\documentclass{article}

\usepackage[main,final]{neurips_2026}

\usepackage[utf8]{inputenc}
\usepackage[T1]{fontenc}
\usepackage{xcolor}
\definecolor{mydarkblue}{rgb}{0,0.08,0.45}
\usepackage[colorlinks=true,linkcolor=mydarkblue,citecolor=mydarkblue,urlcolor=mydarkblue,hypertexnames=false]{hyperref}
\usepackage{url}
\usepackage{upquote}
\usepackage{graphicx}
\usepackage{float}
\usepackage{placeins}
\usepackage{pgf}
\usepackage{booktabs}
\usepackage{multirow}
\usepackage{longtable}
\usepackage{array}
\usepackage{amsmath}
\usepackage{amssymb}
\usepackage{nicefrac}
\usepackage{microtype}
\usepackage{CJKutf8}
\makeatletter
\DeclareRobustCommand\fontseries[1]{\edef\f@series{#1}}
\DeclareRobustCommand\fontshape[1]{\edef\f@shape{#1}}
\makeatother
\usepackage{amsthm}
\usepackage{needspace}
\usepackage{algorithm}
\usepackage{algpseudocode}
\newtheorem{proposition}{Proposition}
\newcolumntype{L}[1]{>{\raggedright\arraybackslash}p{#1}}

\newcommand{\paperConditionCount}{9}
\newcommand{\paperVocabOneSixk}{16K}
\newcommand{\paperVocabThreeTwok}{32K}
\newcommand{\paperVocabSixFourk}{64K}
\newcommand{\paperVocabOneTwoEightk}{128K}
\newcommand{\paperVocabTwoFiveSixk}{256K}

\newcommand{\paperVocabSixFourkDisplay}{64,000}

\newcommand{\mergeVOneSixkBbpe}{15,739}
\newcommand{\mergeVOneSixkUbe}{15,483}
\newcommand{\mergeVOneSixkOverheadPct}{1.6\%}
\newcommand{\mergeVThreeTwokBbpe}{31,739}
\newcommand{\mergeVThreeTwokUbe}{31,483}
\newcommand{\mergeVThreeTwokOverheadPct}{0.8\%}
\newcommand{\mergeVSixFourkBbpe}{63,739}
\newcommand{\mergeVSixFourkUbe}{63,483}
\newcommand{\mergeVSixFourkOverheadPct}{0.4\%}
\newcommand{\mergeVOneTwoEightkBbpe}{127,739}
\newcommand{\mergeVOneTwoEightkUbe}{127,483}
\newcommand{\mergeVOneTwoEightkOverheadPct}{0.2\%}
\newcommand{\mergeVTwoFiveSixkBbpe}{255,739}
\newcommand{\mergeVTwoFiveSixkUbe}{255,483}
\newcommand{\mergeVTwoFiveSixkOverheadPct}{0.1\%}
\newcommand{\utfEightName}{UTF-8}
\newcommand{\utfSixteenName}{UTF-16}
\newcommand{\bbpeSixteenName}{BBPE16}

\newcommand{\gptTwoName}{GPT-2}
\newcommand{\floresTwoHundredName}{FLORES-200}

\newcommand{\llamaFourName}{Llama 4}

\newcommand{\byTFiveName}{ByT5}
\newcommand{\mrTFiveName}{MrT5}
\newcommand{\bbpeAlphabetSize}{256}
\newcommand{\ubeAlphabetSize}{512}
\newcommand{\utfEightAsciiFloor}{1}
\newcommand{\utfEightMidFloor}{2}
\newcommand{\utfEightHighFloor}{3}
\newcommand{\utfEightNonBmpFloor}{4}

\newcommand{\ubeRouteUtfEightMaxBytes}{2}

\newcommand{\paperPhaseCountWord}{two}
\newcommand{\paperPhaseOne}{1}
\newcommand{\paperPhaseTwo}{2}
\newcommand{\paperPhaseThree}{3}

\newcommand{\paperUnderOnePct}{<1\%}

\newcommand{\paperMinFrequencyShard}{2}
\newcommand{\paperMinFrequencyFinalize}{2}

\newcommand{\appendixDecodeVerifiedLangCount}{204}
\newcommand{\appendixDecodeVerifiedSentenceCount}{206,448}
\newcommand{\appendixDecodeVerifiedPassRate}{100\%}

\newcommand{\relatedRanaFertilityGainPct}{39.5\%}
\newcommand{\relatedRanaIndianLangCount}{22}

\newcommand{\eovOneSixkBetter}{37}
\newcommand{\eovOneSixkWorse}{148}
\newcommand{\eovOneSixkNeutral}{19}

\newcommand{\eovOneSixkGiniPct}{-13.9\%}

\newcommand{\eovThreeTwokBetter}{41}
\newcommand{\eovThreeTwokWorse}{140}
\newcommand{\eovThreeTwokNeutral}{23}

\newcommand{\eovThreeTwokGiniPct}{-11.9\%}

\newcommand{\eovThreeTwokWorseMedianDelta}{0.016}
\newcommand{\eovThreeTwokWorseMeanDelta}{0.238}
\newcommand{\eovThreeTwokMaxGainDelta}{155.591}

\newcommand{\eovSixFourkBetter}{37}
\newcommand{\eovSixFourkWorse}{85}
\newcommand{\eovSixFourkNeutral}{82}

\newcommand{\eovSixFourkGiniPct}{-7.3\%}

\newcommand{\fullvOneSixkBetter}{202}
\newcommand{\fullvOneSixkWorse}{2}
\newcommand{\fullvOneSixkNeutral}{0}

\newcommand{\fullvOneSixkGiniPct}{-8.7\%}

\newcommand{\fullvThreeTwokBetter}{191}
\newcommand{\fullvThreeTwokWorse}{13}
\newcommand{\fullvThreeTwokNeutral}{0}

\newcommand{\fullvThreeTwokGiniPct}{-7.2\%}

\newcommand{\fullvSixFourkBetter}{169}
\newcommand{\fullvSixFourkWorse}{11}
\newcommand{\fullvSixFourkNeutral}{24}

\newcommand{\fullvSixFourkGiniPct}{-6.2\%}

\newcommand{\oneOhOnevOneSixkBetter}{196}
\newcommand{\oneOhOnevOneSixkWorse}{8}
\newcommand{\oneOhOnevOneSixkNeutral}{0}
\newcommand{\oneOhOnevOneSixkGiniBBPE}{0.2194}
\newcommand{\oneOhOnevOneSixkGiniUBE}{0.1927}
\newcommand{\oneOhOnevOneSixkGiniPct}{-12.1\%}

\newcommand{\oneOhOnevThreeTwokBetter}{189}
\newcommand{\oneOhOnevThreeTwokWorse}{15}
\newcommand{\oneOhOnevThreeTwokNeutral}{0}
\newcommand{\oneOhOnevThreeTwokGiniBBPE}{0.2183}
\newcommand{\oneOhOnevThreeTwokGiniUBE}{0.1979}
\newcommand{\oneOhOnevThreeTwokGiniPct}{-9.4\%}
\newcommand{\oneOhOnevThreeTwokMeanRatio}{0.9920}

\newcommand{\oneOhOnevSixFourkBetter}{187}
\newcommand{\oneOhOnevSixFourkWorse}{17}
\newcommand{\oneOhOnevSixFourkNeutral}{0}
\newcommand{\oneOhOnevSixFourkGiniBBPE}{0.2225}
\newcommand{\oneOhOnevSixFourkGiniUBE}{0.2043}
\newcommand{\oneOhOnevSixFourkGiniPct}{-8.2\%}

\newcommand{\oneOhOnevOneTwoEightkBetter}{174}
\newcommand{\oneOhOnevOneTwoEightkWorse}{25}
\newcommand{\oneOhOnevOneTwoEightkNeutral}{5}
\newcommand{\oneOhOnevOneTwoEightkGiniBBPE}{0.2258}
\newcommand{\oneOhOnevOneTwoEightkGiniUBE}{0.2140}
\newcommand{\oneOhOnevOneTwoEightkGiniPct}{-5.2\%}

\newcommand{\oneOhOnevTwoFiveSixkBetter}{150}
\newcommand{\oneOhOnevTwoFiveSixkWorse}{41}
\newcommand{\oneOhOnevTwoFiveSixkNeutral}{13}
\newcommand{\oneOhOnevTwoFiveSixkGiniBBPE}{0.2246}
\newcommand{\oneOhOnevTwoFiveSixkGiniUBE}{0.2204}
\newcommand{\oneOhOnevTwoFiveSixkGiniPct}{-1.9\%}

\newcommand{\OneZeroOnelangvThreeTwoZeroZeroZeroGiniCIlo}{-0.0358}
\newcommand{\OneZeroOnelangvThreeTwoZeroZeroZeroGiniCIhi}{-0.0047}

\newcommand{\enonlyvOneSixZeroZeroZeroGiniCIlo}{-0.0616}
\newcommand{\enonlyvOneSixZeroZeroZeroGiniCIhi}{-0.0323}

\newcommand{\enzhkojahithvSixFourZeroZeroZeroGiniCIlo}{-0.0299}
\newcommand{\enzhkojahithvSixFourZeroZeroZeroGiniCIhi}{-0.0044}

\newcommand{\paperStatVarReductionMinPct}{24\%}
\newcommand{\paperStatVarReductionMaxPct}{49\%}
\newcommand{\paperStatGiniReductionMinPct}{6\%}
\newcommand{\paperStatGiniReductionMaxPct}{14\%}

\newcommand{\paperStatWSigCount}{8}

\newcommand{\paperStatWMaxP}{0.32}

\newcommand{\tcBBPEOneLangTrainB}{184.0}
\newcommand{\tcBBPEOneLangEvalB}{7.5}
\newcommand{\tcBBPEOneLangStepsPerEpoch}{350,991}
\newcommand{\tcBBPEOneLangTwoHundredKPct}{57.0}
\newcommand{\tcBBPESixLangTrainB}{120.0}
\newcommand{\tcBBPESixLangEvalB}{7.0}
\newcommand{\tcBBPESixLangStepsPerEpoch}{228,967}
\newcommand{\tcBBPESixLangTwoHundredKPct}{87.3}
\newcommand{\tcBBPEOneOhOneLangTrainB}{108.1}
\newcommand{\tcBBPEOneOhOneLangEvalB}{5.3}
\newcommand{\tcBBPEOneOhOneLangStepsPerEpoch}{206,142}
\newcommand{\tcBBPEOneOhOneLangTwoHundredKPct}{97.0}
\newcommand{\tcUBEOneLangTrainB}{166.9}
\newcommand{\tcUBEOneLangEvalB}{7.2}
\newcommand{\tcUBEOneLangStepsPerEpoch}{318,320}
\newcommand{\tcUBEOneLangTwoHundredKPct}{62.8}
\newcommand{\tcUBESixLangTrainB}{119.0}
\newcommand{\tcUBESixLangEvalB}{6.9}
\newcommand{\tcUBESixLangStepsPerEpoch}{227,033}
\newcommand{\tcUBESixLangTwoHundredKPct}{88.1}
\newcommand{\tcUBEOneOhOneLangTrainB}{107.8}
\newcommand{\tcUBEOneOhOneLangEvalB}{5.3}
\newcommand{\tcUBEOneOhOneLangStepsPerEpoch}{205,695}
\newcommand{\tcUBEOneOhOneLangTwoHundredKPct}{97.2}

\newcommand{\repLangSantaliBbpe}{341.287}
\newcommand{\repLangSantaliUbe}{222.944}
\newcommand{\repLangSantaliDeltaPct}{-34.7\%}

\newcommand{\repLangTamazightBbpe}{269.166}
\newcommand{\repLangTamazightUbe}{183.305}
\newcommand{\repLangTamazightDeltaPct}{-31.9\%}

\newcommand{\repLangTibetanBbpe}{336.804}
\newcommand{\repLangTibetanUbe}{286.737}
\newcommand{\repLangTibetanDeltaPct}{-14.9\%}

\newcommand{\repLangArabicBbpe}{60.600}
\newcommand{\repLangArabicUbe}{60.407}
\newcommand{\repLangArabicDeltaPct}{-0.3\%}

\newcommand{\repLangEnglishBbpe}{31.160}
\newcommand{\repLangEnglishUbe}{31.068}
\newcommand{\repLangEnglishDeltaPct}{-0.3\%}

\newcommand{\repLangOneTwoEightkTibetanDeltaPct}{+7.4\%}

\newcommand{\tdSentOneChars}{24}
\newcommand{\tdSentOneBbpeTokens}{63}
\newcommand{\tdSentOneUbeTokens}{39}
\newcommand{\tdSentOneSavePct}{38\%}

\newcommand{\gvaUbeThreeTwokGini}{0.1979}

\newcommand{\gvaSbpeBbpeThreeTwokGini}{0.1922}

\newcommand{\gvaBbpeSixteenThreeTwokGini}{0.1598}

\newcommand{\gvaInteractionOneSixk}{+4.4}
\newcommand{\gvaInteractionThreeTwok}{+0.3}
\newcommand{\gvaInteractionSixFourk}{-1.7}
\newcommand{\gvaInteractionOneTwoEightk}{-3.8}
\newcommand{\gvaInteractionTwoFiveSixk}{-4.4}
\newcommand{\gvaInteractionMean}{-1.1}
\newcommand{\gvaInteractionSixFourkCILow}{-3.2}
\newcommand{\gvaInteractionSixFourkCIHigh}{-0.1}
\newcommand{\gvaInteractionSixFourkBootN}{203}

\newcommand{\floresRobLangs}{204}
\newcommand{\floresRobGiniBBPE}{0.2225}
\newcommand{\floresRobGiniUBE}{0.2043}

\newcommand{\udhrRobLangs}{387}
\newcommand{\udhrRobGiniBBPE}{0.2491}
\newcommand{\udhrRobGiniUBE}{0.2248}

\newcommand{\sibRobLangs}{205}
\newcommand{\sibRobGiniBBPE}{0.2348}
\newcommand{\sibRobGiniUBE}{0.2174}

\newcommand{\mgsmRobLangs}{11}
\newcommand{\mgsmRobGiniBBPE}{0.1984}
\newcommand{\mgsmRobGiniUBE}{0.1980}

\newcommand{\floresRobOneSixkGiniBBPE}{0.2194}
\newcommand{\floresRobOneSixkGiniUBE}{0.1927}

\newcommand{\floresRobThreeTwokGiniBBPE}{0.2183}
\newcommand{\floresRobThreeTwokGiniUBE}{0.1979}
\newcommand{\floresRobThreeTwokGiniPct}{-9.4\%}

\newcommand{\floresRobOneTwoEightkGiniBBPE}{0.2258}
\newcommand{\floresRobOneTwoEightkGiniUBE}{0.2140}

\newcommand{\floresRobTwoFiveSixkGiniBBPE}{0.2246}
\newcommand{\floresRobTwoFiveSixkGiniUBE}{0.2204}

\newcommand{\udhrRobOneSixkGiniBBPE}{0.2449}
\newcommand{\udhrRobOneSixkGiniUBE}{0.2158}

\newcommand{\udhrRobThreeTwokGiniBBPE}{0.2499}
\newcommand{\udhrRobThreeTwokGiniUBE}{0.2221}
\newcommand{\udhrRobThreeTwokGiniPct}{-11.1\%}

\newcommand{\udhrRobOneTwoEightkGiniBBPE}{0.2552}
\newcommand{\udhrRobOneTwoEightkGiniUBE}{0.2322}

\newcommand{\udhrRobTwoFiveSixkGiniBBPE}{0.2574}
\newcommand{\udhrRobTwoFiveSixkGiniUBE}{0.2352}

\newcommand{\sibRobOneSixkGiniBBPE}{0.2289}
\newcommand{\sibRobOneSixkGiniUBE}{0.2035}

\newcommand{\sibRobThreeTwokGiniBBPE}{0.2297}
\newcommand{\sibRobThreeTwokGiniUBE}{0.2102}
\newcommand{\sibRobThreeTwokGiniPct}{-8.5\%}

\newcommand{\sibRobOneTwoEightkGiniBBPE}{0.2391}
\newcommand{\sibRobOneTwoEightkGiniUBE}{0.2280}

\newcommand{\sibRobTwoFiveSixkGiniBBPE}{0.2398}
\newcommand{\sibRobTwoFiveSixkGiniUBE}{0.2358}

\newcommand{\mgsmRobOneSixkGiniBBPE}{0.1699}
\newcommand{\mgsmRobOneSixkGiniUBE}{0.1710}

\newcommand{\mgsmRobThreeTwokGiniBBPE}{0.1841}
\newcommand{\mgsmRobThreeTwokGiniUBE}{0.1843}
\newcommand{\mgsmRobThreeTwokGiniPct}{+0.1\%}

\newcommand{\mgsmRobOneTwoEightkGiniBBPE}{0.2119}
\newcommand{\mgsmRobOneTwoEightkGiniUBE}{0.2119}

\newcommand{\mgsmRobTwoFiveSixkGiniBBPE}{0.2285}
\newcommand{\mgsmRobTwoFiveSixkGiniUBE}{0.2283}

\newcommand{\gsUbeEnRedPct}{24.2\%}

\newcommand{\gsUbeHiRedPct}{5.6\%}
\newcommand{\gsUbeHiPremMin}{2.57}
\newcommand{\gsUbeHiPremMax}{3.20}
\newcommand{\gsUbeHiPremDelta}{+0.63}

\newcommand{\gsUbeArRedPct}{40.7\%}

\newcommand{\gsUbeArPremDelta}{-0.43}
\newcommand{\gsUbeKoRedPct}{37.6\%}

\newcommand{\gsUbeKoPremDelta}{-0.23}
\newcommand{\gsUbePremUp}{75}

\newcommand{\gsUbePremTotal}{204}
\newcommand{\gsUbeVocabRange}{16K--256K}
\newcommand{\hdMedianAbsDelta}{0.0011}
\newcommand{\hdMaxAbsDelta}{0.0087}
\newcommand{\hdNComparisons}{20}

\newcommand{\lmBbpeOneLangBpb}{1.4909}
\newcommand{\lmUbeOneLangBpb}{\textbf{1.4698}}
\newcommand{\lmBbpeOneLangGini}{0.1993}
\newcommand{\lmUbeOneLangGini}{\textbf{0.1989}}
\newcommand{\lmBbpeOneLangXnli}{34.3}
\newcommand{\lmUbeOneLangXnli}{\textbf{34.4}}
\newcommand{\lmBbpeOneLangXcopa}{51.4}
\newcommand{\lmUbeOneLangXcopa}{\textbf{51.5}}
\newcommand{\lmBbpeOneLangBelebele}{\textbf{23.6}}
\newcommand{\lmUbeOneLangBelebele}{23.0}
\newcommand{\lmBbpeOneLangArcE}{38.1}
\newcommand{\lmUbeOneLangArcE}{\textbf{38.5}}
\newcommand{\lmBbpeOneLangArcC}{\textbf{18.3}}
\newcommand{\lmUbeOneLangArcC}{17.0}
\newcommand{\lmBbpeOneLangHSwag}{27.5}
\newcommand{\lmUbeOneLangHSwag}{\textbf{27.8}}
\newcommand{\lmBbpeOneLangPiqa}{61.9}
\newcommand{\lmUbeOneLangPiqa}{\textbf{62.7}}
\newcommand{\lmBbpeOneLangLambada}{\textbf{23.6}}
\newcommand{\lmUbeOneLangLambada}{22.5}
\newcommand{\lmBbpeSixLangBpb}{\textbf{1.3680}}
\newcommand{\lmUbeSixLangBpb}{1.3760}
\newcommand{\lmBbpeSixLangGini}{0.1930}
\newcommand{\lmUbeSixLangGini}{\textbf{0.1880}}
\newcommand{\lmBbpeSixLangXnli}{34.6}
\newcommand{\lmUbeSixLangXnli}{\textbf{34.9}}
\newcommand{\lmBbpeSixLangXcopa}{52.0}
\newcommand{\lmUbeSixLangXcopa}{\textbf{52.2}}
\newcommand{\lmBbpeSixLangBelebele}{\textbf{23.5}}
\newcommand{\lmUbeSixLangBelebele}{23.2}
\newcommand{\lmBbpeSixLangArcE}{35.6}
\newcommand{\lmUbeSixLangArcE}{\textbf{36.7}}
\newcommand{\lmBbpeSixLangArcC}{17.1}
\newcommand{\lmUbeSixLangArcC}{\textbf{18.4}}
\newcommand{\lmBbpeSixLangHSwag}{27.4}
\newcommand{\lmUbeSixLangHSwag}{\textbf{28.0}}
\newcommand{\lmBbpeSixLangPiqa}{61.3}
\newcommand{\lmUbeSixLangPiqa}{\textbf{62.0}}
\newcommand{\lmBbpeSixLangLambada}{\textbf{23.1}}
\newcommand{\lmUbeSixLangLambada}{\textbf{23.1}}
\newcommand{\lmBbpeOneZeroOneLangBpb}{\textbf{1.5174}}
\newcommand{\lmUbeOneZeroOneLangBpb}{1.5237}

\newcommand{\lmBbpeOneZeroOneLangGini}{0.1762}
\newcommand{\lmUbeOneZeroOneLangGini}{\textbf{0.1739}}

\newcommand{\lmBbpeOneZeroOneLangXnli}{34.3}
\newcommand{\lmUbeOneZeroOneLangXnli}{\textbf{34.4}}

\newcommand{\lmBbpeOneZeroOneLangXcopa}{\textbf{52.0}}
\newcommand{\lmUbeOneZeroOneLangXcopa}{51.8}

\newcommand{\lmBbpeOneZeroOneLangBelebele}{\textbf{24.8}}
\newcommand{\lmUbeOneZeroOneLangBelebele}{23.7}

\newcommand{\lmBbpeOneZeroOneLangArcE}{35.5}
\newcommand{\lmUbeOneZeroOneLangArcE}{\textbf{35.9}}

\newcommand{\lmBbpeOneZeroOneLangArcC}{17.0}
\newcommand{\lmUbeOneZeroOneLangArcC}{\textbf{17.8}}

\newcommand{\lmBbpeOneZeroOneLangHSwag}{27.7}
\newcommand{\lmUbeOneZeroOneLangHSwag}{\textbf{27.9}}

\newcommand{\lmBbpeOneZeroOneLangPiqa}{59.3}
\newcommand{\lmUbeOneZeroOneLangPiqa}{\textbf{60.1}}

\newcommand{\lmBbpeOneZeroOneLangLambada}{22.8}
\newcommand{\lmUbeOneZeroOneLangLambada}{\textbf{23.1}}

\newcommand{\lmOneZeroOneLangUbeWins}{37}

\newcommand{\lmOneZeroOneLangMeanDeltaBpb}{+0.0063}

\newcommand{\lmOneZeroOneLangOneByteMedianDelta}{+0.0064}

\newcommand{\lmOneZeroOneLangTwoByteMedianDelta}{+0.0070}

\newcommand{\lmOneZeroOneLangThreeByteMedianDelta}{+0.0029}

\newcommand{\lmMultiseedNSeeds}{3}

\newcommand{\lmMsUbeGiniMean}{0.1822}
\newcommand{\lmMsUbeGiniStd}{0.0074}

\newcommand{\lmMsBbpeGiniMean}{0.1819}
\newcommand{\lmMsBbpeGiniStd}{0.0052}

\newcommand{\lmMsBpbDeltaMean}{+0.0046}
\newcommand{\lmMsBpbDeltaMax}{0.0063}
\newcommand{\infNTrials}{30}

\newcommand{\infNEffective}{28}

\newcommand{\infDeltaSeqengLatn}{+3.9\%}

\newcommand{\infDeltaSeqhinDeva}{+0.1\%}

\newcommand{\infDeltaSeqthaThai}{+0.5\%}

\newcommand{\infDeltaTokamhEthi}{-32.1\%}
\newcommand{\infDeltaSeqamhEthi}{+45.0\%}

\newcommand{\infNSigLangsBH}{8}

\newcommand{\cwThreeTwokSantaliGainPct}{53\%}
\newcommand{\cwThreeTwokEnglishGainPct}{0.3\%}

\newcommand{\vdBbpeOneSixkDeadMerges}{9}

\newcommand{\vdUbeOneSixkDeadMerges}{6}

\newcommand{\vdSbpeBbpeOneSixkDeadMerges}{156}

\newcommand{\vdSbpeUbeOneSixkDeadMerges}{158}

\newcommand{\vdBbpeAlphTokens}{256}
\newcommand{\vdBbpeMergeTokens}{31,739}
\newcommand{\vdBbpeDeadMerges}{13}
\newcommand{\vdBbpeActiveUtil}{99.9\%}
\newcommand{\vdBbpeDepthOnePct}{7.9\%}
\newcommand{\vdBbpeMaxDepth}{8}
\newcommand{\vdUbeAlphTokens}{512}
\newcommand{\vdUbeMergeTokens}{31,483}
\newcommand{\vdUbeDeadMerges}{9}
\newcommand{\vdUbeActiveUtil}{99.8\%}
\newcommand{\vdUbeDepthOnePct}{17.3\%}
\newcommand{\vdUbeMaxDepth}{7}
\newcommand{\vdUbePuaPuaPct}{23.6\%}
\newcommand{\vdUbeGptGptPct}{66.6\%}
\newcommand{\vdUbeCrossPct}{9.8\%}
\newcommand{\vdSbpeBbpeAlphTokens}{256}
\newcommand{\vdSbpeBbpeMergeTokens}{31,739}
\newcommand{\vdSbpeBbpeDeadMerges}{524}
\newcommand{\vdSbpeBbpeActiveUtil}{98.3\%}
\newcommand{\vdSbpeBbpeDepthOnePct}{10.3\%}
\newcommand{\vdSbpeBbpeMaxDepth}{14}
\newcommand{\vdSbpeUbeAlphTokens}{512}
\newcommand{\vdSbpeUbeMergeTokens}{31,483}
\newcommand{\vdSbpeUbeDeadMerges}{531}
\newcommand{\vdSbpeUbeActiveUtil}{98.2\%}
\newcommand{\vdSbpeUbeDepthOnePct}{19.4\%}
\newcommand{\vdSbpeUbeMaxDepth}{14}
\newcommand{\vdSbpeUbePuaPuaPct}{24.6\%}
\newcommand{\vdSbpeUbeGptGptPct}{64.6\%}
\newcommand{\vdSbpeUbeCrossPct}{10.8\%}

\newcommand{\vdBbpeThreeTwokDeadMerges}{13}

\newcommand{\vdUbeThreeTwokDeadMerges}{9}

\newcommand{\vdSbpeBbpeThreeTwokDeadMerges}{524}

\newcommand{\vdSbpeUbeThreeTwokDeadMerges}{531}

\newcommand{\vdBbpeSixFourkDeadMerges}{20}

\newcommand{\vdUbeSixFourkDeadMerges}{16}

\newcommand{\vdSbpeBbpeSixFourkDeadMerges}{1,193}

\newcommand{\vdSbpeUbeSixFourkDeadMerges}{1,203}

\newcommand{\vdBbpeOneTwoEightkDeadMerges}{82}

\newcommand{\vdUbeOneTwoEightkDeadMerges}{78}

\newcommand{\vdSbpeBbpeOneTwoEightkDeadMerges}{3,841}

\newcommand{\vdSbpeUbeOneTwoEightkDeadMerges}{3,858}

\newcommand{\vdBbpeTwoFiveSixkDeadMerges}{224}

\newcommand{\vdUbeTwoFiveSixkDeadMerges}{217}

\newcommand{\vdSbpeBbpeTwoFiveSixkDeadMerges}{12,203}

\newcommand{\vdSbpeUbeTwoFiveSixkDeadMerges}{12,251}

\newcommand{\msNLangs}{8}
\newcommand{\msNDocs}{400}
\newcommand{\msUbeBetterOrTie}{7 of 8}
\newcommand{\msUbeTotalDeltaPct}{-0.2\%}
\newcommand{\msBbpeSixteenTotalDeltaPct}{6.6\%}
\newcommand{\msBengaliDeltaPct}{+0.4\%}
\newcommand{\pltGptfouroVocab}{200,019}
\newcommand{\pltGptfouroEnAvg}{26.554}
\newcommand{\pltGptfouroTotalM}{11.6}
\newcommand{\pltGptfouroGini}{0.2323}
\newcommand{\pltGptfourVocab}{100,277}
\newcommand{\pltGptfourEnAvg}{26.860}
\newcommand{\pltGptfourTotalM}{18.6}
\newcommand{\pltGptfourGini}{0.3287}
\newcommand{\pltLlamathreeVocab}{128,256}
\newcommand{\pltLlamathreeEnAvg}{26.849}
\newcommand{\pltLlamathreeTotalM}{16.6}
\newcommand{\pltLlamathreeGini}{0.3318}
\newcommand{\pltQwenTwoFiveVocab}{151,665}
\newcommand{\pltQwenTwoFiveEnAvg}{27.293}
\newcommand{\pltQwenTwoFiveTotalM}{15.4}
\newcommand{\pltQwenTwoFiveGini}{0.2724}
\newcommand{\pltQwenthreesixVocab}{248,044}
\newcommand{\pltQwenthreesixEnAvg}{27.204}
\newcommand{\pltQwenthreesixTotalM}{12.8}
\newcommand{\pltQwenthreesixGini}{0.2420}
\newcommand{\pltMistralVocab}{131,072}
\newcommand{\pltMistralEnAvg}{27.497}
\newcommand{\pltMistralTotalM}{14.2}
\newcommand{\pltMistralGini}{0.3030}

\newcommand{\cpuTokTrainHrs}{14}
\newcommand{\cpuTokWorkerHrs}{180}

\newcommand{\cpuEvalCoreHrs}{3}
\newcommand{\lmPrimaryDesignRuns}{16}
\newcommand{\lmAdditionalReportedLmRuns}{12}
\newcommand{\lmAdditionalSixFourkRuns}{6}
\newcommand{\lmAdditionalSuperBpeRuns}{2}
\newcommand{\lmAdditionalMultiseedRuns}{4}
\newcommand{\lmTotalReportedLmRuns}{28}
\newcommand{\lmOneEightTwoMRuns}{12}
\newcommand{\lmOneEightTwoMConfiguredSteps}{200{,}000}
\newcommand{\lmOneEightTwoMWallClockRange}{1--2 days}

\newcommand{\lmOneEightTwoMGpuHrsPerRun}{192--384}
\newcommand{\lmOneEightTwoMGpuHrs}{2{,}304--4{,}608}
\newcommand{\lmTwoThreeZeroMRuns}{6}
\newcommand{\lmTwoThreeZeroMConfiguredSteps}{200{,}000}
\newcommand{\lmTwoThreeZeroMWallClockRange}{1.5--3 days}

\newcommand{\lmTwoThreeZeroMGpuHrsPerRun}{288--576}
\newcommand{\lmTwoThreeZeroMGpuHrs}{1{,}728--3{,}456}
\newcommand{\lmThreeTwoNineMRuns}{4}
\newcommand{\lmThreeTwoNineMConfiguredSteps}{200{,}000}
\newcommand{\lmThreeTwoNineMWallClockRange}{2--4 days}

\newcommand{\lmThreeTwoNineMGpuHrsPerRun}{384--768}
\newcommand{\lmThreeTwoNineMGpuHrs}{1{,}536--3{,}072}
\newcommand{\lmOneThreeBRuns}{6}
\newcommand{\lmOneThreeBConfiguredSteps}{205{,}000}
\newcommand{\lmOneThreeBWallClockRange}{8--14 days}

\newcommand{\lmOneThreeBGpuHrsPerRun}{1{,}536--2{,}688}
\newcommand{\lmOneThreeBGpuHrs}{9{,}216--16{,}128}
\newcommand{\downstreamGpuHrs}{6}
\newcommand{\totalAcceleratorHrs}{14.8K--27.3K}

\newcommand{\contamDocsM}{89.47M}
\newcommand{\contamLangsChecked}{101}
\newcommand{\contamPrefixChars}{50}
\newcommand{\contamFloresOverlaps}{0}
\newcommand{\contamFloresSubstringOverlaps}{53}
\newcommand{\contamFloresSentences}{206,448}
\newcommand{\contamFloresSubstringEligibleSentences}{201,686}
\newcommand{\contamUdhrOverlaps}{1}
\newcommand{\contamUdhrSubstringOverlaps}{7,102}
\newcommand{\contamUdhrSentences}{35,870}
\newcommand{\contamUdhrSubstringEligibleSentences}{22,133}
\newcommand{\contamSibOverlaps}{0}
\newcommand{\contamSibSubstringOverlaps}{10}
\newcommand{\contamSibSentences}{41,820}
\newcommand{\contamSibSubstringEligibleSentences}{40,771}
\newcommand{\contamMgsmOverlaps}{0}
\newcommand{\contamMgsmSubstringOverlaps}{0}
\newcommand{\contamMgsmSentences}{2,750}
\newcommand{\contamMgsmSubstringEligibleSentences}{2,717}
\newcommand{\corpusOneLangLangs}{1}
\newcommand{\corpusOneLangTrainShards}{118}
\newcommand{\corpusOneLangTrainGB}{101}
\newcommand{\corpusOneLangTrainDocsM}{32.1}
\newcommand{\corpusOneLangEvalShards}{1}

\newcommand{\corpusOneLangEvalDocsM}{0.27}
\newcommand{\corpusSixLangLangs}{6}
\newcommand{\corpusSixLangTrainShards}{1,330}
\newcommand{\corpusSixLangTrainGB}{253}
\newcommand{\corpusSixLangTrainDocsM}{62.3}
\newcommand{\corpusSixLangEvalShards}{6}

\newcommand{\corpusSixLangEvalDocsM}{0.46}
\newcommand{\corpusOneOhOneLangLangs}{101}
\newcommand{\corpusOneOhOneLangTrainShards}{2,273}
\newcommand{\corpusOneOhOneLangTrainGB}{346}
\newcommand{\corpusOneOhOneLangTrainDocsM}{89.3}
\newcommand{\corpusOneOhOneLangEvalShards}{101}

\newcommand{\corpusOneOhOneLangEvalDocsM}{4.36}
\newcommand{\corpusOneOhOneLangHalfTrainGB}{173}
\newcommand{\detTieFloresOrigGini}{0.2522}
\newcommand{\detTieFloresDetGini}{0.2517}
\newcommand{\detTieFloresDeltaGini}{-0.00050}
\newcommand{\detTieFloresDeltaMean}{-0.0028}
\newcommand{\detTieUdhrOrigGini}{0.2473}
\newcommand{\detTieUdhrDetGini}{0.2471}
\newcommand{\detTieUdhrDeltaGini}{-0.00025}
\newcommand{\detTieUdhrDeltaMean}{+0.0003}
\newcommand{\detTieSibOrigGini}{0.2600}
\newcommand{\detTieSibDetGini}{0.2595}
\newcommand{\detTieSibDeltaGini}{-0.00050}
\newcommand{\detTieSibDeltaMean}{-0.0025}
\newcommand{\detTieMgsmOrigGini}{0.3120}
\newcommand{\detTieMgsmDetGini}{0.3120}
\newcommand{\detTieMgsmDeltaGini}{+0.00002}
\newcommand{\detTieMgsmDeltaMean}{-0.0004}
\newcommand{\detTieMaxAbsDeltaGini}{0.0005}
\newcommand{\detTieMaxAbsDeltaMean}{0.0028}
\newcommand{\floresRobOneSixkGiniPctDetailed}{-12.14\%}
\newcommand{\floresRobOneTwoEightkGiniPctDetailed}{-5.20\%}
\newcommand{\floresRobSixFourkGiniPctDetailed}{-8.20\%}
\newcommand{\floresRobThreeTwokGiniPctDetailed}{-9.36\%}
\newcommand{\floresRobTwoFiveSixkGiniPctDetailed}{-1.86\%}
\newcommand{\gvaBbpeAllDetailed}{66.056}
\newcommand{\gvaBbpeEnDetailed}{28.716}
\newcommand{\gvaBbpeGiniDetailed}{0.22254}
\newcommand{\gvaBbpeLatinDetailed}{52.540}
\newcommand{\gvaBbpeOneSixkAllDetailed}{79.408}
\newcommand{\gvaBbpeOneSixkEnDetailed}{34.863}
\newcommand{\gvaBbpeOneSixkGiniDetailed}{0.21938}
\newcommand{\gvaBbpeOneSixkLatinDetailed}{62.373}

\newcommand{\gvaBbpeOneTwoEightkAllDetailed}{60.594}
\newcommand{\gvaBbpeOneTwoEightkEnDetailed}{27.174}
\newcommand{\gvaBbpeOneTwoEightkGiniDetailed}{0.22576}
\newcommand{\gvaBbpeOneTwoEightkLatinDetailed}{48.751}

\newcommand{\gvaBbpeSixteenAllDetailed}{72.668}
\newcommand{\gvaBbpeSixteenEnDetailed}{34.471}
\newcommand{\gvaBbpeSixteenGiniDetailed}{0.18906}
\newcommand{\gvaBbpeSixteenLatinDetailed}{67.506}
\newcommand{\gvaBbpeSixteenOneSixkAllDetailed}{95.683}
\newcommand{\gvaBbpeSixteenOneSixkEnDetailed}{47.935}
\newcommand{\gvaBbpeSixteenOneSixkGiniDetailed}{0.14588}
\newcommand{\gvaBbpeSixteenOneSixkLatinDetailed}{90.432}

\newcommand{\gvaBbpeSixteenSixFourkAllDetailed}{79.224}
\newcommand{\gvaBbpeSixteenSixFourkEnDetailed}{37.054}
\newcommand{\gvaBbpeSixteenSixFourkGiniDetailed}{0.17218}
\newcommand{\gvaBbpeSixteenSixFourkLatinDetailed}{74.313}

\newcommand{\gvaBbpeSixteenThreeTwokAllDetailed}{87.170}
\newcommand{\gvaBbpeSixteenThreeTwokEnDetailed}{41.317}
\newcommand{\gvaBbpeSixteenThreeTwokGiniDetailed}{0.15982}
\newcommand{\gvaBbpeSixteenThreeTwokLatinDetailed}{82.137}

\newcommand{\gvaBbpeSixteenTwoFiveSixkAllDetailed}{66.468}
\newcommand{\gvaBbpeSixteenTwoFiveSixkEnDetailed}{32.929}
\newcommand{\gvaBbpeSixteenTwoFiveSixkGiniDetailed}{0.19920}
\newcommand{\gvaBbpeSixteenTwoFiveSixkLatinDetailed}{61.839}

\newcommand{\gvaBbpeThreeTwokAllDetailed}{71.889}
\newcommand{\gvaBbpeThreeTwokEnDetailed}{31.160}
\newcommand{\gvaBbpeThreeTwokGiniDetailed}{0.21833}
\newcommand{\gvaBbpeThreeTwokLatinDetailed}{56.743}

\newcommand{\gvaBbpeTwoFiveSixkAllDetailed}{55.389}
\newcommand{\gvaBbpeTwoFiveSixkEnDetailed}{26.272}
\newcommand{\gvaBbpeTwoFiveSixkGiniDetailed}{0.22456}
\newcommand{\gvaBbpeTwoFiveSixkLatinDetailed}{45.100}

\newcommand{\gvaSbpeBbpeAllDetailed}{66.871}
\newcommand{\gvaSbpeBbpeEnDetailed}{29.133}
\newcommand{\gvaSbpeBbpeGiniDetailed}{0.18996}
\newcommand{\gvaSbpeBbpeLatinDetailed}{55.586}
\newcommand{\gvaSbpeBbpeOneSixkAllDetailed}{82.853}
\newcommand{\gvaSbpeBbpeOneSixkEnDetailed}{38.308}
\newcommand{\gvaSbpeBbpeOneSixkGiniDetailed}{0.19967}
\newcommand{\gvaSbpeBbpeOneSixkLatinDetailed}{66.840}

\newcommand{\gvaSbpeBbpeOneTwoEightkAllDetailed}{61.045}
\newcommand{\gvaSbpeBbpeOneTwoEightkEnDetailed}{26.337}
\newcommand{\gvaSbpeBbpeOneTwoEightkGiniDetailed}{0.19042}
\newcommand{\gvaSbpeBbpeOneTwoEightkLatinDetailed}{51.476}

\newcommand{\gvaSbpeBbpeThreeTwokAllDetailed}{74.187}
\newcommand{\gvaSbpeBbpeThreeTwokEnDetailed}{33.143}
\newcommand{\gvaSbpeBbpeThreeTwokGiniDetailed}{0.19217}
\newcommand{\gvaSbpeBbpeThreeTwokLatinDetailed}{61.177}

\newcommand{\gvaSbpeBbpeTwoFiveSixkAllDetailed}{55.220}
\newcommand{\gvaSbpeBbpeTwoFiveSixkEnDetailed}{24.357}
\newcommand{\gvaSbpeBbpeTwoFiveSixkGiniDetailed}{0.18694}
\newcommand{\gvaSbpeBbpeTwoFiveSixkLatinDetailed}{47.479}

\newcommand{\gvaSbpeUbeAllDetailed}{64.957}
\newcommand{\gvaSbpeUbeEnDetailed}{29.050}
\newcommand{\gvaSbpeUbeGiniDetailed}{0.16790}
\newcommand{\gvaSbpeUbeLatinDetailed}{55.491}
\newcommand{\gvaSbpeUbeOneSixkAllDetailed}{80.676}
\newcommand{\gvaSbpeUbeOneSixkEnDetailed}{38.099}
\newcommand{\gvaSbpeUbeOneSixkGiniDetailed}{0.18269}
\newcommand{\gvaSbpeUbeOneSixkLatinDetailed}{66.534}

\newcommand{\gvaSbpeUbeOneTwoEightkAllDetailed}{59.474}
\newcommand{\gvaSbpeUbeOneTwoEightkEnDetailed}{26.314}
\newcommand{\gvaSbpeUbeOneTwoEightkGiniDetailed}{0.16998}
\newcommand{\gvaSbpeUbeOneTwoEightkLatinDetailed}{51.359}

\newcommand{\gvaSbpeUbeThreeTwokAllDetailed}{72.137}
\newcommand{\gvaSbpeUbeThreeTwokEnDetailed}{32.973}
\newcommand{\gvaSbpeUbeThreeTwokGiniDetailed}{0.17248}
\newcommand{\gvaSbpeUbeThreeTwokLatinDetailed}{60.992}

\newcommand{\gvaSbpeUbeTwoFiveSixkAllDetailed}{54.257}
\newcommand{\gvaSbpeUbeTwoFiveSixkEnDetailed}{24.344}
\newcommand{\gvaSbpeUbeTwoFiveSixkGiniDetailed}{0.17280}
\newcommand{\gvaSbpeUbeTwoFiveSixkLatinDetailed}{47.451}

\newcommand{\gvaUbeAllDetailed}{64.448}
\newcommand{\gvaUbeEnDetailed}{28.695}
\newcommand{\gvaUbeGiniDetailed}{0.20428}
\newcommand{\gvaUbeLatinDetailed}{52.500}
\newcommand{\gvaUbeOneSixkAllDetailed}{76.410}
\newcommand{\gvaUbeOneSixkEnDetailed}{34.671}
\newcommand{\gvaUbeOneSixkGiniDetailed}{0.19274}
\newcommand{\gvaUbeOneSixkLatinDetailed}{62.163}

\newcommand{\gvaUbeOneTwoEightkAllDetailed}{59.649}
\newcommand{\gvaUbeOneTwoEightkEnDetailed}{27.158}
\newcommand{\gvaUbeOneTwoEightkGiniDetailed}{0.21401}
\newcommand{\gvaUbeOneTwoEightkLatinDetailed}{48.722}

\newcommand{\gvaUbeThreeTwokAllDetailed}{69.857}
\newcommand{\gvaUbeThreeTwokEnDetailed}{31.068}
\newcommand{\gvaUbeThreeTwokGiniDetailed}{0.19790}
\newcommand{\gvaUbeThreeTwokLatinDetailed}{56.627}

\newcommand{\gvaUbeTwoFiveSixkAllDetailed}{55.105}
\newcommand{\gvaUbeTwoFiveSixkEnDetailed}{26.271}
\newcommand{\gvaUbeTwoFiveSixkGiniDetailed}{0.22039}
\newcommand{\gvaUbeTwoFiveSixkLatinDetailed}{45.097}

\newcommand{\mgsmRobOneSixkGiniPctDetailed}{+0.65\%}
\newcommand{\mgsmRobOneTwoEightkGiniPctDetailed}{-0.03\%}
\newcommand{\mgsmRobSixFourkGiniPctDetailed}{-0.19\%}
\newcommand{\mgsmRobThreeTwokGiniPctDetailed}{+0.14\%}
\newcommand{\mgsmRobTwoFiveSixkGiniPctDetailed}{-0.06\%}
\newcommand{\sibRobOneSixkGiniPctDetailed}{-11.10\%}
\newcommand{\sibRobOneTwoEightkGiniPctDetailed}{-4.67\%}
\newcommand{\sibRobSixFourkGiniPctDetailed}{-7.40\%}
\newcommand{\sibRobThreeTwokGiniPctDetailed}{-8.50\%}
\newcommand{\sibRobTwoFiveSixkGiniPctDetailed}{-1.66\%}
\newcommand{\udhrRobOneSixkGiniPctDetailed}{-11.86\%}
\newcommand{\udhrRobOneTwoEightkGiniPctDetailed}{-9.03\%}
\newcommand{\udhrRobSixFourkGiniPctDetailed}{-9.74\%}
\newcommand{\udhrRobThreeTwokGiniPctDetailed}{-11.11\%}
\newcommand{\udhrRobTwoFiveSixkGiniPctDetailed}{-8.62\%}
\newcommand{\vdBbpeOneSixkActiveUtilDetailed}{99.825\%}
\newcommand{\vdBbpeOneSixkDepthOnePctDetailed}{12.326\%}
\newcommand{\vdBbpeOneTwoEightkActiveUtilDetailed}{99.921\%}
\newcommand{\vdBbpeOneTwoEightkDepthOnePctDetailed}{2.982\%}
\newcommand{\vdBbpeSixFourkActiveUtilDetailed}{99.939\%}
\newcommand{\vdBbpeSixFourkDepthOnePctDetailed}{4.937\%}
\newcommand{\vdBbpeThreeTwokActiveUtilDetailed}{99.900\%}
\newcommand{\vdBbpeThreeTwokDepthOnePctDetailed}{7.889\%}
\newcommand{\vdBbpeTwoFiveSixkActiveUtilDetailed}{99.905\%}
\newcommand{\vdBbpeTwoFiveSixkDepthOnePctDetailed}{1.774\%}
\newcommand{\vdSbpeBbpeOneSixkActiveUtilDetailed}{98.906\%}
\newcommand{\vdSbpeBbpeOneSixkDepthOnePctDetailed}{15.655\%}
\newcommand{\vdSbpeBbpeOneTwoEightkActiveUtilDetailed}{96.984\%}
\newcommand{\vdSbpeBbpeOneTwoEightkDepthOnePctDetailed}{4.233\%}
\newcommand{\vdSbpeBbpeSixFourkActiveUtilDetailed}{98.106\%}
\newcommand{\vdSbpeBbpeSixFourkDepthOnePctDetailed}{6.897\%}
\newcommand{\vdSbpeBbpeThreeTwokActiveUtilDetailed}{98.303\%}
\newcommand{\vdSbpeBbpeThreeTwokDepthOnePctDetailed}{10.334\%}
\newcommand{\vdSbpeBbpeTwoFiveSixkActiveUtilDetailed}{95.225\%}
\newcommand{\vdSbpeBbpeTwoFiveSixkDepthOnePctDetailed}{2.518\%}
\newcommand{\vdSbpeUbeOneSixkActiveUtilDetailed}{98.763\%}
\newcommand{\vdSbpeUbeOneSixkCrossPctDetailed}{9.307\%}
\newcommand{\vdSbpeUbeOneSixkDepthOnePctDetailed}{27.223\%}
\newcommand{\vdSbpeUbeOneSixkPuaPuaPctDetailed}{25.467\%}
\newcommand{\vdSbpeUbeOneTwoEightkActiveUtilDetailed}{96.954\%}
\newcommand{\vdSbpeUbeOneTwoEightkCrossPctDetailed}{13.802\%}
\newcommand{\vdSbpeUbeOneTwoEightkDepthOnePctDetailed}{8.474\%}
\newcommand{\vdSbpeUbeOneTwoEightkPuaPuaPctDetailed}{21.866\%}
\newcommand{\vdSbpeUbeSixFourkActiveUtilDetailed}{98.056\%}
\newcommand{\vdSbpeUbeSixFourkCrossPctDetailed}{12.696\%}
\newcommand{\vdSbpeUbeSixFourkDepthOnePctDetailed}{13.393\%}
\newcommand{\vdSbpeUbeSixFourkPuaPuaPctDetailed}{23.474\%}
\newcommand{\vdSbpeUbeThreeTwokActiveUtilDetailed}{98.216\%}
\newcommand{\vdSbpeUbeThreeTwokCrossPctDetailed}{10.755\%}
\newcommand{\vdSbpeUbeThreeTwokDepthOnePctDetailed}{19.423\%}
\newcommand{\vdSbpeUbeThreeTwokPuaPuaPctDetailed}{24.607\%}
\newcommand{\vdSbpeUbeTwoFiveSixkActiveUtilDetailed}{95.198\%}
\newcommand{\vdSbpeUbeTwoFiveSixkCrossPctDetailed}{15.931\%}
\newcommand{\vdSbpeUbeTwoFiveSixkDepthOnePctDetailed}{5.225\%}
\newcommand{\vdSbpeUbeTwoFiveSixkPuaPuaPctDetailed}{20.902\%}
\newcommand{\vdUbeOneSixkActiveUtilDetailed}{99.713\%}
\newcommand{\vdUbeOneSixkCrossPctDetailed}{9.082\%}
\newcommand{\vdUbeOneSixkDepthOnePctDetailed}{24.467\%}
\newcommand{\vdUbeOneSixkPuaPuaPctDetailed}{24.758\%}
\newcommand{\vdUbeOneTwoEightkActiveUtilDetailed}{99.907\%}
\newcommand{\vdUbeOneTwoEightkCrossPctDetailed}{11.846\%}
\newcommand{\vdUbeOneTwoEightkDepthOnePctDetailed}{7.526\%}
\newcommand{\vdUbeOneTwoEightkPuaPuaPctDetailed}{20.924\%}
\newcommand{\vdUbeSixFourkActiveUtilDetailed}{99.911\%}
\newcommand{\vdUbeSixFourkCrossPctDetailed}{10.595\%}
\newcommand{\vdUbeSixFourkDepthOnePctDetailed}{11.715\%}
\newcommand{\vdUbeSixFourkPuaPuaPctDetailed}{22.297\%}
\newcommand{\vdUbeThreeTwokActiveUtilDetailed}{99.847\%}
\newcommand{\vdUbeThreeTwokCrossPctDetailed}{9.808\%}
\newcommand{\vdUbeThreeTwokDepthOnePctDetailed}{17.298\%}
\newcommand{\vdUbeThreeTwokPuaPuaPctDetailed}{23.587\%}
\newcommand{\vdUbeTwoFiveSixkActiveUtilDetailed}{99.899\%}
\newcommand{\vdUbeTwoFiveSixkCrossPctDetailed}{13.297\%}
\newcommand{\vdUbeTwoFiveSixkDepthOnePctDetailed}{4.612\%}
\newcommand{\vdUbeTwoFiveSixkPuaPuaPctDetailed}{19.641\%}
\newcommand{\gvaBbpeCjkDetailed}{37.237}
\newcommand{\gvaBbpeOneSixkCjkDetailed}{48.521}
\newcommand{\gvaBbpeOneSixkOtherBmpDetailed}{308.522}
\newcommand{\gvaBbpeOneTwoEightkCjkDetailed}{33.353}
\newcommand{\gvaBbpeOneTwoEightkOtherBmpDetailed}{217.459}
\newcommand{\gvaBbpeOtherBmpDetailed}{250.389}
\newcommand{\gvaBbpeSixteenCjkDetailed}{35.364}
\newcommand{\gvaBbpeSixteenOneSixkCjkDetailed}{51.045}
\newcommand{\gvaBbpeSixteenOneSixkOtherBmpDetailed}{228.970}
\newcommand{\gvaBbpeSixteenOtherBmpDetailed}{201.591}
\newcommand{\gvaBbpeSixteenSixFourkCjkDetailed}{39.576}
\newcommand{\gvaBbpeSixteenSixFourkOtherBmpDetailed}{208.599}
\newcommand{\gvaBbpeSixteenThreeTwokCjkDetailed}{45.083}
\newcommand{\gvaBbpeSixteenThreeTwokOtherBmpDetailed}{222.753}
\newcommand{\gvaBbpeSixteenTwoFiveSixkCjkDetailed}{31.876}
\newcommand{\gvaBbpeSixteenTwoFiveSixkOtherBmpDetailed}{174.085}
\newcommand{\gvaBbpeThreeTwokCjkDetailed}{42.444}
\newcommand{\gvaBbpeThreeTwokOtherBmpDetailed}{275.681}
\newcommand{\gvaBbpeTwoFiveSixkCjkDetailed}{30.259}
\newcommand{\gvaBbpeTwoFiveSixkOtherBmpDetailed}{176.498}
\newcommand{\gvaSbpeBbpeCjkDetailed}{38.710}
\newcommand{\gvaSbpeBbpeOneSixkCjkDetailed}{49.655}
\newcommand{\gvaSbpeBbpeOneSixkOtherBmpDetailed}{273.239}
\newcommand{\gvaSbpeBbpeOneTwoEightkCjkDetailed}{34.349}
\newcommand{\gvaSbpeBbpeOneTwoEightkOtherBmpDetailed}{215.979}
\newcommand{\gvaSbpeBbpeOtherBmpDetailed}{231.689}
\newcommand{\gvaSbpeBbpeThreeTwokCjkDetailed}{43.595}
\newcommand{\gvaSbpeBbpeThreeTwokOtherBmpDetailed}{245.653}
\newcommand{\gvaSbpeBbpeTwoFiveSixkCjkDetailed}{30.977}
\newcommand{\gvaSbpeBbpeTwoFiveSixkOtherBmpDetailed}{187.800}
\newcommand{\gvaSbpeUbeCjkDetailed}{38.808}
\newcommand{\gvaSbpeUbeOneSixkCjkDetailed}{49.281}
\newcommand{\gvaSbpeUbeOneSixkOtherBmpDetailed}{221.587}
\newcommand{\gvaSbpeUbeOneTwoEightkCjkDetailed}{34.539}
\newcommand{\gvaSbpeUbeOneTwoEightkOtherBmpDetailed}{172.071}
\newcommand{\gvaSbpeUbeOtherBmpDetailed}{178.255}
\newcommand{\gvaSbpeUbeThreeTwokCjkDetailed}{43.496}
\newcommand{\gvaSbpeUbeThreeTwokOtherBmpDetailed}{192.162}
\newcommand{\gvaSbpeUbeTwoFiveSixkCjkDetailed}{31.223}
\newcommand{\gvaSbpeUbeTwoFiveSixkOtherBmpDetailed}{159.784}
\newcommand{\gvaUbeCjkDetailed}{37.438}
\newcommand{\gvaUbeOneSixkCjkDetailed}{48.052}
\newcommand{\gvaUbeOneSixkOtherBmpDetailed}{228.410}
\newcommand{\gvaUbeOneTwoEightkCjkDetailed}{33.596}
\newcommand{\gvaUbeOneTwoEightkOtherBmpDetailed}{189.992}
\newcommand{\gvaUbeOtherBmpDetailed}{204.090}
\newcommand{\gvaUbeThreeTwokCjkDetailed}{42.418}
\newcommand{\gvaUbeThreeTwokOtherBmpDetailed}{219.709}
\newcommand{\gvaUbeTwoFiveSixkCjkDetailed}{30.498}
\newcommand{\gvaUbeTwoFiveSixkOtherBmpDetailed}{167.665}

\newcommand{\covExOneSixkEnBbpe}{33.629}
\newcommand{\covExOneSixkEnUbe}{33.451}
\newcommand{\covExOneSixkEnDeltaPct}{-0.5\%}
\newcommand{\covExOneSixkAmharicBbpe}{219.321}
\newcommand{\covExOneSixkAmharicUbe}{140.388}
\newcommand{\covExOneSixkAmharicDeltaPct}{-36.0\%}
\newcommand{\covExOneSixkLaoBbpe}{343.176}
\newcommand{\covExOneSixkLaoUbe}{234.068}
\newcommand{\covExOneSixkLaoDeltaPct}{-31.8\%}
\newcommand{\covExOneSixkOlChikiBbpe}{363.850}
\newcommand{\covExOneSixkOlChikiUbe}{230.685}
\newcommand{\covExOneSixkOlChikiDeltaPct}{-36.6\%}
\newcommand{\covExOneSixkTeluguBbpe}{263.163}
\newcommand{\covExOneSixkTeluguUbe}{224.362}
\newcommand{\covExOneSixkTeluguDeltaPct}{-14.7\%}
\newcommand{\covExOneSixkKannadaBbpe}{354.337}
\newcommand{\covExOneSixkKannadaUbe}{238.641}
\newcommand{\covExOneSixkKannadaDeltaPct}{-32.7\%}
\newcommand{\covExOneSixkTibetanBbpe}{336.849}
\newcommand{\covExOneSixkTibetanUbe}{287.103}
\newcommand{\covExOneSixkTibetanDeltaPct}{-14.8\%}
\newcommand{\covExEnBbpe}{30.249}
\newcommand{\covExEnUbe}{30.169}
\newcommand{\covExEnDeltaPct}{-0.3\%}
\newcommand{\covExThreeTwokEnBbpe}{30.249}
\newcommand{\covExThreeTwokEnUbe}{30.169}
\newcommand{\covExThreeTwokEnDeltaPct}{-0.3\%}
\newcommand{\covExAmharicBbpe}{202.775}
\newcommand{\covExAmharicUbe}{137.539}
\newcommand{\covExAmharicDeltaPct}{-32.2\%}
\newcommand{\covExThreeTwokAmharicBbpe}{202.775}
\newcommand{\covExThreeTwokAmharicUbe}{137.539}
\newcommand{\covExThreeTwokAmharicDeltaPct}{-32.2\%}
\newcommand{\covExLaoBbpe}{342.100}
\newcommand{\covExLaoUbe}{233.445}
\newcommand{\covExLaoDeltaPct}{-31.8\%}
\newcommand{\covExThreeTwokLaoBbpe}{342.100}
\newcommand{\covExThreeTwokLaoUbe}{233.445}
\newcommand{\covExThreeTwokLaoDeltaPct}{-31.8\%}
\newcommand{\covExOlChikiBbpe}{341.287}
\newcommand{\covExOlChikiUbe}{222.707}
\newcommand{\covExOlChikiDeltaPct}{-34.7\%}
\newcommand{\covExThreeTwokOlChikiBbpe}{341.287}
\newcommand{\covExThreeTwokOlChikiUbe}{222.707}
\newcommand{\covExThreeTwokOlChikiDeltaPct}{-34.7\%}
\newcommand{\covExTeluguBbpe}{237.530}
\newcommand{\covExTeluguUbe}{220.807}
\newcommand{\covExTeluguDeltaPct}{-7.0\%}
\newcommand{\covExThreeTwokTeluguBbpe}{237.530}
\newcommand{\covExThreeTwokTeluguUbe}{220.807}
\newcommand{\covExThreeTwokTeluguDeltaPct}{-7.0\%}
\newcommand{\covExKannadaBbpe}{251.956}
\newcommand{\covExKannadaUbe}{237.506}
\newcommand{\covExKannadaDeltaPct}{-5.7\%}
\newcommand{\covExThreeTwokKannadaBbpe}{251.956}
\newcommand{\covExThreeTwokKannadaUbe}{237.506}
\newcommand{\covExThreeTwokKannadaDeltaPct}{-5.7\%}
\newcommand{\covExTibetanBbpe}{264.246}
\newcommand{\covExTibetanUbe}{247.286}
\newcommand{\covExTibetanDeltaPct}{-6.4\%}
\newcommand{\covExThreeTwokTibetanBbpe}{264.246}
\newcommand{\covExThreeTwokTibetanUbe}{247.286}
\newcommand{\covExThreeTwokTibetanDeltaPct}{-6.4\%}
\newcommand{\covExSixFourkEnBbpe}{28.078}
\newcommand{\covExSixFourkEnUbe}{28.035}
\newcommand{\covExSixFourkEnDeltaPct}{-0.2\%}
\newcommand{\covExSixFourkAmharicBbpe}{163.896}
\newcommand{\covExSixFourkAmharicUbe}{134.837}
\newcommand{\covExSixFourkAmharicDeltaPct}{-17.7\%}
\newcommand{\covExSixFourkLaoBbpe}{256.355}
\newcommand{\covExSixFourkLaoUbe}{231.798}
\newcommand{\covExSixFourkLaoDeltaPct}{-9.6\%}
\newcommand{\covExSixFourkOlChikiBbpe}{341.274}
\newcommand{\covExSixFourkOlChikiUbe}{213.132}
\newcommand{\covExSixFourkOlChikiDeltaPct}{-37.5\%}
\newcommand{\covExSixFourkTeluguBbpe}{184.658}
\newcommand{\covExSixFourkTeluguUbe}{181.593}
\newcommand{\covExSixFourkTeluguDeltaPct}{-1.7\%}
\newcommand{\covExSixFourkKannadaBbpe}{220.944}
\newcommand{\covExSixFourkKannadaUbe}{219.524}
\newcommand{\covExSixFourkKannadaDeltaPct}{-0.6\%}
\newcommand{\covExSixFourkTibetanBbpe}{253.668}
\newcommand{\covExSixFourkTibetanUbe}{247.167}
\newcommand{\covExSixFourkTibetanDeltaPct}{-2.6\%}
\newcommand{\covExStressSavingMinPct}{6}
\newcommand{\covExStressSavingMaxPct}{35}
\newcommand{\covExStressContentGainMinPct}{6}
\newcommand{\covExStressContentGainMaxPct}{53}

\newcommand{\regimeOverallReductionRange}{5.4--7.2}
\newcommand{\regimeEnglishIncreaseRange}{18.5--19.9}
\newcommand{\regimeCjkIncreaseRange}{34.7--39.6}
\newcommand{\regimeCjkUbeReductionRange}{2.4--3.8}
\newcommand{\regimeGiniUbeReductionRange}{8.5--10.5}
\newcommand{\regimePrimaryPremium}{2.314}
\newcommand{\regimePrimaryGini}{0.2183}
\newcommand{\regimePrimaryBbpeActive}{5,834}
\newcommand{\regimePrimaryUbeActive}{5,864}
\newcommand{\regimeCappedPremium}{1.841}
\newcommand{\regimeCappedGini}{0.2281}
\newcommand{\regimeCappedBbpeActive}{4,500}
\newcommand{\regimeCappedUbeActive}{4,504}
\newcommand{\regimeMaskedConfigurations}{16}
\newcommand{\regimeMinLower}{181}
\newcommand{\regimeMaxLower}{185}
\newcommand{\regimeCjkBbpeThreePiecePrimaryRange}{0.2--0.6}
\newcommand{\regimeCjkBbpeThreePieceCappedRange}{3.5--5.5}
\newcommand{\regimeCjkUbeWholeSplitPrimaryRange}{1.7--2.2}
\newcommand{\regimeCjkBbpeWholeSplitPrimaryRange}{0.1--0.7}
\newcommand{\regimeEnglishSharePrimary}{29.2}
\newcommand{\regimeEnglishShareCapped}{1.0}
\newcommand{\regimeCjkSharePrimary}{26.5}
\newcommand{\regimeCjkShareCapped}{3.0}
\newcommand{\regimeHindiThaiSharePrimary}{17.4}
\newcommand{\regimeHindiThaiShareCapped}{2.0}
\newcommand{\regimeRemainingSharePrimary}{26.8}
\newcommand{\regimeRemainingShareCapped}{94.1}
\newcommand{\regimeBurmeseSharePrimary}{0.3}
\newcommand{\regimeBurmeseShareCapped}{1.0}
\newcommand{\regimeCjkSplitPrimaryRange}{9.7--13.8}
\newcommand{\regimeCjkSplitCappedRange}{34.3--38.7}
\newcommand{\regimeOtherBmpMaxChange}{3.9}
\newcommand{\regimeMyanmarReductionRange}{6.6--19.3}
\newcommand{\regimeAbsentScriptMaxChange}{7.5}
\newcommand{\regimeLatinArabicCyrillicDropShareRange}{96--104}

\newif\ifcodereleased
\newcommand{\codeurl}{}
\newcommand{\codeavailability}{\ifcodereleased The UBE tokenizer implementation, the Unicode fidelity audit, and the MYTE-UBE transport and parity-verification code are available at \expandafter\url\expandafter{\codeurl}.\else We plan to release the UBE tokenizer implementation, the Unicode fidelity audit, and the MYTE-UBE transport and parity-verification code publicly.\fi}

\title{Universal Byte-Level Encoding: UTF-8/UTF-16 Routing to Reduce Cross-Script Token-Budget Disparities}

\author{%
  Hyunsik Kim \\
  Samsung Research \\
  \texttt{hyunsik777.kim@samsung.com}
  \And
  Youngmoon Jung \\
  Samsung Research \\
  \texttt{youngm.jung@samsung.com}
}

\begin{document}

\maketitle

\begin{abstract}
Byte-level byte-pair encoding (BBPE) tokenizers are attractive for large language models (LLMs) in multilingual settings because they cover all Unicode text. In \utfEightName{}-based BBPE, however, many scripts start from a higher byte-level fallback cost than English: when no learned merges can be applied, a multibyte character requires multiple byte-derived symbols. We call this worst-case pre-merge cost the \emph{encoding floor}. A higher floor can increase token counts, shrink usable context, and increase per-request cost. Changing the underlying text encoding can reduce this gap, but a single global encoding can also make already-efficient English spans more expensive in mixed-script text. We propose Universal Byte-Level Encoding (UBE), a dual-alphabet tokenizer that keeps characters encoded in 1--2 \utfEightName{} bytes on the \utfEightName{} path while routing 3--4-byte \utfEightName{} characters through \utfSixteenName{}. This lowers the encoding floor for 3-byte Basic Multilingual Plane (BMP) characters in scripts with high token premiums (token counts relative to English) without raising it for already-efficient spans in mixed-script text. UBE changes only the byte representation presented to byte-pair encoding (BPE): the BPE merge rule remains standard, and exact decoding is preserved. UBE also composes with alternative boundary policies and morphology-based representations. In a Unicode 17 audit, UBE exactly round-trips all Unicode scalar values and all inputs in the official normalization, grapheme-break, and emoji test suites. Across intrinsic tokenization evaluations, UBE lowers dispersion in English-normalized token-count ratios, reducing cross-lingual token-budget disparity. In multilingual language model (LM) experiments, UBE matches BBPE's LM quality. In the main multilingual settings, UBE reduces token counts most for high-premium scripts and also slightly lowers English token counts; these reductions translate into more usable context under fixed token budgets and faster prompt processing in content-matched benchmarks.
\end{abstract}

\section{Introduction}

Tokenization is often treated as a preprocessing detail, but in multilingual language modeling it also determines how fixed context and compute budgets are spent across languages. When a tokenizer fragments one language into many more units than another, semantically comparable content consumes more of the usable context window and requires processing longer sequences.

For autoregressive transformers, prompt token count determines the sequence length processed during prefill, so cross-lingual differences in tokenized length, which we call token-budget disparity under a fixed tokenizer, directly affect usable context, latency, and prompt-processing cost. This disparity is not merely a side effect of multilingual data imbalance; it is also a property of tokenizer design itself~\citep{petrov2023unfairness,arnett2025explaining}.

We use fairness in an operational sense: under a shared tokenizer, semantically comparable text should not require systematically different token budgets solely because characters from different scripts have different per-character byte-encoding costs. Even when high-resource languages dominate overall usage, this disparity remains operationally consequential because it reduces usable context under fixed token budgets and affects latency and per-request cost for users writing in scripts with higher byte costs. As multilingual large language models (LLMs) are deployed across APIs, embedded applications, and instruction-tuned assistants, the same disparity also becomes a pricing concern in pay-per-token interfaces~\citep{ahia2023languages,artola2025overcharging}.

\citet{petrov2023unfairness} show that, under widely deployed tokenizers such as GPT-4's cl100k\_base, parallel Shan text can require up to $15\times$ as many tokens as English; by our calculation, a fixed 128K-token context then holds only as much Shan content as about 8.5K English tokens. The problem is especially acute for compact or on-device language models (LMs), where English-centric compact models often use 32K--50K tokenizers~\citep{allal2025smollm2,liu2024mobilellm}; prior work also suggests that tokenizer quality matters more at smaller model scales~\citep{goldman2024unpacking}. Under a compact 16K vocabulary on \floresTwoHundredName{}~\citep{costajussa2022nllb}, a byte-level byte-pair encoding (BBPE) tokenizer averages \gvaBbpeOneSixkOtherBmpDetailed{} tokens per parallel sentence for seven Tibetan, Myanmar, Ol~Chiki, and Tifinagh configurations, whose 3-byte Basic Multilingual Plane (BMP) scripts have little dedicated training data, but only \gvaBbpeOneSixkEnDetailed{} for English (Appendix Table~\ref{tab:gini-vs-tokens}), so the same context window holds only about one-ninth as much parallel content in these scripts after normalizing by English token counts. Tokenizer design therefore becomes especially consequential for compact multilingual LMs serving mixed-script users under fixed context budgets.

BBPE is a useful setting for isolating this issue because it is simple, reversible, and language-agnostic: every Unicode string can be represented using a fixed \bbpeAlphabetSize{}-symbol byte alphabet, and byte-pair encoding (BPE) merges can be learned without language-specific segmentation rules. This simplicity, however, comes with an asymmetry inherited from \utfEightName{}: before any merges are learned, characters from different scripts are represented by byte sequences of different lengths. In BBPE, each character is first decomposed into \utfEightName{} bytes, and BPE merges then combine frequent adjacent byte-derived symbols into larger tokens. Because a finite multilingual vocabulary cannot cover all subword patterns that arise across languages and text domains, low-coverage spans continue to fall back toward byte-level decomposition. We define the \emph{encoding floor} for this substrate-imposed fallback as a scalar-level fallback cost: the number of base byte symbols required for one Unicode scalar value under the underlying byte encoding when no learned merge covers it. It is not a lower bound on post-BPE token counts; learned merges can compress frequent characters and multi-character spans below the floor, while underrepresented spans often remain close to it. Under \utfEightName{}, this floor is script- and code-point-dependent (Section~\ref{sec:method}, Table~\ref{tab:encoding-floor}).

\looseness=-1
The encoding floor is a scalar-level, pre-merge analogue of the corpus-level \emph{byte premium}~\citep{arnett2024bitproblem}, which measures \utfEightName{} byte-length ratios between content-matched texts. Larger vocabularies absorb more frequent patterns, although added tokens become progressively less frequent~\citep{reddy2025diminishing}, and token-premium disparities persist across comparably trained monolingual tokenizers even as vocabulary size increases~\citep{arnett2025explaining}; spans that no learned merge covers still pay the full encoding floor, motivating our aligned token premiums and English-normalized dispersion metrics.

Recent work shows that replacing the byte-encoding substrate can lower this floor for some scripts~\citep{kim2026bbpe16}; however, a uniform replacement can also raise costs for spans already efficient under \utfEightName{} in mixed-script text. In this work, we propose Universal Byte-Level Encoding (UBE), a dual-alphabet BBPE tokenizer with deterministic per-character routing based on \utfEightName{} byte length. Characters encoded in 1--2 \utfEightName{} bytes stay on the \utfEightName{} path, while 3-byte BMP characters are routed through \utfSixteenName{}, where they require 2 bytes. Because routing is per character rather than per language, UBE lowers the encoding floor for scripts most exposed to byte-level fragmentation without imposing the uniform 2-byte floor on already-efficient spans in mixed-script text. The method leaves BPE training and merge learning unchanged, preserves exact round-trip decoding for encoded text, and requires no language identification or model-architecture changes.

\section{Related work}
\citet{petrov2023unfairness} show on parallel text that tokenized length can vary sharply across languages, and \citet{ahia2023languages} show that the same content-matched disparity appears in API pricing. We therefore use aligned sentence-level token premiums as our main intrinsic measure of cross-lingual token-budget disparity (Section~\ref{sec:experiments}). Follow-up work links tokenization disparity to unequal utility, downstream performance, latency, and infrastructure/access costs~\citep{lundin2025tokentax,dixit2026scripttax,teklehaymanot2025infrastructure}.

\citet{arnett2025explaining} analyze causes and mitigations of cross-lingual tokenizer inequity, while \citet{shani2026roots} situate tokenizer choices among broader design choices behind multilingual performance gaps. Related metrics quantify disparity at the tokenizer level (Single Token Retention Rate~\citep{nayeem2025strr}), downstream LM-task level (Performance Realization Ratio~\citep{hu2025quantifying}), and dialectal NLP task level (Tokenization Parity~\citep{petrov2023unfairness}, as applied by \citet{kanjirangat2025tokenization}). \citet{nayeem2026tokenizers} argue for tokenizer--model co-design and standardized evaluation.

SentencePiece supports BPE and unigram segmentation directly from raw text~\citep{kudo2018sentencepiece}; it is not restricted to byte-based representations. Mitigations fall into three groups. Within-substrate methods change merge or segmentation objectives, pre-tokenization, vocabulary allocation, or post-hoc vocabulary additions while keeping the byte encoding fixed~\citep{liu2025superbpe,schmidt2025boundless,abagyan2025onetokenizer,churchill2026reducing}; these can compose with UBE but leave the \utfEightName{} encoding floor unchanged. Substrate-changing methods alter the input symbols: \bbpeSixteenName{} applies \utfSixteenName{} uniformly~\citep{kim2026bbpe16}, SCRIPT-BPE combines a script-aware representation with script/category boundaries and constrained merges~\citep{land2025script}, MYTE learns morphology-driven byte codes~\citep{limisiewicz2024myte}, and Bit-level BPE re-splits the byte stream at the bit level~\citep{moon2025bitbpe}. UBE instead routes each character between \utfEightName{} and \utfSixteenName{} inside one reversible tokenizer; Section~\ref{sec:alternative-tokenizer-comparisons} and Appendix~\ref{app:comparative-analyses} compare these representation, boundary, and morphology-based alternatives at the tokenizer and model levels. Tokenization-free and byte-level models operate directly on characters or bytes: CANINE~\citep{clark2022canine}, ByT5~\citep{xue2022byt5}, BLT~\citep{pagnoni2025blt}, MAGNET~\citep{ahia2024magnet}, and H-Net~\citep{hwang2025hnet}; they reduce reliance on fixed tokenizers but change the model's architecture or input granularity, whereas UBE preserves standard BPE training and exact decoding. The appendix surveys further byte-level architectures (Appendix~\ref{app:extended-related-work}).

\section{Background and method}
\label{sec:method}

\begin{table}[H]
\caption{Encoding-floor values before BPE merges. \bbpeSixteenName{}~\citep{kim2026bbpe16} is a \utfSixteenName{}-only baseline; bold marks UBE's target trade-off: ASCII stays at the 1-byte \utfEightName{} floor, while 3-byte BMP scripts (e.g., CJK, Devanagari, and Thai) fall from 3 to 2 via the \utfSixteenName{} Private Use Area (PUA) path; 4-byte non-BMP code points remain at a 4-symbol floor. CJK stands for Chinese, Japanese, and Korean.}
\label{tab:encoding-floor}
\centering
\small
\begin{tabular}{lcccc}
\toprule
 & ASCII & Cyrillic/Arabic/Hebrew & \shortstack{CJK,\\Devanagari, Thai} & Non-BMP \\
\midrule
BBPE & 1 & 2 & 3 & 4 \\
\bbpeSixteenName{} & 2 & 2 & 2 & 4 \\
UBE & \textbf{1} & 2 & \textbf{2} & 4 \\
\bottomrule
\end{tabular}

\end{table}

\paragraph{Encoding floor under standard BBPE.}
BBPE~\citep{wang2020bbpe} converts text to \utfEightName{} bytes and learns BPE merges over the resulting byte stream; in \gptTwoName{}-style implementations, following the encoder released with \gptTwoName{}~\citep{radford2019gpt2}, each byte is first mapped to a reversible visible symbol. Under the encoding-floor definition introduced above, BBPE's floor equals the \utfEightName{} byte length: \utfEightAsciiFloor{} for ASCII, \utfEightMidFloor{} for many Cyrillic, Arabic, and Hebrew characters, \utfEightHighFloor{} for Chinese, Japanese, and Korean (CJK) and other 3-byte scripts, and \utfEightNonBmpFloor{} for non-BMP code points (Table~\ref{tab:encoding-floor}).

\begin{figure}[t]
\centering
\includegraphics[width=0.85\linewidth]{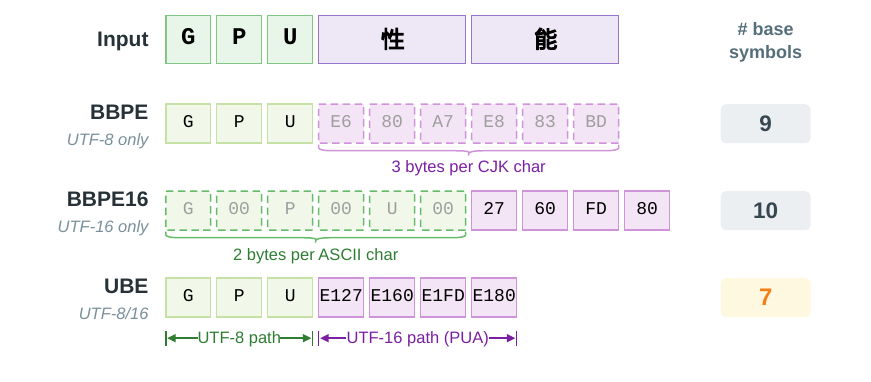}
\caption{Pre-merge base-symbol decomposition of ``GPU\begin{CJK}{UTF8}{gbsn}性能\end{CJK}''. UBE, BBPE, and \bbpeSixteenName{} use 7, 9, and 10 base symbols, respectively. PUA denotes the Private Use Area; Appendix~\ref{app:ube-architecture} gives the full algorithm.}
\label{fig:ube-pipeline}
\end{figure}

\paragraph{Design target and compatibility constraints.}
UBE targets one specific source of multilingual tokenization disparity in BBPE under finite vocabularies: 3-byte BMP scripts that remain close to byte fallback continue to pay a higher per-character cost because their encoding floor (\utfEightHighFloor{} bytes) exceeds the \utfEightAsciiFloor{}- and \utfEightMidFloor{}-byte cases (Table~\ref{tab:encoding-floor}). Non-BMP 4-byte code points, including many emoji, are routed through the same deterministic branch but remain 4 base symbols as \utfSixteenName{} surrogate pairs; UBE therefore matches BBPE's 4-symbol encoding floor for these characters, and unless otherwise noted we focus on BMP 3-byte scripts. To preserve compatibility with existing BPE pipelines, we impose three constraints: exact reversibility, unchanged BPE training and vocabulary learning, and no explicit marker tokens for encoding choice, whose signal could later be absorbed by BPE merges.

\paragraph{UBE: universal byte-level encoding.}
\bbpeSixteenName{}~\citep{kim2026bbpe16} demonstrated that replacing \utfEightName{} with \utfSixteenName{} lowers the CJK floor to \utfEightMidFloor{} while preserving a language-agnostic byte-level representation, which is well matched to multilingual automatic speech recognition (ASR) and other settings dominated by CJK text. In the mixed-script LM setting studied here, however, the same global substrate also raises the ASCII floor to \utfEightMidFloor{} (Table~\ref{tab:encoding-floor}; Figure~\ref{fig:ube-pipeline} illustrates the contrast on a mixed-script example). UBE avoids this trade-off with a single deterministic \utfEightName{}/\utfSixteenName{} routing rule based on the number of bytes in each character's \utfEightName{} encoding:
\[
r(c)=
\begin{cases}
\text{\utfEightName{}} & \text{if } |\mathrm{utf8}(c)| \leq \ubeRouteUtfEightMaxBytes,\\
\text{\utfSixteenName{}} & \text{otherwise.}
\end{cases}
\]
Thus, the routing rule forms a complete partition of Unicode scalar values. The \utfSixteenName{} branch serializes code units as UTF-16LE without a byte-order mark. Here, ``character'' means one Unicode scalar value rather than a user-perceived grapheme cluster; combining-mark sequences and zero-width joiner sequences are routed scalar by scalar, with no additional normalization. Two disjoint \bbpeAlphabetSize{}-symbol alphabets---the \gptTwoName{} byte alphabet for \utfEightName{} and a Private Use Area (PUA) block (U+E100--E1FF) for \utfSixteenName{}\footnote{Using this PUA range for internal byte symbols does not introduce ambiguity: input characters in the same range are first encoded as UTF-16LE bytes and then mapped to those symbols (Appendix~\ref{app:encoding-example}).}---let the decoder infer the active encoding from each symbol's alphabet membership, without marker tokens. The full encoder/decoder details, including pseudocode (Algorithms~\ref{alg:ube-encode} and~\ref{alg:ube-decode}) and dual-buffer flush-on-switch decoding, are in Appendix~\ref{app:ube-architecture}.

\paragraph{BPE compatibility and exact decoding.}
UBE changes the initial symbol stream presented to a standard BPE trainer: the base alphabet becomes the \ubeAlphabetSize{}-symbol union of both byte alphabets, and the vocabulary is learned from that routed stream. The pair-counting interface and merge-selection criterion remain the standard BPE ones; UBE introduces no script-aware merge rule or marker token. UBE therefore wraps an unchanged BPE core with a routing pre-tokenizer, disjoint base alphabet, and custom dual-buffer tokenizer runtime; the Transformer architecture is unchanged. Proposition~\ref{prop:roundtrip} in Appendix~\ref{app:decoder} proves exact round-trip decoding. Empirically, UBE has no decoding failures on \floresTwoHundredName{} across \appendixDecodeVerifiedLangCount{} languages (\appendixDecodeVerifiedPassRate{} pass rate) and recovers every tested raw input in an exhaustive Unicode 17 audit spanning all Unicode scalar values and the official normalization, grapheme-break, and emoji test suites. Among the audited tokenizers, only UBE recovers every input: MYTE recovers 1,138,306 of the 1,155,724 inputs, and SCRIPT-BPE applies Unicode Normalization Form C (NFC) and filters Private-Use scalars, unassigned scalars, and Unicode 17 additions (Appendix Table~\ref{tab:cr-unicode-fidelity}).

\paragraph{Alphabet size and marker-free design.}
Keeping ASCII at a 1-byte floor while mapping common BMP 3-byte \utfEightName{} characters to a 2-byte path requires the decoder to distinguish which byte encoding produced each base symbol. UBE uses a \ubeAlphabetSize{}-symbol disjoint union alphabet as a simple marker-free construction under unrestricted BPE merges: BPE may form merge tokens whose expanded base symbols span both alphabets (e.g., a Latin--CJK adjacent pair within a single regex chunk); after merge-token expansion, alphabet membership identifies the decoder branch internal to the merge token, so the dual-buffer decoder flushes at the alphabet boundary even when merges span routing boundaries (Appendix~\ref{app:vocab-diagnostics} reports cross-alphabet merge frequency across vocabulary sizes).

\paragraph{Vocabulary overhead.}
The trade-off is vocabulary budget: with total vocabulary size $V$ and $S=5$ shared special tokens, UBE leaves $(V-S-\ubeAlphabetSize)$ merge slots instead of $(V-S-\bbpeAlphabetSize)$ for BBPE---the same \bbpeAlphabetSize{}-slot difference, an overhead of roughly $\bbpeAlphabetSize{}/V$, still \paperUnderOnePct{} at typical LM vocabulary sizes ($V \geq \paperVocabThreeTwok{}$).

\section{Experiments}
\label{sec:experiments}

\subsection{Intrinsic tokenizer evaluation}

\paragraph{Metrics and evaluation scope.}
Our primary intrinsic metric is the language-level token premium computed from sentence-level token counts on aligned text. For a tokenizer $\tau$, language $\ell$, and aligned sentence set $S_\ell$, let $m_\tau(\ell) = \frac{1}{|S_\ell|}\sum_{s \in S_\ell} |\tau(s)|$ denote the mean number of tokens per sentence. We define the premium as $\pi_\tau(\ell) = m_\tau(\ell) / m_\tau(\text{en})$; thus $\pi_\tau(\ell)>1$ means that aligned content occupies more positions in the context window than English under the same tokenizer. For dispersion statistics, the English reference row is excluded, so Gini and variance are computed over $\{\pi_\tau(\ell): \ell \neq \text{en}\}$ (definitions in Appendix~\ref{app:intrinsic-significance}). For the additional robustness corpora, without using sentence alignment, we compute $m_\tau(\ell)$ as the mean token count over the available text units (nonempty UDHR line segments, SIB-200 examples, or MGSM questions) and form the same English-normalized non-English ratios before computing Gini; these are descriptive corpus-level comparisons rather than content-matched sentence premiums. Because BBPE tokenizers produce a tokenization for arbitrary Unicode strings rather than emitting unknown tokens~\citep{wang2020bbpe,radford2019gpt2}, we describe such text as byte-fallback-heavy rather than out-of-vocabulary. These tokenizer-level statistics quantify allocation under fixed context and compute budgets: for semantically matched inputs, they determine how much context capacity and prefill computation each language consumes.

\paragraph{Tokenizer setup and evaluation.}
In the core UBE--BBPE comparison, both variants use the same BPE pipeline; only the initial symbol stream differs. We compare three fixed mC4~\citep{xue2021mt5} raw-text shard configurations: \texttt{En-only} (encoding-stress), \texttt{6-lang} (\{en, zh, ko, ja, hi, th\}, narrow-coverage), and \texttt{101-lang} (full-coverage, used for the main matched LM comparison). The core matrix trains all three configurations at 16K/32K/64K vocabularies, and the \texttt{101-lang} vocabulary-scaling sweep extends to 128K and 256K; within each configuration--vocabulary pair, the raw input data are held fixed across tokenizer families so token count is a measured outcome. The primary intrinsic benchmark is the FLORES-200 devtest split (1,012 parallel sentences $\times$ 204 language--script configurations, hereafter languages); pipeline details (shard-local counting, singleton filtering, merge aggregation), corpus construction, and per-configuration token-count statistics are in Appendix Sections~\ref{app:tokenizer-training} and~\ref{app:corpus-accounting}.

\paragraph{FLORES-200 results.}
UBE reduces the Gini coefficient of token premiums in every setting of the \texttt{En-only}, \texttt{6-lang}, and \texttt{101-lang} $\times$ 16K/32K/64K core matrix, and also in the \texttt{101-lang} 128K and 256K extensions.
Across the nine core comparisons, relative reductions range from \paperStatGiniReductionMinPct{} to \paperStatGiniReductionMaxPct{} (Appendix Table~\ref{tab:config-ablation-app}).
The relative Gini reduction is largest at \texttt{En-only/16K} and smallest at \texttt{6-lang/64K}.
Appendix Tables~\ref{tab:config-ablation-app} and~\ref{tab:gini-vs-tokens} give the full configuration and vocabulary sweeps.
The three-way matrix identifies vocabulary size as the dominant source of variation: broader training coverage narrows the UBE--BBPE gap only after the merge budget is already moderate.
Even at \texttt{101-lang/16K}, however, UBE still lowers Gini (\oneOhOnevOneSixkGiniPct{}), showing that merge-budget dilution across many scripts keeps many languages near the encoding floor.
A paired language bootstrap over non-English language-level premiums directly supports the dispersion claim: the 95\% confidence interval (CI) for $\Delta$Gini (UBE $-$ BBPE) is below zero in all nine core comparisons, including [\OneZeroOnelangvThreeTwoZeroZeroZeroGiniCIlo{}, \OneZeroOnelangvThreeTwoZeroZeroZeroGiniCIhi{}] in the primary \texttt{101-lang/32K} setting (Appendix Table~\ref{tab:cr-intrinsic-statistics}).
As a separate paired token-count check, two-sided Wilcoxon signed-rank tests over language-level mean token counts are significant in \paperStatWSigCount{}/9 core comparisons at $p < 0.01$.
The same \paperStatWSigCount{}/9 token-count decision holds after Holm correction at $\alpha=0.05$ and Benjamini--Hochberg false discovery rate (FDR) correction at $q=0.05$ over the nine prespecified core comparisons.
The only non-significant token-count comparison is \texttt{En-only/64K}, where UBE and BBPE have nearly matched language-level token counts ($p = \paperStatWMaxP{}$).

\paragraph{Encoding-floor mechanism under vocabulary scaling.}
Under \texttt{101-lang/32K}, UBE reduces Gini from \oneOhOnevThreeTwokGiniBBPE{} to \oneOhOnevThreeTwokGiniUBE{} (\oneOhOnevThreeTwokGiniPct{}) with a mean UBE/BBPE ratio of \oneOhOnevThreeTwokMeanRatio{}. The largest token-count reductions occur in low-coverage 3-byte scripts such as Santali (Ol~Chiki), Central Atlas Tamazight (Tifinagh), and Tibetan, where UBE reduces the byte-level fallback cost from 3 bytes to 2 bytes per character. Across all 204 languages, UBE produces fewer tokens in \oneOhOnevThreeTwokBetter{} languages and more tokens in \oneOhOnevThreeTwokWorse{} languages.
Figure~\ref{fig:gini-vs-tokens} shows that UBE's Gini stays below BBPE's from 16K to 256K, although the gap narrows at larger vocabularies (\oneOhOnevTwoFiveSixkGiniPct{} at 256K); Appendix Table~\ref{tab:gini-vs-tokens} reports the full numerical sweep. The overall upward trend in Gini reflects heterogeneous compression gains among non-English languages (Appendix Sections~\ref{app:gini-rises} and~\ref{app:token-cost-model}). Because the half-data ablation (Appendix~\ref{app:halfdata-ablation}) preserves per-language sampling proportions, it tests corpus-size effects rather than language-mixture effects; the near-overlap of full- and half-data curves across all four tokenizer families argues that the scaling pattern is not explained simply by insufficient tokenizer-training volume. Appendix~\ref{app:per-lang-distribution} provides the per-language distribution view, and Appendix~\ref{app:script-family} provides the script-family view of the same effect.

\begin{figure}[t]
\centering
\resizebox{\linewidth}{!}{\input{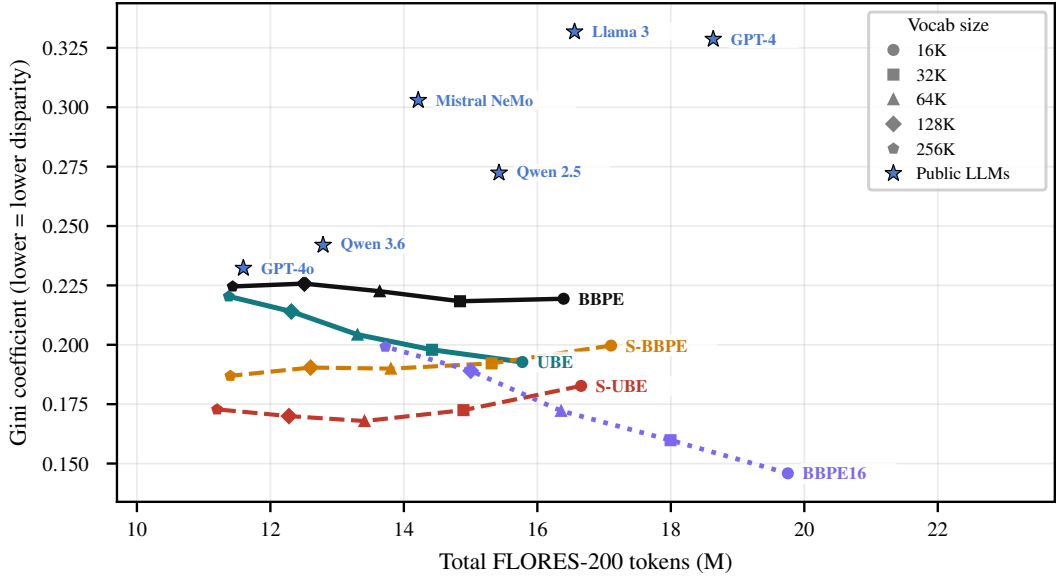}}
\caption{Token-premium Gini vs.\ total FLORES-200 token count (\texttt{101-lang}; 16K--256K vocabularies). S-UBE and S-BBPE apply SuperBPE-style boundary relaxation to UBE and BBPE, respectively. Public LLM tokenizers are reference points (Appendix Table~\ref{tab:public-llm-gini}); lower-left is better.}
\label{fig:gini-vs-tokens}
\end{figure}

\paragraph{Cross-dataset robustness and contamination.}
To test robustness on additional evaluation sets, we repeat the \texttt{101-lang/32K} comparison on three additional corpora (Table~\ref{tab:cross-dataset-robustness}). The \href{https://github.com/nltk/nltk_data/blob/gh-pages/packages/corpora/udhr2.xml}{NLTK Unicode UDHR corpus} (\texttt{udhr2}; \udhrRobLangs{} entries, many with no mC4 training data) shows a larger Gini reduction under UBE (\udhrRobThreeTwokGiniPct{}) than on FLORES-200 (\floresRobThreeTwokGiniPct{}), consistent with the encoding-floor mechanism being most visible where many entries remain near the byte-level fallback floor. Because UDHR is widely reproduced and our mC4-overlap audit finds frequent UDHR prefix matches, we treat it as a public-text stress test rather than a clean held-out set (Appendix~\ref{app:contamination}). SIB-200~\citep{adelani2024sib200}, a topic-labeled subset of FLORES-200 sentences, also improves (\sibRobThreeTwokGiniPct{}), while MGSM~\citep{shi2022mgsm} is effectively unchanged (\mgsmRobThreeTwokGiniPct{}), as expected for its small set of languages, all of which appear in the 101-language mC4 corpus. Appendix Table~\ref{tab:cross-dataset-vocab} reports the full 16K--256K sweep for all four corpora, including the 32K row summarized here. A scan of all mC4 training shards (held-out evaluation shards excluded) finds no exact FLORES-200 matches across \contamDocsM{} documents and \contamLangsChecked{} languages; a \contamPrefixChars{}-character prefix-substring scan finds \contamFloresSubstringOverlaps{} FLORES-200 overlaps among \contamFloresSubstringEligibleSentences{} eligible sentences (Appendix~\ref{app:contamination}).

\paragraph{Alphabet separation trade-off.}
\label{sec:alphabet-tradeoff}
UBE's dual alphabet separates the byte identities of the \utfEightName{} and \utfSixteenName{} paths, so it can lose some shared-byte compression relative to BBPE\@.
Its \ubeAlphabetSize{}-symbol initial alphabet also consumes \bbpeAlphabetSize{} merge slots that standard BBPE uses for subword merges.
The \texttt{En-only} regime tests the worst-case cost: with only English training data, UBE wins just \eovOneSixkBetter{} of 204 languages at 16K, because most languages cannot recoup the lost merge slots.
Two observations mitigate this.
First, typical regressions are small: at \texttt{En-only/32K}, the median token-count increase for languages where UBE is worse is just \eovThreeTwokWorseMedianDelta{} tokens per sentence (mean \eovThreeTwokWorseMeanDelta{}), while UBE's token-count reductions on scripts that remain near the encoding floor reach up to \eovThreeTwokMaxGainDelta{} tokens per sentence.
Second, the Gini reduction at \texttt{En-only/16K} (\eovOneSixkGiniPct{}) confirms that UBE concentrates its benefits precisely where disparity is largest.
The coverage threshold for the PUA alphabet is low: moving from \texttt{En-only} to just 6 languages spanning Latin text and several 3-byte BMP script families jumps UBE's per-language wins from \eovOneSixkBetter{}/204 to \fullvOneSixkBetter{}/204 at 16K.
At \texttt{6-lang/32K}, UBE wins \fullvThreeTwokBetter{}/204 languages.
At \texttt{101-lang/32K}, it wins \oneOhOnevThreeTwokBetter{}/204, showing that even modest script coverage enables the PUA alphabet to learn productive merges and offset the alphabet overhead.
Word-level token composition and decomposition examples are collected in Appendix~\ref{app:token-composition}.
Vocabulary diagnostics support the capacity interpretation: at \texttt{101-lang/32K}, active utilization remains nearly identical to BBPE (\vdBbpeActiveUtil{} vs.\ \vdUbeActiveUtil{}), dead merges fall from \vdBbpeDeadMerges{} to \vdUbeDeadMerges{}, and depth-1 merges roughly double (\vdBbpeDepthOnePct{}$\to$\vdUbeDepthOnePct{}), suggesting that the extra alphabet capacity is not simply stranded (Appendix~\ref{app:vocab-diagnostics}).

\begin{table}[t]
\caption{Cross-dataset token-premium Gini (\texttt{101-lang/32K}). UDHR is a public-text stress test with prefix overlaps. The full 16K--256K sweep is in Appendix Table~\ref{tab:cross-dataset-vocab}.}
\label{tab:cross-dataset-robustness}
\centering
\small
\begin{tabular}{lrrrrl}
\toprule
Dataset & Eval.\ units & UBE Gini & BBPE Gini & $\Delta$Gini & Note \\
\midrule
FLORES-200 & \floresRobLangs{} & \floresRobThreeTwokGiniUBE{} & \floresRobThreeTwokGiniBBPE{} & \floresRobThreeTwokGiniPct{} & parallel sentences \\
UDHR & \udhrRobLangs{} & \udhrRobThreeTwokGiniUBE{} & \udhrRobThreeTwokGiniBBPE{} & \udhrRobThreeTwokGiniPct{} & public-text stress test \\
SIB-200 & \sibRobLangs{} & \sibRobThreeTwokGiniUBE{} & \sibRobThreeTwokGiniBBPE{} & \sibRobThreeTwokGiniPct{} & topic classification \\
MGSM & \mgsmRobLangs{} & \mgsmRobThreeTwokGiniUBE{} & \mgsmRobThreeTwokGiniBBPE{} & \mgsmRobThreeTwokGiniPct{} & math reasoning \\
\bottomrule
\end{tabular}
\vspace{-8pt}
\end{table}

\paragraph{Gini vs.\ absolute token counts.}
Gini alone does not capture absolute token cost: \bbpeSixteenName{} achieves the lowest Gini (\gvaBbpeSixteenThreeTwokGini{}), consistent with its CJK-heavy ASR target domain, but its uniform 2-byte floor narrows relative premiums at the cost of higher absolute token counts for English (\gvaBbpeThreeTwokEnDetailed{}$\to$\gvaBbpeSixteenThreeTwokEnDetailed{}) in mixed-script LM settings. At the 32K main-text point (Figure~\ref{fig:gini-vs-tokens}; full numerical breakdown in Appendix Table~\ref{tab:gini-vs-tokens}), UBE reduces Gini while also slightly lowering the overall average token count (\gvaBbpeThreeTwokAllDetailed{}$\to$\gvaUbeThreeTwokAllDetailed{}). SuperBPE-style boundary relaxation is complementary to UBE but shifts the trade-off: S-BBPE lowers Gini further at 32K (\gvaSbpeBbpeThreeTwokGini{} vs.\ \gvaUbeThreeTwokGini{}) while slightly inflating English tokens (\gvaBbpeThreeTwokEnDetailed{}$\to$\gvaSbpeBbpeThreeTwokEnDetailed{}), and S-UBE lowers Gini further still while plain UBE keeps fewer total FLORES-200 tokens at 16K--64K. Unlike these alternatives, UBE does not increase English cost: in the main multilingual settings (\texttt{6-lang} and \texttt{101-lang}; \texttt{En-only} excluded), it lowers English and Other BMP token counts at every vocabulary size while keeping CJK counts within 1\% of BBPE, and at 16K--64K it has the lowest total FLORES-200 count among the standalone methods tested. English's base stream is unchanged; the small difference emerges from other scripts changing global pair counts and merge competition, not its encoding floor. Appendix Tables~\ref{tab:gini-vs-tokens} and~\ref{tab:gini-vs-tokens-superbpe} report the full 16K--256K standalone and composition sweeps.

\begin{table}[b]
\vspace{-4pt}
\caption{Representative per-language token counts on \floresTwoHundredName{} (\texttt{101-lang/32K}; mean tokens per sentence). Other vocabulary sizes appear in Appendix Table~\ref{tab:rep-language-vocab}.}
\label{tab:rep-language-main}
\centering
\small
\setlength{\tabcolsep}{4pt}
\begin{tabular}{llrrr}
\toprule
Language & Script & UBE & BBPE & $\Delta$\% \\
\midrule
Santali & Ol~Chiki & \repLangSantaliUbe{} & \repLangSantaliBbpe{} & \repLangSantaliDeltaPct{} \\
Central Atlas Tamazight & Tifinagh & \repLangTamazightUbe{} & \repLangTamazightBbpe{} & \repLangTamazightDeltaPct{} \\
Tibetan & Tibetan & \repLangTibetanUbe{} & \repLangTibetanBbpe{} & \repLangTibetanDeltaPct{} \\
Arabic & Arabic & \repLangArabicUbe{} & \repLangArabicBbpe{} & \repLangArabicDeltaPct{} \\
English & Latin & \repLangEnglishUbe{} & \repLangEnglishBbpe{} & \repLangEnglishDeltaPct{} \\
\bottomrule
\end{tabular}
\end{table}

\paragraph{Mixed-script text and representative languages.}
On \msNDocs{} naturally occurring mixed-script mC4 documents (Appendix~\ref{app:mixed-script}), UBE matches or slightly improves on BBPE in \msUbeBetterOrTie{} tested languages, with Bengali showing a small regression (\msBengaliDeltaPct{}); \bbpeSixteenName{} applies the same uniform 2-byte floor to both CJK-heavy and already-efficient spans in this slice, so its aggregate token count is \msBbpeSixteenTotalDeltaPct{} above BBPE\@. The same pattern is visible in the core compact setting: under \texttt{101-lang/32K}, UBE cuts token counts sharply for low-coverage 3-byte scripts such as Santali, Central Atlas Tamazight, and Tibetan, while token counts also decrease slightly for English and Arabic (Table~\ref{tab:rep-language-main}). On the 100-task multilingual code workload~\citep{jimenez2024swebench,hofman2026maps}, UBE uses fewer tokens than BBPE in all 11 language versions and both conditions, including the 99.24\%-ASCII code condition (Appendix Table~\ref{tab:cr-coding-workload}). A complementary 6-language stress-set slice appears in Appendix~\ref{app:coverage-examples}; full per-language values, including the small regression cases, appear in Appendix~\ref{app:per-language-results}.

\subsection{Language model evaluation}

\paragraph{Language model protocol.}
To connect intrinsic disparity to downstream behavior, we train matched LM comparisons over \{UBE, BBPE\} $\times$ \{\texttt{En-only}, \texttt{6-lang}, \texttt{101-lang}\}; within each setting, the same 101-language mC4 corpus, architecture, optimizer, and token-step budget are held fixed while the tokenizer configuration changes. The primary attribution setting is \texttt{101-lang/32K}, where UBE and BBPE are exposure-matched to within $<0.3\%$ raw-text coverage; lower-coverage rows are fixed-compute stress tests where tokenizer quality and raw-text exposure co-vary. The 32K vocabulary was chosen to maximize the UBE--BBPE floor contrast; larger-vocabulary variants, a 1.3B scale check, and appendix-only SuperBPE composition runs are reported in Appendix~\ref{app:lm-eval-full}, with training details in Appendix~\ref{app:lm-training-config}.

We evaluate held-out per-language bits per byte (BPB) on deterministic mC4 terminal shards and zero-shot accuracy on multilingual and English benchmarks. BPB always uses the raw \utfEightName{} byte count of the underlying text as the denominator for both tokenizer conditions, so UBE and BBPE likelihoods are normalized on the same content scale. Exact model sizes, token budgets, optimizer settings, sequence length, BPB sampling, and evaluation splits are reported in Appendix~\ref{app:lm-training-config}. Appendix Table~\ref{tab:prior-work-scale} situates the reported suite against recent tokenizer-comparison papers, and Appendix~\ref{app:compact-lm-tokenizers} surveys compact public LM tokenizer scales.

\paragraph{LM results (182M scale).}
Table~\ref{tab:lm-eval} summarizes the matched 182M/32K LM comparison across the three tokenizer-coverage regimes.

\begin{table}[t]
\caption{LM evaluation (182M, 32K vocabulary, 200K mC4 steps; primary seed; multi-seed in Appendix~\ref{app:multiseed}) across three \emph{tokenizer-training} configurations. Every LM is trained on the same 101-language mC4 corpus; En-only/6-lang/101-lang identify only the corpus used to learn the tokenizer vocabulary. Bits per byte (BPB) and BPB Gini use 101 held-out mC4 language splits; benchmark columns report zero-shot accuracy (Bele = Belebele, HSwag = HellaSwag, LAMB = LAMBADA).}
\label{tab:lm-eval}
\centering
\footnotesize
\setlength{\tabcolsep}{2.5pt}
\begin{tabular}{llcccccccccc}
\toprule
Config & Tok & BPB$\downarrow$ & BPB Gini$\downarrow$ & XNLI & Bele & XCOPA & ARC-E & ARC-C & HSwag & PIQA & LAMB \\
\midrule
\multirow{2}{*}{En-only}
 & UBE & \lmUbeOneLangBpb{} & \lmUbeOneLangGini{} & \lmUbeOneLangXnli{} & \lmUbeOneLangBelebele{} & \lmUbeOneLangXcopa{} & \lmUbeOneLangArcE{} & \lmUbeOneLangArcC{} & \lmUbeOneLangHSwag{} & \lmUbeOneLangPiqa{} & \lmUbeOneLangLambada{} \\
 & BBPE & \lmBbpeOneLangBpb{} & \lmBbpeOneLangGini{} & \lmBbpeOneLangXnli{} & \lmBbpeOneLangBelebele{} & \lmBbpeOneLangXcopa{} & \lmBbpeOneLangArcE{} & \lmBbpeOneLangArcC{} & \lmBbpeOneLangHSwag{} & \lmBbpeOneLangPiqa{} & \lmBbpeOneLangLambada{} \\
\midrule
\multirow{2}{*}{6-lang}
 & UBE & \lmUbeSixLangBpb{} & \lmUbeSixLangGini{} & \lmUbeSixLangXnli{} & \lmUbeSixLangBelebele{} & \lmUbeSixLangXcopa{} & \lmUbeSixLangArcE{} & \lmUbeSixLangArcC{} & \lmUbeSixLangHSwag{} & \lmUbeSixLangPiqa{} & \lmUbeSixLangLambada{} \\
 & BBPE & \lmBbpeSixLangBpb{} & \lmBbpeSixLangGini{} & \lmBbpeSixLangXnli{} & \lmBbpeSixLangBelebele{} & \lmBbpeSixLangXcopa{} & \lmBbpeSixLangArcE{} & \lmBbpeSixLangArcC{} & \lmBbpeSixLangHSwag{} & \lmBbpeSixLangPiqa{} & \lmBbpeSixLangLambada{} \\
\midrule
\multirow{2}{*}{101-lang}
 & UBE & \lmUbeOneZeroOneLangBpb{} & \lmUbeOneZeroOneLangGini{} & \lmUbeOneZeroOneLangXnli{} & \lmUbeOneZeroOneLangBelebele{} & \lmUbeOneZeroOneLangXcopa{} & \lmUbeOneZeroOneLangArcE{} & \lmUbeOneZeroOneLangArcC{} & \lmUbeOneZeroOneLangHSwag{} & \lmUbeOneZeroOneLangPiqa{} & \lmUbeOneZeroOneLangLambada{} \\
 & BBPE & \lmBbpeOneZeroOneLangBpb{} & \lmBbpeOneZeroOneLangGini{} & \lmBbpeOneZeroOneLangXnli{} & \lmBbpeOneZeroOneLangBelebele{} & \lmBbpeOneZeroOneLangXcopa{} & \lmBbpeOneZeroOneLangArcE{} & \lmBbpeOneZeroOneLangArcC{} & \lmBbpeOneZeroOneLangHSwag{} & \lmBbpeOneZeroOneLangPiqa{} & \lmBbpeOneZeroOneLangLambada{} \\
\bottomrule
\end{tabular}
\end{table}

In the primary seed, UBE has lower BPB Gini than BBPE in all three coverage regimes.
Across \lmMultiseedNSeeds{} seeds at \texttt{101-lang/32K}, however, the mean BPB Gini is effectively tied (UBE \lmMsUbeGiniMean{}$\pm$\lmMsUbeGiniStd{} vs.\ BBPE \lmMsBbpeGiniMean{}$\pm$\lmMsBbpeGiniStd{}; Appendix Table~\ref{tab:lm-eval-multiseed}), with cross-seed variability comparable to the gap.
We therefore do not claim a systematic LM-level BPB-dispersion advantage, and we treat the matched end-of-training BPB dispersion as neutral.
Across macro BPB, multilingual benchmarks, and English regression tasks, deltas are small and mixed.

\paragraph{BPB breakdown and scale checks.}
\label{sec:bpb-breakdown}
Disaggregated by encoding-floor group, held-out BPB differences are small in the median: median $\Delta$BPB is \lmOneZeroOneLangOneByteMedianDelta{} for 1-byte, \lmOneZeroOneLangTwoByteMedianDelta{} for 2-byte, and \lmOneZeroOneLangThreeByteMedianDelta{} for 3-byte scripts; in the primary seed, UBE wins \lmOneZeroOneLangUbeWins{}/101 languages despite near-tied aggregate BPB (\lmOneZeroOneLangMeanDeltaBpb{}), a pattern consistent with redistribution rather than uniform improvement. Larger language-specific shifts occur in both directions: Lao improves, whereas Burmese and Sinhala regress. We report BPB Gini as the primary cross-lingual dispersion measure; absolute per-language BPB values (from which any non-ratio dispersion statistic can be recomputed) are in Appendix Table~\ref{tab:perlang-bpb}. At 64K and 128K, intrinsic gaps narrow as merge coverage improves; held-out BPB and downstream benchmark differences stay small and mixed across the 64K, 128K, and 1.3B/32K scale checks, consistent with prior evidence that tokenizer compression or token-count changes need not directly predict downstream LM quality~\citep{schmidt2024compression}; the corresponding LM results are collected in Appendix~\ref{app:lm-eval-full}, and full per-language XNLI, XCOPA, and BPB tables (\texttt{101-lang/32K}) appear in Appendix~\ref{app:perlang-benchmarks}.

\paragraph{Source-byte-matched LM comparison.}
With the primary \texttt{101-lang/32K} UBE/BBPE tokenizers, we compare UBE, BBPE, and BBPE16 at matched supervised UTF-8 byte exposure and on identical held-out text. UBE has the lowest common-text BPB and BPB Gini in this three-way comparison; task accuracy is mixed (Appendix Table~\ref{tab:cr-source-matched-accuracy}).

\paragraph{Context-window and prompt-processing effects.}
Token-count reductions translate into fixed-budget capacity and prefill effects: defining usable-context gain under a fixed token budget as $m_{\mathrm{BBPE}}(\ell)/m_{\mathrm{UBE}}(\ell)-1$, Santali (Ol~Chiki) gains \cwThreeTwokSantaliGainPct{} more usable context at \texttt{101-lang/32K}, with English also gaining \cwThreeTwokEnglishGainPct{}; in a separate \texttt{6-lang/32K} same-content prompt-processing benchmark, end-to-end sequence-processing rate improves by \infDeltaSeqamhEthi{} for Amharic and \infDeltaSeqengLatn{} for English (Appendix Table~\ref{tab:inference-benchmark}). We report end-to-end rather than isolating self-attention because tokenizer-induced length changes propagate through multilayer perceptron (MLP) and memory-bandwidth costs as well.

\subsection{Composition and tokenizer comparisons}
\label{sec:alternative-tokenizer-comparisons}

UBE acts as a reusable representation layer: Table~\ref{tab:cr-composition} brings together matched substitutions under standard, SuperBPE-style, and SCRIPT-style boundaries and MYTE's morphology maps. In each comparison, UBE lowers the English, Other BMP, and overall mean sequence lengths while preserving the complementary component; with the primary corpus, CJK means stay within 0.3\% of BBPE under the three BPE boundary policies (Appendix~\ref{app:corpus-regime} explains why).

\begin{table}[ht]
\caption{Composition with UBE on FLORES-200 for 101-language/32K BPE vocabularies trained on the English-heavy primary corpus or on the 50 MiB/language prefix. Values are mean BPE tokens per sentence; in the MYTE row, the UBE and BBPE columns report MYTE-UBE and MYTE base symbols. CJK (5) and Other BMP (7) are the configuration groups defined in Appendix~\ref{app:gini-vs-tokens}. Gini uses 203 non-English premiums.}
\label{tab:cr-composition}
\centering
\small
\setlength{\tabcolsep}{2.5pt}
\begin{tabular}{@{}lrrrrrrrrrr@{}}
\toprule
 & \multicolumn{2}{c}{English} & \multicolumn{2}{c}{CJK (5)} & \multicolumn{2}{c}{Other BMP (7)} & \multicolumn{2}{c}{Overall (204)} & \multicolumn{2}{c}{Gini} \\
\cmidrule(lr){2-3}\cmidrule(lr){4-5}\cmidrule(lr){6-7}\cmidrule(lr){8-9}\cmidrule(lr){10-11}
Fixed component & UBE & BBPE & UBE & BBPE & UBE & BBPE & UBE & BBPE & UBE & BBPE \\
\midrule
\multicolumn{11}{@{}l}{\textit{Primary corpus}} \\
Standard & \textbf{31.068} & 31.160 & \textbf{42.418} & 42.444 & \textbf{219.709} & 275.681 & \textbf{69.857} & 71.889 & \textbf{0.1979} & 0.2183 \\
SuperBPE-style & \textbf{32.973} & 33.143 & \textbf{43.496} & 43.595 & \textbf{192.162} & 245.653 & \textbf{72.137} & 74.187 & \textbf{0.1725} & 0.1922 \\
SCRIPT-style & \textbf{31.652} & 31.698 & 42.104 & \textbf{42.054} & \textbf{228.269} & 286.663 & \textbf{68.001} & 69.957 & \textbf{0.1824} & 0.2036 \\
\multicolumn{11}{@{}l}{\textit{50 MiB/language prefix}} \\
Standard & \textbf{36.938} & 36.959 & \textbf{57.385} & 58.956 & \textbf{222.054} & 272.617 & \textbf{66.079} & 67.875 & \textbf{0.2086} & 0.2281 \\
SuperBPE-style & \textbf{39.524} & 39.586 & \textbf{58.576} & 60.871 & \textbf{189.674} & 236.155 & \textbf{67.208} & 68.858 & \textbf{0.1736} & 0.1916 \\
SCRIPT-style & \textbf{37.539} & 37.569 & \textbf{56.741} & 58.133 & \textbf{223.497} & 281.154 & \textbf{63.918} & 65.677 & \textbf{0.1832} & 0.2048 \\
\midrule
MYTE maps & \textbf{109.094} & 109.140 & \textbf{94.610} & 124.612 & \textbf{247.622} & 349.398 & \textbf{141.354} & 148.142 & --- & --- \\
\bottomrule
\end{tabular}
\end{table}

\paragraph{Composition with boundary policies.}
SuperBPE-style boundary relaxation~\citep{liu2025superbpe} and SCRIPT-style pre-tokenization~\citep{land2025script} act on a different component from UBE's byte representation. The SCRIPT-style comparison shares the official script/category split but retains UBE/BBPE symbols and standard, unconstrained BPE merges. Compared with standard pre-tokenization, SCRIPT-style achieves a lower overall mean primarily through reductions on South Asian scripts (Appendix Table~\ref{tab:cr-primary-cohorts}). Replacing BBPE with UBE also lowers Gini under all three boundary policies (Appendix~\ref{app:tokenizer-composition}). The same pattern holds with a 50 MiB/language training cap (Appendix Tables~\ref{tab:cr-scriptbpe} and~\ref{tab:cr-superbpe}).

\paragraph{Composition with morphology maps.}
Keeping MYTE's released decomposition and morphology-merge maps fixed~\citep{limisiewicz2024myte}, we route only residual characters through UBE. MYTE-UBE preserves MYTE's decoded output while reducing total base-symbol positions by 4.58\% (201/204 language--script configurations lower, three tied; Appendix Table~\ref{tab:cr-myte}). These are pre-BPE base symbols, not learned BPE tokens.

\paragraph{Comparison with SCRIPT-BPE.}
In the separate five-tokenizer comparison in Appendix Table~\ref{tab:cr-scriptbpe} (50 MiB/language), official SCRIPT-BPE attains the lowest overall mean and token-premium Gini. UBE uses fewer English tokens and fewer tokens in 169/204 language--script configurations (Appendix Table~\ref{tab:cr-scriptbpe-cohorts}). SCRIPT-BPE's gains concentrate in specific script cohorts; this pattern persists across vocabulary sizes, larger training prefixes, and XNLI/XCOPA/MGSM inputs (Appendix~\ref{app:comparative-analyses}).

\paragraph{Sensitivity to the tokenizer-training corpus.}
The two corpus blocks of Table~\ref{tab:cr-composition} compare tokenizers trained on the English-heavy primary corpus and on the 50 MiB/language prefix. Across the six composition variants, the capped tokenizers use \regimeOverallReductionRange{}\% fewer tokens overall, but \regimeEnglishIncreaseRange{}\% more for English and \regimeCjkIncreaseRange{}\% more for CJK, while the Other BMP mean changes by at most \regimeOtherBmpMaxChange{}\%. The cap lowers the training share of English from \regimeEnglishSharePrimary{}\% to \regimeEnglishShareCapped{}\%, that of Chinese, Japanese, and Korean from \regimeCjkSharePrimary{}\% to \regimeCjkShareCapped{}\%, and that of Hindi and Thai from \regimeHindiThaiSharePrimary{}\% to \regimeHindiThaiShareCapped{}\%, so their text receives fewer of the frequency-ordered BPE merges~\citep{conneau2019xlm,costajussa2022nllb,hayase2024datamixture}. The remaining 95 training languages rise from \regimeRemainingSharePrimary{}\% to \regimeRemainingShareCapped{}\%, which shortens the many non-English Latin, Arabic, and Cyrillic configurations that dominate the equally weighted FLORES-200 average, as flatter vocabulary-sampling mixes do for low-resource languages in multilingual translation~\citep{arivazhagan2019massively}; most Other BMP scripts are not used by any training language. Corpus mixture and total training volume change together. Within the capped condition, replacing BBPE with UBE reduces CJK counts by \regimeCjkUbeReductionRange{}\% and keeps English counts lower under every boundary policy. UBE also lowers Gini in both corpus regimes. Appendix~\ref{app:corpus-regime} reports all script groups and paired intervals; the five-tokenizer LM comparison appears in Appendix Table~\ref{tab:cr-matched-exposure-lm}.

\FloatBarrier

\section{Conclusion}

UBE replaces BBPE's single \utfEightName{} substrate with exactly reversible per-character \utfEightName{}/\utfSixteenName{} routing while keeping standard BPE intact; the routing itself requires neither language identification nor a Unicode script table. It reduces token-budget disparity across the main intrinsic evaluations, especially at 16K--64K, and composes with SuperBPE-style and SCRIPT-style boundary policies and MYTE's morphology maps while preserving each complementary mechanism. Downstream differences remain small and mixed; in the source-byte-matched BBPE/BBPE16 comparison UBE has the lowest BPB and BPB Gini, while same-content prompt processing improves by up to \infDeltaSeqamhEthi{}. UBE is most useful for compact vocabularies and low-coverage 3-byte BMP scripts near the encoding floor.

\paragraph{Limitations.}
UBE addresses only the substrate-level source of token-budget disparity; per-language vocabulary sizing, data balance, pre-tokenization, and merge-rule interventions remain orthogonal. Its \bbpeAlphabetSize{} extra PUA symbols reduce wins at very narrow coverage (\texttt{En-only/16K}: \eovOneSixkBetter{}/204) and forgo BBPE's coincidental byte sharing across scripts, causing small regressions in a few high-coverage languages. LM evaluation covers 182M/32K, 230M/64K, 329M/128K, and 1.3B/32K settings, with \lmMultiseedNSeeds{} seeds only at \texttt{101-lang/32K}. The largest intrinsic token-count gains occur for some low-coverage 3-byte scripts outside the 101-language mC4 LM corpus, so downstream gains for those exact languages remain untested. Because UBE changes the tokenizer vocabulary, it is not a drop-in replacement for existing checkpoints. A model-level comparison includes SCRIPT-BPE and derived 32K MYTE+BPE; the released 256-symbol MYTE itself and Bit-level BPE remain future work. Constrained-decoding throughput is also left for future work (Appendix~\ref{app:additional-limitations}, including broader impacts).

\clearpage

\begin{ack}
Hyunsik Kim thanks his wife, Jiyoung Kim, and his daughter, Nahyun Kim, for their unwavering support and understanding throughout this work. He is grateful to Youngmoon Jung for meticulously conducting and reviewing this research. He is also grateful to Haeri Kim and Sichen Jin of Carnegie Mellon University for inspiring the idea behind UBE. Youngmoon Jung is grateful to Hyunsik Kim for the great research opportunity and dedicated collaboration.

This research received no funding. The authors declare no competing interests beyond the affiliation shown above.
\end{ack}

\bibliographystyle{plainnat}
\bibliography{references}

\begin{thebibliography}{86}
\providecommand{\natexlab}[1]{#1}
\providecommand{\url}[1]{\texttt{#1}}
\expandafter\ifx\csname urlstyle\endcsname\relax
  \providecommand{\doi}[1]{doi: #1}\else
  \providecommand{\doi}{doi: \begingroup \urlstyle{rm}\Url}\fi

\bibitem[Abagyan et~al.(2025)Abagyan, Salamanca, Cruz-Salinas, Cao, Lin,
  Locatelli, Fadaee, {\"U}st{\"u}n, and Hooker]{abagyan2025onetokenizer}
Diana Abagyan, Alejandro~R. Salamanca, Andres~Felipe Cruz-Salinas, Kris Cao,
  Hangyu Lin, Acyr Locatelli, Marzieh Fadaee, Ahmet {\"U}st{\"u}n, and Sara
  Hooker.
\newblock One tokenizer to rule them all: Emergent language plasticity via
  multilingual tokenizers.
\newblock \emph{arXiv preprint arXiv:2506.10766}, 2025.
\newblock \doi{10.48550/arXiv.2506.10766}.

\bibitem[Abouelenin et~al.(2025)Abouelenin, Ashfaq, Atkinson, Awadalla, Bach,
  Bao, Benhaim, Cai, Chaudhary, Chen, Chen, Chen, Chen, Chen, Chen, Chen, Dai,
  Dai, Fan, Gao, et~al.]{abouelenin2025phi4mini}
Abdelrahman Abouelenin, Atabak Ashfaq, Adam Atkinson, Hany Awadalla, Nguyen
  Bach, Jianmin Bao, Alon Benhaim, Martin Cai, Vishrav Chaudhary, Congcong
  Chen, Dong Chen, Dongdong Chen, Junkun Chen, Weizhu Chen, Yen-Chun Chen,
  Yi-ling Chen, Qi~Dai, Xiyang Dai, Ruchao Fan, Mei Gao, et~al.
\newblock {Phi-4-Mini} technical report: Compact yet powerful multimodal
  language models via mixture-of-{LoRAs}.
\newblock \emph{arXiv preprint arXiv:2503.01743}, 2025.
\newblock \doi{10.48550/arXiv.2503.01743}.

\bibitem[Adelani et~al.(2024)Adelani, Liu, Shen, Vassilyev, Alabi, Mao, Gao,
  and Lee]{adelani2024sib200}
David~Ifeoluwa Adelani, Hannah Liu, Xiaoyu Shen, Nikita Vassilyev, Jesujoba~O.
  Alabi, Yanke Mao, Haonan Gao, and En-Shiun~Annie Lee.
\newblock {SIB}-200: A simple, inclusive, and big evaluation dataset for topic
  classification in 200+ languages and dialects.
\newblock In \emph{Proceedings of the 18th Conference of the European Chapter
  of the Association for Computational Linguistics (Volume 1: Long Papers)},
  pages 226--245, St. Julian{'}s, Malta, March 2024. Association for
  Computational Linguistics.
\newblock \doi{10.18653/v1/2024.eacl-long.14}.

\bibitem[Ahia et~al.(2023)Ahia, Kumar, Gonen, Kasai, Mortensen, Smith, and
  Tsvetkov]{ahia2023languages}
Orevaoghene Ahia, Sachin Kumar, Hila Gonen, Jungo Kasai, David Mortensen, Noah
  Smith, and Yulia Tsvetkov.
\newblock Do all languages cost the same? {Tokenization} in the era of
  commercial language models.
\newblock In \emph{Proceedings of the 2023 Conference on Empirical Methods in
  Natural Language Processing}, pages 9904--9923, Singapore, December 2023.
  Association for Computational Linguistics.
\newblock \doi{10.18653/v1/2023.emnlp-main.614}.

\bibitem[Ahia et~al.(2024)Ahia, Kumar, Gonen, Hofmann, Limisiewicz, Tsvetkov,
  and Smith]{ahia2024magnet}
Orevaoghene Ahia, Sachin Kumar, Hila Gonen, Valentin Hofmann, Tomasz
  Limisiewicz, Yulia Tsvetkov, and Noah Smith.
\newblock {MAGNET}: Improving the multilingual fairness of language models with
  adaptive gradient-based tokenization.
\newblock In \emph{Advances in Neural Information Processing Systems},
  volume~37, pages 47790--47814. Curran Associates, Inc., 2024.
\newblock \doi{10.52202/079017-1514}.

\bibitem[Allal et~al.(2025)Allal, Lozhkov, Bakouch, Blazquez, Penedo, Tunstall,
  Marafioti, Lajar{\'\i}n, Kydl{\'\i}{\v c}ek, Srivastav, Lochner, Fahlgren,
  Nguyen, Burtenshaw, Fourrier, Zhao, Larcher, Morlon, Zakka, Raffel,
  Von~Werra, and Wolf]{allal2025smollm2}
Loubna~Ben Allal, Anton Lozhkov, Elie Bakouch, Gabriel~Martin Blazquez,
  Guilherme Penedo, Lewis Tunstall, Andr{\'e}s Marafioti, Agust{\'\i}n~Piqueres
  Lajar{\'\i}n, Hynek Kydl{\'\i}{\v c}ek, Vaibhav Srivastav, Joshua Lochner,
  Caleb Fahlgren, Xuan~Son Nguyen, Ben Burtenshaw, Cl{\'e}mentine Fourrier,
  Haojun Zhao, Hugo Larcher, Mathieu Morlon, Cyril Zakka, Colin Raffel, Leandro
  Von~Werra, and Thomas Wolf.
\newblock {SmolLM2}: When smol goes big {\textemdash} data-centric training of
  a fully open small language model.
\newblock In \emph{Second Conference on Language Modeling}, 2025.

\bibitem[Alqahtani et~al.(2026)Alqahtani, Nayeem, Laskar, Mohiuddin, and
  Bari]{nayeem2026tokenizers}
Sawsan Alqahtani, Mir~Tafseer Nayeem, Md~Tahmid~Rahman Laskar, Tasnim
  Mohiuddin, and M~Saiful Bari.
\newblock Stop taking tokenizers for granted: They are core design decisions in
  large language models.
\newblock In \emph{Proceedings of the 19th Conference of the European Chapter
  of the Association for Computational Linguistics (Volume 1: Long Papers)},
  pages 8410--8432, Rabat, Morocco, March 2026. Association for Computational
  Linguistics.
\newblock \doi{10.18653/v1/2026.eacl-long.394}.

\bibitem[Arivazhagan et~al.(2019)Arivazhagan, Bapna, Firat, Lepikhin, Johnson,
  Krikun, Chen, Cao, Foster, Cherry, Macherey, Chen, and
  Wu]{arivazhagan2019massively}
Naveen Arivazhagan, Ankur Bapna, Orhan Firat, Dmitry Lepikhin, Melvin Johnson,
  Maxim Krikun, Mia~Xu Chen, Yuan Cao, George Foster, Colin Cherry, Wolfgang
  Macherey, Zhifeng Chen, and Yonghui Wu.
\newblock Massively multilingual neural machine translation in the wild:
  Findings and challenges.
\newblock \emph{arXiv preprint arXiv:1907.05019}, 2019.
\newblock \doi{10.48550/arXiv.1907.05019}.

\bibitem[Arnett et~al.(2024)Arnett, Chang, and Bergen]{arnett2024bitproblem}
Catherine Arnett, Tyler~A. Chang, and Benjamin Bergen.
\newblock A bit of a problem: Measurement disparities in dataset sizes across
  languages.
\newblock In \emph{Proceedings of the 3rd Annual Meeting of the Special
  Interest Group on Under-resourced Languages @ LREC-COLING 2024}, pages 1--9,
  Torino, Italia, May 2024. ELRA and ICCL.

\bibitem[Arnett et~al.(2025)Arnett, Chang, Biderman, and
  Bergen]{arnett2025explaining}
Catherine Arnett, Tyler Chang, Stella Biderman, and Benjamin Bergen.
\newblock Explaining and mitigating crosslingual tokenizer inequities.
\newblock In \emph{Advances in Neural Information Processing Systems}, volume
  38, Main Conference, pages 63277--63313. Curran Associates, Inc., 2025.
\newblock \doi{10.52202/085713-2126}.

\bibitem[Artola~Velasco et~al.(2025)Artola~Velasco, Tsirtsis, Okati, and
  Gomez~Rodriguez]{artola2025overcharging}
Ander Artola~Velasco, Stratis Tsirtsis, Nastaran Okati, and Manuel
  Gomez~Rodriguez.
\newblock Is your {LLM} overcharging you? {Tokenization}, transparency, and
  incentives.
\newblock In \emph{EurIPS 2025 Workshop on Private {AI} Governance}, 2025.

\bibitem[Bakouch et~al.(2025)Bakouch, Ben~Allal, Lozhkov, Tazi, Tunstall,
  Pati{\~n}o, Beeching, Roucher, Reedi, Gallou{\'e}dec, Rasul, Habib, Fourrier,
  Kydlicek, Penedo, Larcher, Morlon, Srivastav, Lochner, Nguyen, Raffel, von
  Werra, and Wolf]{bakouch2025smollm3}
Elie Bakouch, Loubna Ben~Allal, Anton Lozhkov, Nouamane Tazi, Lewis Tunstall,
  Carlos~Miguel Pati{\~n}o, Edward Beeching, Aymeric Roucher, Aksel~Joonas
  Reedi, Quentin Gallou{\'e}dec, Kashif Rasul, Nathan Habib, Cl{\'e}mentine
  Fourrier, Hynek Kydlicek, Guilherme Penedo, Hugo Larcher, Mathieu Morlon,
  Vaibhav Srivastav, Joshua Lochner, Xuan-Son Nguyen, Colin Raffel, Leandro von
  Werra, and Thomas Wolf.
\newblock {SmolLM3}: smol, multilingual, long-context reasoner.
\newblock \url{https://huggingface.co/blog/smollm3}, 2025.

\bibitem[Bandarkar et~al.(2024)Bandarkar, Liang, Muller, Artetxe, Shukla, Husa,
  Goyal, Krishnan, Zettlemoyer, and Khabsa]{bandarkar2024belebele}
Lucas Bandarkar, Davis Liang, Benjamin Muller, Mikel Artetxe, Satya~Narayan
  Shukla, Donald Husa, Naman Goyal, Abhinandan Krishnan, Luke Zettlemoyer, and
  Madian Khabsa.
\newblock The {Belebele} benchmark: a parallel reading comprehension dataset in
  122 language variants.
\newblock In \emph{Proceedings of the 62nd Annual Meeting of the Association
  for Computational Linguistics (Volume 1: Long Papers)}, pages 749--775,
  Bangkok, Thailand, August 2024. Association for Computational Linguistics.
\newblock \doi{10.18653/v1/2024.acl-long.44}.

\bibitem[Bellagente et~al.(2024)Bellagente, Tow, Mahan, Phung, Zhuravinskyi,
  Adithyan, Baicoianu, Brooks, Cooper, Datta, Lee, Mostaque, Pieler,
  Pinnaparju, Rocha, Saini, Teufel, Zanichelli, and
  Riquelme]{bellagente2024stablelm2}
Marco Bellagente, Jonathan Tow, Dakota Mahan, Duy Phung, Maksym Zhuravinskyi,
  Reshinth Adithyan, James Baicoianu, Ben Brooks, Nathan Cooper, Ashish Datta,
  Meng Lee, Emad Mostaque, Michael Pieler, Nikhil Pinnaparju, Paulo Rocha,
  Harry Saini, Hannah Teufel, Niccolo Zanichelli, and Carlos Riquelme.
\newblock Stable {LM} 2 1.6{B} technical report.
\newblock \emph{arXiv preprint arXiv:2402.17834}, 2024.
\newblock \doi{10.48550/arXiv.2402.17834}.

\bibitem[Biderman et~al.(2023)Biderman, Schoelkopf, Anthony, Bradley, O'Brien,
  Hallahan, Khan, Purohit, Prashanth, Raff, Skowron, Sutawika, and Van
  Der~Wal]{biderman2023pythia}
Stella Biderman, Hailey Schoelkopf, Quentin~Gregory Anthony, Herbie Bradley,
  Kyle O'Brien, Eric Hallahan, Mohammad~Aflah Khan, Shivanshu Purohit,
  Usvsn~Sai Prashanth, Edward Raff, Aviya Skowron, Lintang Sutawika, and Oskar
  Van Der~Wal.
\newblock Pythia: A suite for analyzing large language models across training
  and scaling.
\newblock In \emph{Proceedings of the 40th International Conference on Machine
  Learning}, volume 202 of \emph{Proceedings of Machine Learning Research},
  pages 2397--2430. PMLR, 23--29 Jul 2023.

\bibitem[{BigScience Workshop} et~al.(2023){BigScience Workshop}, Le~Scao, Fan,
  Akiki, Pavlick, Ili{\'c}, Hesslow, Castagn{\'e}, Luccioni, Yvon, Gall{\'e},
  Tow, Rush, Biderman, Webson, et~al.]{bigscience2022bloom}
{BigScience Workshop}, Teven Le~Scao, Angela Fan, Christopher Akiki, Ellie
  Pavlick, Suzana Ili{\'c}, Daniel Hesslow, Roman Castagn{\'e}, Alexandra~Sasha
  Luccioni, Fran{\c{c}}ois Yvon, Matthias Gall{\'e}, Jonathan Tow, Alexander~M.
  Rush, Stella Biderman, Albert Webson, et~al.
\newblock {BLOOM}: A {176B}-parameter open-access multilingual language model.
\newblock \emph{arXiv preprint arXiv:2211.05100}, 2023.
\newblock \doi{10.48550/arXiv.2211.05100}.

\bibitem[Bisk et~al.(2020)Bisk, Zellers, Le~Bras, Gao, and Choi]{bisk2020piqa}
Yonatan Bisk, Rowan Zellers, Ronan Le~Bras, Jianfeng Gao, and Yejin Choi.
\newblock {PIQA}: Reasoning about physical commonsense in natural language.
\newblock \emph{Proceedings of the AAAI Conference on Artificial Intelligence},
  34\penalty0 (05):\penalty0 7432--7439, 2020.
\newblock \doi{10.1609/aaai.v34i05.6239}.

\bibitem[Brahma et~al.(2025)Brahma, Karthika, Singh, Adiga, Bhate,
  Ramakrishnan, Saluja, and Desarkar]{brahma2025morphtok}
Maharaj Brahma, N~J Karthika, Atul~Kumar Singh, Devaraja Adiga, Smruti Bhate,
  Ganesh Ramakrishnan, Rohit Saluja, and Maunendra~Sankar Desarkar.
\newblock {MorphTok}: Morphologically grounded tokenization for {Indic}
  languages.
\newblock In \emph{Tokenization Workshop}, 2025.

\bibitem[Churchill and Skiena(2026)]{churchill2026reducing}
Geoffrey Churchill and Steven Skiena.
\newblock Reducing tokenization premiums for low-resource languages.
\newblock \emph{arXiv preprint arXiv:2601.13328}, 2026.
\newblock \doi{10.48550/arXiv.2601.13328}.

\bibitem[Clark et~al.(2022)Clark, Garrette, Turc, and Wieting]{clark2022canine}
Jonathan~H. Clark, Dan Garrette, Iulia Turc, and John Wieting.
\newblock {CANINE}: Pre-training an efficient tokenization-free encoder for
  language representation.
\newblock \emph{Transactions of the Association for Computational Linguistics},
  10:\penalty0 73--91, 2022.
\newblock \doi{10.1162/tacl_a_00448}.

\bibitem[Clark et~al.(2018)Clark, Cowhey, Etzioni, Khot, Sabharwal, Schoenick,
  and Tafjord]{clark2018arc}
Peter Clark, Isaac Cowhey, Oren Etzioni, Tushar Khot, Ashish Sabharwal, Carissa
  Schoenick, and Oyvind Tafjord.
\newblock Think you have solved question answering? {Try} {ARC}, the {AI2}
  reasoning challenge.
\newblock \emph{arXiv preprint arXiv:1803.05457}, 2018.
\newblock \doi{10.48550/arXiv.1803.05457}.

\bibitem[Conneau and Lample(2019)]{conneau2019xlm}
Alexis Conneau and Guillaume Lample.
\newblock Cross-lingual language model pretraining.
\newblock In \emph{Advances in Neural Information Processing Systems},
  volume~32, pages 7059--7069. Curran Associates, Inc., 2019.

\bibitem[Conneau et~al.(2018)Conneau, Rinott, Lample, Williams, Bowman,
  Schwenk, and Stoyanov]{conneau2018xnli}
Alexis Conneau, Ruty Rinott, Guillaume Lample, Adina Williams, Samuel~R.
  Bowman, Holger Schwenk, and Veselin Stoyanov.
\newblock {XNLI}: Evaluating cross-lingual sentence representations.
\newblock In \emph{Proceedings of the 2018 Conference on Empirical Methods in
  Natural Language Processing}, pages 2475--2485, Brussels, Belgium,
  October-November 2018. Association for Computational Linguistics.
\newblock \doi{10.18653/v1/D18-1269}.

\bibitem[Dao et~al.(2022)Dao, Fu, Ermon, Rudra, and
  R{\'e}]{dao2022flashattention}
Tri Dao, Dan Fu, Stefano Ermon, Atri Rudra, and Christopher R{\'e}.
\newblock {FlashAttention}: Fast and memory-efficient exact attention with
  {IO}-awareness.
\newblock In \emph{Advances in Neural Information Processing Systems},
  volume~35, pages 16344--16359. Curran Associates, Inc., 2022.

\bibitem[Darshana(2026)]{wwho2026sgpe}
Kusal Darshana.
\newblock Separate before you compress: The {WWHO} tokenization architecture.
\newblock \emph{arXiv preprint arXiv:2603.25309}, 2026.
\newblock \doi{10.48550/arXiv.2603.25309}.

\bibitem[Davis and Holbrook(2025)]{unicode2025uts51}
Mark Davis and Ned Holbrook.
\newblock Unicode technical standard \#51: {Unicode} emoji, version 17.0.
\newblock \url{https://www.unicode.org/reports/tr51/tr51-29.html}, 2025.

\bibitem[Dixit and Dixit(2026)]{dixit2026scripttax}
Aradhya Dixit and Shreem Dixit.
\newblock The script tax: Measuring tokenization-driven efficiency and latency
  disparities in multilingual language models.
\newblock \emph{arXiv preprint arXiv:2602.11174}, 2026.
\newblock \doi{10.48550/arXiv.2602.11174}.

\bibitem[Eberhard et~al.(2024)Eberhard, Simons, and Fennig]{ethnologue2024}
David~M. Eberhard, Gary~F. Simons, and Charles~D. Fennig, editors.
\newblock \emph{Ethnologue: Languages of the World}.
\newblock SIL International, Dallas, Texas, 27th edition, 2024.

\bibitem[Georges et~al.(2007)Georges, Buytaert, and
  Eeckhout]{georges2007statistically}
Andy Georges, Dries Buytaert, and Lieven Eeckhout.
\newblock Statistically rigorous {Java} performance evaluation.
\newblock In \emph{Proceedings of the 22nd Annual ACM SIGPLAN Conference on
  Object-Oriented Programming, Systems, Languages, and Applications (OOPSLA)},
  pages 57--76. ACM, October 2007.
\newblock \doi{10.1145/1297027.1297033}.

\bibitem[Goldman et~al.(2024)Goldman, Caciularu, Eyal, Cao, Szpektor, and
  Tsarfaty]{goldman2024unpacking}
Omer Goldman, Avi Caciularu, Matan Eyal, Kris Cao, Idan Szpektor, and Reut
  Tsarfaty.
\newblock Unpacking tokenization: Evaluating text compression and its
  correlation with model performance.
\newblock In \emph{Findings of the Association for Computational Linguistics:
  ACL 2024}, pages 2274--2286, Bangkok, Thailand, August 2024. Association for
  Computational Linguistics.
\newblock \doi{10.18653/v1/2024.findings-acl.134}.

\bibitem[Grattafiori et~al.(2024)Grattafiori, Dubey, Jauhri,
  et~al.]{grattafiori2024llama3}
Aaron Grattafiori, Abhimanyu Dubey, Abhinav Jauhri, et~al.
\newblock The {Llama} 3 herd of models.
\newblock \emph{arXiv preprint arXiv:2407.21783}, 2024.
\newblock \doi{10.48550/arXiv.2407.21783}.

\bibitem[Hayase et~al.(2024)Hayase, Liu, Choi, Oh, and
  Smith]{hayase2024datamixture}
Jonathan Hayase, Alisa Liu, Yejin Choi, Sewoong Oh, and Noah Smith.
\newblock Data mixture inference attack: {BPE} tokenizers reveal training data
  compositions.
\newblock In \emph{Advances in Neural Information Processing Systems},
  volume~37, pages 8956--8983. Curran Associates, Inc., 2024.
\newblock \doi{10.52202/079017-0285}.

\bibitem[Hoefler and Belli(2015)]{hoefler2015scientific}
Torsten Hoefler and Roberto Belli.
\newblock Scientific benchmarking of parallel computing systems: Twelve ways to
  tell the masses when reporting performance results.
\newblock In \emph{Proceedings of the International Conference for High
  Performance Computing, Networking, Storage and Analysis (SC)}, pages
  73:1--73:12, Austin, TX, USA, November 2015. ACM.
\newblock ISBN 978-1-4503-3723-6.
\newblock \doi{10.1145/2807591.2807644}.

\bibitem[Hoffmann et~al.(2022)Hoffmann, Borgeaud, Mensch, Buchatskaya, Cai,
  Rutherford, de~Las~Casas, Hendricks, Welbl, Clark, Hennigan, Noland,
  Millican, van~den Driessche, Damoc, Guy, Osindero, Simonyan, Elsen, Vinyals,
  Rae, and Sifre]{hoffmann2022chinchilla}
Jordan Hoffmann, Sebastian Borgeaud, Arthur Mensch, Elena Buchatskaya, Trevor
  Cai, Eliza Rutherford, Diego de~Las~Casas, Lisa~Anne Hendricks, Johannes
  Welbl, Aidan Clark, Thomas Hennigan, Eric Noland, Katherine Millican, George
  van~den Driessche, Bogdan Damoc, Aurelia Guy, Simon Osindero, Kar{\'e}n
  Simonyan, Erich Elsen, Oriol Vinyals, Jack Rae, and Laurent Sifre.
\newblock An empirical analysis of compute-optimal large language model
  training.
\newblock In \emph{Advances in Neural Information Processing Systems},
  volume~35, pages 30016--30030. Curran Associates, Inc., 2022.

\bibitem[Hofman et~al.(2026)Hofman, Brokman, Rachmil, Bose, Pahuja, Shimizu,
  Starostina, Marchisio, Goldfarb-Tarrant, and Vainshtein]{hofman2026maps}
Omer Hofman, Jonathan Brokman, Oren Rachmil, Shamik Bose, Vikas Pahuja, Toshiya
  Shimizu, Trisha Starostina, Kelly Marchisio, Seraphina Goldfarb-Tarrant, and
  Roman Vainshtein.
\newblock {MAPS}: A multilingual benchmark for agent performance and security.
\newblock In \emph{Findings of the Association for Computational Linguistics:
  EACL 2026}, pages 821--845. Association for Computational Linguistics, 2026.

\bibitem[Hu et~al.(2025)Hu, Vuli{\'c}, and Korhonen]{hu2025quantifying}
Songbo Hu, Ivan Vuli{\'c}, and Anna Korhonen.
\newblock Quantifying language disparities in multilingual large language
  models.
\newblock In \emph{Proceedings of the 2025 Conference on Empirical Methods in
  Natural Language Processing}, pages 4003--4018, Suzhou, China, November 2025.
  Association for Computational Linguistics.
\newblock \doi{10.18653/v1/2025.emnlp-main.199}.

\bibitem[{HuggingFaceTB}(2025)]{huggingfacetb2025smollm3card}
{HuggingFaceTB}.
\newblock {SmolLM3-3B} model card and configuration.
\newblock \url{https://huggingface.co/HuggingFaceTB/SmolLM3-3B}, 2025.

\bibitem[Hwang et~al.(2026)Hwang, Wang, and Gu]{hwang2025hnet}
Sukjun Hwang, Brandon Wang, and Albert Gu.
\newblock Dynamic chunking for end-to-end hierarchical sequence modeling.
\newblock In \emph{The Fourteenth International Conference on Learning
  Representations}, 2026.

\bibitem[Jimenez et~al.(2024)Jimenez, Yang, Wettig, Yao, Pei, Press, and
  Narasimhan]{jimenez2024swebench}
Carlos~E. Jimenez, John Yang, Alexander Wettig, Shunyu Yao, Kexin Pei, Ofir
  Press, and Karthik Narasimhan.
\newblock {SWE-bench}: Can language models resolve real-world {GitHub} issues?
\newblock In \emph{International Conference on Learning Representations}, 2024.

\bibitem[Kallini et~al.(2025)Kallini, Murty, Manning, Potts, and
  Csord{\'a}s]{kallini2025mrt5}
Julie Kallini, Shikhar Murty, Christopher~D. Manning, Christopher Potts, and
  R{\'o}bert Csord{\'a}s.
\newblock {MrT5}: Dynamic token merging for efficient byte-level language
  models.
\newblock In \emph{The Thirteenth International Conference on Learning
  Representations}, 2025.

\bibitem[Kanjirangat et~al.(2025)Kanjirangat, Samardzic, Dolamic, and
  Rinaldi]{kanjirangat2025tokenization}
Vani Kanjirangat, Tanja Samardzic, Ljiljana Dolamic, and Fabio Rinaldi.
\newblock Tokenization and representation biases in multilingual models on
  dialectal {NLP} tasks.
\newblock In \emph{Proceedings of the 2025 Conference on Empirical Methods in
  Natural Language Processing}, pages 23992--24010, Suzhou, China, November
  2025. Association for Computational Linguistics.
\newblock \doi{10.18653/v1/2025.emnlp-main.1224}.

\bibitem[Kim et~al.(2026)Kim, Kim, Lee, and Lee]{kim2026bbpe16}
Hyunsik Kim, Haeri Kim, Munhak Lee, and Kyungmin Lee.
\newblock {BBPE16}: {UTF}-16-based byte-level byte-pair encoding for improved
  multilingual speech recognition.
\newblock In \emph{ICASSP 2026 -- 2026 IEEE International Conference on
  Acoustics, Speech and Signal Processing (ICASSP)}, pages 17067--17071. IEEE,
  2026.
\newblock \doi{10.1109/ICASSP55912.2026.11460641}.

\bibitem[Kudo and Richardson(2018)]{kudo2018sentencepiece}
Taku Kudo and John Richardson.
\newblock {S}entence{P}iece: A simple and language independent subword
  tokenizer and detokenizer for neural text processing.
\newblock In \emph{Proceedings of the 2018 Conference on Empirical Methods in
  Natural Language Processing: System Demonstrations}, pages 66--71, Brussels,
  Belgium, November 2018. Association for Computational Linguistics.
\newblock \doi{10.18653/v1/D18-2012}.

\bibitem[Land and Arnett(2025)]{land2025script}
Sander Land and Catherine Arnett.
\newblock {BPE} stays on {SCRIPT}: Structured encoding for robust multilingual
  pretokenization.
\newblock In \emph{Tokenization Workshop}, 2025.

\bibitem[Limisiewicz et~al.(2024)Limisiewicz, Blevins, Gonen, Ahia, and
  Zettlemoyer]{limisiewicz2024myte}
Tomasz Limisiewicz, Terra Blevins, Hila Gonen, Orevaoghene Ahia, and Luke
  Zettlemoyer.
\newblock {MYTE}: Morphology-driven byte encoding for better and fairer
  multilingual language modeling.
\newblock In \emph{Proceedings of the 62nd Annual Meeting of the Association
  for Computational Linguistics (Volume 1: Long Papers)}, pages 15059--15076,
  Bangkok, Thailand, 2024. Association for Computational Linguistics.
\newblock \doi{10.18653/v1/2024.acl-long.804}.

\bibitem[Liu et~al.(2025)Liu, Hayase, Hofmann, Oh, Smith, and
  Choi]{liu2025superbpe}
Alisa Liu, Jonathan Hayase, Valentin Hofmann, Sewoong Oh, Noah~A. Smith, and
  Yejin Choi.
\newblock {SuperBPE}: Space travel for language models.
\newblock In \emph{Second Conference on Language Modeling}, 2025.

\bibitem[Liu et~al.(2024)Liu, Zhao, Iandola, Lai, Tian, Fedorov, Xiong, Chang,
  Shi, Krishnamoorthi, Lai, and Chandra]{liu2024mobilellm}
Zechun Liu, Changsheng Zhao, Forrest Iandola, Chen Lai, Yuandong Tian, Igor
  Fedorov, Yunyang Xiong, Ernie Chang, Yangyang Shi, Raghuraman Krishnamoorthi,
  Liangzhen Lai, and Vikas Chandra.
\newblock {M}obile{LLM}: Optimizing sub-billion parameter language models for
  on-device use cases.
\newblock In \emph{Proceedings of the 41st International Conference on Machine
  Learning}, volume 235 of \emph{Proceedings of Machine Learning Research},
  pages 32431--32454. PMLR, 21--27 Jul 2024.

\bibitem[Lundin et~al.(2026)Lundin, Zhang, Karim, Louzan, Wei, Adelani, and
  Carroll]{lundin2025tokentax}
Jessica~M. Lundin, Ada Zhang, Nihal Karim, Hamza Louzan, Guohao Wei,
  David~Ifeoluwa Adelani, and Cody Carroll.
\newblock The token tax: Systematic bias in multilingual tokenization.
\newblock In \emph{Proceedings of the 7th Workshop on African Natural Language
  Processing (AfricaNLP 2026)}, pages 103--112, Rabat, Morocco, March 2026.
  Association for Computational Linguistics.
\newblock \doi{10.18653/v1/2026.africanlp-main.10}.

\bibitem[{Meta}(2024)]{meta2024llama32card}
{Meta}.
\newblock {Llama 3.2} model card.
\newblock
  \url{https://github.com/meta-llama/llama-models/blob/main/models/llama3_2/MODEL_CARD.md},
  2024.

\bibitem[{Microsoft}(2024)]{microsoft2024phi35mini}
{Microsoft}.
\newblock {Phi-3.5-mini-instruct} model card.
\newblock \url{https://huggingface.co/microsoft/Phi-3.5-mini-instruct}, 2024.

\bibitem[Minixhofer et~al.(2026)Minixhofer, Murray, Limisiewicz, Korhonen,
  Zettlemoyer, Smith, Ponti, Soldaini, and Hofmann]{minixhofer2025bolmo}
Benjamin Minixhofer, Tyler Murray, Tomasz Limisiewicz, Anna Korhonen, Luke
  Zettlemoyer, Noah~A. Smith, Edoardo~M. Ponti, Luca Soldaini, and Valentin
  Hofmann.
\newblock {Bolmo}: Byteifying the next generation of language models.
\newblock \emph{arXiv preprint arXiv:2512.15586}, 2026.
\newblock \doi{10.48550/arXiv.2512.15586}.

\bibitem[{Mistral AI}(2024{\natexlab{a}})]{mistralai2024mistralnemo}
{Mistral AI}.
\newblock {Mistral NeMo}.
\newblock \url{https://mistral.ai/news/mistral-nemo/}, 2024{\natexlab{a}}.

\bibitem[{Mistral AI}(2024{\natexlab{b}})]{mistralai2024mistralnemocard}
{Mistral AI}.
\newblock {Mistral-Nemo-Base-2407} model card and configuration.
\newblock \url{https://huggingface.co/mistralai/Mistral-Nemo-Base-2407},
  2024{\natexlab{b}}.

\bibitem[Moon et~al.(2025)Moon, Hiraoka, and Okazaki]{moon2025bitbpe}
Sangwhan Moon, Tatsuya Hiraoka, and Naoaki Okazaki.
\newblock Bit-level {BPE}: Below the byte boundary.
\newblock \emph{arXiv preprint arXiv:2506.07541}, 2025.
\newblock \doi{10.48550/arXiv.2506.07541}.

\bibitem[Nayeem et~al.(2025)Nayeem, Alqahtani, Laskar, Mohiuddin, and
  Bari]{nayeem2025strr}
Mir~Tafseer Nayeem, Sawsan Alqahtani, Md~Tahmid~Rahman Laskar, Tasnim
  Mohiuddin, and M~Saiful Bari.
\newblock Beyond fertility: Analyzing {STRR} as a metric for multilingual
  tokenization evaluation.
\newblock In \emph{NeurIPS 2025 Workshop on Evaluating the Evolving {LLM}
  Lifecycle: Benchmarks, Emergent Abilities, and Scaling}, 2025.

\bibitem[{NLLB Team}(2024)]{costajussa2022nllb}
{NLLB Team}.
\newblock Scaling neural machine translation to 200 languages.
\newblock \emph{Nature}, 630:\penalty0 841--846, 2024.
\newblock \doi{10.1038/s41586-024-07335-x}.

\bibitem[{OpenAI}(2023)]{openai2023tiktoken}
{OpenAI}.
\newblock tiktoken.
\newblock \url{https://github.com/openai/tiktoken}, 2023.

\bibitem[Owodunni et~al.(2026)Owodunni, Ahia, and
  Kumar]{owodunni2025flexitokens}
Abraham~Toluwase Owodunni, Orevaoghene Ahia, and Sachin Kumar.
\newblock {FLEXITOKENS}: Flexible tokenization for evolving language models.
\newblock In \emph{Findings of the Association for Computational Linguistics:
  ACL 2026}, pages 17170--17190, San Diego, California, United States, July
  2026. Association for Computational Linguistics.
\newblock \doi{10.18653/v1/2026.findings-acl.848}.

\bibitem[Pagnoni et~al.(2025)Pagnoni, Pasunuru, Rodriguez, Nguyen, Muller, Li,
  Zhou, Yu, Weston, Zettlemoyer, Ghosh, Lewis, Holtzman, and
  Iyer]{pagnoni2025blt}
Artidoro Pagnoni, Ramakanth Pasunuru, Pedro Rodriguez, John Nguyen, Benjamin
  Muller, Margaret Li, Chunting Zhou, Lili Yu, Jason~E. Weston, Luke
  Zettlemoyer, Gargi Ghosh, Mike Lewis, Ari Holtzman, and Srini Iyer.
\newblock Byte latent transformer: Patches scale better than tokens.
\newblock In \emph{Proceedings of the 63rd Annual Meeting of the Association
  for Computational Linguistics (Volume 1: Long Papers)}, pages 9238--9258,
  Vienna, Austria, July 2025. Association for Computational Linguistics.
\newblock \doi{10.18653/v1/2025.acl-long.453}.

\bibitem[Paperno et~al.(2016)Paperno, Kruszewski, Lazaridou, Pham, Bernardi,
  Pezzelle, Baroni, Boleda, and Fern{\'a}ndez]{paperno2016lambada}
Denis Paperno, Germ{\'a}n Kruszewski, Angeliki Lazaridou, Ngoc~Quan Pham,
  Raffaella Bernardi, Sandro Pezzelle, Marco Baroni, Gemma Boleda, and Raquel
  Fern{\'a}ndez.
\newblock The {LAMBADA} dataset: Word prediction requiring a broad discourse
  context.
\newblock In \emph{Proceedings of the 54th Annual Meeting of the Association
  for Computational Linguistics (Volume 1: Long Papers)}, pages 1525--1534,
  Berlin, Germany, August 2016. Association for Computational Linguistics.
\newblock \doi{10.18653/v1/P16-1144}.

\bibitem[Petrov et~al.(2023)Petrov, La~Malfa, Torr, and
  Bibi]{petrov2023unfairness}
Aleksandar Petrov, Emanuele La~Malfa, Philip Torr, and Adel Bibi.
\newblock Language model tokenizers introduce unfairness between languages.
\newblock In \emph{Advances in Neural Information Processing Systems},
  volume~36, pages 36963--36990. Curran Associates, Inc., 2023.

\bibitem[Ponti et~al.(2020)Ponti, Glava{\v{s}}, Majewska, Liu, Vuli{\'c}, and
  Korhonen]{ponti2020xcopa}
Edoardo~Maria Ponti, Goran Glava{\v{s}}, Olga Majewska, Qianchu Liu, Ivan
  Vuli{\'c}, and Anna Korhonen.
\newblock {XCOPA}: A multilingual dataset for causal commonsense reasoning.
\newblock In \emph{Proceedings of the 2020 Conference on Empirical Methods in
  Natural Language Processing (EMNLP)}, pages 2362--2376, Online, November
  2020. Association for Computational Linguistics.
\newblock \doi{10.18653/v1/2020.emnlp-main.185}.

\bibitem[Pope et~al.(2023)Pope, Douglas, Chowdhery, Devlin, Bradbury, Heek,
  Xiao, Agrawal, and Dean]{pope2023efficiently}
Reiner Pope, Sholto Douglas, Aakanksha Chowdhery, Jacob Devlin, James Bradbury,
  Jonathan Heek, Kefan Xiao, Shivani Agrawal, and Jeff Dean.
\newblock Efficiently scaling transformer inference.
\newblock In \emph{Proceedings of Machine Learning and Systems}, volume~5,
  pages 606--624, 2023.

\bibitem[{Qwen Team}(2024)]{qwenteam2024qwen25card}
{Qwen Team}.
\newblock {{Qwen}2.5}: A party of foundation models.
\newblock \url{https://qwenlm.github.io/blog/qwen2.5/}, September 2024.

\bibitem[{Qwen Team}(2025)]{yang2024qwen25report}
{Qwen Team}.
\newblock {{Qwen}2.5} technical report.
\newblock \emph{arXiv preprint arXiv:2412.15115}, 2025.
\newblock \doi{10.48550/arXiv.2412.15115}.

\bibitem[{Qwen Team}(2026)]{qwenteam2026qwen36card}
{Qwen Team}.
\newblock {{Qwen}3.6-27B}: Flagship-level coding in a {27B} dense model.
\newblock \url{https://qwen.ai/blog?id=qwen3.6-27b}, April 2026.

\bibitem[Radford et~al.(2019)Radford, Wu, Child, Luan, Amodei, and
  Sutskever]{radford2019gpt2}
Alec Radford, Jeffrey Wu, Rewon Child, David Luan, Dario Amodei, and Ilya
  Sutskever.
\newblock Language models are unsupervised multitask learners.
\newblock \emph{OpenAI Blog}, 2019.

\bibitem[Rana et~al.(2026)Rana, Menezes, Kulkarni, Khatri, and
  Agarwal]{rana2025indicsuper}
Souvik Rana, Arul Menezes, Ashish Kulkarni, Chandra Khatri, and Shubham
  Agarwal.
\newblock {MUTANT}: A recipe for multilingual tokenizer design.
\newblock In \emph{Proceedings of the 64th Annual Meeting of the Association
  for Computational Linguistics (Volume 1: Long Papers)}, pages 46261--46277,
  San Diego, California, United States, July 2026. Association for
  Computational Linguistics.
\newblock \doi{10.18653/v1/2026.acl-long.2146}.

\bibitem[Reddi et~al.(2020)Reddi, Cheng, Kanter, Mattson, Schmuelling, Wu,
  Anderson, Breughe, Charlebois, Chou, Chukka, Coleman, Davis, Deng, Diamos,
  Duke, Fick, Gardner, Hubara, Idgunji, Jablin, Jiao, John, Kanwar, Lee, Liao,
  Lokhmotov, Massa, Meng, Micikevicius, Osborne, Pekhimenko, Rajan, Sequeira,
  Sirasao, Sun, Tang, Thomson, Wei, Wu, Xu, Yamada, Yu, Yuan, Zhong, Zhang, and
  Zhou]{reddi2020mlperf}
Vijay~Janapa Reddi, Christine Cheng, David Kanter, Peter Mattson, Guenther
  Schmuelling, Carole-Jean Wu, Brian Anderson, Maximilien Breughe, Mark
  Charlebois, William Chou, Ramesh Chukka, Cody Coleman, Sam Davis, Pan Deng,
  Greg Diamos, Jared Duke, Dave Fick, J.~Scott Gardner, Itay Hubara, Sachin
  Idgunji, Thomas~B. Jablin, Jeff Jiao, Tom~St. John, Pankaj Kanwar, David Lee,
  Jeffery Liao, Anton Lokhmotov, Francisco Massa, Peng Meng, Paulius
  Micikevicius, Colin Osborne, Gennady Pekhimenko, Arun Tejusve~Raghunath
  Rajan, Dilip Sequeira, Ashish Sirasao, Fei Sun, Hanlin Tang, Michael Thomson,
  Frank Wei, Ephrem Wu, Lingjie Xu, Koichi Yamada, Bing Yu, George Yuan, Aaron
  Zhong, Peizhao Zhang, and Yuchen Zhou.
\newblock {MLPerf} inference benchmark.
\newblock In \emph{2020 ACM/IEEE 47th Annual International Symposium on
  Computer Architecture (ISCA)}, pages 446--459, May 2020.
\newblock \doi{10.1109/ISCA45697.2020.00045}.

\bibitem[Reddy et~al.(2025)Reddy, Schmidt, Pinter, and
  Tanner]{reddy2025diminishing}
Varshini Reddy, Craig~W. Schmidt, Yuval Pinter, and Chris Tanner.
\newblock How much is enough? {The} diminishing returns of tokenization
  training data.
\newblock In \emph{Tokenization Workshop}, 2025.

\bibitem[Schmidt et~al.(2024)Schmidt, Reddy, Zhang, Alameddine, Uzan, Pinter,
  and Tanner]{schmidt2024compression}
Craig~W. Schmidt, Varshini Reddy, Haoran Zhang, Alec Alameddine, Omri Uzan,
  Yuval Pinter, and Chris Tanner.
\newblock Tokenization is more than compression.
\newblock In \emph{Proceedings of the 2024 Conference on Empirical Methods in
  Natural Language Processing}, pages 678--702, Miami, Florida, USA, November
  2024. Association for Computational Linguistics.
\newblock \doi{10.18653/v1/2024.emnlp-main.40}.

\bibitem[Schmidt et~al.(2025)Schmidt, Reddy, Tanner, and
  Pinter]{schmidt2025boundless}
Craig~W. Schmidt, Varshini Reddy, Chris Tanner, and Yuval Pinter.
\newblock Boundless byte pair encoding: Breaking the pre-tokenization barrier.
\newblock In \emph{Second Conference on Language Modeling}, 2025.

\bibitem[Shani et~al.(2026)Shani, Reif, Roll, Jurafsky, and
  Shutova]{shani2026roots}
Chen Shani, Yuval Reif, Nathan Roll, Dan Jurafsky, and Ekaterina Shutova.
\newblock The roots of performance disparity in multilingual language models:
  Intrinsic modeling difficulty or design choices?
\newblock \emph{arXiv preprint arXiv:2601.07220}, 2026.
\newblock \doi{10.48550/arXiv.2601.07220}.

\bibitem[Shazeer(2019)]{shazeer2019fast}
Noam Shazeer.
\newblock Fast transformer decoding: One write-head is all you need.
\newblock \emph{arXiv preprint arXiv:1911.02150}, 2019.
\newblock \doi{10.48550/arXiv.1911.02150}.

\bibitem[Shi et~al.(2023)Shi, Suzgun, Freitag, Wang, Srivats, Vosoughi, Chung,
  Tay, Ruder, Zhou, Das, and Wei]{shi2022mgsm}
Freda Shi, Mirac Suzgun, Markus Freitag, Xuezhi Wang, Suraj Srivats, Soroush
  Vosoughi, Hyung~Won Chung, Yi~Tay, Sebastian Ruder, Denny Zhou, Dipanjan Das,
  and Jason Wei.
\newblock Language models are multilingual {Chain-of-Thought} reasoners.
\newblock In \emph{The Eleventh International Conference on Learning
  Representations}, 2023.

\bibitem[Slagle(2024)]{slagle2024spacebyte}
Kevin Slagle.
\newblock {SpaceByte}: Towards deleting tokenization from large language
  modeling.
\newblock In \emph{Advances in Neural Information Processing Systems},
  volume~37, pages 124925--124950. Curran Associates, Inc., 2024.
\newblock \doi{10.52202/079017-3967}.

\bibitem[Teklehaymanot and Nejdl(2025)]{teklehaymanot2025infrastructure}
Hailay~Kidu Teklehaymanot and Wolfgang Nejdl.
\newblock Tokenization disparities as infrastructure bias: How subword systems
  create inequities in {LLM} access and efficiency.
\newblock In \emph{2025 3rd International Conference on Foundation and Large
  Language Models ({FLLM})}, pages 822--828. {IEEE}, November 2025.
\newblock \doi{10.1109/FLLM67465.2025.11390956}.

\bibitem[Thakur et~al.(2025)Thakur, Nagpal, Savarkar, Pundalik, Dosi, Sawarkar,
  Thakur, Saluja, Desarkar, and Ramakrishnan]{artofbreaking2025indic}
Aamod Thakur, Ajay Nagpal, Atharva Savarkar, Kundeshwar Pundalik, Siddhesh
  Dosi, Piyush Sawarkar, Viraj Thakur, Rohit Saluja, Maunendra~Sankar Desarkar,
  and Ganesh Ramakrishnan.
\newblock The art of breaking words: Rethinking multilingual tokenizer design.
\newblock \emph{arXiv preprint arXiv:2508.06533}, 2025.
\newblock \doi{10.48550/arXiv.2508.06533}.

\bibitem[{The Unicode Consortium}(2025)]{unicode2025}
{The Unicode Consortium}.
\newblock The {Unicode} standard, version 17.0.0.
\newblock \url{https://www.unicode.org/versions/Unicode17.0.0/}, 2025.

\bibitem[Vaswani et~al.(2017)Vaswani, Shazeer, Parmar, Uszkoreit, Jones, Gomez,
  Kaiser, and Polosukhin]{vaswani2017attention}
Ashish Vaswani, Noam Shazeer, Niki Parmar, Jakob Uszkoreit, Llion Jones,
  Aidan~N. Gomez, {\L}ukasz Kaiser, and Illia Polosukhin.
\newblock Attention is all you need.
\newblock In \emph{Advances in Neural Information Processing Systems},
  volume~30. Curran Associates, Inc., 2017.

\bibitem[Velayuthan and Sarveswaran(2025)]{velayuthan2025gpe}
Menan Velayuthan and Kengatharaiyer Sarveswaran.
\newblock Egalitarian language representation in language models: It all begins
  with tokenizers.
\newblock In \emph{Proceedings of the 31st International Conference on
  Computational Linguistics}, pages 5987--5996, Abu Dhabi, UAE, January 2025.
  Association for Computational Linguistics.

\bibitem[Wang et~al.(2020)Wang, Cho, and Gu]{wang2020bbpe}
Changhan Wang, Kyunghyun Cho, and Jiatao Gu.
\newblock Neural machine translation with byte-level subwords.
\newblock \emph{Proceedings of the AAAI Conference on Artificial Intelligence},
  34\penalty0 (05):\penalty0 9154--9160, 2020.
\newblock \doi{10.1609/aaai.v34i05.6451}.

\bibitem[Xue et~al.(2021)Xue, Constant, Roberts, Kale, Al-Rfou, Siddhant,
  Barua, and Raffel]{xue2021mt5}
Linting Xue, Noah Constant, Adam Roberts, Mihir Kale, Rami Al-Rfou, Aditya
  Siddhant, Aditya Barua, and Colin Raffel.
\newblock m{T}5: A massively multilingual pre-trained text-to-text transformer.
\newblock In \emph{Proceedings of the 2021 Conference of the North American
  Chapter of the Association for Computational Linguistics: Human Language
  Technologies}, pages 483--498, Online, June 2021. Association for
  Computational Linguistics.
\newblock \doi{10.18653/v1/2021.naacl-main.41}.

\bibitem[Xue et~al.(2022)Xue, Barua, Constant, Al-Rfou, Narang, Kale, Roberts,
  and Raffel]{xue2022byt5}
Linting Xue, Aditya Barua, Noah Constant, Rami Al-Rfou, Sharan Narang, Mihir
  Kale, Adam Roberts, and Colin Raffel.
\newblock {B}y{T}5: Towards a token-free future with pre-trained byte-to-byte
  models.
\newblock \emph{Transactions of the Association for Computational Linguistics},
  10:\penalty0 291--306, 2022.
\newblock \doi{10.1162/tacl_a_00461}.

\bibitem[Zellers et~al.(2019)Zellers, Holtzman, Bisk, Farhadi, and
  Choi]{zellers2019hellaswag}
Rowan Zellers, Ari Holtzman, Yonatan Bisk, Ali Farhadi, and Yejin Choi.
\newblock {HellaSwag}: Can a machine really finish your sentence?
\newblock In \emph{Proceedings of the 57th Annual Meeting of the Association
  for Computational Linguistics}, pages 4791--4800, Florence, Italy, July 2019.
  Association for Computational Linguistics.
\newblock \doi{10.18653/v1/P19-1472}.

\bibitem[Zhang et~al.(2024)Zhang, Zeng, Wang, and Lu]{zhang2024tinyllama}
Peiyuan Zhang, Guangtao Zeng, Tianduo Wang, and Wei Lu.
\newblock {TinyLlama}: An open-source small language model.
\newblock \emph{arXiv preprint arXiv:2401.02385}, 2024.
\newblock \doi{10.48550/arXiv.2401.02385}.

\end{thebibliography}

\clearpage
\appendix
\renewcommand{\thesection}{\Alph{section}}
\renewcommand{\thesubsection}{\thesection.\arabic{subsection}}
\numberwithin{table}{section}
\numberwithin{figure}{section}
\renewcommand{\theHsection}{appendix.\arabic{section}}
\renewcommand{\theHsubsection}{appendix.\arabic{section}.\arabic{subsection}}
\renewcommand{\theHtable}{appendix.\arabic{section}.table.\arabic{table}}
\renewcommand{\theHfigure}{appendix.\arabic{section}.figure.\arabic{figure}}

\section{Universal byte-level encoding (UBE) algorithm and correctness}
\label{app:ube-architecture}

\subsection{Pre-tokenizer: \utfEightName{}-byte-length routing}

The pre-tokenizer applies three stages sequentially:

\paragraph{Stage \paperPhaseOne{}: \gptTwoName{} regex splitting.}
UBE uses the same \gptTwoName{} byte-level pre-tokenizer pattern as the BBPE baseline in this paper to split input text into word-like chunks:
\begin{center}
\footnotesize
\verb@'s|'t|'re|'ve|'m|'ll|'d| ?\p{L}+| ?\p{N}+@\\
\verb@| ?[^\s\p{L}\p{N}]+|\s+(?!\S)|\s+@
\end{center}
In the core UBE/BBPE tokenizer artifacts, the Stage-\paperPhaseOne{} pre-tokenizer is configured with \texttt{use\_regex=true} and \texttt{add\_prefix\_space=false}; no prefix space is inserted before regex splitting during tokenization. The tokenizer-level normalizer is \texttt{null}, so inputs are used as received, with no Unicode normalization, lowercasing, accent stripping, or compatibility folding. All tokenizer families reserve five added special tokens with IDs 0--4: \texttt{<s>}, \texttt{<pad>}, \texttt{</s>}, \texttt{<unk>}, and \texttt{<mask>}; each is serialized with \texttt{special=true} and \texttt{normalized=false}. This stage is held fixed across tokenizer families, so UBE changes the byte substrate rather than the word-boundary policy or special-token convention.

\paragraph{Stage \paperPhaseTwo{}: per-character routing.}
Each character is routed by a deterministic rule based only on its \utfEightName{} byte length. Here, ``character'' refers to one Unicode scalar value (the Rust \texttt{char} unit used by the implementation), not a grapheme cluster; isolated surrogate code points are outside this well-formed input domain:
\begin{itemize}
    \item \textbf{\utfEightName{} path}: characters whose \utfEightName{} encoding is 1--2 bytes (e.g., ASCII and many Latin Extended, Cyrillic, Greek, Arabic, and Hebrew code points).
    \item \textbf{\utfSixteenName{} path}: characters whose \utfEightName{} encoding is 3--4 bytes. This route includes all 3-byte BMP scalar values and 4-byte non-BMP scalar values; only the former receive a byte-floor reduction, from 3 \utfEightName{} bytes to 2 \utfSixteenName{} bytes. Examples of 3-byte BMP scripts and code points include CJK characters, Devanagari, Bengali, Tamil, Thai, Georgian, Tibetan, Ethiopic, Myanmar, and Khmer.
\end{itemize}

\paragraph{Stage \paperPhaseThree{}: byte serialization and alphabet mapping.}
\begin{itemize}
    \item \textbf{\utfEightName{} path}: character $\to$ \utfEightName{} bytes $\to$ \gptTwoName{} visible-byte alphabet (\bbpeAlphabetSize{} symbols).
    \item \textbf{\utfSixteenName{} path}: character $\to$ UTF-16LE bytes $\to$ PUA U+E100--E1FF (\bbpeAlphabetSize{} symbols).
\end{itemize}
On the \utfSixteenName{} path, BMP code points emit 2 bytes, while non-BMP code points emit 4 bytes via surrogate pairs.
The combined initial alphabet has $\bbpeAlphabetSize{} + \bbpeAlphabetSize{} = \ubeAlphabetSize{}$ base symbols. The two alphabets share no characters.

\subsection{Reference pseudocode}
\label{app:ube-pseudocode}

The following pseudocode (Algorithms~\ref{alg:ube-encode} and~\ref{alg:ube-decode}) summarizes the reference implementation at the level needed to reproduce the routing and exact round-trip behavior. Standard BPE merge learning, ranking, and vocabulary construction are unchanged. BPE is applied independently to each \gptTwoName{} regex chunk, matching standard \gptTwoName{}-style BBPE and preventing learned merges from crossing pre-tokenizer chunk boundaries. Exact round-trip decoding for encoded Unicode strings is proved in Proposition~\ref{prop:roundtrip}.

\begin{algorithm}[ht]
\caption{UBE encode}
\label{alg:ube-encode}
\begin{algorithmic}[1]
\Require Unicode string $s$; trained BPE vocabulary over the \ubeAlphabetSize{}-symbol union alphabet
\Ensure Token ID sequence
\If{\texttt{add\_prefix\_space} and $s$ does not start with whitespace}
    \State $s \leftarrow$ \texttt{\char34\ \char34} $\|$ $s$
\EndIf
\State $\mathit{chunks} \leftarrow \text{GPT-2-regex-split}(s)$ \Comment{Stage 1}
\State $\mathit{tokens} \leftarrow [\,]$
\For{each chunk in $\mathit{chunks}$}
    \State $\mathit{chunkSymbols} \leftarrow [\,]$
    \For{each character $c$ in chunk} \Comment{Stage 2: routing}
        \If{$|\mathrm{utf8}(c)| \leq \ubeRouteUtfEightMaxBytes$}
            \State Encode $c$ as \utfEightName{} bytes; map each byte $b$ to $\text{GPT-2}[b]$ \Comment{Stage 3}
        \Else
            \State Encode $c$ as UTF-16LE bytes; map each byte $b$ to $\text{PUA}[b]$ (U+E100+$b$) \Comment{Stage 3}
        \EndIf
        \State Append mapped symbols to $\mathit{chunkSymbols}$
    \EndFor
    \State Append $\text{BPE-merge}(\mathit{chunkSymbols})$ to $\mathit{tokens}$ \Comment{No merge crosses chunk boundaries}
\EndFor
\State \Return $\mathit{tokens}$
\end{algorithmic}
\end{algorithm}

\begin{algorithm}[ht]
\caption{UBE decode (dual-buffer flush-on-switch)}
\label{alg:ube-decode}
\begin{algorithmic}[1]
\Require Expanded symbol sequence from BPE decoding over the mixed alphabet
\Ensure Decoded Unicode string
\State $B_8 \leftarrow [\,]$; $B_{16} \leftarrow [\,]$; $\mathit{out} \leftarrow [\,]$ \Comment{buffers + output}
\For{each symbol $x$ in expanded BPE symbol stream}
    \If{$x \in \text{GPT-2 alphabet}$}
        \State Flush $B_{16}$: decode as UTF-16LE, append to $\mathit{out}$, clear $B_{16}$
        \State Append corresponding byte to $B_8$
    \ElsIf{$x \in \text{PUA alphabet}$}
        \State Flush $B_8$: decode as \utfEightName{}, append to $\mathit{out}$, clear $B_8$
        \State Append $(\text{ord}(x) - \text{0xE100})$ to $B_{16}$
    \Else \Comment{special token boundary}
        \State Flush $B_8$ and $B_{16}$; append $x$ to $\mathit{out}$
    \EndIf
\EndFor
\State Flush remaining $B_8$ and $B_{16}$
\State \Return $\text{concat}(\mathit{out})$
\end{algorithmic}
\end{algorithm}

\FloatBarrier

\subsection{Encoding example}
\label{app:encoding-example}

We complement Figure~\ref{fig:ube-pipeline} (which gives the 7/9/10 base-symbol totals for ``GPU\begin{CJK}{UTF8}{gbsn}性能\end{CJK}'' under UBE, BBPE, and \bbpeSixteenName{}) with the byte-level mapping for the CJK character \begin{CJK}{UTF8}{gbsn}性\end{CJK} (U+6027): BBPE encodes it as 3 \utfEightName{} bytes (\texttt{0xE6, 0x80, 0xA7}), mapped to the \gptTwoName{} alphabet, while UBE encodes it as 2 UTF-16LE bytes (\texttt{0x27, 0x60}), mapped to PUA (U+E127, U+E160).

\paragraph{Private-use input characters.}
The PUA block U+E100--E1FF is used only as an internal visible-byte alphabet for the \utfSixteenName{} path; raw input scalar values from the same Unicode range are still encoded as ordinary input characters before alphabet mapping. For example, a raw U+E100 character is a 3-byte \utfEightName{} scalar and is routed to the \utfSixteenName{} path. Its UTF-16LE bytes are \texttt{0x00, 0xE1}, which are represented by the two internal symbols U+E100 and U+E1E1. During decoding, these symbols are interpreted as byte values, reconstructed as the \utfSixteenName{} code unit 0xE100, and emitted as the original raw scalar value. Thus there is no separate collision case: alphabet membership identifies the byte buffer, while the reconstructed code unit determines the final Unicode scalar.

\subsection{Correctness and implementation notes}
\label{app:decoder}

\begin{proposition}[Exact round-trip]
\label{prop:roundtrip}
Let $s$ be a well-formed Unicode scalar-value string with no isolated surrogate code points, and let $s^{+}$ denote the string after the tokenizer's optional prefix-space preprocessing. If BPE merge tokens are expanded to base symbols before buffer routing, then $\text{decode}(\text{encode}(s)) = s^{+}$. In particular, $\text{decode}(\text{encode}(s)) = s$ whenever prefix-space preprocessing does not modify $s$.
\end{proposition}
\begin{proof} Prefix-space preprocessing, when enabled and applicable, deterministically replaces $s$ with $s^{+}$ before regex splitting; the remaining argument is therefore about reconstructing $s^{+}$. BPE decoding first expands each merge token into a sequence over the two base alphabets. The encoder maps each input scalar value of $s^{+}$ into exactly one of two disjoint symbol alphabets: the visible-byte alphabet used by the \utfEightName{} path, and the private-use alphabet used by the \utfSixteenName{} path. Because the alphabets are disjoint, every emitted symbol uniquely identifies which byte buffer must receive it during decoding. When the symbol stream switches paths, the previously active buffer is flushed and decoded before any byte from the new path is appended, so the decoder never forms a mixed \utfEightName{}/\utfSixteenName{} byte string. Within a same-path run, the decoder does not need to know character boundaries: the encoder wrote a concatenation of complete \utfEightName{} byte sequences or complete \utfSixteenName{} code-unit sequences, and the standard \utfEightName{} and \utfSixteenName{} decodings are bijective on well-formed byte or code-unit sequences (including surrogate-paired non-BMP scalar values), so the original scalar values are recovered. Special tokens trigger a flush of both buffers and are passed through unchanged, preserving token boundaries. Therefore every encoded span is reconstructed exactly, and the final concatenation is $s^{+}$; if prefix-space preprocessing is disabled or leaves $s$ unchanged, this is the original string $s$. We additionally verified this round-trip property empirically on \floresTwoHundredName{} (\appendixDecodeVerifiedLangCount{} languages, \appendixDecodeVerifiedSentenceCount{} sentences, \appendixDecodeVerifiedPassRate{} pass rate).
\end{proof}

This proposition is a round-trip guarantee for token sequences produced by the encoder, not a claim that every arbitrary model-generated token sequence expands to a well-formed \utfEightName{} or \utfSixteenName{} byte string. The UBE decoder follows the usual robust byte-tokenizer convention: malformed same-path byte runs are decoded with replacement characters rather than raising an exception, so generation remains total while exactness is guaranteed only on encoder-produced streams.

The Rust core implements the routing pre-tokenizer, dual-buffer decoder, and offset post-processor, with Python wrappers exposing standard Hugging Face training APIs; the incremental training pipeline (Appendix~\ref{app:tokenizer-training}) is a separate systems layer. UBE implements the Hugging Face \texttt{PostProcessor} interface for downstream span-extraction compatibility; our configurations set \texttt{trim\_offsets=false} so that leading/trailing whitespace offsets are returned as-is.

\paragraph{Streaming and constrained decoding audit.}
The decoder state needed for streaming consists only of the active alphabet tag and one pending byte buffer. The local invariant after any emitted base symbol is: all flushed output is valid Unicode, and the only unflushed bytes belong to exactly one path, either \utfEightName{} or \utfSixteenName{}\@. When the next expanded symbol has the same alphabet, its byte is appended to the pending buffer; when its alphabet differs, the pending buffer is decoded and emitted before the new buffer is started. Special tokens also force both buffers to flush. Therefore cross-alphabet BPE merges do not create mixed byte strings after token expansion, and streaming implementations can expose committed output at flush points while retaining at most the current same-alphabet byte run as pending state. Guided or grammar-constrained decoders need access to the same per-symbol alphabet tag before applying byte-level constraints; otherwise a constraint engine cannot tell whether a byte value is being interpreted as \utfEightName{} or \utfSixteenName{}\@. We audit this as an implementation requirement rather than a throughput result; constrained-decoding latency is not benchmarked in this paper.

\section{Training and evaluation pipeline}
\label{app:training-evaluation-pipeline}

\subsection{Tokenizer training pipeline}
\label{app:tokenizer-training}

\paragraph{Overview.}
All tokenizers (UBE, BBPE, \bbpeSixteenName{}) are trained using a \paperPhaseCountWord{}-phase distributed BPE pipeline. The canonical entry point is \texttt{train\_tokenizer.py}, which dispatches phase-\paperPhaseOne{} shard processing and phase-\paperPhaseTwo{} merge/finalization. Phase \paperPhaseOne{} computes shard-local states once, and phase \paperPhaseTwo{} merges and finalizes those states into the target vocabulary. This design separates shard-level word-count collection from merge selection, enabling reuse of shard states across vocabulary sizes for a fixed tokenizer type and fixed shard manifest.

\paragraph{Phase \paperPhaseOne{}: shard-local state collection.}
Each training shard (an Arrow file containing raw text) is processed independently. For each shard, we:
\begin{enumerate}
    \item Extract the \texttt{text} column from the Arrow file.
    \item Initialize a fresh \texttt{BpeTrainer} with \texttt{vocab\_size} set to the maximum target for the sweep being produced (\paperVocabSixFourkDisplay{} for the core 16K/32K/64K sweep, with larger targets used for extended diagnostics) and \texttt{min\_frequency=\paperMinFrequencyShard{}}.
    \item Run \texttt{train\_from\_iterator\_incremental(finalize=False)}, which accumulates pre-tokenized word counts without performing BPE merge selection. Pair statistics are derived from the retained word table during finalization.
    \item Save the resulting word-count state to a JSON file via \texttt{trainer.save\_state(path)}.
\end{enumerate}
The save-state threshold \texttt{min\_frequency=\paperMinFrequencyShard{}} filters per-shard singletons (word patterns appearing only once in a single shard), reducing state file size. This is a noise-reduction step, not a final vocabulary decision. Shard states are independent and can be computed in parallel across workers.
\paragraph{Phase \paperPhaseTwo{}: state merging.}
All shard states are merged by summing retained word counts in a single parallel reduce pass with no merge-time pruning or post-reduction filtering in the canonical path. Because addition is commutative and associative, the retained word-count totals are order-independent and exact over the retained shard states. This is not equivalent to exact whole-corpus BPE, however, because phase~\paperPhaseOne{} already removes per-shard singletons: patterns appearing once in multiple shards can be lost before the reduce ever sees them. For the 101-language configuration this canonical merge retains approximately 139M unique word patterns after phase-\paperPhaseOne{} filtering, and all retained patterns are carried into the subsequent BPE merge step.

\paragraph{Finalization within phase \paperPhaseTwo{}.}
The merged state is finalized by running the full BPE merge loop up to the target vocabulary size:
\begin{itemize}
    \item \texttt{finalize(vocab\_size=target, min\_frequency=\paperMinFrequencyFinalize{})}: the final BPE merge loop runs with the same numerical threshold used when saving shard states.
    \item Setting \texttt{min\_frequency=\paperMinFrequencyFinalize{}} at finalization preserves the behavior of the reported tokenizers.
\end{itemize}
The resulting tokenizer is saved as a Hugging Face-compatible \texttt{tokenizer.json}.

\paragraph{Frequency metadata.}
The merge routine records token-frequency metadata derived from the merged filtered state, including a \texttt{token\_counts.json} sidecar alongside the tokenizer artifact; the vocabulary-utilization diagnostics in Appendix~\ref{app:vocab-diagnostics} use these counts to compute dead-merge and active-utilization statistics.

\paragraph{Vocabulary sizes and alphabet overhead.}
The core sweep used for the paper's \paperConditionCount{} significance-tested conditions trains tokenizers at three vocabulary sizes: \paperVocabOneSixk{}, \paperVocabThreeTwok{}, and \paperVocabSixFourk{}, with \texttt{101-lang} extensions at \paperVocabOneTwoEightk{} and \paperVocabTwoFiveSixk{}. All rows include the same $S=5$ special tokens, so the available merge slots are $V-S-\ubeAlphabetSize{}$ for UBE and $V-S-\bbpeAlphabetSize{}$ for BBPE; Table~\ref{tab:vocab-merge-slots} reports the resulting accounting.

\begin{table}[ht]
\centering
\caption{Merge-slot accounting by vocabulary size. All rows reserve the same five special tokens; UBE spends an additional \bbpeAlphabetSize{} slots on its disjoint \utfSixteenName{} alphabet relative to BBPE\@.}
\label{tab:vocab-merge-slots}
\begin{tabular}{lccc}
\toprule
Vocab size & UBE merges & BBPE merges & UBE overhead \\
\midrule
\paperVocabOneSixk{} & \mergeVOneSixkUbe{} & \mergeVOneSixkBbpe{} & \mergeVOneSixkOverheadPct{} \\
\paperVocabThreeTwok{} & \mergeVThreeTwokUbe{} & \mergeVThreeTwokBbpe{} & \mergeVThreeTwokOverheadPct{} \\
\paperVocabSixFourk{} & \mergeVSixFourkUbe{} & \mergeVSixFourkBbpe{} & \mergeVSixFourkOverheadPct{} \\
\paperVocabOneTwoEightk{} & \mergeVOneTwoEightkUbe{} & \mergeVOneTwoEightkBbpe{} & \mergeVOneTwoEightkOverheadPct{} \\
\paperVocabTwoFiveSixk{} & \mergeVTwoFiveSixkUbe{} & \mergeVTwoFiveSixkBbpe{} & \mergeVTwoFiveSixkOverheadPct{} \\
\bottomrule
\end{tabular}
\end{table}

\paragraph{Reproducibility.}
Our preparation scripts construct Arrow shards under \texttt{data\_dir/\allowbreak\{lang\}/\allowbreak*.arrow}, reserve one terminal shard per language for held-out evaluation as \texttt{99999.arrow}, and hold the raw-data manifest fixed across tokenizer families. The frozen 101-language mC4~\citep{xue2021mt5} mirror therefore contains 2,273 train shards and 101 held-out evaluation shards overall. As discussed in the main text, \texttt{101-lang/32K} is the cleanest downstream attribution setting because UBE and BBPE see nearly identical raw-text coverage there; the lower-coverage settings are retained as fixed-compute stress tests.

\paragraph{Merge-order tie-breaking.}
The BPE merge loop is deterministic given the merged word-count table, but hash-map iteration order can change how equal-frequency ties near frequency boundaries are resolved. A retrain without deterministic tie handling is therefore \emph{frequency-exact} (all merged pair counts are identical across runs) but not necessarily \emph{bit-for-bit reproducible}: a small number of merges near frequency boundaries may be reordered.

The tokenizer directories used in this work are fixed and define the canonical artifacts for all reported experiments. For deterministic reruns of tie-prone training stages, our training scripts expose \texttt{-{}-deterministic-tie-breaks}, which passes \texttt{deterministic\_tie\_breaks=True} to the custom BPE trainer and applies stable secondary keys to equal-frequency alphabet and post-merge-pruning ties.

As a sensitivity check, we retrained the \texttt{6-lang/32K} UBE tokenizer with this deterministic-tie mode. The retrained tokenizer keeps the same 32K vocabulary size and five special-token IDs and round-trips multilingual samples without \texttt{<unk>} emission. Table~\ref{tab:deterministic-tie-break-check} compares the original tokenizer to this deterministic-tie variant on the same intrinsic datasets. The differences are at most \detTieMaxAbsDeltaGini{} in Gini and \detTieMaxAbsDeltaMean{} in mean premium, so the paper's intrinsic conclusions are not driven by hash-map tie order in this stress setting. We treat this as a targeted sensitivity check rather than as a full deterministic retrain sweep across all vocabulary sizes.

\begin{table}[ht]
\centering
\small
\caption{Deterministic tie-break sensitivity check for UBE \texttt{6-lang/32K}. ``Det.'' is the tokenizer trained with deterministic secondary ordering for equal-frequency merge candidates. $\Delta$ columns report Det.\ minus the original UBE tokenizer.}
\label{tab:deterministic-tie-break-check}
\begin{tabular}{lrrrr}
\toprule
Dataset & Original Gini & Det.\ Gini & $\Delta$Gini & $\Delta$ mean premium \\
\midrule
FLORES-200 & \detTieFloresOrigGini{} & \detTieFloresDetGini{} & \detTieFloresDeltaGini{} & \detTieFloresDeltaMean{} \\
UDHR       & \detTieUdhrOrigGini{}   & \detTieUdhrDetGini{}   & \detTieUdhrDeltaGini{}   & \detTieUdhrDeltaMean{} \\
SIB-200    & \detTieSibOrigGini{}    & \detTieSibDetGini{}    & \detTieSibDeltaGini{}    & \detTieSibDeltaMean{} \\
MGSM       & \detTieMgsmOrigGini{}   & \detTieMgsmDetGini{}   & \detTieMgsmDeltaGini{}   & \detTieMgsmDeltaMean{} \\
\bottomrule
\end{tabular}
\end{table}

\subsection{Language-model training and evaluation configuration}
\label{app:lm-training-config}

The primary two-scale LM matrix spans the 182M/32K coverage matrix and the Chinchilla-style~\citep{hoffmann2022chinchilla} 1.3B/32K scale check. The reported core LM suite contains \lmPrimaryDesignRuns{} runs: this primary matrix plus a separate 329M/128K composition check for S-UBE and S-BBPE\@. We additionally report \lmAdditionalReportedLmRuns{} appendix LM runs: \lmAdditionalSixFourkRuns{} 230M/64K vocabulary variants, \lmAdditionalSuperBpeRuns{} 182M/32K SuperBPE composition rows, and \lmAdditionalMultiseedRuns{} extra random-seed repeats for the matched \texttt{101-lang/32K} setting. All runs keep the LM corpus fixed to the 101-language mC4 release; runs differ in tokenizer coverage, vocabulary size, model scale, or seed. Within each setting, architecture, optimizer, and token-step budget are matched while the tokenizer configuration changes; raw-text exposure can therefore differ in the lower-coverage fixed-compute stress rows. The LM dataloader applies no language-level temperature sampling or resampling: it deterministically shuffles a fixed, capped set of per-language shards, so the retained shard volumes determine the language mixture. Table~\ref{tab:lm-scaleup-plan} summarizes the reported core suite, and Table~\ref{tab:lm-token-accounting} reports the 32K token-budget estimates per configuration.

\paragraph{Transformer architecture.}
All LM runs use the same decoder-only LitGPT/LLaMA-style architecture. The 182M, 230M, and 329M runs use 14L/12H/768d (head dimension 64, MLP 3072); the 230M and 329M variants differ from the 182M/32K setting only by the 64K and 128K vocabularies, respectively. The 1.3B runs use 24L/16H/2048d (head dimension 128, MLP 5504). Blocks use pre-norm \texttt{RMSNorm} ($\epsilon=10^{-5}$), sequential residuals (\texttt{parallel\_residual=false}), standard causal multi-head attention (MHA) with \texttt{n\_query\_groups=n\_head} (not grouped-query or multi-query attention), no query-key (QK) normalization, no sliding-window attention, and no attention/MLP/projection biases. Positional encoding is RoPE only, applied to all query/key head dimensions (\texttt{rotary\_percentage=1.0}, \texttt{rope\_base=10000}, \texttt{rope\_condense\_ratio=1}), with no learned absolute position embeddings. The feed-forward sublayer is a LLaMA gated SiLU MLP, $\mathrm{SiLU}(xW_1)\odot xW_2$, followed by an output projection. Token embeddings and the LM head are untied; the vocabulary is read from each tokenizer artifact, padded to a multiple of 512, and uses the same five special tokens listed in Appendix~\ref{app:ube-architecture}.

\begin{table}[ht]
\caption{Reported core LM suite on 8$\times$A100 80\,GB GPUs. ``Batch recipe'' is per-device microbatch $\times$ gradient-accumulation steps; all rows use sequence length 2048. The 182M and 1.3B rows form the primary two-scale UBE/BBPE matrix; the 329M row is a separate 128K composition check for S-UBE and S-BBPE\@.}
\label{tab:lm-scaleup-plan}
\centering
{\footnotesize
\setlength{\tabcolsep}{3pt}
\begin{tabular}{lcccccL{4.0cm}}
\toprule
Setting & Runs & Layers/heads/width & Params & Batch recipe & Global tok/step & Role \\
\midrule
182M & 6 & 14/12/768 & 182M & 32 $\times$ 1 & 524,288 & full coverage matrix (200K steps, 105B tok) \\
329M & 4 & 14/12/768 & 329M & 8 $\times$ 4 & 524,288 & 128K composition check (200K steps, 105B tok) \\
1.3B & 6 & 24/16/2048 & 1.35B actual & 8 $\times$ 1 & 131,072 & scale verification (205K steps, $\sim$26.9B tok) \\
\bottomrule
\end{tabular}
}
\end{table}

Thus, the 1.3B row processes $8$ GPUs $\times$ $8$ sequences per device $\times$ $1$ accumulation step $\times$ $2048$ tokens $=131{,}072$ global tokens per optimizer step. This row is a matched UBE/BBPE scale-verification experiment rather than the primary attribution setting.

\begin{table}[ht]
\caption{Estimated 32K LM token accounting under the shared 101-language mC4 release. Train counts use one-shard-per-language byte-weighted extrapolation; evaluation counts enumerate held-out terminal shards. Steps/epoch are train-token totals divided by 524,288 global tokens/step, and fixed 200K steps correspond to 104.9B training tokens.}
\label{tab:lm-token-accounting}
\centering
{\footnotesize
\setlength{\tabcolsep}{3pt}
\resizebox{\linewidth}{!}{%
\begin{tabular}{lcccccc}
\toprule
Tokenizer coverage & UBE train & BBPE train & UBE eval & BBPE eval & Steps/epoch & 200K coverage \\
\midrule
\texttt{En-only} on 101-lang mC4 & $\sim$\tcUBEOneLangTrainB{}B & $\sim$\tcBBPEOneLangTrainB{}B & $\sim$\tcUBEOneLangEvalB{}B & $\sim$\tcBBPEOneLangEvalB{}B & \tcUBEOneLangStepsPerEpoch{} / \tcBBPEOneLangStepsPerEpoch{} & \tcUBEOneLangTwoHundredKPct{}\% / \tcBBPEOneLangTwoHundredKPct{}\% \\
\texttt{6-lang} on 101-lang mC4 & $\sim$\tcUBESixLangTrainB{}B & $\sim$\tcBBPESixLangTrainB{}B & $\sim$\tcUBESixLangEvalB{}B & $\sim$\tcBBPESixLangEvalB{}B & \tcUBESixLangStepsPerEpoch{} / \tcBBPESixLangStepsPerEpoch{} & \tcUBESixLangTwoHundredKPct{}\% / \tcBBPESixLangTwoHundredKPct{}\% \\
\texttt{101-lang} on 101-lang mC4 & $\sim$\tcUBEOneOhOneLangTrainB{}B & $\sim$\tcBBPEOneOhOneLangTrainB{}B & $\sim$\tcUBEOneOhOneLangEvalB{}B & $\sim$\tcBBPEOneOhOneLangEvalB{}B & \tcUBEOneOhOneLangStepsPerEpoch{} / \tcBBPEOneOhOneLangStepsPerEpoch{} & \tcUBEOneOhOneLangTwoHundredKPct{}\% / \tcBBPEOneOhOneLangTwoHundredKPct{}\% \\
\bottomrule
\end{tabular}
}
\vspace{2pt}
\begin{minipage}{0.95\linewidth}
\raggedright\scriptsize
Entries with two values are UBE/BBPE\@. The 64K LM training runs (\texttt{En-only}/\texttt{6-lang}/\texttt{101-lang}) and 128K LM training runs (\texttt{101-lang} only) in Tables~\ref{tab:lm-eval-64k} and~\ref{tab:lm-eval-128k} share the same 101-language mC4 corpus and 200K-step budget. Their token totals follow the same one-shard-per-language methodology and trend slightly downward at larger vocabularies as more frequent patterns absorb into multi-byte merges (cf.\ the \texttt{All avg} column of Table~\ref{tab:gini-vs-tokens}).
\end{minipage}
}
\end{table}

\paragraph{Optimizer and training details.}
All LM runs use the same LitGPT-based trainer with AdamW ($\beta_1=0.9$, $\beta_2=0.95$), weight decay $0.1$, and gradient clipping at $\lVert g \rVert \leq 1.0$. Training uses a linear learning-rate (LR) warmup over 4{,}000 steps followed by cosine decay to $10\%$ of the peak learning rate. The 182M, 230M, and 329M settings use peak LR $6\times 10^{-4}$ with min LR $6\times 10^{-5}$; the 1.3B setting uses peak LR $3\times 10^{-4}$ with min LR $3\times 10^{-5}$. Sequence length is $2048$ throughout. Mixed-precision training is bf16-mixed with \texttt{torch.compile}. All single-seed LM runs use the same seed; the multi-seed stability study (Appendix~\ref{app:multiseed}) adds two independent seeds. During training, a lightweight held-out perplexity check runs every $5{,}000$ steps over $70$ iterations on the terminal evaluation shard. Dropout is disabled. Optimizer states are reset per run; no checkpoint resumption is used across tokenizer conditions.

\paragraph{Evaluation details.}
Held-out predictive loss is evaluated as per-language BPB on deterministic terminal-shard splits across all 101 mC4 language directories, with 1{,}600 context blocks per language to equalize representation. Let $L_{\mathrm{tok}}$ sum natural-log negative log-likelihood over all predicted targets and $L_{\mathrm{text}}$ sum it only over targets with positive raw-text byte counts, excluding EOS and other zero-byte targets. With $N_{\mathrm{tok}}$ predicted targets and $N_{\mathrm{byte}}$ raw \utfEightName{} bytes, the token-level perplexity is $\mathrm{PPL}=\exp(L_{\mathrm{tok}}/N_{\mathrm{tok}})$ and $\mathrm{BPB}=L_{\mathrm{text}}/(N_{\mathrm{byte}}\ln 2)$. Thus, token-level perplexity is not obtained by converting BPB alone. The BPB denominator uses the raw byte count of the underlying text for both UBE and BBPE, rather than a tokenizer-dependent encoded-byte count. Downstream evaluation uses the test splits for XNLI~\citep{conneau2018xnli} (15 languages), XCOPA~\citep{ponti2020xcopa} (11), Belebele~\citep{bandarkar2024belebele} (122), ARC-Easy/ARC-Challenge~\citep{clark2018arc}, and LAMBADA~\citep{paperno2016lambada}; HellaSwag~\citep{zellers2019hellaswag} and PIQA~\citep{bisk2020piqa} use their validation splits.

\subsection{Corpus and token-count accounting}
\label{app:corpus-accounting}

The primary corpus uses tiered per-language targets measured in stored Arrow data: approximately 100\,GiB for English, 30\,GiB each for Korean, Chinese, Japanese, Hindi, and Thai, and 1\,GiB for each of the remaining 95 languages, subject to data availability and shard granularity. As in other multilingual tokenizers, these targets depart from the natural size distribution of the source data: prior work resamples languages so that smaller ones are well represented in the vocabulary \citep{conneau2019xlm,xue2021mt5,costajussa2022nllb}, and one byte-level BPE study fixes English at 30\% of its tokenizer data \citep{abagyan2025onetokenizer}. Our tiers give English the largest share, give five languages whose primary scripts use 3-byte UTF-8 characters a data-rich partition, and keep the other 95 languages in tokenizer training at a smaller fixed size. Because the targets are byte counts, the content of a partition also depends on its byte premium \citep{arnett2024bitproblem}. Table~\ref{tab:corpus-accounting} summarizes the three tokenizer-training subsets, which share the same underlying shards and differ only in language coverage. Within each configuration--vocabulary comparison, tokenizer families use identical raw-text shards, allowing us to compare byte representations under a fixed, non-uniform language mixture; Appendix~\ref{app:corpus-regime} repeats the comparison on a 50\,MiB/language prefix. The LM corpus is held fixed to the full 101-language release across all primary and appendix LM runs.

\begin{table}[ht]
\centering
\small
\setlength{\tabcolsep}{4pt}
\caption{Tokenizer-training corpus size by configuration. Rows aggregate the language subset used to train each tokenizer; the LM corpus always uses the full 101-language mC4 release. Evaluation columns enumerate held-out terminal shards, document counts are byte-weighted one-shard estimates, and raw byte totals are exact Arrow file sizes.}
\label{tab:corpus-accounting}
\begin{tabular}{lrrrrrr}
\toprule
\multirow{2}{*}{Config} & \multirow{2}{*}{Langs} & \multicolumn{3}{c}{Train} & \multicolumn{2}{c}{Eval} \\
\cmidrule(lr){3-5} \cmidrule(lr){6-7}
 & & Shards & GiB & Docs & Shards & Docs \\
\midrule
\texttt{En-only}   & \corpusOneLangLangs{}       & \corpusOneLangTrainShards{}      & \corpusOneLangTrainGB{}      & \corpusOneLangTrainDocsM{}M & \corpusOneLangEvalShards{}      & \corpusOneLangEvalDocsM{}M \\
\texttt{6-lang}    & \corpusSixLangLangs{}       & \corpusSixLangTrainShards{}      & \corpusSixLangTrainGB{}      & \corpusSixLangTrainDocsM{}M & \corpusSixLangEvalShards{}      & \corpusSixLangEvalDocsM{}M \\
\texttt{101-lang}  & \corpusOneOhOneLangLangs{}  & \corpusOneOhOneLangTrainShards{} & \corpusOneOhOneLangTrainGB{} & \corpusOneOhOneLangTrainDocsM{}M & \corpusOneOhOneLangEvalShards{} & \corpusOneOhOneLangEvalDocsM{}M \\
\bottomrule
\end{tabular}
\end{table}

The \texttt{6-lang} set comprises \{en, zh, ko, ja, hi, th\}, chosen as a compact mixed-script stress set with Latin text (en), CJK scripts for zh/ko/ja, and Indic/Thai scripts for hi/th. It does not include a training language written in a 2-byte \utfEightName{} script such as Arabic, Hebrew, Greek, or Cyrillic. The half-data ablation in Appendix~\ref{app:halfdata-ablation} uses approximately half of the \corpusOneOhOneLangTrainGB{}\,GiB \texttt{101-lang} corpus ($\sim$\corpusOneOhOneLangHalfTrainGB{}\,GiB), obtained by per-language uniform shard downsampling at fraction $0.5$.

\paragraph{Tokenizer-dependent token counts.}
Token counts in Table~\ref{tab:lm-token-accounting} are estimated by one-shard-per-language byte-weighted sampling (middle interior train shard only, excluding the last potentially smaller shard). Evaluation shards are fully enumerated. Both tokenizers sample the same shard per language (fixed seed) for fair comparison. The table reports the 32K primary attribution scale and folds train-token totals, evaluation-token totals, steps per epoch, and 200K-step coverage into one accounting table. The 64K LM training runs cover all three coverage configs and the 128K LM training runs cover \texttt{101-lang} only (Tables~\ref{tab:lm-eval-64k} and~\ref{tab:lm-eval-128k}); they share the same 101-language mC4 release, the same one-shard-per-language sampling, and the same fixed 200K-step budget. Their per-language token-per-byte rates decrease with vocabulary size as more frequent patterns are absorbed into multi-byte merges (the \texttt{All avg} column of Table~\ref{tab:gini-vs-tokens} illustrates the same intrinsic scaling pattern under \texttt{101-lang}).

All six tokenizers are applied to the same 101-language mC4 corpus. \texttt{En-only} tokenizers inflate non-English text to more tokens because low-coverage spans remain closer to raw-byte decomposition, so the same raw bytes yield more training tokens per epoch. The 200K fixed-step budget keeps all runs under one epoch. The \texttt{101-lang} comparison is therefore the closest tokenizer-isolation setting, whereas \texttt{En-only} and \texttt{6-lang} should be read as fixed-compute stress tests that intentionally include raw-text exposure differences.

\section{Intrinsic examples and checks}
\label{app:intrinsic-examples-checks}

\subsection{Representative language examples}
\label{app:coverage-examples}

Table~\ref{tab:coverage-examples} lists representative languages outside the 6-language tokenizer training set where BBPE remains close to the encoding floor. Unlike the core \texttt{101-lang/32K} examples in the main text, this narrower \texttt{6-lang/32K} setup is intentionally a stress regime, so the same scripts show larger token savings when coverage is sparse. Table~\ref{tab:coverage-examples-vocab} extends the same six languages to the 16K and 64K \texttt{6-lang} vocabularies; the \texttt{6-lang} setting is not trained at 128K--256K (only \texttt{101-lang} is extended that far, see Appendix~\ref{app:config-ablation}). Table~\ref{tab:rep-language-vocab} extends the main-text 101-lang representative comparison (Table~\ref{tab:rep-language-main}) across all five vocabulary sizes (16K--256K).

\begin{table}[H]
\caption{Per-language UBE impact on \floresTwoHundredName{} in the \texttt{6-lang/32K} stress setting (mean tokens per sentence). The listed non-English languages are outside tokenizer training; UBE cuts token counts by \covExStressSavingMinPct{}--\covExStressSavingMaxPct{}\%, equivalent to \covExStressContentGainMinPct{}--\covExStressContentGainMaxPct{}\% more fixed-budget content. The 16K and 64K results for the same languages appear in Table~\ref{tab:coverage-examples-vocab}.}
\label{tab:coverage-examples}
\centering
\small
\setlength{\tabcolsep}{3pt}
\begin{tabular}{llrrrrl}
\toprule
Language & Region & Speakers & UBE & BBPE & $\Delta$\% & Script \\
\midrule
Amharic  & Ethiopia & $\sim$57M & \covExAmharicUbe{} & \covExAmharicBbpe{} & \covExAmharicDeltaPct{} & Ethiopic \\
Lao      & Laos     & $\sim$30M & \covExLaoUbe{} & \covExLaoBbpe{} & \covExLaoDeltaPct{} & Lao \\
Santali  & India    & $\sim$8M  & \covExOlChikiUbe{} & \covExOlChikiBbpe{} & \covExOlChikiDeltaPct{} & Ol~Chiki \\
Telugu   & India    & $\sim$82M & \covExTeluguUbe{} & \covExTeluguBbpe{} & \covExTeluguDeltaPct{} & Telugu \\
Kannada  & India    & $\sim$44M & \covExKannadaUbe{} & \covExKannadaBbpe{} & \covExKannadaDeltaPct{} & Kannada \\
Tibetan  & China/Nepal & $\sim$6M & \covExTibetanUbe{} & \covExTibetanBbpe{} & \covExTibetanDeltaPct{} & Tibetan \\
\midrule
English  & (ref.)   & --- & \covExEnUbe{} & \covExEnBbpe{} & \covExEnDeltaPct{} & Latin \\
\bottomrule
\end{tabular}
\vspace{2pt}
\begin{minipage}{0.95\linewidth}
\raggedright\scriptsize
\emph{Speaker-count note.} Approximate Ethnologue 27th edition figures may include macrolanguage clusters and L2 speakers; the Lao count includes closely related Isan varieties, and the Tibetan count covers the Tibetic language cluster rather than Standard Tibetan alone~\citep{ethnologue2024}.
\end{minipage}
\end{table}

\begin{table}[ht]
\caption{Per-language UBE impact on \floresTwoHundredName{} across the \texttt{6-lang} vocabulary sweep (16K/32K/64K; mean tokens per sentence). The same six representative non-English languages as Table~\ref{tab:coverage-examples} are shown. $\Delta$\% is computed as $(\mathrm{UBE} - \mathrm{BBPE}) / \mathrm{BBPE}$.}
\label{tab:coverage-examples-vocab}
\centering
\small
\setlength{\tabcolsep}{3pt}
\begin{tabular}{llrrrrrrrrr}
\toprule
\multirow{2}{*}{Language} & \multirow{2}{*}{Script} & \multicolumn{3}{c}{16K} & \multicolumn{3}{c}{32K} & \multicolumn{3}{c}{64K} \\
\cmidrule(lr){3-5} \cmidrule(lr){6-8} \cmidrule(lr){9-11}
 & & UBE & BBPE & $\Delta$\% & UBE & BBPE & $\Delta$\% & UBE & BBPE & $\Delta$\% \\
\midrule
Amharic & Ethiopic & \covExOneSixkAmharicUbe{} & \covExOneSixkAmharicBbpe{} & \covExOneSixkAmharicDeltaPct{} & \covExThreeTwokAmharicUbe{} & \covExThreeTwokAmharicBbpe{} & \covExThreeTwokAmharicDeltaPct{} & \covExSixFourkAmharicUbe{} & \covExSixFourkAmharicBbpe{} & \covExSixFourkAmharicDeltaPct{} \\
Lao & Lao & \covExOneSixkLaoUbe{} & \covExOneSixkLaoBbpe{} & \covExOneSixkLaoDeltaPct{} & \covExThreeTwokLaoUbe{} & \covExThreeTwokLaoBbpe{} & \covExThreeTwokLaoDeltaPct{} & \covExSixFourkLaoUbe{} & \covExSixFourkLaoBbpe{} & \covExSixFourkLaoDeltaPct{} \\
Santali & Ol~Chiki & \covExOneSixkOlChikiUbe{} & \covExOneSixkOlChikiBbpe{} & \covExOneSixkOlChikiDeltaPct{} & \covExThreeTwokOlChikiUbe{} & \covExThreeTwokOlChikiBbpe{} & \covExThreeTwokOlChikiDeltaPct{} & \covExSixFourkOlChikiUbe{} & \covExSixFourkOlChikiBbpe{} & \covExSixFourkOlChikiDeltaPct{} \\
Telugu & Telugu & \covExOneSixkTeluguUbe{} & \covExOneSixkTeluguBbpe{} & \covExOneSixkTeluguDeltaPct{} & \covExThreeTwokTeluguUbe{} & \covExThreeTwokTeluguBbpe{} & \covExThreeTwokTeluguDeltaPct{} & \covExSixFourkTeluguUbe{} & \covExSixFourkTeluguBbpe{} & \covExSixFourkTeluguDeltaPct{} \\
Kannada & Kannada & \covExOneSixkKannadaUbe{} & \covExOneSixkKannadaBbpe{} & \covExOneSixkKannadaDeltaPct{} & \covExThreeTwokKannadaUbe{} & \covExThreeTwokKannadaBbpe{} & \covExThreeTwokKannadaDeltaPct{} & \covExSixFourkKannadaUbe{} & \covExSixFourkKannadaBbpe{} & \covExSixFourkKannadaDeltaPct{} \\
Tibetan & Tibetan & \covExOneSixkTibetanUbe{} & \covExOneSixkTibetanBbpe{} & \covExOneSixkTibetanDeltaPct{} & \covExThreeTwokTibetanUbe{} & \covExThreeTwokTibetanBbpe{} & \covExThreeTwokTibetanDeltaPct{} & \covExSixFourkTibetanUbe{} & \covExSixFourkTibetanBbpe{} & \covExSixFourkTibetanDeltaPct{} \\
\midrule
English & Latin & \covExOneSixkEnUbe{} & \covExOneSixkEnBbpe{} & \covExOneSixkEnDeltaPct{} & \covExThreeTwokEnUbe{} & \covExThreeTwokEnBbpe{} & \covExThreeTwokEnDeltaPct{} & \covExSixFourkEnUbe{} & \covExSixFourkEnBbpe{} & \covExSixFourkEnDeltaPct{} \\
\bottomrule
\end{tabular}%
\end{table}

\begin{table}[ht]
\centering
\small
\caption{Representative per-language UBE impact on \floresTwoHundredName{} across \texttt{101-lang} vocabulary sizes (average tokens per sentence). Extends Table~\ref{tab:rep-language-main}.}
\label{tab:rep-language-vocab}
\begin{tabular}{lllrrr}
\toprule
Language & Script & Vocab & UBE & BBPE & $\Delta$\% \\
\midrule
Santali & Ol~Chiki & 16K & 237.326 & 341.312 & -30.467\% \\
 &  & 32K & 222.944 & 341.287 & -34.676\% \\
 &  & 64K & 212.692 & 341.271 & -37.677\% \\
 &  & 128K & 211.712 & 251.125 & -15.695\% \\
 &  & 256K & 207.379 & 228.583 & -9.276\% \\
\addlinespace[2pt]
Central Atlas Tamazight & Tifinagh & 16K & 192.112 & 269.200 & -28.636\% \\
 &  & 32K & 183.305 & 269.166 & -31.899\% \\
 &  & 64K & 182.243 & 269.123 & -32.282\% \\
 &  & 128K & 180.664 & 269.102 & -32.864\% \\
 &  & 256K & 180.648 & 203.206 & -11.101\% \\
\addlinespace[2pt]
Tibetan & Tibetan & 16K & 287.065 & 432.437 & -33.617\% \\
 &  & 32K & 286.737 & 336.804 & -14.865\% \\
 &  & 64K & 247.171 & 264.208 & -6.449\% \\
 &  & 128K & 230.166 & 214.252 & +7.428\% \\
 &  & 256K & 163.501 & 161.215 & +1.418\% \\
\addlinespace[2pt]
Arabic & Arabic & 16K & 68.952 & 69.872 & -1.317\% \\
 &  & 32K & 60.407 & 60.600 & -0.318\% \\
 &  & 64K & 52.653 & 52.753 & -0.189\% \\
 &  & 128K & 46.245 & 46.296 & -0.111\% \\
 &  & 256K & 40.890 & 40.905 & -0.036\% \\
\addlinespace[2pt]
English & Latin & 16K & 34.671 & 34.863 & -0.550\% \\
 &  & 32K & 31.068 & 31.160 & -0.295\% \\
 &  & 64K & 28.695 & 28.716 & -0.076\% \\
 &  & 128K & 27.158 & 27.174 & -0.058\% \\
 &  & 256K & 26.271 & 26.272 & -0.004\% \\
\addlinespace[2pt]
\bottomrule
\end{tabular}
\end{table}

\subsection{Naturally occurring mixed-script documents}
\label{app:mixed-script}

The main-text mixed-script claim is measured on \msNDocs{} naturally occurring mC4 documents drawn uniformly from \msNLangs{} languages (\texttt{ja}, \texttt{ko}, \texttt{zh}, \texttt{hi}, \texttt{th}, \texttt{ar}, \texttt{ru}, \texttt{bn}; 50 documents each), tokenized under the \texttt{6-lang/32K} tokenizer family. Relative to BBPE, UBE matches or slightly improves token counts in \msUbeBetterOrTie{} languages; Bengali is the only language that regresses (\msBengaliDeltaPct{}). Aggregated over all \msNDocs{} documents, UBE changes total tokens by \msUbeTotalDeltaPct{} relative to BBPE, while \bbpeSixteenName{} is \msBbpeSixteenTotalDeltaPct{} longer because Latin runs in the same documents inherit the uniform 2-byte floor. We report this slice at the matched \texttt{6-lang/32K} setting because it uses the same compact tokenizer as the 6-language stress examples in Table~\ref{tab:coverage-examples}; the qualitative \bbpeSixteenName{}/UBE Latin-floor pattern is consistent across vocabulary sizes (cf.\ \texttt{EN avg} and \texttt{Latin avg} columns of Table~\ref{tab:gini-vs-tokens}, and per-script-family rows of Table~\ref{tab:script-family} from 16K to 256K). The per-document inputs are selected deterministically by our sampling script from the same public mC4 subset used for tokenizer training, so the selection depends only on public inputs.

\subsection{Contamination checks}
\label{app:contamination}

Following prior overlap-based contamination audits such as \citet{arnett2025explaining}, we compare the mC4 scan corpus (\contamDocsM{} documents across \contamLangsChecked{} languages) against the evaluation sentences of each benchmark using two surface-form checks: exact document--sentence equality and a \contamPrefixChars{}-character prefix-substring scan. The \contamDocsM{} figure is the number of mC4 documents loaded by this scan, not the full tokenizer-training document estimate in Table~\ref{tab:corpus-accounting}. Table~\ref{tab:contamination-audit} reports the overlap counts.

\begin{table}[ht]
\centering
\small
\caption{Surface-form contamination audit against the mC4 scan corpus. Prefix overlaps use the \contamPrefixChars{}-character prefix-substring check described in the text.}
\label{tab:contamination-audit}
\begin{tabular}{lrrr}
\toprule
Evaluation suite & Eval sentences & Exact overlaps & Prefix overlaps \\
\midrule
FLORES-200 & \contamFloresSentences{} & \contamFloresOverlaps{} & \contamFloresSubstringOverlaps{} \\
UDHR       & \contamUdhrSentences{}   & \contamUdhrOverlaps{}   & \contamUdhrSubstringOverlaps{}   \\
SIB-200    & \contamSibSentences{}    & \contamSibOverlaps{}    & \contamSibSubstringOverlaps{}    \\
MGSM       & \contamMgsmSentences{}   & \contamMgsmOverlaps{}   & \contamMgsmSubstringOverlaps{}   \\
\bottomrule
\end{tabular}
\end{table}

The prefix-substring scan covers only sentences with at least \contamPrefixChars{} characters: \contamFloresSubstringEligibleSentences{} FLORES-200, \contamUdhrSubstringEligibleSentences{} UDHR, \contamSibSubstringEligibleSentences{} SIB-200, and \contamMgsmSubstringEligibleSentences{} MGSM sentences. The exact-match scan found no FLORES-200, SIB-200, or MGSM overlaps, and one UDHR overlap. Prefix overlaps are rare for FLORES-200 and SIB-200 and absent for MGSM, while UDHR has many prefix overlaps, consistent with its status as a short, widely reproduced public UN text. We therefore treat the UDHR result as a robustness stress test rather than as evidence from a fully held-out corpus.

\section{Intrinsic scaling and robustness diagnostics}
\label{app:intrinsic-scaling-diagnostics}

\subsection{Training config ablation}
\label{app:config-ablation}

\begin{table}[ht]
\caption{FLORES-200 language-level outcome summary across configurations (204 languages). Columns show languages where UBE produces fewer tokens (wins), ties, and more tokens (losses), plus the corresponding Gini change. The 128K and 256K rows are only run for \texttt{101-lang}.}
\label{tab:config-ablation-app}
\centering
\small
\begin{tabular}{llcr}
\toprule
Config & Vocab & UBE wins/ties/losses & $\Delta$Gini \\
\midrule
\texttt{En-only} & 16K & \eovOneSixkBetter{}\,/\,\eovOneSixkNeutral{}\,/\,\eovOneSixkWorse{} & \eovOneSixkGiniPct{} \\
\texttt{En-only} & 32K & \eovThreeTwokBetter{}\,/\,\eovThreeTwokNeutral{}\,/\,\eovThreeTwokWorse{} & \eovThreeTwokGiniPct{} \\
\texttt{En-only} & 64K & \eovSixFourkBetter{}\,/\,\eovSixFourkNeutral{}\,/\,\eovSixFourkWorse{} & \eovSixFourkGiniPct{} \\
\midrule
\texttt{6-lang} & 16K & \fullvOneSixkBetter{}\,/\,\fullvOneSixkNeutral{}\,/\,\fullvOneSixkWorse{} & \fullvOneSixkGiniPct{} \\
\texttt{6-lang} & 32K & \fullvThreeTwokBetter{}\,/\,\fullvThreeTwokNeutral{}\,/\,\fullvThreeTwokWorse{} & \fullvThreeTwokGiniPct{} \\
\texttt{6-lang} & 64K & \fullvSixFourkBetter{}\,/\,\fullvSixFourkNeutral{}\,/\,\fullvSixFourkWorse{} & \fullvSixFourkGiniPct{} \\
\midrule
\texttt{101-lang} & 16K & \oneOhOnevOneSixkBetter{}\,/\,\oneOhOnevOneSixkNeutral{}\,/\,\oneOhOnevOneSixkWorse{} & \oneOhOnevOneSixkGiniPct{} \\
\texttt{101-lang} & 32K & \oneOhOnevThreeTwokBetter{}\,/\,\oneOhOnevThreeTwokNeutral{}\,/\,\oneOhOnevThreeTwokWorse{} & \oneOhOnevThreeTwokGiniPct{} \\
\texttt{101-lang} & 64K & \oneOhOnevSixFourkBetter{}\,/\,\oneOhOnevSixFourkNeutral{}\,/\,\oneOhOnevSixFourkWorse{} & \oneOhOnevSixFourkGiniPct{} \\
\texttt{101-lang} & 128K & \oneOhOnevOneTwoEightkBetter{}\,/\,\oneOhOnevOneTwoEightkNeutral{}\,/\,\oneOhOnevOneTwoEightkWorse{} & \oneOhOnevOneTwoEightkGiniPct{} \\
\texttt{101-lang} & 256K & \oneOhOnevTwoFiveSixkBetter{}\,/\,\oneOhOnevTwoFiveSixkNeutral{}\,/\,\oneOhOnevTwoFiveSixkWorse{} & \oneOhOnevTwoFiveSixkGiniPct{} \\
\bottomrule
\end{tabular}
\end{table}

Across the 9 core tokenizer--vocabulary configurations (16K--64K), UBE reduces token-premium variance by \paperStatVarReductionMinPct{}--\paperStatVarReductionMaxPct{} and Gini coefficient by \paperStatGiniReductionMinPct{}--\paperStatGiniReductionMaxPct{}. For the primary dispersion statistic, a paired language bootstrap over non-English token premiums (10K resamples) places $\Delta$Gini (UBE $-$ BBPE) below zero in every core setting; representative 95\% CIs are [\enonlyvOneSixZeroZeroZeroGiniCIlo{}, \enonlyvOneSixZeroZeroZeroGiniCIhi{}] for \texttt{En-only/16K}, [\OneZeroOnelangvThreeTwoZeroZeroZeroGiniCIlo{}, \OneZeroOnelangvThreeTwoZeroZeroZeroGiniCIhi{}] for the primary \texttt{101-lang/32K} setting, and [\enzhkojahithvSixFourZeroZeroZeroGiniCIlo{}, \enzhkojahithvSixFourZeroZeroZeroGiniCIhi{}] for \texttt{6-lang/64K}. Separately, Wilcoxon signed-rank tests over language-level mean token counts are significant in \paperStatWSigCount{}/9 comparisons at $p < 0.01$; the same \paperStatWSigCount{}/9 token-count decision is unchanged under Holm correction at $\alpha=0.05$ and Benjamini--Hochberg false discovery rate (FDR) correction at $q=0.05$ over the nine prespecified core comparisons. The lone token-count exception is the near-tie \texttt{En-only/64K} setting (max $p = \paperStatWMaxP{}$). Theil-$T$ and Theil-$L$ also decrease in all nine settings, so the disparity ordering is not specific to Gini. The largest qualitative examples are not the typical case: outside the most floor-bound scripts, many per-language token-count differences are small, while the aggregate distribution still shifts consistently toward lower disparity.

Only the \texttt{101-lang} setting was extended to 256K, so the corresponding standalone and composition sweeps are collected in Tables~\ref{tab:gini-vs-tokens} (Appendix~\ref{app:gini-vs-tokens}) and~\ref{tab:gini-vs-tokens-superbpe} (Appendix~\ref{app:superbpe-composition}).

\subsection{Gini coefficient and absolute token counts}
\label{app:gini-vs-tokens}

We report two fixed diagnostic groups of 3-byte BMP configurations alongside overall results for all 204 configurations. CJK (5) comprises the Chinese (Simplified and Traditional), Cantonese, Japanese, and Korean configurations; Other BMP (7) comprises the Tibetan, Myanmar, Ol~Chiki, and Tifinagh configurations (Dzongkha, Tibetan, Burmese, Shan, Santali, Tamasheq, and Central Atlas Tamazight). In both groups, a core character costs three UTF-8 bytes, which UBE lowers to two base symbols. The groups differ in the other factor of a byte premium, the number of characters needed for the same content \citep{arnett2024bitproblem}: relative to English, parallel FLORES-200 text uses 0.31--0.51 times as many characters in CJK (5) but 0.89--1.42 times as many in Other BMP (7), and Other BMP (7) holds seven of the eight highest token premiums under GPT-4's cl100k\_base \citep{petrov2023unfairness}. Training coverage differs as well: the Chinese, Japanese, and Korean partitions each hold about 30\,GiB of the primary corpus, whereas Other BMP has only the 1\,GiB Burmese partition and no dedicated partition for its other six configurations. The split therefore separates scripts that BPE largely covers with learned character and multi-character tokens from scripts that remain near the encoding floor. Both groups are unchanged across vocabulary sizes and composition settings.

\begin{table}[ht]
\caption{FLORES-200 Gini for UBE and BBPE under the primary \texttt{101-lang} training regime. Table~\ref{tab:gini-vs-tokens} gives the token-count breakdown.}
\label{tab:headline-gini-app}
\centering
\small
\begin{tabular}{lcccc}
\toprule
Vocab & UBE Gini & BBPE Gini & $\Delta$Gini & UBE wins \\
\midrule
16K  & \oneOhOnevOneSixkGiniUBE{}  & \oneOhOnevOneSixkGiniBBPE{}  & \oneOhOnevOneSixkGiniPct{}  & \oneOhOnevOneSixkBetter{}/204 \\
32K  & \oneOhOnevThreeTwokGiniUBE{}  & \oneOhOnevThreeTwokGiniBBPE{}  & \oneOhOnevThreeTwokGiniPct{}  & \oneOhOnevThreeTwokBetter{}/204 \\
64K  & \oneOhOnevSixFourkGiniUBE{}  & \oneOhOnevSixFourkGiniBBPE{}  & \oneOhOnevSixFourkGiniPct{}  & \oneOhOnevSixFourkBetter{}/204 \\
128K & \oneOhOnevOneTwoEightkGiniUBE{} & \oneOhOnevOneTwoEightkGiniBBPE{} & \oneOhOnevOneTwoEightkGiniPct{} & \oneOhOnevOneTwoEightkBetter{}/204 \\
256K & \oneOhOnevTwoFiveSixkGiniUBE{} & \oneOhOnevTwoFiveSixkGiniBBPE{} & \oneOhOnevTwoFiveSixkGiniPct{} & \oneOhOnevTwoFiveSixkBetter{}/204 \\
\bottomrule
\end{tabular}
\end{table}

\begin{table}[ht]
\centering
\small
\caption{Gini coefficient and groupwise mean tokens per sentence on FLORES-200 (\texttt{101-lang}; 16K--256K). Latin avg covers Latin-script languages; CJK (5) and Other BMP (7) are the configuration groups defined in Appendix~\ref{app:gini-vs-tokens}. $^\dagger$\bbpeSixteenName{} lowers Gini via a uniform 2-byte floor that raises Latin token counts; SuperBPE compositions are reported separately.}
\label{tab:gini-vs-tokens}
\begin{tabular}{lrrrrrl}
\toprule
Tokenizer & EN avg & Latin avg & \shortstack{CJK\\(5)} & \shortstack{Other BMP\\(7)} & All avg & Gini \\
\midrule
UBE 16K             & \gvaUbeOneSixkEnDetailed{} & \gvaUbeOneSixkLatinDetailed{} & \gvaUbeOneSixkCjkDetailed{} & \gvaUbeOneSixkOtherBmpDetailed{} & \gvaUbeOneSixkAllDetailed{} & \gvaUbeOneSixkGiniDetailed{} \\
UBE 32K             & \gvaUbeThreeTwokEnDetailed{} & \gvaUbeThreeTwokLatinDetailed{} & \gvaUbeThreeTwokCjkDetailed{} & \gvaUbeThreeTwokOtherBmpDetailed{} & \gvaUbeThreeTwokAllDetailed{} & \gvaUbeThreeTwokGiniDetailed{} \\
UBE 64K             & \gvaUbeEnDetailed{} & \gvaUbeLatinDetailed{} & \gvaUbeCjkDetailed{} & \gvaUbeOtherBmpDetailed{} & \gvaUbeAllDetailed{} & \gvaUbeGiniDetailed{} \\
UBE 128K            & \gvaUbeOneTwoEightkEnDetailed{} & \gvaUbeOneTwoEightkLatinDetailed{} & \gvaUbeOneTwoEightkCjkDetailed{} & \gvaUbeOneTwoEightkOtherBmpDetailed{} & \gvaUbeOneTwoEightkAllDetailed{} & \gvaUbeOneTwoEightkGiniDetailed{} \\
UBE 256K            & \gvaUbeTwoFiveSixkEnDetailed{} & \gvaUbeTwoFiveSixkLatinDetailed{} & \gvaUbeTwoFiveSixkCjkDetailed{} & \gvaUbeTwoFiveSixkOtherBmpDetailed{} & \gvaUbeTwoFiveSixkAllDetailed{} & \gvaUbeTwoFiveSixkGiniDetailed{} \\
\midrule
BBPE 16K            & \gvaBbpeOneSixkEnDetailed{} & \gvaBbpeOneSixkLatinDetailed{} & \gvaBbpeOneSixkCjkDetailed{} & \gvaBbpeOneSixkOtherBmpDetailed{} & \gvaBbpeOneSixkAllDetailed{} & \gvaBbpeOneSixkGiniDetailed{} \\
BBPE 32K            & \gvaBbpeThreeTwokEnDetailed{} & \gvaBbpeThreeTwokLatinDetailed{} & \gvaBbpeThreeTwokCjkDetailed{} & \gvaBbpeThreeTwokOtherBmpDetailed{} & \gvaBbpeThreeTwokAllDetailed{} & \gvaBbpeThreeTwokGiniDetailed{} \\
BBPE 64K            & \gvaBbpeEnDetailed{} & \gvaBbpeLatinDetailed{} & \gvaBbpeCjkDetailed{} & \gvaBbpeOtherBmpDetailed{} & \gvaBbpeAllDetailed{} & \gvaBbpeGiniDetailed{} \\
BBPE 128K           & \gvaBbpeOneTwoEightkEnDetailed{} & \gvaBbpeOneTwoEightkLatinDetailed{} & \gvaBbpeOneTwoEightkCjkDetailed{} & \gvaBbpeOneTwoEightkOtherBmpDetailed{} & \gvaBbpeOneTwoEightkAllDetailed{} & \gvaBbpeOneTwoEightkGiniDetailed{} \\
BBPE 256K           & \gvaBbpeTwoFiveSixkEnDetailed{} & \gvaBbpeTwoFiveSixkLatinDetailed{} & \gvaBbpeTwoFiveSixkCjkDetailed{} & \gvaBbpeTwoFiveSixkOtherBmpDetailed{} & \gvaBbpeTwoFiveSixkAllDetailed{} & \gvaBbpeTwoFiveSixkGiniDetailed{} \\
\midrule
\bbpeSixteenName{} 16K  & \gvaBbpeSixteenOneSixkEnDetailed{} & \gvaBbpeSixteenOneSixkLatinDetailed{} & \gvaBbpeSixteenOneSixkCjkDetailed{} & \gvaBbpeSixteenOneSixkOtherBmpDetailed{} & \gvaBbpeSixteenOneSixkAllDetailed{} & \gvaBbpeSixteenOneSixkGiniDetailed{}$^\dagger$ \\
\bbpeSixteenName{} 32K  & \gvaBbpeSixteenThreeTwokEnDetailed{} & \gvaBbpeSixteenThreeTwokLatinDetailed{} & \gvaBbpeSixteenThreeTwokCjkDetailed{} & \gvaBbpeSixteenThreeTwokOtherBmpDetailed{} & \gvaBbpeSixteenThreeTwokAllDetailed{} & \gvaBbpeSixteenThreeTwokGiniDetailed{}$^\dagger$ \\
\bbpeSixteenName{} 64K  & \gvaBbpeSixteenSixFourkEnDetailed{} & \gvaBbpeSixteenSixFourkLatinDetailed{} & \gvaBbpeSixteenSixFourkCjkDetailed{} & \gvaBbpeSixteenSixFourkOtherBmpDetailed{} & \gvaBbpeSixteenSixFourkAllDetailed{} & \gvaBbpeSixteenSixFourkGiniDetailed{}$^\dagger$ \\
\bbpeSixteenName{} 128K & \gvaBbpeSixteenEnDetailed{} & \gvaBbpeSixteenLatinDetailed{} & \gvaBbpeSixteenCjkDetailed{} & \gvaBbpeSixteenOtherBmpDetailed{} & \gvaBbpeSixteenAllDetailed{} & \gvaBbpeSixteenGiniDetailed{}$^\dagger$ \\
\bbpeSixteenName{} 256K & \gvaBbpeSixteenTwoFiveSixkEnDetailed{} & \gvaBbpeSixteenTwoFiveSixkLatinDetailed{} & \gvaBbpeSixteenTwoFiveSixkCjkDetailed{} & \gvaBbpeSixteenTwoFiveSixkOtherBmpDetailed{} & \gvaBbpeSixteenTwoFiveSixkAllDetailed{} & \gvaBbpeSixteenTwoFiveSixkGiniDetailed{}$^\dagger$ \\
\bottomrule
\end{tabular}
\end{table}

\subsubsection{Why Gini rises with vocabulary size}
\label{app:gini-rises}

Figure~\ref{fig:premium-scaling} visualizes the mechanism behind the Gini increase observed in Table~\ref{tab:gini-vs-tokens}. Script families compress at vastly different rates. Across the \gsUbeVocabRange{} vocabulary sweep, English token counts fall by \gsUbeEnRedPct{}, whereas Hindi (Devanagari) token counts fall by only \gsUbeHiRedPct{}. Because premiums are ratios to English, Hindi's premium \emph{rises} from \gsUbeHiPremMin{} to \gsUbeHiPremMax{} (\gsUbeHiPremDelta{}) even though its absolute count decreases slightly. Meanwhile, languages that can still form long merges (Arabic: \gsUbeArRedPct{} reduction, premium \gsUbeArPremDelta{}; Korean: \gsUbeKoRedPct{}, \gsUbeKoPremDelta{}) converge toward English. The net effect is a widening premium distribution---\gsUbePremUp{} of \gsUbePremTotal{} languages see their premium increase---that drives Gini upward. This pattern is expected for BPE on a multilingual corpus when merge productivity is script-dependent: some script families absorb more productive merges per added vocabulary slot than others. UBE lowers the \emph{intercept} of the premium curve but does not remove this vocabulary-scaling effect.

\begin{figure}[ht]
\centering
\input{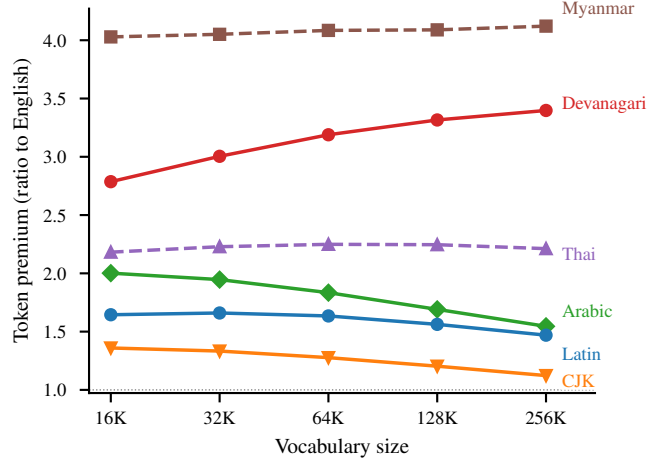}
\caption{Token premium (ratio to English) vs.\ vocabulary size for six script families under UBE (\texttt{101-lang}).}
\label{fig:premium-scaling}
\end{figure}

\subsubsection{First-order token-cost intuition}
\label{app:token-cost-model}

Complementing the empirical premium-scaling decomposition in Appendix~\ref{app:gini-rises}, a simple first-order cost sketch provides intuition for (i) the dataset ordering in Table~\ref{tab:cross-dataset-vocab} (UDHR $>$ FLORES-200 $>$ MGSM) and (ii) the monotonic attenuation of UBE's Gini reduction---that is, $\Delta$Gini approaching zero---as vocabulary grows (Table~\ref{tab:headline-gini-app}). We present it as an \emph{intuition}, not as a quantitative prediction: UBE changes the merge-slot budget and the observed vocabulary utilization (Appendix~\ref{app:vocab-diagnostics}), so the constant-coverage derivation below is a first-order sketch rather than a parameter-free forecast.

\paragraph{Setup.}
For a script $s$ and substrate $E \in \{\text{UBE}, \text{BBPE}\}$, let $f_s^{(E)}$ denote the corresponding main-text encoding-floor value. Under BBPE, $f_s^{(\text{BBPE})} = |\mathrm{utf8}(c)|$ for $c$ drawn from $s$, taking values in $\{1, 2, 3, 4\}$. Under UBE, the routing rule sets $f_s^{(\text{UBE})} = f_s^{(\text{BBPE})}$ for $f_s \leq 2$, $f_s^{(\text{UBE})} = 2$ for $f_s = 3$, and $f_s^{(\text{UBE})} = 4$ for $f_s = 4$. Let $c_s^{(E)} \in [0,1]$ denote the \emph{character-level merge coverage} under substrate $E$: the probability that a random character from script $s$ is absorbed by a learned BPE token whose per-character cost is below $f_s^{(E)}$. Let $m_s^{(E)}$ denote the expected token count per character among covered characters.

\paragraph{Model.}
The expected token cost per character under substrate $E$ is
\[
T_s^{(E)} \;\approx\; c_s^{(E)}\, m_s^{(E)} \;+\; \bigl(1 - c_s^{(E)}\bigr)\, f_s^{(E)}.
\]
This is a first-order approximation that treats coverage as a scalar and ignores correlations between adjacent characters.

\paragraph{UBE effect under a constant-coverage assumption.}
UBE's alphabet, merge-slot budget, and merge-utilization structure all differ from BBPE's (Appendix~\ref{app:vocab-diagnostics}), so $c_s^{(E)}$ and $m_s^{(E)}$ are genuinely substrate-dependent. If we simplify by assuming $c_s^{(\text{UBE})} \approx c_s^{(\text{BBPE})} =: c_s$ and $m_s^{(\text{UBE})} \approx m_s^{(\text{BBPE})} =: m_s$, the UBE--BBPE difference collapses, for BMP 3-byte scripts, to
\[
T_s^{(\text{UBE})} - T_s^{(\text{BBPE})} \;\approx\; (1 - c_s)\,\bigl(f_s^{(\text{UBE})} - f_s^{(\text{BBPE})}\bigr) \;=\; -(1 - c_s).
\]
For ASCII and 2-byte \utfEightName{} scripts, $f_s^{(\text{UBE})} = f_s^{(\text{BBPE})}$, so the first-order saving is zero. For BMP 3-byte scripts, by contrast, the floor change---and hence the first-order saving above---is scaled by the uncovered fraction $1 - c_s$.

\paragraph{Qualitative dataset ordering.}
Let $\bar u_D = \mathbb{E}_{s \sim D,\, f_s = 3}\bigl[1 - c_s^{(\text{BBPE})}\bigr]$ denote the mean uncovered fraction of 3-byte characters in dataset $D$ under BBPE\@. We do not measure $\bar u_D$; the following large, moderate, and near-zero labels follow from dataset coverage breadth alone:
\begin{itemize}
    \item UDHR (\udhrRobLangs{} corpus entries, many with no mC4 training data): $\bar u_D$ is expected to be large. Observed $\Delta$Gini at 32K: \udhrRobThreeTwokGiniPct{}.
    \item FLORES-200 (\floresRobLangs{} languages, mixed coverage): $\bar u_D$ is expected to be moderate. Observed: \floresRobThreeTwokGiniPct{}.
    \item MGSM (\mgsmRobLangs{} languages, all of which appear in the 101-language mC4 corpus): $\bar u_D$ is expected to be near zero. Observed: \mgsmRobThreeTwokGiniPct{}.
\end{itemize}
The observed ordering is consistent with this first-order sketch. We rank $\bar u_D$ from coverage breadth rather than measuring it directly; per-dataset measurement against the trained BBPE tokenizer is left to future work.

\paragraph{Decay with vocabulary size.}
Larger vocabularies increase $c_s^{(\text{BBPE})}$ for frequent patterns, shrinking $(1 - c_s^{(\text{BBPE})})$ and therefore the first-order UBE saving. This is consistent with $\Delta$Gini monotonically approaching zero from \oneOhOnevOneSixkGiniPct{} at 16K to \oneOhOnevTwoFiveSixkGiniPct{} at 256K (Table~\ref{tab:headline-gini-app}) and the per-script premium-scaling decomposition in Appendix~\ref{app:gini-rises}.

\paragraph{Beyond the floor-only saving.}
Under the constant-coverage assumption above, replacing a 3-byte floor with 2 bytes yields at most a $1/3$ premium reduction. Empirically, Santali at \texttt{101-lang/64K} has a UBE/BBPE token-count ratio below the corresponding $2/3$ floor-only limit (Table~\ref{tab:best-worst-flores}), exceeding this floor-only saving. UBE's \utfSixteenName{} path therefore enables BPE merges that absorb content beyond what BBPE's \utfEightName{} path captures, so the constant-coverage assumption underestimates the realized saving.

\subsection{SuperBPE composition results}
\label{app:superbpe-composition}

Complementing the standalone UBE/BBPE and \bbpeSixteenName{} results in Table~\ref{tab:gini-vs-tokens}, Table~\ref{tab:gini-vs-tokens-superbpe} reports the full S-UBE/S-BBPE composition sweep. The table adds two details to the main-text summary: among the SuperBPE composition rows, S-UBE has lower Gini than S-BBPE at every vocabulary size, while plain UBE keeps fewer total FLORES-200 tokens than S-UBE at 16K--64K; and the bounded, sign-changing interaction term supports a qualitative complementarity claim rather than exact additivity.

\paragraph{Implementation of S-UBE and S-BBPE.}
We implement S-UBE and S-BBPE as SuperBPE-style \emph{extension} runs, not as separately designed tokenizers. Each reported target vocabulary size $V \in \{16\text{K}, 32\text{K}, 64\text{K}, 128\text{K}, 256\text{K}\}$ starts from an already trained Stage~1 tokenizer of the same substrate at vocabulary size $V/2$ (UBE for S-UBE, BBPE for S-BBPE), loads its \texttt{merges.txt} as ordered seed merges, and then continues BPE training until the total vocabulary reaches $V$. The Stage~1 seed tokenizers are intermediate artifacts only; all reported S-UBE/S-BBPE result rows use the 16K--256K target vocabularies. This mirrors the two-stage structure of SuperBPE: Stage~1 learns an ordinary whitespace-splitting BPE tokenizer over the chosen substrate, and Stage~2 reuses those merges while relaxing the pre-tokenizer so that cross-word merge candidates can be selected.

The Stage~2 implementation is shared by the two composed variants. Our \texttt{train\_superbpe.py} script first regenerates shard-local count states using a permissive regex pre-tokenizer rather than the normal GPT-2 whitespace-splitting regex; these states cannot be reused from UBE/BBPE training because they contain multi-word chunks that ordinary pre-tokenization never exposes. Shard-local pre-tokenized chunks occurring fewer than five times are discarded before aggregation; new merges use a minimum frequency of two. It then merges the shard states, initializes the BPE trainer with the substrate-specific base alphabet and the Stage~1 seed merges, replays the seed merges in their original order, rebuilds the pair-count priority queue from the permissive-regex corpus state, and resumes the standard greedy BPE merge loop. The only substrate-specific branch is the byte representation used before BPE: S-BBPE uses the standard \utfEightName{} byte-level alphabet, pre-tokenizer, and decoder; S-UBE uses UBE's UniversalByteLevel alphabet, pre-tokenizer, and decoder, so routed \utfSixteenName{} PUA symbols participate in the same SuperBPE Stage~2 merge loop. In both cases we retain SuperBPE's colon-space block and four-word cap (at most four whitespace-separated words per merged token), which avoids unbounded multiword tokens; the reported tokenizer sidecars record the base type, seed-merge count, selected shard-state count, target vocabulary, and these SuperBPE constraints.

\begin{table}[ht]
\centering
\small
\caption{SuperBPE composition results on FLORES-200 (204 evaluation languages; tokenizers trained under the \texttt{101-lang} mC4 regime; 16K--256K vocab). Format matches Table~\ref{tab:gini-vs-tokens}.}
\label{tab:gini-vs-tokens-superbpe}
\begin{tabular}{lrrrrrl}
\toprule
Tokenizer & EN avg & Latin avg & \shortstack{CJK\\(5)} & \shortstack{Other BMP\\(7)} & All avg & Gini \\
\midrule
SuperBPE UBE 16K    & \gvaSbpeUbeOneSixkEnDetailed{} & \gvaSbpeUbeOneSixkLatinDetailed{} & \gvaSbpeUbeOneSixkCjkDetailed{} & \gvaSbpeUbeOneSixkOtherBmpDetailed{} & \gvaSbpeUbeOneSixkAllDetailed{} & \gvaSbpeUbeOneSixkGiniDetailed{} \\
SuperBPE UBE 32K    & \gvaSbpeUbeThreeTwokEnDetailed{} & \gvaSbpeUbeThreeTwokLatinDetailed{} & \gvaSbpeUbeThreeTwokCjkDetailed{} & \gvaSbpeUbeThreeTwokOtherBmpDetailed{} & \gvaSbpeUbeThreeTwokAllDetailed{} & \gvaSbpeUbeThreeTwokGiniDetailed{} \\
SuperBPE UBE 64K    & \gvaSbpeUbeEnDetailed{} & \gvaSbpeUbeLatinDetailed{} & \gvaSbpeUbeCjkDetailed{} & \gvaSbpeUbeOtherBmpDetailed{} & \gvaSbpeUbeAllDetailed{} & \gvaSbpeUbeGiniDetailed{} \\
SuperBPE UBE 128K   & \gvaSbpeUbeOneTwoEightkEnDetailed{} & \gvaSbpeUbeOneTwoEightkLatinDetailed{} & \gvaSbpeUbeOneTwoEightkCjkDetailed{} & \gvaSbpeUbeOneTwoEightkOtherBmpDetailed{} & \gvaSbpeUbeOneTwoEightkAllDetailed{} & \gvaSbpeUbeOneTwoEightkGiniDetailed{} \\
SuperBPE UBE 256K   & \gvaSbpeUbeTwoFiveSixkEnDetailed{} & \gvaSbpeUbeTwoFiveSixkLatinDetailed{} & \gvaSbpeUbeTwoFiveSixkCjkDetailed{} & \gvaSbpeUbeTwoFiveSixkOtherBmpDetailed{} & \gvaSbpeUbeTwoFiveSixkAllDetailed{} & \gvaSbpeUbeTwoFiveSixkGiniDetailed{} \\
\midrule
SuperBPE BBPE 16K   & \gvaSbpeBbpeOneSixkEnDetailed{} & \gvaSbpeBbpeOneSixkLatinDetailed{} & \gvaSbpeBbpeOneSixkCjkDetailed{} & \gvaSbpeBbpeOneSixkOtherBmpDetailed{} & \gvaSbpeBbpeOneSixkAllDetailed{} & \gvaSbpeBbpeOneSixkGiniDetailed{} \\
SuperBPE BBPE 32K   & \gvaSbpeBbpeThreeTwokEnDetailed{} & \gvaSbpeBbpeThreeTwokLatinDetailed{} & \gvaSbpeBbpeThreeTwokCjkDetailed{} & \gvaSbpeBbpeThreeTwokOtherBmpDetailed{} & \gvaSbpeBbpeThreeTwokAllDetailed{} & \gvaSbpeBbpeThreeTwokGiniDetailed{} \\
SuperBPE BBPE 64K   & \gvaSbpeBbpeEnDetailed{} & \gvaSbpeBbpeLatinDetailed{} & \gvaSbpeBbpeCjkDetailed{} & \gvaSbpeBbpeOtherBmpDetailed{} & \gvaSbpeBbpeAllDetailed{} & \gvaSbpeBbpeGiniDetailed{} \\
SuperBPE BBPE 128K  & \gvaSbpeBbpeOneTwoEightkEnDetailed{} & \gvaSbpeBbpeOneTwoEightkLatinDetailed{} & \gvaSbpeBbpeOneTwoEightkCjkDetailed{} & \gvaSbpeBbpeOneTwoEightkOtherBmpDetailed{} & \gvaSbpeBbpeOneTwoEightkAllDetailed{} & \gvaSbpeBbpeOneTwoEightkGiniDetailed{} \\
SuperBPE BBPE 256K  & \gvaSbpeBbpeTwoFiveSixkEnDetailed{} & \gvaSbpeBbpeTwoFiveSixkLatinDetailed{} & \gvaSbpeBbpeTwoFiveSixkCjkDetailed{} & \gvaSbpeBbpeTwoFiveSixkOtherBmpDetailed{} & \gvaSbpeBbpeTwoFiveSixkAllDetailed{} & \gvaSbpeBbpeTwoFiveSixkGiniDetailed{} \\
\bottomrule
\end{tabular}
\end{table}

\paragraph{Interaction-term stability.}
The interaction term of the $2 \times 2$ factorial, defined on signed relative-Gini changes as combined\% $-$ UBE\% $-$ SuperBPE\% (negative individual terms correspond to Gini reductions), is bounded but sign-changing across vocabulary sizes: \gvaInteractionOneSixk{}\% at 16K, \gvaInteractionThreeTwok{}\% at 32K, \gvaInteractionSixFourk{}\% at 64K, \gvaInteractionOneTwoEightk{}\% at 128K, and \gvaInteractionTwoFiveSixk{}\% at 256K (mean \gvaInteractionMean{}\%). A positive value indicates that the composition falls short of the summed individual reductions (interference); a negative value indicates that the composition exceeds the sum (mild synergy). This pattern supports a qualitative complementarity interpretation, not a uniform near-additive law: small-vocabulary alphabet overhead and larger-vocabulary cross-word merge opportunities affect the sign of the interaction. At 64K, a bootstrap over \gvaInteractionSixFourkBootN{} languages (10K resamples) places the interaction at \gvaInteractionSixFourk{}\% with a 95\% CI of [\gvaInteractionSixFourkCILow{}\%, \gvaInteractionSixFourkCIHigh{}\%], indicating that this negative interaction point is small but distinguishable from zero. We do not claim a uniform interaction sign across vocabulary sizes.

\subsection{Data-to-vocabulary ratio ablation}
\label{app:halfdata-ablation}

Figure~\ref{fig:halfdata-ablation} plots the Gini coefficient as a function of vocabulary size for full-data (\corpusOneOhOneLangTrainGB{}\,GiB) and half-data ($\sim$\corpusOneOhOneLangHalfTrainGB{}\,GiB) tokenizers across all four families. Full and half curves nearly overlap at every vocabulary size for all families, with the median $|\Delta\text{Gini}|$ across \hdNComparisons{} comparisons of just \hdMedianAbsDelta{} (max \hdMaxAbsDelta{} at S-BBPE 16K). This supports the interpretation that the overall Gini increase observed for UBE and BBPE is driven mainly by encoding-floor asymmetry and script-dependent merge productivity, rather than by insufficient tokenizer-training data at large vocabulary sizes.

\begin{figure}[ht]
\centering
\input{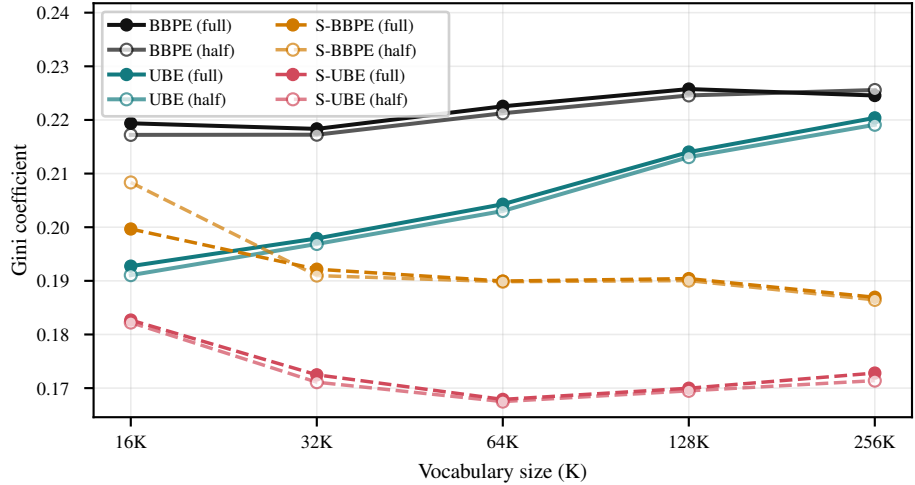}
\caption{Gini coefficient vs.\ vocabulary size for full-data (\corpusOneOhOneLangTrainGB{}\,GiB) and half-data ($\sim$\corpusOneOhOneLangHalfTrainGB{}\,GiB) tokenizer training (\texttt{101-lang}). Solid markers denote full data and open markers denote half data.}
\label{fig:halfdata-ablation}
\end{figure}

\begin{table}[t]
\centering
\small
\setlength{\tabcolsep}{5pt}
\caption{Best and worst 10 languages by UBE/BBPE token ratio on FLORES-200 (\texttt{101-lang}; UBE/BBPE: 64K, BBPE16: 128K). Ratio $<1$ means UBE uses fewer tokens; ratio $>1$ means UBE uses more. BBPE16 is included as a cross-substrate reference.}
\label{tab:best-worst-flores}
\begin{tabular}{lrrrr}
\toprule
Language & UBE avg & BBPE avg & BBPE16 (128K) avg & UBE/BBPE \\
\midrule
\multicolumn{5}{l}{\textit{Top 10: UBE helps most}} \\
\midrule
        Santali & 212.692 & 341.271 & 222.620 & 0.6232 \\
        Central Atlas Tamazight & 182.243 & 269.123 & 183.143 & 0.6772 \\
        Tamashek (Tifinagh) & 185.609 & 270.571 & 187.934 & 0.6860 \\
        Dzongkha & 265.685 & 286.182 & 267.341 & 0.9284 \\
        Tibetan & 247.171 & 264.208 & 247.411 & 0.9355 \\
        Kabyle & 62.457 & 62.967 & 81.116 & 0.9919 \\
        Sicilian & 52.945 & 53.291 & 75.957 & 0.9935 \\
        Lao & 87.657 & 88.180 & 83.831 & 0.9941 \\
        Armenian & 64.792 & 65.068 & 63.260 & 0.9957 \\
        Czech & 47.946 & 48.132 & 53.918 & 0.9961 \\
\midrule
\multicolumn{5}{l}{\textit{Bottom 10: UBE helps least (or regresses)}} \\
\midrule
        Japanese & 40.783 & 40.630 & 37.926 & 1.0037 \\
        Yoruba & 70.871 & 70.604 & 77.817 & 1.0038 \\
        Chinese (Simplified) & 35.015 & 34.764 & 33.126 & 1.0072 \\
        Chinese (Traditional) & 39.077 & 38.741 & 36.869 & 1.0087 \\
        Cantonese & 38.206 & 37.856 & 36.009 & 1.0092 \\
        Mossi & 58.559 & 57.896 & 80.328 & 1.0115 \\
        Odia & 135.328 & 133.265 & 121.959 & 1.0155 \\
        Georgian & 71.468 & 70.135 & 64.738 & 1.0190 \\
        Amharic & 61.581 & 60.261 & 55.687 & 1.0219 \\
        Shan & 218.032 & 204.210 & 190.272 & 1.0677 \\
\bottomrule
\end{tabular}
\end{table}

\subsection{Per-language distribution and mechanism diagnostics}
\label{app:per-lang-distribution}

The full 16K--256K numerical sweep appears in Tables~\ref{tab:rep-language-vocab},~\ref{tab:headline-gini-app}, and~\ref{tab:gini-vs-tokens}, with complete per-language values in Appendix~\ref{app:per-language-results}. Table~\ref{tab:best-worst-flores} lists the top 10 and bottom 10 languages by UBE/BBPE token ratio under the representative \texttt{101-lang/64K} setting, using raw mean token counts. Figure~\ref{fig:best-worst-vocab-panels} shows the best and worst five languages across all 16K--256K vocabulary sizes. Most languages cluster near parity (ratio $\approx 1.0$), with a left tail of substantial improvements concentrated in low-resource 3-byte scripts. Across vocabulary sizes, the regression tail stays narrow and concentrated in a small recurring set of languages---chiefly Georgian, Amharic, Shan, and Tibetan, with Mossi (\texttt{mos\_Latn}) as the Latin-script language that regresses most consistently across 32K--256K---and peaks at \repLangOneTwoEightkTibetanDeltaPct{} for Tibetan at 128K. Dzongkha shares the Tibetan script but does not regress repeatedly; it improves at every vocabulary size except 128K.

\FloatBarrier

\begin{figure}[htbp]
\centering
\resizebox{\linewidth}{!}{\input{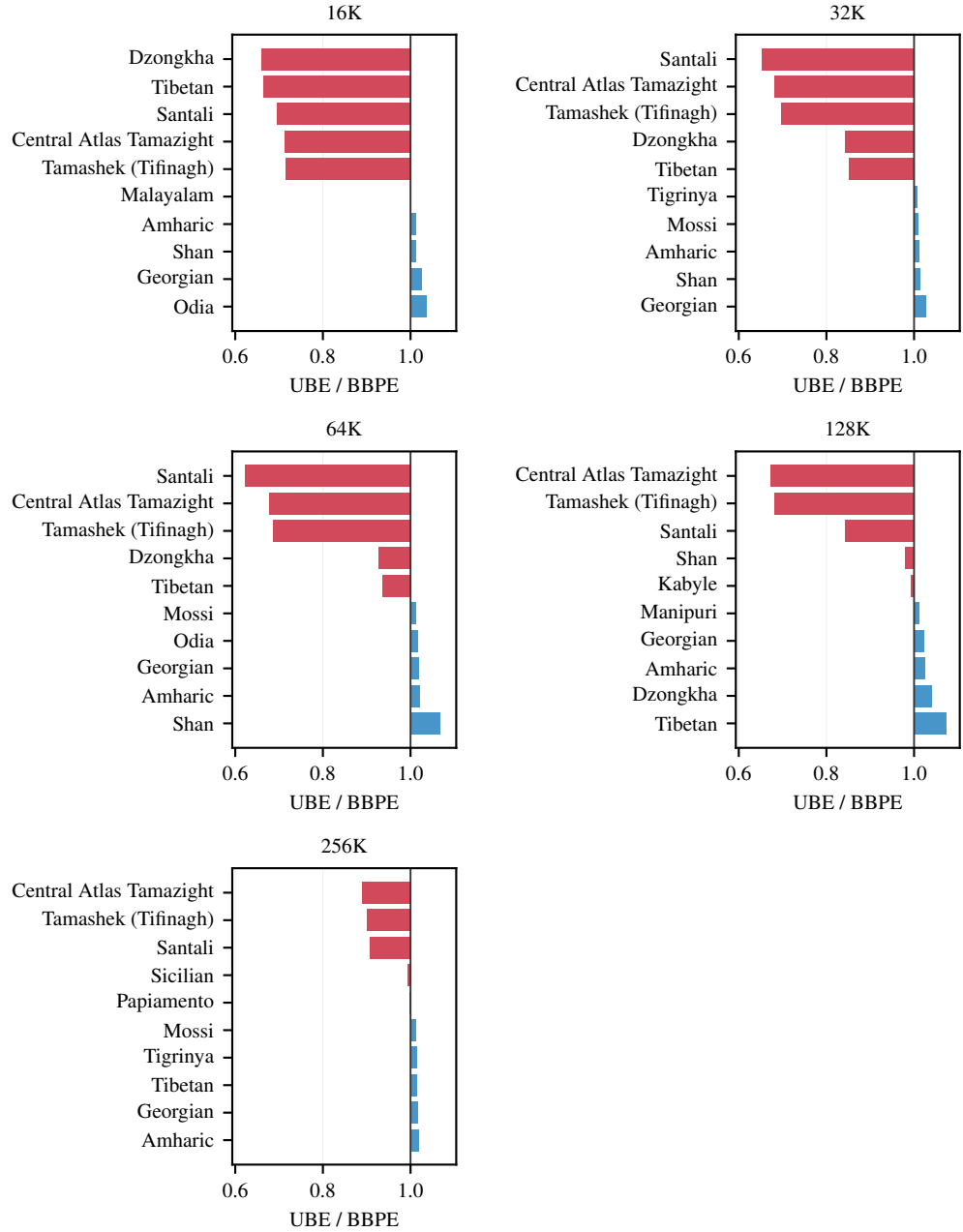}}
\caption{Best and worst five languages by UBE/BBPE token-count ratio across the full \texttt{101-lang} vocabulary sweep (16K--256K). Bars start at the vertical line, which marks equal token counts (ratio~1); red bars mark ratios below 1 (UBE uses fewer tokens) and blue bars ratios of 1 or above.}
\label{fig:best-worst-vocab-panels}
\end{figure}

\subsubsection{Token composition and decomposition}
\label{app:token-composition}

\begingroup
\let\origtable\table
\let\endorigtable\endtable
\renewenvironment{table}[1][]{\origtable[H]}{\endorigtable}
\begin{table}[t]
\centering
\small
\caption{Token composition for the scripts with the largest BBPE-to-UBE token-count gaps in the \texttt{101-lang}/64K setting. Per-script-family token counts across all 16K--256K vocabularies appear in Table~\ref{tab:script-family}.}
\label{tab:token-composition}
\begin{tabular}{llcccccc}
\toprule
Script & Language & Codepoint & \multicolumn{2}{c}{Floor (bytes/char)} & \multicolumn{2}{c}{Avg tokens/sent} & Save \\
\cmidrule(lr){4-5} \cmidrule(lr){6-7}
 & & & \utfEightName{} & \utfSixteenName{} & UBE & BBPE & \\
\midrule
        Ol Chiki & Santali & \texttt{U+1C5E} & 3 & 2 & 212.692 & 341.271 & 37.7\% \\
        Tifinagh & Central Atlas Tamazight & \texttt{U+2D30} & 3 & 2 & 182.243 & 269.123 & 32.3\% \\
        Tibetan & Dzongkha & \texttt{U+0F42} & 3 & 2 & 265.685 & 286.182 & 7.2\% \\
\bottomrule
\end{tabular}
\end{table}

\begin{table}[t]
\centering
\small
\caption{Word- and sentence-level tokenization examples in the \texttt{101-lang}/64K setting. ``Base'' and ``Tokens'' count pre-merge symbols and final BPE tokens. ``Word'' gives a romanized form; ``(\emph{n}~chars)'' denotes sentence fragments. Table~\ref{tab:token-decomposition-vocab} extends the same examples to all 16K--256K vocabularies.}
\label{tab:token-decomposition}
\begin{tabular}{llrcccc@{\hspace{6pt}}c}
\toprule
Script & Word & Chars & \multicolumn{2}{c}{UBE} & \multicolumn{2}{c}{BBPE} & Save \\
\cmidrule(lr){4-5} \cmidrule(lr){6-7}
 & & & Base & Tokens & Base & Tokens & \\
\midrule
        Ol Chiki & ona & 3 & 6 & 6 & 9 & 9 & 33\% \\
        Ol Chiki & kemikal & 7 & 14 & 12 & 21 & 21 & 43\% \\
        Tifinagh & sin & 3 & 6 & 6 & 9 & 9 & 33\% \\
        Tifinagh & acukula & 7 & 14 & 14 & 21 & 21 & 33\% \\
\midrule
        Ol Chiki & (24 chars) & 24 & 45 & 39 & 66 & 63 & 38\% \\
        Tifinagh & (18 chars) & 18 & 34 & 31 & 50 & 48 & 35\% \\
\bottomrule
\end{tabular}
\end{table}

\endgroup

Tables~\ref{tab:token-composition} and~\ref{tab:token-decomposition} illustrate the mechanism at the word level. For short Ol~Chiki and Tifinagh words, BBPE produces exactly \utfEightHighFloor{} tokens per character because it has learned virtually no merges for these scripts even at 64K vocabulary: each character's \utfEightHighFloor{} \utfEightName{} bytes remain as separate base symbols. UBE starts at \utfEightMidFloor{} base symbols per character via the \utfSixteenName{} path, and BPE training has additionally learned cross-character byte merges, pushing the token count below the \utfEightMidFloor{}-per-character baseline. For a \tdSentOneChars{}-character Ol~Chiki sentence fragment, the combined floor and merge advantage yields \tdSentOneBbpeTokens{} BBPE tokens versus \tdSentOneUbeTokens{} UBE tokens (\tdSentOneSavePct{} reduction). Table~\ref{tab:token-decomposition-vocab} extends the same word- and sentence-level examples across the full 16K--256K vocabulary sweep, and Table~\ref{tab:script-family} reports the corresponding per-script-family averages on FLORES-200.

\begin{table}[ht]
\centering
\footnotesize
\setlength{\tabcolsep}{3pt}
\caption{Word- and sentence-level token counts across the 16K--256K vocabulary sweep (\texttt{101-lang}), using the same inputs as Table~\ref{tab:token-decomposition}.}
\label{tab:token-decomposition-vocab}
\resizebox{\textwidth}{!}{%
\begin{tabular}{llrrrrrrrrrrr}
\toprule
\multirow{2}{*}{Script} & \multirow{2}{*}{Word} & \multirow{2}{*}{Chars} & \multicolumn{2}{c}{16K} & \multicolumn{2}{c}{32K} & \multicolumn{2}{c}{64K} & \multicolumn{2}{c}{128K} & \multicolumn{2}{c}{256K} \\
\cmidrule(lr){4-5} \cmidrule(lr){6-7} \cmidrule(lr){8-9} \cmidrule(lr){10-11} \cmidrule(lr){12-13}
 & & & UBE & BBPE & UBE & BBPE & UBE & BBPE & UBE & BBPE & UBE & BBPE \\
\midrule
        Ol Chiki & ona & 3 & 6 & 9 & 6 & 9 & 6 & 9 & 6 & 6 & 6 & 6 \\
        Ol Chiki & kemikal & 7 & 14 & 21 & 13 & 21 & 12 & 21 & 12 & 14 & 10 & 14 \\
        Tifinagh & sin & 3 & 6 & 9 & 6 & 9 & 6 & 9 & 6 & 9 & 6 & 6 \\
        Tifinagh & acukula & 7 & 14 & 21 & 14 & 21 & 14 & 21 & 14 & 21 & 14 & 14 \\
\midrule
        Ol Chiki & (24 chars) & 24 & 42 & 63 & 40 & 63 & 39 & 63 & 39 & 45 & 38 & 42 \\
        Tifinagh & (18 chars) & 18 & 31 & 48 & 31 & 48 & 31 & 48 & 31 & 48 & 31 & 34 \\
\bottomrule
\end{tabular}%
}
\end{table}

\FloatBarrier

\subsubsection{Per-script-family analysis}
\label{app:script-family}

Table~\ref{tab:script-family} breaks down the \texttt{101-lang} FLORES-200 results by Unicode script family across all vocabulary sizes (16K--256K). UBE's largest absolute token reductions are in Ol~Chiki, Tifinagh, and Tibetan---all 3-byte \utfEightName{} scripts with no language-specific mC4 training data, where BBPE remains near its encoding floor. For high-coverage scripts (Latin, Cyrillic, Arabic, Devanagari), UBE's effect is negligible ($<1$ token) at every vocabulary size, consistent with the typical broad-coverage outcome being near-tie behavior rather than large per-language swings.

\begingroup
\footnotesize
\begin{longtable}{lrrlrrrrr}
\caption{\normalsize Per-script-family average token counts on FLORES-200 (\texttt{101-lang}; 16K--256K; 29 scripts). $\Delta = \mathrm{UBE} - \mathrm{BBPE}$ (negative = fewer UBE tokens).}
\label{tab:script-family}\\
\toprule
Script family & Langs & \utfEightName{} & Metric & 16K & 32K & 64K & 128K & 256K \\
\midrule
\endfirsthead
\multicolumn{9}{c}{Table \thetable{} (continued)}\\
\toprule
Script family & Langs & \utfEightName{} & Metric & 16K & 32K & 64K & 128K & 256K \\
\midrule
\endhead
\bottomrule
\endfoot
Latin & 127 & 1--2 & UBE & 62.163 & 56.627 & 52.500 & 48.722 & 45.097 \\*
 &  &  & BBPE & 62.373 & 56.743 & 52.540 & 48.751 & 45.100 \\*
 &  &  & $\Delta$ & $-0.210$ & $-0.117$ & $-0.040$ & $-0.029$ & $-0.003$ \\
\addlinespace[2pt]
Arabic & 22 & 2 & UBE & 77.610 & 69.549 & 62.741 & 56.453 & 51.336 \\*
 &  &  & BBPE & 78.312 & 69.687 & 62.817 & 56.489 & 51.358 \\*
 &  &  & $\Delta$ & $-0.702$ & $-0.138$ & $-0.076$ & $-0.036$ & $-0.022$ \\
\addlinespace[2pt]
Cyrillic & 12 & 2 & UBE & 75.391 & 65.828 & 57.626 & 50.389 & 44.300 \\*
 &  &  & BBPE & 75.875 & 66.044 & 57.703 & 50.435 & 44.315 \\*
 &  &  & $\Delta$ & $-0.484$ & $-0.216$ & $-0.077$ & $-0.046$ & $-0.015$ \\
\addlinespace[2pt]
Hebrew & 2 & 2 & UBE & 74.512 & 63.175 & 55.843 & 48.783 & 43.206 \\*
 &  &  & BBPE & 74.965 & 63.530 & 55.976 & 48.797 & 43.216 \\*
 &  &  & $\Delta$ & $-0.453$ & $-0.355$ & $-0.133$ & $-0.014$ & $-0.010$ \\
\addlinespace[2pt]
Greek & 1 & 2 & UBE & 90.479 & 77.161 & 66.125 & 56.125 & 48.226 \\*
 &  &  & BBPE & 90.726 & 77.342 & 66.307 & 56.179 & 48.243 \\*
 &  &  & $\Delta$ & $-0.247$ & $-0.181$ & $-0.183$ & $-0.054$ & $-0.017$ \\
\addlinespace[2pt]
Armenian & 1 & 2 & UBE & 91.947 & 78.250 & 64.792 & 54.347 & 46.934 \\*
 &  &  & BBPE & 92.957 & 78.649 & 65.068 & 54.425 & 46.963 \\*
 &  &  & $\Delta$ & $-1.010$ & $-0.399$ & $-0.277$ & $-0.078$ & $-0.030$ \\
\addlinespace[2pt]
Devanagari & 10 & 3 & UBE & 93.056 & 89.929 & 88.237 & 86.705 & 85.871 \\*
 &  &  & BBPE & 93.438 & 90.036 & 88.247 & 86.756 & 85.869 \\*
 &  &  & $\Delta$ & $-0.382$ & $-0.107$ & $-0.010$ & $-0.051$ & $+0.001$ \\
\addlinespace[2pt]
Bengali & 3 & 3 & UBE & 119.222 & 110.351 & 104.794 & 100.288 & 96.991 \\*
 &  &  & BBPE & 119.609 & 110.142 & 104.888 & 99.876 & 96.733 \\*
 &  &  & $\Delta$ & $-0.386$ & $+0.209$ & $-0.095$ & $+0.412$ & $+0.258$ \\
\addlinespace[2pt]
Ethiopic & 2 & 3 & UBE & 87.370 & 75.036 & 65.440 & 57.024 & 49.441 \\*
 &  &  & BBPE & 87.440 & 74.330 & 64.830 & 56.061 & 48.660 \\*
 &  &  & $\Delta$ & $-0.071$ & $+0.706$ & $+0.610$ & $+0.962$ & $+0.781$ \\
\addlinespace[2pt]
Tibetan & 2 & 3 & UBE & 299.469 & 298.116 & 256.428 & 237.486 & 170.774 \\*
 &  &  & BBPE & 452.168 & 352.053 & 275.195 & 224.709 & 169.761 \\*
 &  &  & $\Delta$ & $-152.699$ & $-53.937$ & $-18.767$ & $+12.776$ & $+1.013$ \\
\addlinespace[2pt]
Myanmar & 2 & 3 & UBE & 188.222 & 173.614 & 167.616 & 139.157 & 129.968 \\*
 &  &  & BBPE & 186.831 & 172.278 & 160.686 & 141.021 & 129.805 \\*
 &  &  & $\Delta$ & $+1.391$ & $+1.335$ & $+6.930$ & $-1.865$ & $+0.163$ \\
\addlinespace[2pt]
Tifinagh & 2 & 3 & UBE & 193.081 & 185.780 & 183.926 & 182.474 & 182.395 \\*
 &  &  & BBPE & 270.173 & 269.911 & 269.847 & 269.815 & 203.886 \\*
 &  &  & $\Delta$ & $-77.092$ & $-84.130$ & $-85.921$ & $-87.341$ & $-21.491$ \\
\addlinespace[2pt]
Han (Trad.) & 2 & 3 & UBE & 49.420 & 43.913 & 38.641 & 34.964 & 31.997 \\*
 &  &  & BBPE & 50.239 & 44.039 & 38.298 & 34.557 & 31.617 \\*
 &  &  & $\Delta$ & $-0.819$ & $-0.125$ & $+0.343$ & $+0.408$ & $+0.380$ \\
\addlinespace[2pt]
Han (Simp.) & 1 & 3 & UBE & 44.910 & 39.359 & 35.015 & 31.574 & 28.664 \\*
 &  &  & BBPE & 44.920 & 39.183 & 34.764 & 31.310 & 28.437 \\*
 &  &  & $\Delta$ & $-0.010$ & $+0.176$ & $+0.251$ & $+0.264$ & $+0.227$ \\
\addlinespace[2pt]
Hangul & 1 & 3 & UBE & 44.092 & 38.633 & 34.112 & 30.377 & 27.495 \\*
 &  &  & BBPE & 44.539 & 38.766 & 34.193 & 30.375 & 27.481 \\*
 &  &  & $\Delta$ & $-0.447$ & $-0.132$ & $-0.081$ & $+0.003$ & $+0.014$ \\
\addlinespace[2pt]
Japanese & 1 & 3 & UBE & 52.419 & 46.273 & 40.783 & 36.098 & 32.338 \\*
 &  &  & BBPE & 52.668 & 46.194 & 40.630 & 35.965 & 32.145 \\*
 &  &  & $\Delta$ & $-0.249$ & $+0.079$ & $+0.152$ & $+0.132$ & $+0.193$ \\
\addlinespace[2pt]
Khmer & 1 & 3 & UBE & 145.965 & 135.804 & 127.642 & 121.825 & 118.121 \\*
 &  &  & BBPE & 145.946 & 136.047 & 127.747 & 121.860 & 118.116 \\*
 &  &  & $\Delta$ & $+0.020$ & $-0.243$ & $-0.105$ & $-0.035$ & $+0.005$ \\
\addlinespace[2pt]
Lao & 1 & 3 & UBE & 115.340 & 99.961 & 87.657 & 79.355 & 72.767 \\*
 &  &  & BBPE & 115.791 & 100.289 & 88.180 & 79.768 & 72.899 \\*
 &  &  & $\Delta$ & $-0.451$ & $-0.327$ & $-0.523$ & $-0.413$ & $-0.132$ \\
\addlinespace[2pt]
Thai & 1 & 3 & UBE & 75.640 & 69.237 & 64.534 & 60.990 & 58.107 \\*
 &  &  & BBPE & 75.850 & 69.173 & 64.612 & 60.969 & 58.122 \\*
 &  &  & $\Delta$ & $-0.209$ & $+0.064$ & $-0.078$ & $+0.021$ & $-0.015$ \\
\addlinespace[2pt]
Telugu & 1 & 3 & UBE & 114.800 & 110.878 & 105.106 & 100.440 & 98.118 \\*
 &  &  & BBPE & 114.948 & 111.162 & 105.057 & 100.465 & 98.162 \\*
 &  &  & $\Delta$ & $-0.147$ & $-0.284$ & $+0.048$ & $-0.026$ & $-0.044$ \\
\addlinespace[2pt]
Tamil & 1 & 3 & UBE & 130.214 & 122.921 & 118.113 & 115.209 & 113.794 \\*
 &  &  & BBPE & 130.302 & 123.013 & 118.227 & 115.250 & 113.797 \\*
 &  &  & $\Delta$ & $-0.088$ & $-0.092$ & $-0.115$ & $-0.041$ & $-0.003$ \\
\addlinespace[2pt]
Kannada & 1 & 3 & UBE & 121.886 & 114.405 & 108.854 & 104.483 & 101.329 \\*
 &  &  & BBPE & 121.878 & 115.094 & 109.013 & 104.505 & 101.329 \\*
 &  &  & $\Delta$ & $+0.008$ & $-0.689$ & $-0.159$ & $-0.022$ & $0.000$ \\
\addlinespace[2pt]
Malayalam & 1 & 3 & UBE & 132.650 & 126.980 & 120.760 & 115.976 & 113.153 \\*
 &  &  & BBPE & 132.444 & 126.831 & 120.714 & 115.951 & 113.167 \\*
 &  &  & $\Delta$ & $+0.207$ & $+0.149$ & $+0.045$ & $+0.026$ & $-0.014$ \\
\addlinespace[2pt]
Sinhala & 1 & 3 & UBE & 110.344 & 102.037 & 95.809 & 91.664 & 88.621 \\*
 &  &  & BBPE & 110.516 & 102.125 & 95.822 & 91.593 & 88.624 \\*
 &  &  & $\Delta$ & $-0.172$ & $-0.088$ & $-0.013$ & $+0.071$ & $-0.003$ \\
\addlinespace[2pt]
Odia & 1 & 3 & UBE & 188.713 & 161.316 & 135.328 & 118.017 & 109.633 \\*
 &  &  & BBPE & 182.258 & 164.863 & 133.265 & 118.017 & 108.809 \\*
 &  &  & $\Delta$ & $+6.456$ & $-3.546$ & $+2.063$ & $0.000$ & $+0.824$ \\
\addlinespace[2pt]
Gujarati & 1 & 3 & UBE & 105.362 & 99.345 & 93.423 & 89.804 & 86.962 \\*
 &  &  & BBPE & 105.835 & 98.813 & 93.399 & 89.846 & 86.943 \\*
 &  &  & $\Delta$ & $-0.473$ & $+0.532$ & $+0.024$ & $-0.042$ & $+0.020$ \\
\addlinespace[2pt]
Gurmukhi & 1 & 3 & UBE & 106.408 & 98.947 & 93.255 & 88.974 & 86.336 \\*
 &  &  & BBPE & 107.065 & 98.726 & 93.199 & 89.055 & 86.388 \\*
 &  &  & $\Delta$ & $-0.657$ & $+0.220$ & $+0.056$ & $-0.081$ & $-0.052$ \\
\addlinespace[2pt]
Georgian & 1 & 3 & UBE & 100.074 & 85.117 & 71.468 & 59.405 & 49.793 \\*
 &  &  & BBPE & 97.688 & 82.745 & 70.135 & 58.003 & 48.937 \\*
 &  &  & $\Delta$ & $+2.386$ & $+2.372$ & $+1.333$ & $+1.402$ & $+0.857$ \\
\addlinespace[2pt]
Ol Chiki & 1 & 3 & UBE & 237.326 & 222.944 & 212.692 & 211.712 & 207.379 \\*
 &  &  & BBPE & 341.312 & 341.287 & 341.271 & 251.125 & 228.583 \\*
 &  &  & $\Delta$ & $-103.986$ & $-118.343$ & $-128.579$ & $-39.413$ & $-21.204$ \\
\addlinespace[2pt]
\end{longtable}
\endgroup

\FloatBarrier

\subsubsection{Vocabulary utilization diagnostics}
\label{app:vocab-diagnostics}

Table~\ref{tab:vocab-diagnostics} reports vocabulary utilization metrics for the four \texttt{101-lang/32K} tokenizers. A dead merge is a learned merge token with count exactly zero when the finalized tokenizer retokenizes the frequency-weighted, merged filtered training word table used to produce its training-time \texttt{token\_counts} metadata. UBE's \vdUbeAlphTokens{}-symbol alphabet does not appear to strand merge capacity: active utilization is \vdUbeActiveUtil{}, all 256 PUA symbols are used, and UBE has fewer dead merge tokens (\vdUbeDeadMerges{}) than BBPE (\vdBbpeDeadMerges{}). Roughly \vdUbePuaPuaPct{} of UBE's merges join two pure PUA tokens (\utfSixteenName{}-path merges), indicating that the \utfSixteenName{} path learns productive subword patterns. The higher dead-merge counts for S-UBE and S-BBPE (\vdSbpeUbeDeadMerges{}/\vdSbpeBbpeDeadMerges{} vs.\ \vdUbeDeadMerges{}/\vdBbpeDeadMerges{}) reflect SuperBPE's permissive cross-word merges, which create long-tail tokens that are not emitted in this final training-time retokenization. Mixed-alphabet results account for \vdUbeCrossPct{} of merges, including merges with an already mixed operand; the remainder stay entirely within one alphabet. UBE also produces $2\times$ more depth-1 merges (\vdUbeDepthOnePct{} vs.\ \vdBbpeDepthOnePct{}) because routed BMP characters can be completed after a single merge from 2 \utfSixteenName{} bytes, whereas BBPE typically needs 2 successive merges to collapse a 3-byte \utfEightName{} decomposition before character-spanning subword merges become available. Depth-1 share is not a standalone quality metric; here, the higher share reflects that the alphabet has already paid much of the floor cost, so fewer hierarchical merges are needed to reach character-level units. Together with \vdUbeActiveUtil{} active utilization and use of all PUA symbols, this indicates that the extra alphabet capacity is productively spent rather than stranded.

\begin{table}[ht]
\centering
\small
\caption{Vocabulary utilization diagnostics (\texttt{101-lang/32K}). Other vocabulary sizes are in Table~\ref{tab:vocab-diagnostics-vocab}.}
\label{tab:vocab-diagnostics}
\begin{tabular}{lrrrr}
\toprule
Metric & UBE & BBPE & S-UBE & S-BBPE \\
\midrule
Alphabet tokens & \vdUbeAlphTokens{} & \vdBbpeAlphTokens{} & \vdSbpeUbeAlphTokens{} & \vdSbpeBbpeAlphTokens{} \\
Merge tokens & \vdUbeMergeTokens{} & \vdBbpeMergeTokens{} & \vdSbpeUbeMergeTokens{} & \vdSbpeBbpeMergeTokens{} \\
Dead merges & \vdUbeDeadMerges{} & \vdBbpeDeadMerges{} & \vdSbpeUbeDeadMerges{} & \vdSbpeBbpeDeadMerges{} \\
Active utilization & \vdUbeActiveUtil{} & \vdBbpeActiveUtil{} & \vdSbpeUbeActiveUtil{} & \vdSbpeBbpeActiveUtil{} \\
\midrule
Pure PUA--PUA merges & \vdUbePuaPuaPct{} & --- & \vdSbpeUbePuaPuaPct{} & --- \\
GPT-2--GPT-2 merges & \vdUbeGptGptPct{} & --- & \vdSbpeUbeGptGptPct{} & --- \\
Mixed-alphabet merges & \vdUbeCrossPct{} & --- & \vdSbpeUbeCrossPct{} & --- \\
\midrule
Depth-1 merges & \vdUbeDepthOnePct{} & \vdBbpeDepthOnePct{} & \vdSbpeUbeDepthOnePct{} & \vdSbpeBbpeDepthOnePct{} \\
Max merge depth & \vdUbeMaxDepth{} & \vdBbpeMaxDepth{} & \vdSbpeUbeMaxDepth{} & \vdSbpeBbpeMaxDepth{} \\
\bottomrule
\end{tabular}
\end{table}

\begin{table}[ht]
\centering
\footnotesize
\setlength{\tabcolsep}{3.5pt}
\caption{Vocabulary utilization diagnostics across vocabulary sizes (\texttt{101-lang}; 16K--256K), extending Table~\ref{tab:vocab-diagnostics}.}
\label{tab:vocab-diagnostics-vocab}
\begin{tabular}{llrrrrr}
\toprule
Family & Metric & 16K & 32K & 64K & 128K & 256K \\
\midrule
\multirow{5}{*}{UBE}
 & Dead merges & \vdUbeOneSixkDeadMerges{} & \vdUbeThreeTwokDeadMerges{} & \vdUbeSixFourkDeadMerges{} & \vdUbeOneTwoEightkDeadMerges{} & \vdUbeTwoFiveSixkDeadMerges{} \\
 & Active util. & \vdUbeOneSixkActiveUtilDetailed{} & \vdUbeThreeTwokActiveUtilDetailed{} & \vdUbeSixFourkActiveUtilDetailed{} & \vdUbeOneTwoEightkActiveUtilDetailed{} & \vdUbeTwoFiveSixkActiveUtilDetailed{} \\
 & Pure PUA & \vdUbeOneSixkPuaPuaPctDetailed{} & \vdUbeThreeTwokPuaPuaPctDetailed{} & \vdUbeSixFourkPuaPuaPctDetailed{} & \vdUbeOneTwoEightkPuaPuaPctDetailed{} & \vdUbeTwoFiveSixkPuaPuaPctDetailed{} \\
 & Mixed alph. & \vdUbeOneSixkCrossPctDetailed{} & \vdUbeThreeTwokCrossPctDetailed{} & \vdUbeSixFourkCrossPctDetailed{} & \vdUbeOneTwoEightkCrossPctDetailed{} & \vdUbeTwoFiveSixkCrossPctDetailed{} \\
 & Depth-1 & \vdUbeOneSixkDepthOnePctDetailed{} & \vdUbeThreeTwokDepthOnePctDetailed{} & \vdUbeSixFourkDepthOnePctDetailed{} & \vdUbeOneTwoEightkDepthOnePctDetailed{} & \vdUbeTwoFiveSixkDepthOnePctDetailed{} \\
\midrule
\multirow{3}{*}{BBPE}
 & Dead merges & \vdBbpeOneSixkDeadMerges{} & \vdBbpeThreeTwokDeadMerges{} & \vdBbpeSixFourkDeadMerges{} & \vdBbpeOneTwoEightkDeadMerges{} & \vdBbpeTwoFiveSixkDeadMerges{} \\
 & Active util. & \vdBbpeOneSixkActiveUtilDetailed{} & \vdBbpeThreeTwokActiveUtilDetailed{} & \vdBbpeSixFourkActiveUtilDetailed{} & \vdBbpeOneTwoEightkActiveUtilDetailed{} & \vdBbpeTwoFiveSixkActiveUtilDetailed{} \\
 & Depth-1 & \vdBbpeOneSixkDepthOnePctDetailed{} & \vdBbpeThreeTwokDepthOnePctDetailed{} & \vdBbpeSixFourkDepthOnePctDetailed{} & \vdBbpeOneTwoEightkDepthOnePctDetailed{} & \vdBbpeTwoFiveSixkDepthOnePctDetailed{} \\
\midrule
\multirow{5}{*}{S-UBE}
 & Dead merges & \vdSbpeUbeOneSixkDeadMerges{} & \vdSbpeUbeThreeTwokDeadMerges{} & \vdSbpeUbeSixFourkDeadMerges{} & \vdSbpeUbeOneTwoEightkDeadMerges{} & \vdSbpeUbeTwoFiveSixkDeadMerges{} \\
 & Active util. & \vdSbpeUbeOneSixkActiveUtilDetailed{} & \vdSbpeUbeThreeTwokActiveUtilDetailed{} & \vdSbpeUbeSixFourkActiveUtilDetailed{} & \vdSbpeUbeOneTwoEightkActiveUtilDetailed{} & \vdSbpeUbeTwoFiveSixkActiveUtilDetailed{} \\
 & Pure PUA & \vdSbpeUbeOneSixkPuaPuaPctDetailed{} & \vdSbpeUbeThreeTwokPuaPuaPctDetailed{} & \vdSbpeUbeSixFourkPuaPuaPctDetailed{} & \vdSbpeUbeOneTwoEightkPuaPuaPctDetailed{} & \vdSbpeUbeTwoFiveSixkPuaPuaPctDetailed{} \\
 & Mixed alph. & \vdSbpeUbeOneSixkCrossPctDetailed{} & \vdSbpeUbeThreeTwokCrossPctDetailed{} & \vdSbpeUbeSixFourkCrossPctDetailed{} & \vdSbpeUbeOneTwoEightkCrossPctDetailed{} & \vdSbpeUbeTwoFiveSixkCrossPctDetailed{} \\
 & Depth-1 & \vdSbpeUbeOneSixkDepthOnePctDetailed{} & \vdSbpeUbeThreeTwokDepthOnePctDetailed{} & \vdSbpeUbeSixFourkDepthOnePctDetailed{} & \vdSbpeUbeOneTwoEightkDepthOnePctDetailed{} & \vdSbpeUbeTwoFiveSixkDepthOnePctDetailed{} \\
\midrule
\multirow{3}{*}{S-BBPE}
 & Dead merges & \vdSbpeBbpeOneSixkDeadMerges{} & \vdSbpeBbpeThreeTwokDeadMerges{} & \vdSbpeBbpeSixFourkDeadMerges{} & \vdSbpeBbpeOneTwoEightkDeadMerges{} & \vdSbpeBbpeTwoFiveSixkDeadMerges{} \\
 & Active util. & \vdSbpeBbpeOneSixkActiveUtilDetailed{} & \vdSbpeBbpeThreeTwokActiveUtilDetailed{} & \vdSbpeBbpeSixFourkActiveUtilDetailed{} & \vdSbpeBbpeOneTwoEightkActiveUtilDetailed{} & \vdSbpeBbpeTwoFiveSixkActiveUtilDetailed{} \\
 & Depth-1 & \vdSbpeBbpeOneSixkDepthOnePctDetailed{} & \vdSbpeBbpeThreeTwokDepthOnePctDetailed{} & \vdSbpeBbpeSixFourkDepthOnePctDetailed{} & \vdSbpeBbpeOneTwoEightkDepthOnePctDetailed{} & \vdSbpeBbpeTwoFiveSixkDepthOnePctDetailed{} \\
\bottomrule
\end{tabular}
\end{table}
\FloatBarrier

\subsubsection{Cross-dataset robustness by vocabulary size}
\label{app:cross-dataset-vocab}

Table~\ref{tab:cross-dataset-vocab} reports UBE's Gini change across all four evaluation corpora and all available vocabulary sizes under the \texttt{101-lang} regime (16K through 256K). UBE's advantage is generally larger on UDHR from 32K onward, while the FLORES-200 reduction is marginally larger at 16K. The NLTK UDHR corpus has \udhrRobLangs{} entries, nearly twice the number of FLORES-200 languages, and includes many scripts with no mC4 training data. In these byte-fallback-heavy settings, the encoding-floor reduction is most visible because BBPE has learned few or no compensating merges.

\begin{table}[ht]
\centering
\small
\caption{Cross-dataset Gini comparison of UBE and BBPE across vocabulary sizes (\texttt{101-lang}; 16K--256K).}
\label{tab:cross-dataset-vocab}
\begin{tabular}{llrrr}
\toprule
Dataset & Vocab & UBE Gini & BBPE Gini & $\Delta$Gini \\
\midrule
FLORES-200 & 16K & \floresRobOneSixkGiniUBE{} & \floresRobOneSixkGiniBBPE{} & \floresRobOneSixkGiniPctDetailed{} \\
FLORES-200 & 32K & \floresRobThreeTwokGiniUBE{} & \floresRobThreeTwokGiniBBPE{} & \floresRobThreeTwokGiniPctDetailed{} \\
FLORES-200 & 64K & \floresRobGiniUBE{} & \floresRobGiniBBPE{} & \floresRobSixFourkGiniPctDetailed{} \\
FLORES-200 & 128K & \floresRobOneTwoEightkGiniUBE{} & \floresRobOneTwoEightkGiniBBPE{} & \floresRobOneTwoEightkGiniPctDetailed{} \\
FLORES-200 & 256K & \floresRobTwoFiveSixkGiniUBE{} & \floresRobTwoFiveSixkGiniBBPE{} & \floresRobTwoFiveSixkGiniPctDetailed{} \\
\midrule
UDHR & 16K & \udhrRobOneSixkGiniUBE{} & \udhrRobOneSixkGiniBBPE{} & \udhrRobOneSixkGiniPctDetailed{} \\
UDHR & 32K & \udhrRobThreeTwokGiniUBE{} & \udhrRobThreeTwokGiniBBPE{} & \udhrRobThreeTwokGiniPctDetailed{} \\
UDHR & 64K & \udhrRobGiniUBE{} & \udhrRobGiniBBPE{} & \udhrRobSixFourkGiniPctDetailed{} \\
UDHR & 128K & \udhrRobOneTwoEightkGiniUBE{} & \udhrRobOneTwoEightkGiniBBPE{} & \udhrRobOneTwoEightkGiniPctDetailed{} \\
UDHR & 256K & \udhrRobTwoFiveSixkGiniUBE{} & \udhrRobTwoFiveSixkGiniBBPE{} & \udhrRobTwoFiveSixkGiniPctDetailed{} \\
\midrule
SIB-200 & 16K & \sibRobOneSixkGiniUBE{} & \sibRobOneSixkGiniBBPE{} & \sibRobOneSixkGiniPctDetailed{} \\
SIB-200 & 32K & \sibRobThreeTwokGiniUBE{} & \sibRobThreeTwokGiniBBPE{} & \sibRobThreeTwokGiniPctDetailed{} \\
SIB-200 & 64K & \sibRobGiniUBE{} & \sibRobGiniBBPE{} & \sibRobSixFourkGiniPctDetailed{} \\
SIB-200 & 128K & \sibRobOneTwoEightkGiniUBE{} & \sibRobOneTwoEightkGiniBBPE{} & \sibRobOneTwoEightkGiniPctDetailed{} \\
SIB-200 & 256K & \sibRobTwoFiveSixkGiniUBE{} & \sibRobTwoFiveSixkGiniBBPE{} & \sibRobTwoFiveSixkGiniPctDetailed{} \\
\midrule
MGSM & 16K & \mgsmRobOneSixkGiniUBE{} & \mgsmRobOneSixkGiniBBPE{} & \mgsmRobOneSixkGiniPctDetailed{} \\
MGSM & 32K & \mgsmRobThreeTwokGiniUBE{} & \mgsmRobThreeTwokGiniBBPE{} & \mgsmRobThreeTwokGiniPctDetailed{} \\
MGSM & 64K & \mgsmRobGiniUBE{} & \mgsmRobGiniBBPE{} & \mgsmRobSixFourkGiniPctDetailed{} \\
MGSM & 128K & \mgsmRobOneTwoEightkGiniUBE{} & \mgsmRobOneTwoEightkGiniBBPE{} & \mgsmRobOneTwoEightkGiniPctDetailed{} \\
MGSM & 256K & \mgsmRobTwoFiveSixkGiniUBE{} & \mgsmRobTwoFiveSixkGiniBBPE{} & \mgsmRobTwoFiveSixkGiniPctDetailed{} \\
\bottomrule
\end{tabular}
\end{table}

\FloatBarrier

\section{Statistical and comparative analyses}
\label{app:comparative-analyses}

This section examines statistical reliability, alternative tokenizer designs, Unicode fidelity, and exposure-matched language modeling. Unless stated otherwise, learned-tokenizer rows use a 32K vocabulary and report final BPE output tokens, whereas representation-only rows report base-symbol positions before any learned BPE. We keep those measurement units separate.

\subsection{Intrinsic significance analysis}
\label{app:intrinsic-significance}

For the $N$ non-English language premiums, Gini is computed as $G = \frac{1}{2 N^2 \bar{\pi}} \sum_{i,j} |\pi_i - \pi_j|$, where $\bar{\pi}$ is the mean premium. Premium variance uses the sample-corrected $N-1$ denominator and, unlike Gini, depends on the normalization factor.

For each of the nine prespecified core comparisons, the two-sided Wilcoxon signed-rank test uses 204 paired language--script mean token counts, including English. Holm and Benjamini--Hochberg (BH) corrections are applied jointly across those nine tests. The paired bootstrap resamples the 203 non-English language rows with replacement 10,000 times, applies each sampled index to both tokenizers, and reports the percentile interval for $\mathrm{Gini}(\mathrm{UBE})-\mathrm{Gini}(\mathrm{BBPE})$. Table~\ref{tab:cr-intrinsic-statistics} shows that all nine Gini intervals lie below zero; the same eight of nine Wilcoxon comparisons remain significant under the raw, Holm-adjusted, and BH-adjusted criteria. The exception is the token-count test at \texttt{En-only/64K}; its negative Gini interval is not contradictory because the two procedures estimate different statistics.

\begin{table}[ht]
\centering
\scriptsize
\setlength{\tabcolsep}{2.7pt}
\caption{Statistical analysis of the nine core FLORES-200 comparisons. The Gini CI uses paired resampling over non-English premiums; Wilcoxon uses paired mean token counts.}
\label{tab:cr-intrinsic-statistics}
\resizebox{\linewidth}{!}{\begin{tabular}{llrrrrrl}
\toprule
Corpus & Vocab & Wilcoxon $p$ & Holm $p$ & BH $q$ & UBE Gini & BBPE Gini & $\Delta$Gini 95\% CI \\
\midrule
En-only & 16K & $9.21\!\times\!10^{-6}$ & $2.76\!\times\!10^{-5}$ & $1.18\!\times\!10^{-5}$ & 0.3010 & 0.3495 & [-0.0616, -0.0323] \\
En-only & 32K & $1.85\!\times\!10^{-3}$ & $3.70\!\times\!10^{-3}$ & $2.08\!\times\!10^{-3}$ & 0.2978 & 0.3380 & [-0.0537, -0.0244] \\
En-only & 64K & $3.19\!\times\!10^{-1}$ & $3.19\!\times\!10^{-1}$ & $3.19\!\times\!10^{-1}$ & 0.2950 & 0.3183 & [-0.0353, -0.0109] \\
6-lang & 16K & $2.46\!\times\!10^{-33}$ & $2.21\!\times\!10^{-32}$ & $2.21\!\times\!10^{-32}$ & 0.2703 & 0.2960 & [-0.0376, -0.0139] \\
6-lang & 32K & $4.12\!\times\!10^{-26}$ & $2.88\!\times\!10^{-25}$ & $1.24\!\times\!10^{-25}$ & 0.2522 & 0.2717 & [-0.0325, -0.0074] \\
6-lang & 64K & $2.34\!\times\!10^{-22}$ & $1.17\!\times\!10^{-21}$ & $4.20\!\times\!10^{-22}$ & 0.2442 & 0.2603 & [-0.0299, -0.0044] \\
101-lang & 16K & $1.28\!\times\!10^{-29}$ & $1.02\!\times\!10^{-28}$ & $5.75\!\times\!10^{-29}$ & 0.1927 & 0.2194 & [-0.0467, -0.0063] \\
101-lang & 32K & $1.19\!\times\!10^{-22}$ & $7.13\!\times\!10^{-22}$ & $2.68\!\times\!10^{-22}$ & 0.1979 & 0.2183 & [-0.0358, -0.0047] \\
101-lang & 64K & $1.54\!\times\!10^{-20}$ & $6.18\!\times\!10^{-20}$ & $2.32\!\times\!10^{-20}$ & 0.2043 & 0.2225 & [-0.0364, -0.0012] \\
\bottomrule
\end{tabular}
}
\end{table}

\subsection{Tokenizer comparisons and composition}
\label{app:tokenizer-composition}

The primary-corpus block of Table~\ref{tab:cr-composition} summarizes the three \texttt{101-lang/32K} learned-token comparisons; MYTE uses fixed maps without learned BPE. The SCRIPT-style tokenizers use all 2,273 primary training shards, with no additional per-language cap, and deterministic BPE training under the same script/category boundary policy. Replacing BBPE with UBE lowers their token-premium Gini from 0.2036 to 0.1824; the corresponding standard and SuperBPE-style results are in Tables~\ref{tab:gini-vs-tokens} and~\ref{tab:gini-vs-tokens-superbpe}.

Table~\ref{tab:cr-primary-cohorts} partitions all 204 configurations into five disjoint script cohorts. South Asian comprises Bengali, Devanagari, Gujarati, Gurmukhi, Kannada, Malayalam, Odia, Sinhala, Tamil, and Telugu scripts; Remaining is the complement of the four named cohorts. Each configuration contains 1,012 sentences, so the overall mean weights cohort means by their configuration counts. SCRIPT-style's lower overall mean relative to standard boundaries is driven by the South Asian cohort, despite higher English and Other BMP means (Table~\ref{tab:cr-composition}). Within each boundary policy, the effect of substituting UBE varies across cohorts.

\begin{table}[ht]
\caption{Primary-corpus FLORES-200 breakdown (101-language mC4, 32K vocabularies). Entries are final tokens/sentence; $n$ counts language--script configurations.}
\label{tab:cr-primary-cohorts}
\centering
\small
\begin{tabular}{@{}lrrrrrrr@{}}
\toprule
 & & \multicolumn{2}{c}{Standard} & \multicolumn{2}{c}{SuperBPE-style} & \multicolumn{2}{c}{SCRIPT-style} \\
\cmidrule(lr){3-4}\cmidrule(lr){5-6}\cmidrule(lr){7-8}
Script cohort & $n$ & UBE & BBPE & UBE & BBPE & UBE & BBPE \\
\midrule
Latin/Arabic/Cyrillic & 161 & \textbf{59.078} & 59.205 & \textbf{64.262} & 64.534 & \textbf{59.719} & 59.832 \\
South Asian & 21 & \textbf{103.199} & 103.401 & 91.514 & \textbf{91.472} & 78.385 & \textbf{77.316} \\
Han/Japanese & 4 & 43.364 & \textbf{43.363} & 43.465 & \textbf{43.455} & 42.801 & \textbf{42.722} \\
Khmer/Lao/Myanmar/Thai & 5 & 130.446 & \textbf{130.013} & 115.093 & \textbf{113.257} & \textbf{124.854} & 126.617 \\
Remaining & 13 & \textbf{134.332} & 164.495 & \textbf{130.662} & 160.236 & \textbf{139.682} & 170.053 \\
\midrule
All & 204 & \textbf{69.857} & 71.889 & \textbf{72.137} & 74.187 & \textbf{68.001} & 69.957 \\
\bottomrule
\end{tabular}

\end{table}

The following tables separately report capped-corpus tokenizer comparisons and the MYTE base-symbol results.

\paragraph{SCRIPT-BPE.}
SCRIPT-BPE changes three axes jointly: its block/index representation, script/category pre-tokenization, and character-preserving merge constraint~\citep{land2025script}. UBE changes only the reversible base-symbol representation and uses neither a language identifier nor a Unicode script table. Table~\ref{tab:cr-scriptbpe} reports the five-tokenizer comparison under the capped-corpus setting. Under SCRIPT-style pre-tokenization, replacing BBPE with UBE lowers token-premium Gini from 0.2048 to 0.1832 (95\% paired-language bootstrap CI for UBE minus BBPE: $[-0.038, -0.004]$; 20,000 resamples over 203 non-English configurations). The pairwise breakdown in Table~\ref{tab:cr-scriptbpe-cohorts} shows that UBE is lower in 169/204 FLORES-200 configurations, while SCRIPT-BPE's advantages are concentrated in specific script cohorts. Across five matched tokenizer settings (50 MiB/language at 16K and 32K; 300 MiB/language at 16K, 32K, and 64K), the same 163/204 configurations use fewer UBE tokens than SCRIPT-BPE in every setting. On the natural-language input fields of XNLI, XCOPA, and MGSM, UBE is lower for all 25 Latin/Arabic/Cyrillic language--corpus pairs, whereas SCRIPT-BPE is lower for 11/12 other-script pairs.

\begin{table}[ht]
\caption{SCRIPT-BPE comparison on FLORES-200. Means are final BPE tokens/sentence; CJK (5) and Other BMP (7) are the configuration groups defined in Appendix~\ref{app:gini-vs-tokens}. Gini uses 203 English-normalized non-English premiums. Tokenizer training uses the same 101-language mC4 prefix (approximately 50 MiB of UTF-8 text per language; 32K vocabulary). Change rows compare UBE with BBPE under the same boundary policy; ``Lowest'' counts wins among all five methods.}
\label{tab:cr-scriptbpe}
\centering
\small
\setlength{\tabcolsep}{3pt}
\begin{tabular}{@{}llrrrrrr@{}}
\toprule
Pre-tokenization & Method & English & \shortstack{CJK\\(5)} & \shortstack{Other BMP\\(7)} & Overall & Gini & Lowest \\
\midrule
Standard GPT-2 & UBE (ours) & \textbf{36.938} & 57.385 & \textbf{222.054} & 66.079 & 0.2086 & \textbf{162}/204 \\
Standard GPT-2 & BBPE & 36.959 & 58.956 & 272.617 & 67.875 & 0.2281 & 6/204 \\
Official script/category & SCRIPT-BPE & 37.982 & \textbf{51.069} & 223.140 & \textbf{63.785} & \textbf{0.1785} & 25/204 \\
SCRIPT-style & UBE (ours) & 37.539 & 56.741 & 223.497 & 63.918 & 0.1832 & 4/204 \\
SCRIPT-style & BBPE & 37.569 & 58.133 & 281.154 & 65.677 & 0.2048 & 7/204 \\
\bottomrule
\end{tabular}

\par\smallskip
\begin{tabular}{@{}lrrr@{}}
\toprule
UBE change vs. matched BBPE & \shortstack{CJK\\(5)} & \shortstack{Other BMP\\(7)} & Overall \\
\midrule
Standard GPT-2 & -2.66\% & -18.55\% & -2.65\% \\
SCRIPT-style & -2.39\% & -20.51\% & -2.68\% \\
\bottomrule
\end{tabular}

\end{table}

\begin{table}[ht]
\centering
\small
\setlength{\tabcolsep}{3pt}
\caption{Pairwise UBE/SCRIPT-BPE win counts (32K tokenizers trained on 101-language mC4 with approximately 50 MiB of UTF-8 text per language).}
\label{tab:cr-scriptbpe-cohorts}
\begin{tabular}{llrrr}
\toprule
Evaluation & Script cohort & Comparisons & UBE lower & SCRIPT-BPE lower \\
\midrule
FLORES-200 & Latin/Arabic/Cyrillic & 161 & \textbf{156} & 5 \\
FLORES-200 & South Asian scripts & 21 & 1 & \textbf{20} \\
FLORES-200 & Han/Japanese & 4 & 0 & \textbf{4} \\
FLORES-200 & Khmer/Lao/Myanmar/Thai & 5 & 0 & \textbf{5} \\
FLORES-200 & Remaining scripts & 13 & \textbf{12} & 1 \\
FLORES-200 & All & 204 & \textbf{169} & 35 \\
XNLI/XCOPA/MGSM & Latin/Arabic/Cyrillic & 25 & \textbf{25} & 0 \\
XNLI/XCOPA/MGSM & Other scripts & 12 & 1 & \textbf{11} \\
XNLI/XCOPA/MGSM & All & 37 & \textbf{26} & 11 \\
\bottomrule
\end{tabular}
\end{table}

\paragraph{Boundary-policy comparison.}
To separate base representation from boundary policy, Table~\ref{tab:cr-superbpe} crosses UBE/BBPE with standard/SuperBPE-style boundaries using the same 101-language mC4 prefix (approximately 50 MiB/language) and 32K vocabulary size as Table~\ref{tab:cr-scriptbpe}. Each SuperBPE-style tokenizer extends the corresponding 16K standard tokenizer to 32K with relaxed boundaries, using the two-stage recipe in Appendix~\ref{app:superbpe-composition}. Both stages use deterministic tie-breaking. Replacing BBPE with UBE lowers every reported quantity within either fixed boundary policy. With SuperBPE-style boundaries, the Other BMP mean falls by 19.7\%, the CJK mean by 3.8\%, the overall mean by 2.4\%, and Gini by 9.4\%.

\begin{table}[ht]
\centering
\small
\setlength{\tabcolsep}{3pt}
\caption{Representation-by-boundary composition on FLORES-200 (101-language mC4, approximately 50 MiB/language, 32K vocabulary). Means are final BPE output tokens per sentence; CJK (5) and Other BMP (7) are the configuration groups defined in Appendix~\ref{app:gini-vs-tokens}. Gini uses 203 English-normalized non-English premiums.}
\label{tab:cr-superbpe}
\begin{tabular}{llrrrrr}
\toprule
Boundary policy & Representation & English & \shortstack{CJK\\(5)} & \shortstack{Other BMP\\(7)} & Overall & Gini \\
\midrule
Standard & UBE (ours) & \textbf{36.938} & \textbf{57.385} & \textbf{222.054} & \textbf{66.079} & \textbf{0.2086} \\
Standard & BBPE & 36.959 & 58.956 & 272.617 & 67.875 & 0.2281 \\
SuperBPE-style & UBE (ours) & \textbf{39.524} & \textbf{58.576} & \textbf{189.674} & \textbf{67.208} & \textbf{0.1736} \\
SuperBPE-style & BBPE & 39.586 & 60.871 & 236.155 & 68.858 & 0.1916 \\
\bottomrule
\end{tabular}

\end{table}

\paragraph{MYTE composition.}
MYTE is a morphology-driven 256-symbol byte representation learned from language-specific resources and used directly as model input~\citep{limisiewicz2024myte}. We apply its released decomposition and morphology-merge maps unchanged, retain the generated code bytes, and route only residual characters through UBE. The compact 512-symbol MYTE-UBE representation shares one 256-symbol region between residual UTF-8 bytes and MYTE code bytes, which are both copied verbatim into the post-MYTE byte stream, and keeps a separate 256-symbol region for routed UTF-16 bytes. No learned BPE is applied to Table~\ref{tab:cr-myte}; one emitted base symbol is one model-input position. MYTE-UBE reduces the CJK mean by 24.08\%, the Other BMP mean by 29.13\%, and total positions by 4.58\% relative to MYTE, is lower in 201/204 configurations, ties in three, and is higher in none. On all 206,448 FLORES-200 sentences, MYTE-UBE reconstructs the post-MYTE byte stream exactly and matches MYTE's released decoder output.

\begin{table}[ht]
\centering
\small
\setlength{\tabcolsep}{3pt}
\caption{Representation-only FLORES-200 comparison. Means and totals are base-symbol positions, not learned BPE output tokens; CJK (5) and Other BMP (7) are the configuration groups defined in Appendix~\ref{app:gini-vs-tokens}.}
\label{tab:cr-myte}
\begin{tabular}{lrrrrrr}
\toprule
Base representation & Alphabet & English & \shortstack{CJK\\(5)} & \shortstack{Other BMP\\(7)} & Overall & Total positions \\
\midrule
UBE (ours) & 512 & 130.481 & \textbf{91.252} & 277.056 & 176.088 & 36,353,040 \\
UTF-8 bytes & 256 & 130.530 & 133.665 & 406.632 & 196.106 & 40,485,643 \\
MYTE-UBE (ours) & 512 & \textbf{109.094} & 94.610 & \textbf{247.622} & \textbf{141.354} & \textbf{29,182,333} \\
MYTE & 256 & 109.140 & 124.612 & 349.398 & 148.142 & 30,583,555 \\
\bottomrule
\end{tabular}

\end{table}

\subsection{Unicode fidelity}

We audit all 1,112,064 Unicode scalar values and every unique sequence in the Unicode 17 normalization and grapheme-break test files~\citep{unicode2025} and the emoji test file~\citep{unicode2025uts51}. For UBE and MYTE, recovery requires $\mathrm{decode}(\mathrm{encode}(x))=x$; SCRIPT-BPE is also evaluated after its configured NFC normalization and within the character set supported by its released implementation. Across all 1,155,724 suite-level inputs, UBE recovers every raw input and MYTE recovers 1,138,306. SCRIPT-BPE recovers every assigned-through-Unicode-16 scalar in its released valid-character set after NFC, but filters Private-Use and unassigned scalars, and Unicode-17 additions fall outside that released set (Table~\ref{tab:cr-unicode-fidelity}). Under the stricter raw-identity criterion, SCRIPT-BPE recovers 153,943/155,063 scalars in that valid set; NFC changes the remaining 1,120.

\begin{table}[htbp]
\centering
\scriptsize
\setlength{\tabcolsep}{3pt}
\caption{Deterministic Unicode 17 fidelity audit (UBE/SCRIPT-BPE tokenizers from Table~\ref{tab:cr-scriptbpe}). ``After NFC'' includes SCRIPT-BPE's normalization; no aggregate is reported across excluded inputs. Sequence files are versioned stress tests, not conformance claims.}
\label{tab:cr-unicode-fidelity}
\resizebox{\linewidth}{!}{\begin{tabular}{lrrrr}
\toprule
Unicode 17 input set & Inputs & \shortstack{UBE\\raw exact} & \shortstack{SCRIPT-BPE\\after NFC} & \shortstack{MYTE\\raw exact} \\
\midrule
\begin{tabular}[t]{@{}l@{}}Assigned through Unicode 16\\(excluding Private-Use)\end{tabular} & 155,063 & \textbf{All recovered} & \textbf{All recovered} & 154,023/155,063 \\
Assigned in Unicode 17 & 4,803 & \textbf{All recovered} & Outside released coverage & \textbf{All recovered} \\
Private-Use (Co) & 137,468 & \textbf{All recovered} & Filtered by design & \textbf{All recovered} \\
Unassigned/reserved (Cn) & 814,730 & \textbf{All recovered} & Filtered by design & \textbf{All recovered} \\
\begin{tabular}[t]{@{}l@{}}\texttt{NormalizationTest.txt}\\unique inputs\end{tabular} & 37,671 & \textbf{All recovered} & 166 outside coverage & 21,373/37,671 \\
\begin{tabular}[t]{@{}l@{}}\texttt{GraphemeBreakTest.txt}\\unique inputs\end{tabular} & 764 & \textbf{All recovered} & 74 outside coverage & 684/764 \\
\begin{tabular}[t]{@{}l@{}}\texttt{emoji-test.txt}\\unique inputs\end{tabular} & 5,225 & \textbf{All recovered} & 67 outside coverage & \textbf{All recovered} \\
All suite-level cases & 1,155,724 & \textbf{All recovered} & \begin{tabular}[t]{@{}r@{}}No aggregate across\\excluded inputs\end{tabular} & 1,138,306/1,155,724 \\
\bottomrule
\end{tabular}
}
\end{table}

\subsection{Source-byte-matched language-model comparison}

Table~\ref{tab:cr-source-matched-accuracy} compares UBE, BBPE, and BBPE16 at matched cumulative supervised UTF-8 source-byte exposure. All three checkpoints use the same frozen 101-language mC4 corpus and data-loading protocol, 182M architecture, 32K vocabulary, a shared seed, and the same 807,999,969-byte held-out source. The selected BBPE and BBPE16 checkpoints differ from UBE's 361,441,233,445-byte target by less than 0.0005\%. UBE's BPB is 0.91\% below BBPE and 0.78\% below BBPE16; BBPE16 requires 18.61\% more model-input training tokens to reach the same source-byte exposure. Task-level accuracy leaders vary.

\begin{table}[ht]
\centering
\small
\setlength{\tabcolsep}{2.2pt}
\caption{Source-byte-matched language-model comparison. Top: training exposure (steps, model-input positions, supervised UTF-8 bytes) and held-out BPB/BPB Gini. Bottom: zero-shot accuracy (\%); English mean weights the five English tasks equally.}
\label{tab:cr-source-matched-accuracy}
\begin{tabular}{@{}c@{}}
\begin{tabular}{lrrrrr}
\toprule
Tokenizer & Steps & \begin{tabular}[b]{@{}c@{}}Model-input\\positions\end{tabular} & \begin{tabular}[b]{@{}c@{}}Supervised\\UTF-8 bytes\end{tabular} & \begin{tabular}[b]{@{}c@{}}Common-text\\BPB$\downarrow$\end{tabular} & BPB Gini$\downarrow$ \\
\midrule
\begin{tabular}[t]{@{}l@{}}UBE (ours,\\primary tokenizer)\end{tabular} & 200,000 & 104,857,600,000 & 361,441,233,445 & \textbf{1.5298} & \textbf{0.1735} \\
BBPE & 200,437 & 105,086,713,856 & 361,439,652,484 & 1.5439 & 0.1740 \\
BBPE16 & 237,228 & 124,375,793,664 & 361,440,525,559 & 1.5418 & 0.1882 \\
\bottomrule
\end{tabular} \\
\addlinespace
\begin{tabular}{lrrrrrrrrr}
\toprule
Tokenizer & XNLI & XCOPA & Bele & ARC-E & ARC-C & HSwag & PIQA & LAMB & \begin{tabular}[b]{@{}c@{}}English\\mean\end{tabular} \\
\midrule
\begin{tabular}[t]{@{}l@{}}UBE (ours,\\primary tokenizer)\end{tabular} & \textbf{34.4} & 51.8 & 23.7 & \textbf{35.9} & 17.8 & \textbf{27.9} & 60.1 & 23.1 & \textbf{33.0} \\
BBPE & \textbf{34.4} & 51.5 & \textbf{24.0} & 35.3 & 16.8 & \textbf{27.9} & 60.3 & \textbf{23.3} & 32.7 \\
BBPE16 & \textbf{34.4} & \textbf{51.9} & 23.3 & 35.5 & \textbf{18.0} & 27.6 & \textbf{60.8} & 19.5 & 32.3 \\
\bottomrule
\end{tabular}
\end{tabular}
\end{table}
\FloatBarrier

\subsection{Five-tokenizer language-model comparison at matched exposure}
\label{app:matched-exposure-lm}

Table~\ref{tab:cr-matched-exposure-lm} extends the model-level comparison to five 181,688,064-parameter models with 32K vocabularies. Their BPE vocabularies are trained on the same 101-language mC4 prefix (approximately 50 MiB of UTF-8 text per language); all runs share the random seed, 101-language training stream, Transformer blocks, optimizer, evaluation suite, and 807,999,969-byte held-out source. Training steps vary to match UBE's trainer-recorded target of 283,363,416,364 supervised UTF-8 bytes. These tokenizers and the exposure target differ from those in the source-byte-matched comparison above, so BPB values should not be compared across the two protocols.

Multilingual mean equally weights XNLI, XCOPA, and Belebele; English mean equally weights ARC-Easy, ARC-Challenge, HellaSwag, PIQA, and LAMBADA. UBE has the lowest BPB Gini and the highest multilingual and English benchmark-group means, while SCRIPT-BPE has the lowest common-text BPB; UBE's BPB is lower than BBPE and MYTE+BPE. The released 256-symbol MYTE representation is not included directly because matching the target at its observed supervised-byte rate would require approximately 520,943 steps, 2.60$\times$ UBE's 200,000-step run. We instead use a derived, parameter-matched 32K BPE variant over the released MYTE representation. Because MYTE+BPE and SCRIPT-BPE can canonicalize or filter inputs, their common-text BPB is representation-conditioned loss per source byte rather than lossless raw-text code length.

\begin{table}[ht]
\centering
\small
\setlength{\tabcolsep}{5pt}
\caption{Language-model comparison across five tokenizers (182M parameters; 32K vocabularies trained on approximately 50 MiB/language). Single-seed runs match 283.36B supervised UTF-8 bytes. Group means weight benchmarks equally. BPB for MYTE+BPE and SCRIPT-BPE is representation-conditioned; see Appendix~\ref{app:matched-exposure-lm} for details.}
\label{tab:cr-matched-exposure-lm}
\begin{tabular}{lrrrr}
\toprule
Method & Common-text BPB$\downarrow$ & BPB Gini$\downarrow$ & Multi. mean$\uparrow$ & English mean$\uparrow$ \\
\midrule
UBE (ours) & 1.5643 & \textbf{0.1937} & \textbf{36.8} & \textbf{33.0} \\
BBPE & 1.5908 & 0.2006 & 36.6 & 32.7 \\
BBPE16 & 1.5561 & 0.1996 & 36.3 & 32.5 \\
MYTE+BPE (derived 32K) & 1.6474 & 0.2029 & 36.5 & 32.2 \\
SCRIPT-BPE & \textbf{1.5401} & 0.1951 & 36.4 & 32.9 \\
\bottomrule
\end{tabular}
\end{table}

\FloatBarrier

\subsection{Evaluation coverage and code-heavy inputs}

The downstream LM benchmarks XNLI, Belebele, and XCOPA, as well as the tokenizer-level MGSM robustness corpus, include languages whose primary scripts use 3-byte BMP encodings; none of the listed evaluations includes a primary 4-byte non-BMP group (Table~\ref{tab:cr-benchmark-coverage}). Held-out BPB covers all 101 mC4 languages, including 20 primary 3-byte entries. These counts describe each evaluation's standard language entries rather than the number of LM training languages.

\begin{table}[ht]
\centering
\small
\setlength{\tabcolsep}{4pt}
\caption{Language-entry coverage by primary encoding-floor group. Each English benchmark denotes ARC-Easy, ARC-Challenge, HellaSwag, PIQA, or LAMBADA separately.}
\label{tab:cr-benchmark-coverage}
\begin{tabular}{lrrr}
\toprule
Evaluation & Entries & Primary 3-byte BMP & Primary 4-byte non-BMP \\
\midrule
XNLI & 15 & 3 & 0 \\
Belebele & 122 & 26 & 0 \\
XCOPA & 11 & 3 & 0 \\
MGSM & 11 & 5 & 0 \\
Each English benchmark & 1 & 0 & 0 \\
Held-out BPB & 101 & 20 & 0 \\
\bottomrule
\end{tabular}

\end{table}

Because retaining the ASCII floor is central to UBE's design, we also test code-heavy multilingual inputs using 100 SWE-bench tasks~\citep{jimenez2024swebench} and the ten MAPS translations in the public release~\citep{hofman2026maps} (the Chinese translation described in the MAPS paper is not released). Token counts use the primary \texttt{101-lang/32K} UBE and BBPE tokenizers and a 32K BBPE16 tokenizer trained on the same corpus. Aggregated within each language over the 100 aligned tasks, UBE uses fewer final BPE tokens than BBPE for English and every one of the ten translations under both input conditions. Table~\ref{tab:cr-coding-workload} aggregates the 1,000 non-English inputs; intervals use 10,000 paired bootstrap resamples of the 100 aligned task clusters. UBE remains slightly below BBPE even when complete reference-edit file contents make the inputs 99.24\% ASCII bytes, whereas uniform UTF-16 BBPE16 is substantially longer.

\begin{table}[ht]
\centering
\scriptsize
\setlength{\tabcolsep}{3pt}
\caption{Final BPE output-token changes on ten MAPS translations of 100 SWE-bench tasks, using 32K tokenizers trained on the primary 101-language corpus. Negative values favor the named method over BBPE.}
\label{tab:cr-coding-workload}
\resizebox{\linewidth}{!}{\begin{tabular}{lrrr}
\toprule
Input content & ASCII-byte share & UBE $\Delta$ vs. BBPE & BBPE16 $\Delta$ vs. BBPE \\
\midrule
Issue text & 73.91\% & -0.27\% [-0.47\%, -0.15\%] & +21.92\% [+20.56\%, +23.19\%] \\
Issue text + edit-file contents & 99.24\% & -0.18\% [-0.21\%, -0.15\%] & +24.86\% [+24.35\%, +25.40\%] \\
\bottomrule
\end{tabular}
}
\end{table}

\FloatBarrier

\subsection{Corpus-regime sensitivity}
\label{app:corpus-regime}

\paragraph{Protocol.}
We compare the six learned-token composition variants in Table~\ref{tab:cr-composition} across two tokenizer-training regimes: the primary 101-language corpus and an approximately 50 MiB UTF-8-text prefix per language, always at 32K vocabulary size. The cap jointly changes corpus volume and language mixture; it is not an isolated intervention on language balance. Equal byte budgets also do not imply equal semantic content across scripts. The evaluation keeps all 1,012 aligned FLORES-200 devtest sentences in each of the 204 language--script configurations fixed. Following the distinction between parallel-content token counts and word-based fertility~\citep{arnett2025explaining}, we report final output tokens per sentence, without added special tokens, and group by the dataset's script tags. These tags identify configurations, not a character-level partition of mixed-script sentences. Fixed MYTE maps do not depend on the mC4 cap and remain a separate base-symbol comparison (Table~\ref{tab:cr-myte}).

\paragraph{Redistribution across scripts.}
The capped vocabularies use fewer tokens in \regimeMinLower{}--\regimeMaxLower{} of the 204 configurations, depending on boundary policy and representation. The aggregate improvement is therefore not an English or CJK improvement: both groups require more tokens under the cap, whereas the 160 non-English Latin, Arabic, and Cyrillic configurations account for \regimeLatinArabicCyrillicDropShareRange{}\% of the net decrease in the overall mean. Table~\ref{tab:cr-regime-scripts} reports every script tag, including increases for well-covered primary-corpus languages and decreases elsewhere. Within each fixed corpus and boundary policy, substituting UBE for BBPE lowers token-premium Gini by \regimeGiniUbeReductionRange{}\%. Its CJK effect is also corpus-dependent: it is near zero under the primary recipe but reduces CJK tokens under every capped boundary policy (Table~\ref{tab:cr-regime-intervals}). The representation benefit thus persists when the English-heavy training recipe is replaced by the capped one.

\paragraph{Why the cap redistributes tokens.}
The cap changes the training mixture: English falls from \regimeEnglishSharePrimary{}\% to \regimeEnglishShareCapped{}\%, Chinese, Japanese, and Korean from \regimeCjkSharePrimary{}\% to \regimeCjkShareCapped{}\%, and Hindi and Thai from \regimeHindiThaiSharePrimary{}\% to \regimeHindiThaiShareCapped{}\%, while the remaining 95 languages rise from \regimeRemainingSharePrimary{}\% to \regimeRemainingShareCapped{}\%. BPE adds merges in order of corpus frequency, so a language's share sets how much of the shared vocabulary serves its text~\citep{conneau2019xlm,costajussa2022nllb,hayase2024datamixture}. Consistently, English uses fewer distinct tokens on identical text under the cap (see below), and the fraction of CJK character occurrences split below the character rises from \regimeCjkSplitPrimaryRange{}\% to \regimeCjkSplitCappedRange{}\%. In Other BMP, the Myanmar-script configurations, whose Burmese share rises from \regimeBurmeseSharePrimary{}\% to \regimeBurmeseShareCapped{}\%, use \regimeMyanmarReductionRange{}\% fewer tokens, whereas Tibetan, Ol~Chiki, and Tifinagh are not used by any training language and change by at most \regimeAbsentScriptMaxChange{}\%. Total volume falls with the mixture, so these shares describe the redistribution rather than isolate its cause.

\begingroup
\small
\setlength{\tabcolsep}{4pt}
\begin{longtable}{@{}lrlrrrrrr@{}}
\caption{\normalsize Mean output tokens per FLORES-200 sentence by script tag for the six composition variants (32K vocabularies), with tokenizers trained on the primary corpus or on the 50 MiB/language prefix. Both rows of a script use the same sentences; $n$ counts configurations, and Latin includes English.}
\label{tab:cr-regime-scripts}\\
\toprule
 & & & \multicolumn{2}{c}{Standard} & \multicolumn{2}{c}{SuperBPE-style} & \multicolumn{2}{c}{SCRIPT-style} \\
\cmidrule(lr){4-5}\cmidrule(lr){6-7}\cmidrule(lr){8-9}
FLORES script & $n$ & Corpus & UBE & BBPE & UBE & BBPE & UBE & BBPE \\
\midrule
\endfirsthead
\multicolumn{9}{c}{Table \thetable{} (continued)}\\
\toprule
 & & & \multicolumn{2}{c}{Standard} & \multicolumn{2}{c}{SuperBPE-style} & \multicolumn{2}{c}{SCRIPT-style} \\
\cmidrule(lr){4-5}\cmidrule(lr){6-7}\cmidrule(lr){8-9}
FLORES script & $n$ & Corpus & UBE & BBPE & UBE & BBPE & UBE & BBPE \\
\midrule
\endhead
\bottomrule
\endfoot
Arabic & 22 & Primary & 69.549 & 69.687 & 77.190 & 77.862 & 70.394 & 70.450 \\*
 &  & 50 MiB & 62.995 & 63.039 & 68.894 & 68.940 & 63.335 & 63.413 \\
\addlinespace[2pt]
Armenian & 1 & Primary & 78.250 & 78.649 & 91.867 & 92.875 & 79.762 & 80.015 \\*
 &  & 50 MiB & 63.933 & 64.063 & 73.844 & 74.226 & 64.669 & 64.707 \\
\addlinespace[2pt]
Bengali & 3 & Primary & 110.351 & 110.142 & 114.986 & 115.807 & 101.364 & 94.897 \\*
 &  & 50 MiB & 105.400 & 105.439 & 104.012 & 104.017 & 88.834 & 83.498 \\
\addlinespace[2pt]
Cyrillic & 12 & Primary & 65.828 & 66.044 & 75.168 & 75.626 & 66.698 & 66.912 \\*
 &  & 50 MiB & 57.978 & 58.022 & 64.967 & 65.036 & 58.779 & 58.830 \\
\addlinespace[2pt]
Devanagari & 10 & Primary & 89.929 & 90.036 & 67.171 & 67.198 & 58.677 & 58.790 \\*
 &  & 50 MiB & 94.952 & 94.933 & 85.407 & 85.468 & 74.666 & 71.307 \\
\addlinespace[2pt]
Ethiopic & 2 & Primary & 75.036 & 74.330 & 87.280 & 87.651 & 76.581 & 75.698 \\*
 &  & 50 MiB & 65.912 & 65.060 & 74.994 & 73.750 & 66.837 & 65.781 \\
\addlinespace[2pt]
Georgian & 1 & Primary & 85.117 & 82.745 & 100.157 & 97.806 & 86.253 & 83.613 \\*
 &  & 50 MiB & 73.824 & 73.454 & 85.892 & 84.165 & 74.694 & 74.425 \\
\addlinespace[2pt]
Greek & 1 & Primary & 77.161 & 77.342 & 90.232 & 90.522 & 78.051 & 78.100 \\*
 &  & 50 MiB & 67.721 & 67.772 & 77.343 & 77.516 & 68.462 & 68.509 \\
\addlinespace[2pt]
Gujarati & 1 & Primary & 99.345 & 98.813 & 97.409 & 97.932 & 84.574 & 81.210 \\*
 &  & 50 MiB & 93.346 & 93.236 & 85.798 & 83.561 & 69.764 & 67.950 \\
\addlinespace[2pt]
Gurmukhi & 1 & Primary & 98.947 & 98.726 & 94.967 & 95.315 & 83.107 & 82.551 \\*
 &  & 50 MiB & 92.489 & 92.596 & 80.943 & 81.054 & 68.602 & 68.986 \\
\addlinespace[2pt]
Hangul & 1 & Primary & 38.633 & 38.766 & 43.622 & 44.159 & 39.314 & 39.381 \\*
 &  & 50 MiB & 53.995 & 55.464 & 63.085 & 66.668 & 54.941 & 56.510 \\
\addlinespace[2pt]
Han (simplified) & 1 & Primary & 39.359 & 39.183 & 37.628 & 37.506 & 39.681 & 39.416 \\*
 &  & 50 MiB & 53.877 & 54.926 & 51.170 & 51.963 & 53.694 & 54.287 \\
\addlinespace[2pt]
Han (traditional) & 2 & Primary & 43.913 & 44.039 & 43.901 & 43.958 & 43.812 & 43.836 \\*
 &  & 50 MiB & 56.774 & 59.484 & 55.337 & 58.238 & 56.514 & 58.666 \\
\addlinespace[2pt]
Hebrew & 2 & Primary & 63.175 & 63.530 & 74.486 & 74.938 & 64.405 & 64.614 \\*
 &  & 50 MiB & 54.887 & 54.912 & 61.632 & 61.747 & 55.925 & 56.023 \\
\addlinespace[2pt]
Japanese & 1 & Primary & 46.273 & 46.194 & 48.428 & 48.396 & 43.900 & 43.797 \\*
 &  & 50 MiB & 65.507 & 65.421 & 67.953 & 69.249 & 62.043 & 62.537 \\
\addlinespace[2pt]
Khmer & 1 & Primary & 135.804 & 136.047 & 117.627 & 116.502 & 112.703 & 113.635 \\*
 &  & 50 MiB & 127.890 & 127.778 & 96.629 & 96.395 & 90.826 & 93.156 \\
\addlinespace[2pt]
Kannada & 1 & Primary & 114.405 & 115.094 & 105.599 & 105.614 & 90.495 & 87.544 \\*
 &  & 50 MiB & 108.404 & 108.449 & 94.325 & 94.241 & 74.691 & 72.421 \\
\addlinespace[2pt]
Lao & 1 & Primary & 99.961 & 100.289 & 102.399 & 103.144 & 108.217 & 103.767 \\*
 &  & 50 MiB & 85.275 & 85.630 & 82.771 & 85.199 & 84.385 & 79.887 \\
\addlinespace[2pt]
Latin & 127 & Primary & 56.627 & 56.743 & 60.992 & 61.177 & 57.210 & 57.324 \\*
 &  & 50 MiB & 52.236 & 52.279 & 55.379 & 55.435 & 52.710 & 52.749 \\
\addlinespace[2pt]
Malayalam & 1 & Primary & 126.980 & 126.831 & 96.841 & 96.669 & 92.365 & 88.169 \\*
 &  & 50 MiB & 123.319 & 123.284 & 86.055 & 85.563 & 83.323 & 78.650 \\
\addlinespace[2pt]
Myanmar & 2 & Primary & 173.614 & 172.278 & 147.969 & 143.524 & 173.686 & 178.840 \\*
 &  & 50 MiB & 161.930 & 160.966 & 119.463 & 119.620 & 151.898 & 159.353 \\
\addlinespace[2pt]
Ol Chiki & 1 & Primary & 222.944 & 341.287 & 230.274 & 341.440 & 242.811 & 364.459 \\*
 &  & 50 MiB & 239.684 & 341.934 & 242.051 & 342.068 & 251.140 & 364.458 \\
\addlinespace[2pt]
Odia & 1 & Primary & 161.316 & 164.863 & 187.490 & 182.281 & 148.375 & 164.074 \\*
 &  & 50 MiB & 172.301 & 171.531 & 192.737 & 187.334 & 152.069 & 164.080 \\
\addlinespace[2pt]
Sinhala & 1 & Primary & 102.037 & 102.125 & 106.787 & 107.439 & 88.953 & 84.163 \\*
 &  & 50 MiB & 95.973 & 96.134 & 95.025 & 95.128 & 74.749 & 69.862 \\
\addlinespace[2pt]
Tamil & 1 & Primary & 122.921 & 123.013 & 115.659 & 115.820 & 77.683 & 75.619 \\*
 &  & 50 MiB & 121.695 & 121.851 & 104.768 & 104.531 & 73.694 & 70.895 \\
\addlinespace[2pt]
Telugu & 1 & Primary & 110.878 & 111.162 & 100.376 & 100.439 & 89.684 & 87.714 \\*
 &  & 50 MiB & 104.795 & 104.852 & 85.994 & 86.088 & 73.845 & 71.519 \\
\addlinespace[2pt]
Tifinagh & 2 & Primary & 185.780 & 269.911 & 191.536 & 270.293 & 204.546 & 289.680 \\*
 &  & 50 MiB & 196.078 & 270.108 & 195.998 & 270.280 & 204.581 & 289.885 \\
\addlinespace[2pt]
Thai & 1 & Primary & 69.237 & 69.173 & 59.498 & 59.591 & 55.979 & 58.003 \\*
 &  & 50 MiB & 84.894 & 84.610 & 79.968 & 79.458 & 85.782 & 87.061 \\
\addlinespace[2pt]
Tibetan & 2 & Primary & 298.116 & 352.053 & 217.926 & 275.250 & 299.304 & 352.571 \\*
 &  & 50 MiB & 299.339 & 352.121 & 227.372 & 265.607 & 300.190 & 352.571 \\
\end{longtable}
\endgroup

\paragraph{Absolute cost and English normalization.}
Let $m_r(\ell)$ be a language's mean token count under corpus regime $r$, and $p_r(\ell)=m_r(\ell)/m_r(\mathrm{English})$. Then
\begin{equation}
\log\frac{p_{\mathrm{cap}}(\ell)}{p_{\mathrm{primary}}(\ell)}
=\log\frac{m_{\mathrm{cap}}(\ell)}{m_{\mathrm{primary}}(\ell)}
-\log\frac{m_{\mathrm{cap}}(\mathrm{English})}{m_{\mathrm{primary}}(\mathrm{English})}.
\label{eq:regime-premium-decomposition}
\end{equation}
For standard BBPE, the mean non-English premium falls from \regimePrimaryPremium{} to \regimeCappedPremium{}, yet \regimeMaskedConfigurations{} non-English configurations have lower premiums while using more tokens. Gini is invariant to the common English denominator: its increase from \regimePrimaryGini{} to \regimeCappedGini{} reflects a change in the distribution of non-English counts, not English normalization. Reporting both absolute counts and dispersion avoids treating a higher English cost as a multilingual compression gain.

\paragraph{Vocabulary use on identical text.}
The number of distinct token IDs used on the English evaluation sentences falls from \regimePrimaryUbeActive{} to \regimeCappedUbeActive{} for standard UBE and from \regimePrimaryBbpeActive{} to \regimeCappedBbpeActive{} for standard BBPE. These are observed active types on the same held-out text, not the number of vocabulary entries exclusively allocated to English. Together with the count changes, they describe a different use of the fixed vocabulary budget; they do not isolate a causal effect of language balance from corpus volume.

\paragraph{Capped-corpus comparisons.}
Table~\ref{tab:cr-regime-capped} adds official SCRIPT-BPE and derived MYTE+BPE to the six composition variants under the same capped training condition. The seven disjoint cohorts cover all 204 configurations. South Asian uses the ten script tags listed in Appendix~\ref{app:tokenizer-composition}; CJK includes Han (simplified and traditional), Japanese, and Hangul; Southeast Asian includes Khmer, Lao, Myanmar, and Thai. Other is the complement of these named groups. SCRIPT-BPE and MYTE+BPE retain their respective text transforms, whose fidelity is evaluated separately in Table~\ref{tab:cr-unicode-fidelity}. SCRIPT-BPE's smallest overall mean coexists with higher English and non-English Latin counts than standard UBE, while SCRIPT-style boundaries chiefly reduce South Asian counts. These cohort means expose differences that an overall average alone cannot identify.

\begin{table}[htbp]
\centering
\small
\setlength{\tabcolsep}{2.5pt}
\caption{Capped-corpus output tokens per sentence on disjoint FLORES cohorts (approximately 50 MiB/language; 32K vocabularies). SE denotes Southeast; Other is the residual cohort. Overall means weight cohorts by $n$.}
\label{tab:cr-regime-capped}
\begin{tabular}{@{}lrrrrrrrr@{}}
\toprule
Method & EN & \shortstack{Other\\Latin} & \shortstack{Arabic/\\Cyrillic} & \shortstack{South\\Asian} & CJK & \shortstack{SE\\Asian} & Other & All \\
$n$ & 1 & 126 & 34 & 21 & 5 & 5 & 12 & 204 \\
\midrule
Standard UBE & 36.938 & \textbf{52.357} & \textbf{61.224} & 103.716 & 57.385 & 124.384 & 139.800 & 66.079 \\
Standard BBPE & 36.959 & 52.401 & 61.269 & 103.694 & 58.956 & 123.990 & 169.302 & 67.875 \\
SuperBPE-style UBE & 39.524 & 55.505 & 67.508 & 94.845 & 58.576 & \textbf{99.659} & \textbf{133.260} & 67.208 \\
SuperBPE-style BBPE & 39.586 & 55.561 & 67.562 & 94.487 & 60.871 & 100.058 & 160.062 & 68.858 \\
SCRIPT-style UBE & 37.539 & 52.831 & 61.727 & 80.186 & 56.741 & 112.958 & 142.836 & 63.918 \\
SCRIPT-style BBPE & 37.569 & 52.869 & 61.795 & \textbf{77.520} & 58.133 & 115.762 & 175.052 & 65.677 \\
SCRIPT-BPE & 37.982 & 53.276 & 62.648 & 78.162 & \textbf{51.069} & 101.340 & 144.001 & \textbf{63.785} \\
MYTE+BPE & \textbf{36.661} & 52.893 & 61.802 & 104.538 & 58.360 & 124.674 & 159.254 & 67.764 \\
\bottomrule
\end{tabular}

\end{table}

\paragraph{Paired uncertainty.}
For each UBE/BBPE pair, Table~\ref{tab:cr-regime-intervals} resamples the 1,012 aligned sentence indices jointly across all languages and both tokenizers, with 2,000 bootstrap replicates. Each replicate computes the percentage change in the cohort's mean count. The percentile intervals describe sentence-sampling uncertainty conditional on these saved vocabularies and the fixed language set, not variability across tokenizer retraining. All six English and overall intervals favor UBE. For CJK, all three capped intervals favor UBE, whereas the primary SCRIPT-style pair has a small increase. As in the main LM results, shorter output sequences alone do not establish downstream quality gains~\citep{schmidt2024compression}.

\paragraph{Why the CJK effect depends on the corpus.}
UBE's shorter floor saves tokens only where BBPE needs three or more byte pieces for a character. Under the primary corpus this happens for \regimeCjkBbpeThreePiecePrimaryRange{}\% of Han, kana, and Hangul occurrences in the CJK (5) sentences, because frequent characters are whole tokens and partial-character byte tokens cover most others in two pieces~\citep{wang2020bbpe}; under the 50~MiB/language cap it happens for \regimeCjkBbpeThreePieceCappedRange{}\%, and UBE lowers CJK counts by \regimeCjkUbeReductionRange{}\%. UBE also splits more characters whose whole-character token is in its vocabulary (\regimeCjkUbeWholeSplitPrimaryRange{}\% of occurrences under the primary corpus, versus \regimeCjkBbpeWholeSplitPrimaryRange{}\% for BBPE). \utfEightName{} distinguishes lead from continuation bytes, but UBE writes both bytes of each \utfSixteenName{} code unit with the same 256 PUA symbols, so the high byte of one character and the low byte of the next can form a learned merge, sometimes another character's token (in \begin{CJK}{UTF8}{gbsn}典广\end{CJK} they spell \begin{CJK}{UTF8}{gbsn}网\end{CJK}, U+7F51). Under the primary corpus the CJK (5) mean stays within 0.3\% of BBPE; per configuration, UBE is higher only for Simplified Chinese and Japanese, where BBPE splits the fewest characters into three pieces. SCRIPT's two-symbol encoding likewise matches byte-level BPE on Chinese and Japanese once both are trained with the same pretokenizer~\citep{land2025script}.

\begin{table}[htbp]
\centering
\small
\setlength{\tabcolsep}{4pt}
\caption{BBPE-to-UBE token-count changes (\%) and 95\% paired-sentence bootstrap intervals at fixed corpus and boundary policy. Negative values favor UBE.}
\label{tab:cr-regime-intervals}
\begin{tabular}{@{}llrrr@{}}
\toprule
Corpus & Boundary & English & CJK (5) & All (204) \\
\midrule
Primary & Standard & -0.29 [-0.36, -0.23] & -0.06 [-0.13, +0.01] & -2.83 [-2.84, -2.81] \\
 & SuperBPE-style & -0.51 [-0.60, -0.43] & -0.23 [-0.31, -0.14] & -2.76 [-2.78, -2.75] \\
 & SCRIPT-style & -0.14 [-0.19, -0.10] & +0.12 [+0.05, +0.18] & -2.80 [-2.81, -2.78] \\
\addlinespace[3pt]
50 MiB & Standard & -0.06 [-0.08, -0.04] & -2.66 [-2.78, -2.55] & -2.65 [-2.66, -2.63] \\
 & SuperBPE-style & -0.16 [-0.20, -0.12] & -3.77 [-3.88, -3.67] & -2.40 [-2.41, -2.38] \\
 & SCRIPT-style & -0.08 [-0.11, -0.05] & -2.39 [-2.49, -2.30] & -2.68 [-2.70, -2.66] \\
\bottomrule
\end{tabular}

\end{table}
\FloatBarrier

\section{LM and inference benchmark results}
\label{app:lm-inference-results}

\subsection{Full LM benchmark results}
\label{app:lm-eval-full}

Table~\ref{tab:lm-eval-full} reports the complete per-benchmark zero-shot accuracy for the main 182M/32K LM runs and adds the SuperBPE composition rows for UBE and BBPE (S-UBE and S-BBPE) omitted from the main text. Tables~\ref{tab:lm-eval-1.3b},~\ref{tab:lm-eval-64k}, and~\ref{tab:lm-eval-128k} report the corresponding 1.3B, 64K, and 128K variants.

\begin{table}[ht]
\centering
\footnotesize
\setlength{\tabcolsep}{2.5pt}
\caption{Full LM benchmark results (182M, 32K, 200K mC4 steps). Benchmark columns report zero-shot accuracy. XNLI/Belebele/XCOPA are multilingual; ARC-Easy/ARC-Challenge, HellaSwag, PIQA, and LAMBADA are English-only.}
\label{tab:lm-eval-full}
\begin{tabular}{llcccccccccc}
\toprule
Config & Tok & BPB & BPB Gini & XNLI & Bele & XCOPA & ARC-E & ARC-C & HSwag & PIQA & LAMB \\
\midrule
En-only & UBE & \textbf{1.4698} & \textbf{0.1989} & \textbf{34.4} & 23.0 & \textbf{51.5} & \textbf{38.5} & 17.0 & \textbf{27.8} & \textbf{62.7} & 22.5 \\
En-only & BBPE & 1.4909 & 0.1993 & 34.3 & \textbf{23.6} & 51.4 & 38.1 & \textbf{18.3} & 27.5 & 61.9 & \textbf{23.6} \\
\midrule
6-lang & UBE & 1.3760 & \textbf{0.1880} & \textbf{34.9} & 23.2 & \textbf{52.2} & \textbf{36.7} & \textbf{18.4} & \textbf{28.0} & \textbf{62.0} & \textbf{23.1} \\
6-lang & BBPE & \textbf{1.3680} & 0.1930 & 34.6 & \textbf{23.5} & 52.0 & 35.6 & 17.1 & 27.4 & 61.3 & \textbf{23.1} \\
\midrule
101-lang & UBE & 1.5237 & \textbf{0.1739} & \textbf{34.4} & 23.7 & 51.8 & \textbf{35.9} & \textbf{17.8} & \textbf{27.9} & \textbf{60.1} & \textbf{23.1} \\
101-lang & BBPE & \textbf{1.5174} & 0.1762 & 34.3 & \textbf{24.8} & \textbf{52.0} & 35.5 & 17.0 & 27.7 & 59.3 & 22.8 \\
\midrule
101-lang (SuperBPE) & S-UBE & \textbf{1.4607} & 0.1822 & 34.1 & \textbf{24.5} & 52.1 & 33.0 & 17.7 & \textbf{27.7} & \textbf{59.7} & \textbf{23.9} \\
101-lang (SuperBPE) & S-BBPE & 1.4876 & \textbf{0.1802} & \textbf{34.3} & 24.4 & \textbf{52.4} & \textbf{33.7} & \textbf{18.1} & 27.6 & \textbf{59.7} & 22.2 \\
\bottomrule
\end{tabular}
\end{table}

\begin{table}[ht]
\centering
\small
\setlength{\tabcolsep}{2.5pt}
\caption{LM evaluation (1.3B, 32K vocabulary, 205K Chinchilla-style mC4 steps) across the same three tokenizer-training configurations as Table~\ref{tab:lm-eval}. Lower-coverage rows (\texttt{En-only}, \texttt{6-lang}) are fixed-compute stress tests. Benchmark columns report zero-shot accuracy.}
\label{tab:lm-eval-1.3b}
\begin{tabular}{llcccccccccc}
\toprule
Config & Tok & BPB & BPB Gini & XNLI & Bele & XCOPA & ARC-E & ARC-C & HSwag & PIQA & LAMB \\
\midrule
En-only & UBE & \textbf{1.3668} & \textbf{0.1956} & \textbf{34.8} & \textbf{22.9} & \textbf{52.8} & 37.3 & \textbf{19.0} & 28.2 & 62.3 & \textbf{27.0} \\
En-only & BBPE & 1.4193 & 0.1971 & 34.5 & \textbf{22.9} & 52.3 & \textbf{39.4} & 17.4 & \textbf{28.5} & \textbf{62.6} & 25.6 \\
\midrule
6-lang & UBE & 1.3106 & 0.1973 & \textbf{35.0} & \textbf{23.9} & 52.6 & 39.6 & 18.2 & \textbf{29.2} & 63.0 & 28.6 \\
6-lang & BBPE & \textbf{1.3071} & \textbf{0.1956} & 34.8 & 23.1 & \textbf{52.9} & \textbf{40.4} & \textbf{18.9} & \textbf{29.2} & \textbf{63.5} & \textbf{28.9} \\
\midrule
101-lang & UBE & 1.4707 & 0.1983 & 34.5 & \textbf{23.4} & \textbf{51.9} & \textbf{39.9} & \textbf{18.4} & \textbf{29.1} & 63.2 & 29.1 \\
101-lang & BBPE & \textbf{1.4673} & \textbf{0.1972} & \textbf{35.0} & 23.0 & 51.1 & 39.1 & 18.1 & \textbf{29.1} & \textbf{64.0} & \textbf{29.9} \\
\bottomrule
\end{tabular}
\end{table}

\begin{table}[ht]
\centering
\small
\setlength{\tabcolsep}{2.5pt}
\caption{LM evaluation (230M, 64K vocabulary, 200K mC4 steps) across the same three tokenizer-training configurations as Table~\ref{tab:lm-eval}. Benchmark columns report zero-shot accuracy.}
\label{tab:lm-eval-64k}
\begin{tabular}{llcccccccccc}
\toprule
Config & Tok & BPB & BPB Gini & XNLI & Bele & XCOPA & ARC-E & ARC-C & HSwag & PIQA & LAMB \\
\midrule
En-only & UBE & \textbf{1.4588} & \textbf{0.2002} & \textbf{34.2} & \textbf{23.2} & 51.9 & \textbf{40.2} & \textbf{18.4} & \textbf{27.6} & \textbf{61.9} & \textbf{22.6} \\
En-only & BBPE & 1.4879 & 0.2006 & \textbf{34.2} & 23.0 & \textbf{52.1} & \textbf{40.2} & 18.1 & \textbf{27.6} & 61.2 & 20.6 \\
\midrule
6-lang & UBE & 1.4456 & 0.1812 & 34.5 & 22.9 & 51.6 & \textbf{38.1} & \textbf{19.3} & 27.8 & \textbf{62.1} & \textbf{22.8} \\
6-lang & BBPE & \textbf{1.4329} & \textbf{0.1751} & \textbf{34.7} & \textbf{23.0} & \textbf{52.0} & 37.7 & 18.9 & \textbf{28.0} & 61.3 & 22.4 \\
\midrule
101-lang & UBE & 1.4435 & 0.2047 & \textbf{34.8} & \textbf{23.4} & \textbf{51.9} & 38.2 & \textbf{18.1} & 27.6 & 61.0 & \textbf{24.7} \\
101-lang & BBPE & \textbf{1.4431} & \textbf{0.2037} & \textbf{34.8} & 23.2 & 51.5 & \textbf{39.5} & 17.6 & \textbf{28.1} & \textbf{61.4} & 23.2 \\
\bottomrule
\end{tabular}
\end{table}

\begin{table}[ht]
\centering
\small
\setlength{\tabcolsep}{2.5pt}
\caption{LM evaluation (329M, 128K vocabulary, 200K mC4 steps; 101-lang). S-UBE and S-BBPE provide SuperBPE-style composition checks. Benchmark columns report zero-shot accuracy.}
\label{tab:lm-eval-128k}
\begin{tabular}{llcccccccccc}
\toprule
Config & Tok & BPB & BPB Gini & XNLI & Bele & XCOPA & ARC-E & ARC-C & HSwag & PIQA & LAMB \\
\midrule
101-lang & UBE & \textbf{1.5028} & \textbf{0.2021} & \textbf{34.9} & \textbf{23.1} & \textbf{51.6} & 40.0 & \textbf{19.4} & \textbf{28.7} & 62.2 & \textbf{26.7} \\
101-lang & BBPE & 1.5044 & 0.2029 & 34.6 & 23.0 & 51.4 & \textbf{41.0} & \textbf{19.4} & 28.3 & \textbf{62.4} & 25.5 \\
\midrule
101-lang & S-UBE & \textbf{1.4929} & 0.1914 & \textbf{34.3} & 23.4 & \textbf{52.4} & \textbf{39.9} & 17.8 & 28.2 & \textbf{62.6} & 26.7 \\
101-lang & S-BBPE & 1.5064 & \textbf{0.1912} & \textbf{34.3} & \textbf{24.2} & 52.3 & 39.4 & \textbf{18.0} & \textbf{28.4} & 60.8 & \textbf{29.9} \\
\bottomrule
\end{tabular}
\end{table}

\FloatBarrier

\subsection{Multi-seed stability}
\label{app:multiseed}

To verify that the 182M results are not overly sensitive to random initialization, we repeat the \texttt{101-lang/32K} configuration (182M) with \lmMultiseedNSeeds{} independent seeds (Table~\ref{tab:lm-eval-multiseed}). Absolute macro BPB varies across seeds, but the paired same-seed UBE--BBPE gap is small: mean $\Delta$BPB is \lmMsBpbDeltaMean{}, and the largest absolute same-seed gap is \lmMsBpbDeltaMax{}. We therefore interpret the multi-seed result as showing no systematic tokenizer-level quality advantage in this setting.

\begin{table}[ht]
\centering
\small
\setlength{\tabcolsep}{4pt}
\caption{Multi-seed LM evaluation (182M, 32K vocabulary, 101-lang, 3 seeds). Mean rows report mean $\pm$ cross-seed standard deviation for BPB, BPB Gini, and benchmarks. Benchmark columns report zero-shot accuracy. Matching run labels use the same seed.}
\label{tab:lm-eval-multiseed}
\begin{tabular}{llrrrrr}
\toprule
\multicolumn{7}{l}{Language modeling and multilingual accuracy} \\
Tok & Run & BPB & BPB Gini & XNLI & Bele & XCOPA \\
\midrule
UBE & A & 1.4717 & 0.1879 & 34.9 & 23.5 & 51.8 \\
 & B & 1.4747 & 0.1849 & 34.7 & 23.4 & 51.7 \\
 & Primary & 1.5237 & 0.1739 & 34.4 & 23.7 & 51.8 \\
\cmidrule{2-7}
 & mean & 1.4900{$\pm$0.0292} & 0.1822{$\pm$0.0074} & 34.7{$\pm$0.3} & 23.5{$\pm$0.2} & 51.8{$\pm$0.0} \\
\midrule
BBPE & A & 1.4674 & 0.1863 & 34.7 & 23.6 & 51.8 \\
 & B & 1.4715 & 0.1833 & 34.3 & 24.5 & 51.8 \\
 & Primary & 1.5174 & 0.1762 & 34.3 & 24.8 & 52.0 \\
\cmidrule{2-7}
 & mean & 1.4854{$\pm$0.0278} & 0.1819{$\pm$0.0052} & 34.4{$\pm$0.2} & 24.3{$\pm$0.7} & 51.9{$\pm$0.1} \\
\bottomrule
\end{tabular}
\par\medskip
\begin{tabular}{llrrrrr}
\toprule
\multicolumn{7}{l}{English accuracy} \\
Tok & Run & ARC-E & ARC-C & HSwag & PIQA & LAMB \\
\midrule
UBE & A & 37.2 & 17.6 & 27.8 & 60.9 & 24.4 \\
 & B & 36.8 & 17.8 & 27.9 & 61.4 & 24.0 \\
 & Primary & 35.9 & 17.8 & 27.9 & 60.1 & 23.1 \\
\cmidrule{2-7}
 & mean & 36.6{$\pm$0.7} & 17.7{$\pm$0.1} & 27.9{$\pm$0.1} & 60.8{$\pm$0.7} & 23.9{$\pm$0.7} \\
\midrule
BBPE & A & 36.7 & 17.4 & 27.9 & 60.3 & 24.7 \\
 & B & 36.3 & 17.8 & 27.7 & 60.0 & 25.1 \\
 & Primary & 35.5 & 17.0 & 27.7 & 59.3 & 22.8 \\
\cmidrule{2-7}
 & mean & 36.2{$\pm$0.6} & 17.4{$\pm$0.4} & 27.8{$\pm$0.1} & 59.9{$\pm$0.5} & 24.2{$\pm$1.2} \\
\bottomrule
\end{tabular}
\end{table}

\subsection{Context-window and prompt-processing benchmark}
\label{app:inference-benchmark}

Table~\ref{tab:inference-benchmark} reports same-content prompt-processing measurements for UBE and BBPE\@. We include this benchmark as a deployment-side check of the main context-window claim: shorter prompts for the same semantic content move byte-fallback-heavy scripts farther from the truncation boundary. We describe the full methodology below, informed by the benchmarking principles of \citet{hoefler2015scientific} and \citet{georges2007statistically}.

\begin{table}[ht]
\centering
\small
\caption{Same-content prompt-processing benchmark on 5-sentence FLORES-200 bundles (\texttt{6-lang/32K}, 182M, A100 80\,GB, batch~1, FP32, \texttt{torch.compile}). Token columns are bundle lengths; rates use 30 trials with 95\% CI half-widths over 28 samples after symmetric min/max trimming. $\Delta$tok = (UBE$-$BBPE)/BBPE and $\Delta$spd = relative sequence-rate change.}
\label{tab:inference-benchmark}
\setlength{\tabcolsep}{3pt}
\resizebox{\linewidth}{!}{%
\begin{tabular}{lrrrrrr}
\toprule
Language & UBE tok & BBPE tok & $\Delta$tok & UBE rate (seq/s, $\pm$95\% CI) & BBPE rate (seq/s, $\pm$95\% CI) & $\Delta$spd \\
\midrule
Amharic & 976 & 1438 & -32.1\% & 46.0\,$\pm$\,0.02 & 31.8\,$\pm$\,0.01 & +45.01\% \\
Arabic & 957 & 975 & -1.8\% & 48.1\,$\pm$\,0.02 & 46.1\,$\pm$\,0.01 & +4.29\% \\
English & 223 & 226 & -1.3\% & 151.0\,$\pm$\,0.41 & 145.3\,$\pm$\,0.15 & +3.90\% \\
Hindi & 672 & 672 & 0.0\% & 65.3\,$\pm$\,0.02 & 65.2\,$\pm$\,0.03 & +0.08\% \\
Japanese & 347 & 354 & -2.0\% & 108.7\,$\pm$\,0.05 & 107.5\,$\pm$\,0.05 & +1.07\% \\
Korean & 253 & 260 & -2.7\% & 143.9\,$\pm$\,0.10 & 140.1\,$\pm$\,0.10 & +2.76\% \\
Burmese & 1074 & 1080 & -0.6\% & 43.4\,$\pm$\,0.01 & 43.4\,$\pm$\,0.01 & +0.01\% \\
Russian & 811 & 822 & -1.3\% & 54.7\,$\pm$\,0.02 & 54.6\,$\pm$\,0.01 & +0.12\% \\
Thai & 480 & 482 & -0.4\% & 85.9\,$\pm$\,0.03 & 85.5\,$\pm$\,0.04 & +0.49\% \\
Chinese (Simplified) & 278 & 276 & +0.7\% & 139.5\,$\pm$\,0.09 & 139.4\,$\pm$\,0.10 & +0.04\% \\
\bottomrule
\end{tabular}%
}
\end{table}

\paragraph{Hardware and software.}
All measurements are conducted on a single NVIDIA A100 80\,GB GPU with exclusive access, using PyTorch \texttt{torch.compile} (Inductor backend, \texttt{dynamic=True}), batch size~1, FP32 precision, and single-stream scheduling consistent with the MLPerf Inference single-stream scenario~\citep{reddi2020mlperf}. The \texttt{6-lang/32K} tokenizer configuration is used to separate well-covered scripts (en, zh, ko, ja, hi, th) from byte-fallback-heavy scripts (my, am, ar, ru).

\paragraph{Timing and trials.}
Each language is measured independently with CUDA event timing. Following \citet{hoefler2015scientific}, we report the timing mechanism and explicitly control for timer overhead, synchronization, warmup, and measurement variability. We execute 5 warmup iterations (to stabilize \texttt{torch.compile} code generation, cuDNN autotuning, and GPU thermal state), followed by \infNTrials{} independent trials of 20 forward passes each for each tokenizer condition. We apply symmetric min/max outlier removal (1 from each tail, $\sim$3.3\% trim) to guard against cold-cache and scheduling artifacts. Following \citet{georges2007statistically}, we report $t$-distribution confidence intervals, appropriate for the \infNEffective{} effective samples remaining after trimming ($n < 30$). The primary summary is the trimmed mean $\pm$ 95\% CI; the significance tests below use all \infNTrials{} trials without trimming.

\paragraph{Prompt-processing measurement.}
We report only prompt processing. Transformer prefill contains a self-attention term that scales as $O(n^2 d)$ in sequence length~\citep{vaswani2017attention}; I/O-optimized implementations reduce memory traffic but do not remove this sequence-length dependence~\citep{dao2022flashattention}. Prompt processing is therefore the most direct deployment-side measurement we report for tokenizer-induced context-window pressure and truncation risk. Autoregressive decode, by contrast, is often memory-bandwidth-bound~\citep{shazeer2019fast,pope2023efficiently}, with per-step attention cost scaling linearly in key--value (KV) cache length, so input-length reductions still help but through a different cost path.

\paragraph{Expected vs.\ observed speedup.}
Let $r = n_{\text{UBE}}/n_{\text{BBPE}}$ denote the UBE/BBPE token-count ratio on the same prompt (derivable from the $\Delta$tok column in Table~\ref{tab:inference-benchmark}). The floating-point operation (FLOP)-only scaling terms $1/r - 1$ for MLP-dominated prefill ($O(n d^2)$) and $1/r^2 - 1$ for attention-dominated prefill ($O(n^2 d)$) provide reference points, but launch overheads, embedding/output layers, and memory effects shift the observed rate. At the width tested here ($d = 768$, 182M scale) and the prompt lengths in Table~\ref{tab:inference-benchmark}, the MLP term $n d^2$ is comparable to or larger than the attention term $n^2 d$ for most languages, so observed rates fall closer to linear-in-$n$ than to quadratic. This matches the data: Amharic (\infDeltaTokamhEthi{}, \infDeltaSeqamhEthi{}) shows the largest speedup; Hindi and Thai, whose token counts barely change, show small but statistically significant speed changes (\infDeltaSeqhinDeva{}, \infDeltaSeqthaThai{}); and English and Arabic deviate from FLOP-only estimates, which we attribute to fixed and kernel-level effects.

\paragraph{Fair comparison metric.}
We report sequences/second (seq/s) rather than tokens/second. Because different tokenizers produce different token counts for the same text---with disparities up to $15\times$ across languages~\citep{petrov2023unfairness}---tokens/second conflates token count with processing rate. Normalizing by the original text content, following prior content-matched multilingual cost comparisons~\citep{ahia2023languages}, yields a metric that directly compares how quickly each system ingests the same natural-language input under the same model and hardware budget.

\paragraph{Statistical tests.}
For each language, we compute Welch's $t$-test (unpaired, unequal variance) and Mann--Whitney $U$ as a distribution-free check. Welch's test is appropriate because UBE and BBPE are benchmarked as separate timing samples and produce different sequence lengths for the same text. To control for multiple comparisons across 10 languages, we apply Benjamini--Hochberg false discovery rate (FDR) correction at $\alpha = 0.05$. The compact table reports rates and confidence intervals; after FDR correction, \infNSigLangsBH{}/10 language pairs are significant, while Chinese (Simplified) and Burmese are not, matching their near-zero measured speed changes. Error bars capture run-to-run variability under otherwise fixed hardware and software conditions; residual variance arises primarily from GPU scheduling, thermal drift, and measurement noise.

\FloatBarrier
\section{Reproducibility and compute disclosure}
\label{app:supplementary}

\codeavailability{} The five-way LM comparison in Table~\ref{tab:cr-matched-exposure-lm} uses private training and evaluation infrastructure; its runner, runtime, and configurations are not part of this release. Raw web-corpus shards are reconstructed from public sources rather than redistributed.

\subsection{Tokenizer artifact card}
\label{app:tokenizer-card}

The UBE tokenizer JSONs used in this work are documented as follows:
\begin{itemize}
    \item \textbf{Scope and use.} The paper's FLORES-200 evaluations and LM runs; each artifact is a standard Hugging Face \texttt{tokenizer.json} with unchanged BPE merge semantics.
    \item \textbf{Data and procedure.} Fixed-capped mC4 subsets and two-phase BPE training are documented in Appendix Sections~\ref{app:tokenizer-training} and~\ref{app:corpus-accounting}.
    \item \textbf{Evaluation and limitations.} Cross-dataset robustness, LM evaluation, and per-language FLORES-200 tables are reported in Appendix Sections~\ref{app:cross-dataset-vocab},~\ref{app:lm-eval-full}, and~\ref{app:per-language-results}; limitations are collected in Appendix~\ref{app:additional-limitations}.
    \item \textbf{License.} The UBE implementation builds on the Hugging Face \texttt{tokenizers} library and follows its Apache~2.0 license (Table~\ref{tab:asset-licenses-software}).
\end{itemize}

\subsection{Compute disclosure}
\label{app:compute-disclosure}
The primary 28-run LM suite and its single-GPU evaluations ran on internal Linux nodes with NVIDIA A100 80\,GB GPUs. Tokenizer training used our two-phase CPU pipeline. Table~\ref{tab:compute-disclosure} gives approximate active compute for this suite only; it excludes the additional comparisons in Tables~\ref{tab:cr-source-matched-accuracy} and~\ref{tab:cr-matched-exposure-lm}. For LM rows, GPU-hours per run are wall-clock hours multiplied by 8 GPUs; wall-clock ranges are coarse model-size-adjusted estimates based on the 150M/450M/1.2B reference budgets and exclude queueing, download, preprocessing, and checkpoint-transfer time.

\begin{table}[ht]
\centering
\footnotesize
\setlength{\tabcolsep}{4pt}
\caption{Approximate active compute budget for the primary 28-run LM suite. GPU-hours per run are wall-clock hours multiplied by 8 GPUs for LM training rows.}
\label{tab:compute-disclosure}
\begin{tabular}{L{3.5cm}L{1.7cm}L{0.9cm}L{1.8cm}L{1.8cm}L{2.0cm}}
\toprule
Stage & Configured steps & Runs & Wall-clock/run & GPU-hours/run & Total GPU-hours \\
\midrule
182M/32K LM pre-training & \lmOneEightTwoMConfiguredSteps{} & \lmOneEightTwoMRuns{} & \lmOneEightTwoMWallClockRange{} & \lmOneEightTwoMGpuHrsPerRun{} & \lmOneEightTwoMGpuHrs{} \\
230M/64K LM pre-training & \lmTwoThreeZeroMConfiguredSteps{} & \lmTwoThreeZeroMRuns{} & \lmTwoThreeZeroMWallClockRange{} & \lmTwoThreeZeroMGpuHrsPerRun{} & \lmTwoThreeZeroMGpuHrs{} \\
329M/128K LM pre-training & \lmThreeTwoNineMConfiguredSteps{} & \lmThreeTwoNineMRuns{} & \lmThreeTwoNineMWallClockRange{} & \lmThreeTwoNineMGpuHrsPerRun{} & \lmThreeTwoNineMGpuHrs{} \\
1.3B/32K LM pre-training & \lmOneThreeBConfiguredSteps{} & \lmOneThreeBRuns{} & \lmOneThreeBWallClockRange{} & \lmOneThreeBGpuHrsPerRun{} & \lmOneThreeBGpuHrs{} \\
Downstream eval.\ + inference & n/a & agg. & $<$\downstreamGpuHrs{} h & $<$\downstreamGpuHrs{} & $<$\downstreamGpuHrs{} \\
\bottomrule
\end{tabular}
\end{table}

Multiplying the primary-suite per-run ranges by their run counts gives approximately \totalAcceleratorHrs{} A100~80\,GB GPU-hours, plus $<$\downstreamGpuHrs{} GPU-hours for downstream evaluation and inference benchmarks. The CPU-only tokenizer matrix used about \cpuTokTrainHrs{} hours wall-clock and \cpuTokWorkerHrs{} CPU-worker-hours; intrinsic evaluation, contamination checks, and statistics used about \cpuEvalCoreHrs{} additional CPU-hours. The primary-corpus SCRIPT-style vocabulary construction additionally took 4.09 hours for UBE and 9.23 hours for BBPE using eight CPU threads, excluding the preceding collection of corpus statistics.

Table~\ref{tab:cr-source-matched-accuracy} reuses the primary UBE checkpoint and adds two training runs: BBPE on H100 and BBPE16 on A100. Table~\ref{tab:additional-compute} reports the recorded training times for the five additional runs underlying Table~\ref{tab:cr-matched-exposure-lm}. The available records do not establish complete accelerator-time accounting for these additional comparisons, so they are not folded into the primary-suite A100 estimate.

\begin{table}[ht]
\centering
\small
\caption{Recorded training wall-clock time for the five-tokenizer matched-exposure comparison (Table~\ref{tab:cr-matched-exposure-lm}), separate from the primary-suite budget.}
\label{tab:additional-compute}
\begin{tabular}{lrr}
\toprule
Method & Steps & Training time (hours) \\
\midrule
UBE & 200{,}000 & 12.94 \\
BBPE & 178{,}549 & 24.76 \\
BBPE16 & 177{,}699 & 24.68 \\
MYTE+BPE & 179{,}526 & 12.06 \\
SCRIPT-BPE & 170{,}191 & 11.39 \\
\bottomrule
\end{tabular}
\end{table}

\FloatBarrier

\section{Positioning, related work, and limitations}
\label{app:prior-work-scale}

Table~\ref{tab:prior-work-scale} summarizes the experimental scope of selected tokenizer studies. UBE combines tokenizer training on 101 languages, intrinsic evaluation across 204 language configurations, and a primary suite of \lmTotalReportedLmRuns{} LM-training runs. For \citet{abagyan2025onetokenizer}, the two multilingual benchmarks are Belebele and M-MMLU; Aya denotes the separate open-ended generation evaluation.

\begin{table}[ht]
\centering
\footnotesize
\setlength{\tabcolsep}{3pt}
\caption{Experimental scope of selected tokenizer studies. Tok.\ langs = tokenizer-training languages, eval langs = intrinsic-evaluation languages, and benchmarks = downstream benchmarks. NR = intrinsic-evaluation language count not separately reported; mono/multi = monolingual/multilingual. The UBE row covers the primary 28-run suite.}
\label{tab:prior-work-scale}
\begin{tabular}{@{}llcccccc@{}}
\toprule
Paper & Venue & \begin{tabular}[c]{@{}c@{}}Tok.\\langs\end{tabular} & \begin{tabular}[c]{@{}c@{}}Eval\\langs\end{tabular} & LM size & LM tokens & Benchmarks & Vocab \\
\midrule
Arnett+ & NeurIPS'25 & 97 & 97 & --- & --- & --- & 8--262K \\
Schmidt+ & EMNLP'24 & 1 (en) & 1 & \begin{tabular}[c]{@{}c@{}}350M--2.4B\\(64 runs)\end{tabular} & 200B & 10 EN & 33--49K \\
SuperBPE & COLM'25 & 1 (en) & 1 & 680M--11B & \begin{tabular}[c]{@{}c@{}}$\sim$330B\\(8B baseline)\end{tabular} & 30 EN & 200K \\
Abagyan+ & arXiv'25 & 62 & NR & 3.3B & 100B+10.5B & \begin{tabular}[c]{@{}c@{}}2 multi+11 EN\\+ Aya\end{tabular} & 100--250K \\
Land \& Arnett & TokShop'25 & 12 (mono) & 12 (mono) & --- & --- & --- & \begin{tabular}[c]{@{}c@{}}64K (mono);\\256K merges\\(multi)\end{tabular} \\
MAGNET & NeurIPS'24 & 9 & 9 & 126M & $\sim$10B bytes & 4 multi+2 & \begin{tabular}[c]{@{}c@{}}50--250K\\(BPE baselines)\end{tabular} \\
\midrule
UBE (ours) & this work & 101 & 204 & \begin{tabular}[c]{@{}c@{}}182M--1.35B\\(\lmTotalReportedLmRuns{} runs)\end{tabular} & 105B+27B & 3 multi+5 EN & 16--256K \\
\bottomrule
\end{tabular}

\vspace{2pt}
\begin{minipage}{0.95\linewidth}
\raggedright\scriptsize
Sources: Arnett+~\citep{arnett2025explaining}; Schmidt+~\citep{schmidt2024compression}; SuperBPE~\citep{liu2025superbpe}; Abagyan+~\citep{abagyan2025onetokenizer}; Land \& Arnett~\citep{land2025script}; MAGNET~\citep{ahia2024magnet}. X+ denotes X et al.
\end{minipage}
\end{table}

\subsection{Public large language model tokenizer comparison}
\label{app:public-llm-gini}

Table~\ref{tab:public-llm-gini} reports the FLORES-200 Gini coefficient and total token count for six tokenizers used by public LLMs, corresponding to the reference points in Figure~\ref{fig:gini-vs-tokens}. These tokenizers span a wide vocabulary range (100K--248K) and represent distinct design philosophies: tiktoken-based or tiktoken-derived BPE (GPT-4, GPT-4o, Llama~3), Qwen-family BPE variants, and Mistral~NeMo's Tekken tokenizer. The Qwen~2.5 row uses the tokenizer released with Qwen2.5-7B, which is identical to the Qwen2.5-3B tokenizer; the Qwen~3.6 row follows the official model-card/configuration metadata for the larger Qwen3.6-27B multimodal tokenizer and its special-token metadata~\citep{qwenteam2026qwen36card}. All measurements use the same 204-language FLORES-200 devtest evaluation as the paper's intrinsic benchmarks.

\begin{table}[ht]
\centering
\footnotesize
\setlength{\tabcolsep}{4pt}
\caption{FLORES-200 Gini coefficient and total token count for tokenizers used by public LLMs, included as Figure~\ref{fig:gini-vs-tokens} reference points.}
\label{tab:public-llm-gini}
\begin{tabular}{lrrrrl}
\toprule
Tokenizer & Vocab & EN avg & Total (M) & Gini & Type \\
\midrule
GPT-4o (o200k)    & \pltGptfouroVocab{} & \pltGptfouroEnAvg{} & \pltGptfouroTotalM{} & \pltGptfouroGini{} & tiktoken \\
GPT-4 (cl100k)    & \pltGptfourVocab{} & \pltGptfourEnAvg{} & \pltGptfourTotalM{} & \pltGptfourGini{} & tiktoken \\
Llama~3           & \pltLlamathreeVocab{} & \pltLlamathreeEnAvg{} & \pltLlamathreeTotalM{} & \pltLlamathreeGini{} & tiktoken-BPE \\
Qwen~2.5          & \pltQwenTwoFiveVocab{} & \pltQwenTwoFiveEnAvg{} & \pltQwenTwoFiveTotalM{} & \pltQwenTwoFiveGini{} & BPE \\
Qwen~3.6 multimodal tokenizer$^{\dagger}$ & \pltQwenthreesixVocab{} & \pltQwenthreesixEnAvg{} & \pltQwenthreesixTotalM{} & \pltQwenthreesixGini{} & BPE \\
Mistral~NeMo      & \pltMistralVocab{} & \pltMistralEnAvg{} & \pltMistralTotalM{} & \pltMistralGini{} & Tekken \\
\bottomrule
\end{tabular}

\vspace{2pt}
\begin{minipage}{0.95\linewidth}
\raggedright\scriptsize
$^{\dagger}$ The Qwen~3.6 multimodal tokenizer is taken from the \texttt{Qwen/Qwen3.6-27B} release. The official configuration reports a padded \texttt{vocab\_size} of 248{,}320, while this table reports the measured tokenizer vocabulary scale used by our FLORES-200 token-count script (\pltQwenthreesixVocab{}); the row is therefore an upper-end reference point and is not directly comparable to the Qwen~2.5 text-model row.\par
Tokenizer citations: tiktoken~\citep{openai2023tiktoken}; Llama~3~\citep{grattafiori2024llama3}; Qwen~2.5 report and tokenizer metadata~\citep{yang2024qwen25report,qwenteam2024qwen25card}; Qwen~3.6 model-card metadata~\citep{qwenteam2026qwen36card}; Mistral~NeMo release and tokenizer metadata~\citep{mistralai2024mistralnemo,mistralai2024mistralnemocard}.
\end{minipage}
\end{table}

\subsection{Representative compact language models and tokenizer scales}
\label{app:compact-lm-tokenizers}

Public papers, technical reports, and model cards rarely disclose the exact number of languages used for tokenizer training. Table~\ref{tab:compact-lm-tokenizers} therefore reports the public language scope stated in official papers, technical reports, model cards, blogs, or configs rather than the tokenizer-training language count, and should be read as a deployment-facing survey rather than a like-for-like tokenizer-training comparison. When a row spans a model family, the Scale column reports the public training-token range across the listed parameter sizes. The pattern is nevertheless informative. Many English-centric compact models still operate with 32K--50K vocabularies, while compact public multilingual models either remain 32K-class (Phi-3.5-mini) or jump to much larger family tokenizers (Llama~3.2, Qwen~2.5). This is the regime split most relevant to UBE: small or moderate vocabularies leave less room to absorb script-specific merge coverage, whereas very large multilingual vocabularies partially compensate through scale alone. Table~\ref{tab:compact-lm-tokenizers} compiles representative public examples; sources are listed below the table to identify the source for each row's language scope, vocabulary scale, and training scale.

\begin{table}[ht]
\centering
\footnotesize
\setlength{\tabcolsep}{4pt}
\caption{Representative compact public LMs grouped by tokenizer vocabulary scale: compact (32K--50K), large (100K+), and family-inherited (128K+). Tok.\ source identifies the tokenizer source or family.}
\label{tab:compact-lm-tokenizers}
\begin{tabular}{@{}lllrll@{}}
\toprule
Model & Lang scope & Params & Vocab & Tok.\ source & Scale \\
\midrule
\multicolumn{6}{l}{\emph{Compact tokenizer (32K--50K)}} \\
SmolLM2 & English & 135M--1.7B & 49,152 & self-trained & 2T--11T tok \\
MobileLLM & English & 125M--1B & 32,000 & Llama~2 & 1T tok \\
Pythia & English & 160M--2.8B & 50,277 & GPT-NeoX & 300B tok \\
TinyLlama & English + code & 1.1B & 32,000 & Llama~2 & 3T tok \\
Phi-3.5-mini & 23 langs & 3.8B & 32,064 & Llama~2 & 3.4T tok \\
\midrule
\multicolumn{6}{l}{\emph{Large tokenizer (100K+)}} \\
StableLM~2 & 7 langs & 1.6B & 100,352 & Arcade100k (cl100k ext.) & 2T tok \\
Phi-4-mini & 23 langs & 3.8B & 200,064 & o200k & 5T tok \\
\midrule
\multicolumn{6}{l}{\emph{Inherited from larger model (128K+)}} \\
Llama~3.2 & 8 langs & 1B--3B & 128,256 & Llama~3.1 & 9T tok \\
SmolLM3 & 6 langs & 3B & 128,256 & Llama~3.2 & 11T tok \\
Qwen~2.5 & 29+ langs & 0.5B--3B & 151,936\rlap{$^{\ast}$} & Qwen family & 18T tok \\
BLOOM & 46 langs & 560M--3B & 250,680 & BLOOM-176B & 341B tok \\
\bottomrule
\end{tabular}

\vspace{2pt}
\begin{minipage}{0.95\linewidth}
\raggedright\scriptsize
$^{\ast}$\,Padded \texttt{vocab\_size} from the \texttt{Qwen/Qwen2.5-3B} \texttt{config.json} (\url{https://huggingface.co/Qwen/Qwen2.5-3B}); the Qwen~2.5 technical report~\citep{yang2024qwen25report} reports 151{,}643 regular tokens plus 22 control tokens (151{,}665 in total). Pythia's configuration pads its 50{,}277-token vocabulary to 50{,}304.\par
Sources: SmolLM2~\citep{allal2025smollm2}; MobileLLM~\citep{liu2024mobilellm}; Pythia~\citep{biderman2023pythia}; TinyLlama~\citep{zhang2024tinyllama}; Phi-3.5-mini~\citep{microsoft2024phi35mini}; StableLM~2~\citep{bellagente2024stablelm2}; Phi-4-mini~\citep{abouelenin2025phi4mini}; Llama~3.2~\citep{meta2024llama32card,grattafiori2024llama3}; SmolLM3~\citep{bakouch2025smollm3,huggingfacetb2025smollm3card}; Qwen~2.5~\citep{yang2024qwen25report,qwenteam2024qwen25card}; BLOOM~\citep{bigscience2022bloom}.
\end{minipage}
\end{table}

\subsection{Extended related work}
\label{app:extended-related-work}

\paragraph{Script-specific tokenization.}
A growing cluster of work addresses tokenization for specific script families. \citet{rana2025indicsuper} propose MUTANT, a recipe for multilingual tokenizer design with language-aware pre-tokenization and subword/multiword-aware training; their MUTANT-Indic tokenizer improves average fertility by \relatedRanaFertilityGainPct{} over \llamaFourName{} across English, \relatedRanaIndianLangCount{} Indian languages, and code. \citet{brahma2025morphtok} propose Constrained BPE (CBPE), which enforces script-specific constraints for Indic scripts---keeping dependent vowels cohesive with base consonants. \citet{velayuthan2025gpe} replace byte-level character extraction with grapheme-aware units, benefiting abugida scripts whose written units (graphemes) span multiple code points. \citet{artofbreaking2025indic} systematically study vocabulary size, pre-tokenization rules, and training-corpus composition for Indic scripts, linking these choices to both token-to-word efficiency and model quality. More recently, \citet{wwho2026sgpe} proposes a three-layer syllable-aware architecture that separates script-specific segmentation from BPE compression and reports large token reductions for Sinhala and Hindi. These papers show that script-specific solutions are increasingly common, but each targets a particular script family. UBE instead applies one reversible, language-agnostic routing rule based on \utfEightName{} byte length, without requiring per-language or per-script rules.

\paragraph{Byte-level model architectures.}
An alternative strategy is to reduce reliance on a fixed tokenizer altogether. CANINE~\citep{clark2022canine} operates directly on Unicode characters, and \byTFiveName{}~\citep{xue2022byt5} operates on \utfEightName{} bytes. MAGNET~\citep{ahia2024magnet} learns adaptive segmentation boundaries within the model. SpaceByte~\citep{slagle2024spacebyte} and BLT~\citep{pagnoni2025blt} show that raw-byte language modeling can become competitive when the architecture allocates computation unevenly across a byte stream. More recently, Bolmo~\citep{minixhofer2025bolmo} introduces a family of open byte-level LMs using a two-stage ``byteification'' procedure that converts subword models into byte-level models with minimal additional training. FLEXITOKENS~\citep{owodunni2025flexitokens} uses learnable boundary predictors between bytes, reducing over-fragmentation and improving performance on token-classification and generative tasks. H-Net~\citep{hwang2025hnet} learns content- and context-dependent byte segmentation jointly with the model, outperforming BPE-tokenized Transformers with especially large gains for Chinese and code. \mrTFiveName{}~\citep{kallini2025mrt5} adds a learned delete gate to \byTFiveName{} that dynamically shortens byte sequences, adapting compression rates per language. These approaches show that segmentation can be moved into the model, but they change the model's architecture or input granularity and are generally not direct replacements in existing subword-tokenizer training and inference pipelines.

\subsection{Additional scope and limitations}
\label{app:additional-limitations}

The following limitations supplement the core scope statement in the main text.

\paragraph{Potential societal impacts.}\label{app:broader-impacts} The intended positive impact of UBE is more equitable multilingual access: by lowering token budgets for scripts that are structurally penalized under \utfEightName{} BBPE, the method can increase usable context and reduce truncation pressure, compute cost, and latency disparities for users whose languages rely on BMP 3-byte \utfEightName{} encodings and therefore remain close to the byte fallback floor. The main negative risk is that the same context-capacity gains can be folded into more capable multilingual LMs, including systems that may be deployed for surveillance, misinformation, or other harmful applications. We therefore frame UBE as a tokenizer-level context-capacity and allocative-fairness intervention rather than as a complete solution to broader social fairness or misuse risks.

\paragraph{Evaluation methodology.} Our primary evaluation relies on translated parallel text (\floresTwoHundredName{}), which may exhibit translationese artifacts; we check consistency on three additional corpora (UDHR, SIB-200, MGSM), using corpus-level aggregation without sentence alignment. Token premiums are computed relative to English throughout, and both Gini and variance exclude the English reference row. For a fixed set of non-English languages, changing only the common normalization factor leaves Gini unchanged; it does not imply invariance to pairwise-ratio aggregation.

\paragraph{Low-coverage training regimes.} UBE's dual alphabet reserves \bbpeAlphabetSize{} extra initial symbols relative to BBPE, so at very narrow tokenizer-training coverage and small vocabularies (the \texttt{En-only/16K} stress regime) UBE wins in only \eovOneSixkBetter{} of 204 languages; the main-text \hyperref[sec:alphabet-tradeoff]{alphabet-separation trade-off discussion} notes that the absolute per-language regressions are small (median \eovThreeTwokWorseMedianDelta{} tokens per sentence at \texttt{En-only/32K}) and that Gini still falls at \texttt{En-only/16K} (\eovOneSixkGiniPct{}), but users of very-narrow-coverage compact tokenizers should expect a reduced per-language win rate rather than a uniform improvement.

\paragraph{Downstream coverage for most-affected scripts.} The largest intrinsic token-count reductions are on low-coverage 3-byte scripts such as Santali, Central Atlas Tamazight, and Tibetan. These languages are not included in the 101-language mC4 LM corpus used for downstream training and evaluation, so the downstream experiments test the tokenizer mechanism indirectly rather than measuring gains for those exact languages.

\paragraph{SuperBPE composition.} Because the SuperBPE composition changes both encoding substrate and pre-tokenization, we interpret the interaction analysis narrowly. In the $2 \times 2$ factorial design, the interaction term is bounded but sign-changing (\gvaInteractionMean{}\% on average across 16K--256K; per-vocabulary values in Appendix~\ref{app:superbpe-composition}), supporting qualitative complementarity rather than uniform additivity in the setting we test. Our SuperBPE re-implementation follows the reference recipe~\citep{liu2025superbpe}: the same permissive regex, the four-word cap on merged tokens (Appendix~A.1.2 of that paper), and the colon-space block (Appendix~A.1.4).

\paragraph{Cross-script byte sharing.} UBE's dual-alphabet design forgoes a form of coincidental sharing available to BBPE's single \utfEightName{} byte alphabet: BBPE can sometimes learn merges that reuse shared byte values across scripts or around universal characters such as whitespace and punctuation. This occasionally helps high-coverage languages like Georgian and Amharic, whereas UBE keeps \utfEightName{}-path and \utfSixteenName{}-path bytes distinct to preserve marker-free decoding (see the main-text \hyperref[sec:alphabet-tradeoff]{alphabet-separation trade-off discussion}); future work could explore selective sharing for a small set of universal characters to mitigate these regressions while preserving the disjoint-alphabet decoding guarantee.

\paragraph{Pre-tokenizer dependence.} The core intrinsic matrix uses the GPT-2 regex pre-tokenizer; results under other boundary policies may differ. The SuperBPE- and SCRIPT-style matched comparisons in Appendix~\ref{app:comparative-analyses} show that replacing BBPE with UBE still reduces the reported sequence-length quantities under two substantially different pre-tokenization policies, but this does not establish invariance to all possible boundary rules.

\paragraph{Character routing and grapheme clusters.} UBE routes characters independently; it is not a grapheme-aware tokenizer and does not normalize canonically equivalent strings beyond the configured upstream pre-tokenizer/normalizer. Scripts whose written units span multiple characters, zero-width joiner (ZWJ) sequences, and combining-mark sequences may therefore still require multiple routed units even when each character follows the intended encoding path.

\paragraph{Comparison scope.} The downstream LM experiments focus on isolating the byte-substrate change within otherwise matched training and evaluation pipelines. Appendix~\ref{app:comparative-analyses} includes tokenizer-level SCRIPT-BPE comparisons, a representation-only MYTE composition, a source-byte-matched LM comparison including BBPE16, and a separate matched-exposure model-level comparison with SCRIPT-BPE and a derived 32K MYTE+BPE variant. The latter changes several tokenizer design axes jointly and uses one matched seed, so we interpret it as a comparison across tokenizer designs rather than causal isolation of the representation. Direct model-level comparisons with the released 256-symbol MYTE itself~\citep{limisiewicz2024myte} and Bit-level BPE~\citep{moon2025bitbpe} remain future work.

\paragraph{Non-BMP and symbol-heavy text.} Because 4-byte \utfEightName{} code points remain 4 bytes under \utfSixteenName{} surrogate pairs, UBE provides no encoding-floor reduction for supplementary-plane characters. We have not evaluated UBE on symbol-heavy social-media or chat domains, so we do not claim robustness for those settings.

\paragraph{Streaming and constrained decoding.} The dual-buffer decoder's streaming and constrained-decoding requirements are detailed in Appendix~\ref{app:decoder}; our implementation supports both modes, but we have not benchmarked constrained-decoding throughput.

\paragraph{Matched shard-local singleton filtering.} The two-phase distributed BPE pipeline applies the same shard-local singleton filtering before global merge selection for both UBE and BBPE\@. Because filtering operates at the word level---and both tokenizers share the same Stage~1 regex split, word boundaries, and per-word frequencies on the same text---it removes the same raw word types in both conditions; only the byte-level representation of retained words differs. The filtering step therefore does not act as a confound between UBE and BBPE\@.

\section{Appendix long tables}
\label{app:appendix-longtables}

\subsection{Existing assets and licenses}
\label{app:asset-licenses}
Tables~\ref{tab:asset-licenses-data} and~\ref{tab:asset-licenses-software} summarize the external datasets and software assets used in the paper, together with the license or usage constraint that governed their use in our experiments. To keep long tables at the very end of the appendix, we collect those inventories here. We do not redistribute web-scale corpora such as mC4; our experiments use subsets reconstructed locally from the public sources.

{\footnotesize
\setlength{\tabcolsep}{3pt}
\begin{longtable}{L{1.7cm}L{2.6cm}L{1.6cm}L{5.8cm}}
\caption{Existing dataset and benchmark assets used in this work.}
\label{tab:asset-licenses-data}\\
\toprule
Asset & Access point / version used & License & Terms / handling in this work \\
\midrule
\endfirsthead
\multicolumn{4}{l}{\tablename\ \thetable\ -- continued from previous page}\\
\toprule
Asset & Access point / version used & License & Terms / handling in this work \\
\midrule
\endhead
\midrule
\multicolumn{4}{r}{Continued on next page}\\
\endfoot
\bottomrule
\endlastfoot
mC4 pretraining corpus & Hugging Face \nolinkurl{allenai/c4} multilingual; deterministic 101-language local Arrow mirror & ODC-BY & Derived from Common Crawl snapshots. Fixed capped subsets are used for tokenizer and LM training; the local Arrow mirror is reconstructed from public sources and not redistributed. \\
FLORES-200 & Hugging Face \nolinkurl{facebook/flores}, devtest split & CC BY-SA 4.0 & Used for the primary parallel intrinsic benchmark. \\
UDHR & NLTK \nolinkurl{udhr2} corpus (Unicode UDHR) and United Nations declaration text & Public UN text; public corpus distribution & Used only for robustness evaluation. Our loader downloads the public corpus when absent and records the local corpus checksum. \\
SIB-200 & Hugging Face \nolinkurl{Davlan/sib200}, test split & CC BY-SA 4.0 & Used only for robustness evaluation. Our loader uses a pinned public revision and records a text checksum for the measured split. \\
MGSM & Hugging Face \nolinkurl{juletxara/mgsm}, test split & CC BY-SA 4.0 & Used for robustness evaluation. Our preprocessing script builds the evaluation JSON with the original question text from the public dataset. \\
XNLI & \nolinkurl{facebookresearch/XNLI}; Hugging Face \nolinkurl{facebook/xnli}, test split & CC BY-NC 4.0 & Used for multilingual downstream evaluation only. We preserve the noncommercial restriction and do not redistribute the data. \\
XCOPA & Hugging Face \nolinkurl{cambridgeltl/xcopa}, test split & CC BY 4.0 & Used for multilingual downstream evaluation only. Our preprocessing script regenerates the local evaluation JSON from the public dataset. \\
Belebele & Hugging Face \nolinkurl{facebook/belebele}, test split & CC BY-SA 4.0 & Used for multilingual downstream evaluation only. \\
ARC-Easy/ARC-Challenge & Hugging Face \nolinkurl{allenai/ai2_arc}, test split for reported metrics & CC BY-SA 4.0 & Used for the English regression suite only. \\
HellaSwag & \nolinkurl{rowanz/hellaswag}, validation split for reported metrics & MIT & Used for the English regression suite only. \\
PIQA & \nolinkurl{ybisk.github.io/piqa} public release, validation split for reported metrics & AFL-3.0 & Used for the English regression suite only. \\
LAMBADA & Hugging Face \nolinkurl{EleutherAI/lambada_openai}, test split & Modified MIT & OpenAI-preprocessed LAMBADA; the original release is CC BY 4.0. Used for the English regression suite only. \\
MAPS (SWE-bench subset) & Hugging Face \nolinkurl{Fujitsu-FRE/MAPS}, rev.~\texttt{2b29573}, with task files at each base commit & CC BY 4.0; repository licenses & Token counts for Table~\ref{tab:cr-coding-workload} only; not redistributed. \\
Unicode 17.0 data files & \nolinkurl{unicode.org/Public/17.0.0} & Unicode License v3 & Fidelity audit (Table~\ref{tab:cr-unicode-fidelity}) only. \\
Ethnologue speaker counts & Ethnologue 27th edition (citation-only lookup) & Proprietary reference work & Used only for the speaker-population annotations in Table~\ref{tab:coverage-examples}. No Ethnologue content is redistributed. \\
\end{longtable}

\begin{longtable}{L{2.3cm}L{2.8cm}L{1.4cm}L{5.1cm}}
\caption{Existing software assets used in this work.}
\label{tab:asset-licenses-software}\\
\toprule
Asset & Access point / version used & License & Terms / handling in this work \\
\midrule
\endfirsthead
\toprule
Asset & Access point / version used & License & Terms / handling in this work \\
\midrule
\endhead
\bottomrule
\endfoot
Hugging Face \nolinkurl{tokenizers} & Upstream \nolinkurl{huggingface/tokenizers}; 0.22.3-dev.778/779 & Apache 2.0 & Primary SCRIPT-style tokenizer training uses dev.779. Our modifications retain the Apache 2.0 license and attribution files. \\
LitGPT & \nolinkurl{Lightning-AI/litgpt}, version 0.5.11 & Apache 2.0 & Our trainer vendors and modifies selected LitGPT modules while preserving the upstream license notice and attribution. \\
lm-evaluation-\newline harness & EleutherAI repo, version 0.4.9.1 & MIT & Used as the evaluation backend for our offline benchmark runner, pinned to this version. \\
SCRIPT-BPE & \nolinkurl{sanderland/script_tok}, commit \texttt{54a1058} & Apache 2.0 & Official SCRIPT-BPE tokenizers in Tables~\ref{tab:cr-scriptbpe}, \ref{tab:cr-unicode-fidelity}, and~\ref{tab:cr-matched-exposure-lm}. \\
MYTE maps and decoder & \nolinkurl{tomlimi/MYTE}, commit \texttt{177299c} & Not stated in the repository & Used unmodified for Tables~\ref{tab:cr-myte}, \ref{tab:cr-unicode-fidelity}, and~\ref{tab:cr-matched-exposure-lm}; not redistributed. \\
\end{longtable}
}

\needspace{12\baselineskip}
\subsection{Complete \texttt{101-lang} per-language results}
\label{app:per-language-results}
\begingroup
\small
\begin{longtable}{llrrrrr}
\caption{Complete FLORES-200 mean tokens per sentence for all 204 language--script configurations under \texttt{101-lang} tokenizer training. Columns give vocabulary sizes; each language has paired UBE and BBPE rows. Means are rounded to three decimal places; the win/tie/loss counts in Table~\ref{tab:config-ablation-app} use unrounded values.}
\label{tab:full-language-results}\\
\toprule
Language & Tokenizer & 16K & 32K & 64K & 128K & 256K \\
\midrule
\endfirsthead
\multicolumn{7}{l}{\tablename\ \thetable\ -- continued from previous page}\\
\toprule
Language & Tokenizer & 16K & 32K & 64K & 128K & 256K \\
\midrule
\endhead
\midrule
\multicolumn{7}{r}{Continued on next page}\\
\endfoot
\bottomrule
\endlastfoot
Acehnese (Arabic) & UBE & 85.927 & 79.600 & 75.741 & 70.241 & 67.480 \\*
 & BBPE & 86.626 & 79.774 & 75.897 & 70.374 & 67.620 \\
\addlinespace[2pt]
Acehnese (Latin) & UBE & 60.001 & 56.095 & 52.427 & 49.271 & 45.328 \\*
 & BBPE & 60.122 & 56.191 & 52.458 & 49.284 & 45.389 \\
\addlinespace[2pt]
Afrikaans & UBE & 52.071 & 47.458 & 43.605 & 40.014 & 36.589 \\*
 & BBPE & 52.299 & 47.583 & 43.651 & 40.024 & 36.590 \\
\addlinespace[2pt]
Akan & UBE & 76.909 & 65.912 & 63.770 & 61.159 & 55.821 \\*
 & BBPE & 77.001 & 65.955 & 63.786 & 61.182 & 55.850 \\
\addlinespace[2pt]
Alemannic & UBE & 61.181 & 56.013 & 51.111 & 46.273 & 41.842 \\*
 & BBPE & 61.465 & 56.144 & 51.289 & 46.292 & 41.854 \\
\addlinespace[2pt]
Amharic & UBE & 82.890 & 70.851 & 61.581 & 52.367 & 44.478 \\*
 & BBPE & 81.970 & 69.995 & 60.261 & 51.053 & 43.655 \\
\addlinespace[2pt]
Arabic & UBE & 68.952 & 60.407 & 52.653 & 46.245 & 40.890 \\*
 & BBPE & 69.872 & 60.600 & 52.753 & 46.296 & 40.905 \\
\addlinespace[2pt]
Arabic (Latin) & UBE & 70.497 & 66.711 & 62.903 & 59.892 & 57.739 \\*
 & BBPE & 70.679 & 66.736 & 62.952 & 59.888 & 57.740 \\
\addlinespace[2pt]
Armenian & UBE & 91.947 & 78.250 & 64.792 & 54.347 & 46.934 \\*
 & BBPE & 92.957 & 78.649 & 65.068 & 54.425 & 46.963 \\
\addlinespace[2pt]
Assamese & UBE & 121.978 & 107.690 & 103.233 & 98.061 & 94.208 \\*
 & BBPE & 122.269 & 107.488 & 103.420 & 98.105 & 94.208 \\
\addlinespace[2pt]
Asturian & UBE & 52.166 & 47.708 & 44.491 & 41.299 & 38.654 \\*
 & BBPE & 52.390 & 47.810 & 44.559 & 41.320 & 38.655 \\
\addlinespace[2pt]
Awadhi & UBE & 89.941 & 87.054 & 85.436 & 84.299 & 83.747 \\*
 & BBPE & 90.202 & 87.123 & 85.446 & 84.552 & 83.743 \\
\addlinespace[2pt]
Ayacucho Quechua & UBE & 63.566 & 57.986 & 53.421 & 50.081 & 47.194 \\*
 & BBPE & 63.680 & 58.379 & 53.437 & 50.100 & 47.196 \\
\addlinespace[2pt]
Aymara & UBE & 63.904 & 60.158 & 56.439 & 53.707 & 51.015 \\*
 & BBPE & 64.060 & 60.242 & 56.540 & 53.711 & 51.020 \\
\addlinespace[2pt]
Azerbaijani & UBE & 69.600 & 60.773 & 52.889 & 45.706 & 39.928 \\*
 & BBPE & 70.343 & 60.985 & 52.938 & 45.656 & 39.887 \\
\addlinespace[2pt]
Balinese & UBE & 55.118 & 51.107 & 47.312 & 42.325 & 38.968 \\*
 & BBPE & 55.354 & 51.183 & 47.353 & 42.358 & 38.996 \\
\addlinespace[2pt]
Bambara & UBE & 72.784 & 60.037 & 58.458 & 56.308 & 53.743 \\*
 & BBPE & 72.858 & 60.068 & 58.469 & 56.314 & 53.773 \\
\addlinespace[2pt]
Banjar (Arabic) & UBE & 88.832 & 80.923 & 75.831 & 69.976 & 66.470 \\*
 & BBPE & 89.302 & 81.011 & 75.842 & 69.978 & 66.504 \\
\addlinespace[2pt]
Banjar (Latin) & UBE & 52.688 & 47.217 & 43.095 & 39.330 & 36.176 \\*
 & BBPE & 52.796 & 47.273 & 43.111 & 39.360 & 36.176 \\
\addlinespace[2pt]
Bashkir & UBE & 86.725 & 75.885 & 65.871 & 59.651 & 54.490 \\*
 & BBPE & 87.014 & 76.205 & 65.915 & 59.676 & 54.504 \\
\addlinespace[2pt]
Basque & UBE & 56.730 & 51.664 & 46.143 & 40.982 & 36.703 \\*
 & BBPE & 56.903 & 51.779 & 46.189 & 41.031 & 36.706 \\
\addlinespace[2pt]
Belarusian & UBE & 82.795 & 72.856 & 64.942 & 56.513 & 49.458 \\*
 & BBPE & 83.050 & 73.020 & 65.007 & 56.563 & 49.490 \\
\addlinespace[2pt]
Bemba & UBE & 65.952 & 60.496 & 57.039 & 54.178 & 50.822 \\*
 & BBPE & 66.029 & 60.809 & 57.066 & 54.188 & 50.833 \\
\addlinespace[2pt]
Bengali & UBE & 111.474 & 104.137 & 99.317 & 95.335 & 92.797 \\*
 & BBPE & 111.689 & 103.932 & 99.507 & 95.321 & 92.844 \\
\addlinespace[2pt]
Bhojpuri & UBE & 89.326 & 86.490 & 85.134 & 83.772 & 83.337 \\*
 & BBPE & 89.237 & 86.524 & 85.141 & 84.063 & 83.328 \\
\addlinespace[2pt]
Bosnian & UBE & 57.041 & 52.513 & 48.725 & 45.691 & 42.523 \\*
 & BBPE & 57.147 & 52.595 & 48.825 & 45.699 & 42.525 \\
\addlinespace[2pt]
Buginese & UBE & 59.070 & 54.659 & 51.357 & 48.553 & 45.399 \\*
 & BBPE & 59.173 & 54.735 & 51.365 & 48.604 & 45.399 \\
\addlinespace[2pt]
Bulgarian & UBE & 68.430 & 59.862 & 52.577 & 45.971 & 41.144 \\*
 & BBPE & 69.152 & 60.066 & 52.661 & 46.033 & 41.144 \\
\addlinespace[2pt]
Burmese & UBE & 139.676 & 125.853 & 117.201 & 111.051 & 108.263 \\*
 & BBPE & 139.770 & 126.255 & 117.161 & 111.058 & 108.256 \\
\addlinespace[2pt]
Cantonese & UBE & 48.922 & 43.513 & 38.206 & 34.518 & 31.580 \\*
 & BBPE & 49.708 & 43.637 & 37.856 & 34.116 & 31.221 \\
\addlinespace[2pt]
Catalan & UBE & 53.407 & 49.471 & 45.687 & 42.175 & 39.055 \\*
 & BBPE & 53.569 & 49.581 & 45.742 & 42.190 & 39.064 \\
\addlinespace[2pt]
Cebuano & UBE & 58.205 & 53.396 & 48.913 & 45.038 & 41.801 \\*
 & BBPE & 58.409 & 53.467 & 48.954 & 45.063 & 41.805 \\
\addlinespace[2pt]
Central Atlas Tamazight & UBE & 192.112 & 183.305 & 182.243 & 180.664 & 180.648 \\*
 & BBPE & 269.200 & 269.166 & 269.123 & 269.102 & 203.206 \\
\addlinespace[2pt]
Central Kanuri (Arabic) & UBE & 91.903 & 88.641 & 86.271 & 84.266 & 82.724 \\*
 & BBPE & 91.996 & 88.654 & 86.276 & 84.267 & 82.729 \\
\addlinespace[2pt]
Central Kanuri (Latin) & UBE & 67.723 & 63.558 & 60.579 & 57.633 & 55.173 \\*
 & BBPE & 67.961 & 63.665 & 60.625 & 57.618 & 55.159 \\
\addlinespace[2pt]
Central Kurdish & UBE & 93.146 & 86.587 & 79.674 & 70.182 & 61.794 \\*
 & BBPE & 95.684 & 86.598 & 79.678 & 70.182 & 61.800 \\
\addlinespace[2pt]
Chhattisgarhi & UBE & 88.629 & 85.262 & 84.000 & 83.224 & 82.785 \\*
 & BBPE & 88.660 & 85.428 & 84.004 & 83.226 & 82.791 \\
\addlinespace[2pt]
Chichewa & UBE & 61.234 & 55.689 & 51.690 & 46.685 & 42.545 \\*
 & BBPE & 61.352 & 55.782 & 51.723 & 46.818 & 42.580 \\
\addlinespace[2pt]
Chinese (Simplified) & UBE & 44.910 & 39.359 & 35.015 & 31.574 & 28.664 \\*
 & BBPE & 44.920 & 39.183 & 34.764 & 31.310 & 28.437 \\
\addlinespace[2pt]
Chinese (Traditional) & UBE & 49.919 & 44.313 & 39.077 & 35.411 & 32.414 \\*
 & BBPE & 50.770 & 44.440 & 38.741 & 34.998 & 32.013 \\
\addlinespace[2pt]
Chokwe & UBE & 57.408 & 52.807 & 49.780 & 47.661 & 45.042 \\*
 & BBPE & 57.561 & 52.890 & 49.803 & 47.688 & 45.080 \\
\addlinespace[2pt]
Crimean Tatar & UBE & 61.907 & 56.758 & 51.302 & 47.011 & 43.093 \\*
 & BBPE & 62.021 & 56.830 & 51.457 & 47.042 & 43.098 \\
\addlinespace[2pt]
Croatian & UBE & 55.903 & 51.588 & 47.775 & 44.815 & 41.678 \\*
 & BBPE & 56.013 & 51.664 & 47.858 & 44.821 & 41.686 \\
\addlinespace[2pt]
Czech & UBE & 59.761 & 53.849 & 47.946 & 42.535 & 37.554 \\*
 & BBPE & 59.944 & 54.008 & 48.132 & 42.559 & 37.570 \\
\addlinespace[2pt]
Danish & UBE & 51.159 & 46.673 & 42.433 & 38.464 & 35.242 \\*
 & BBPE & 51.294 & 46.848 & 42.479 & 38.488 & 35.254 \\
\addlinespace[2pt]
Dari & UBE & 66.958 & 57.616 & 49.862 & 43.035 & 37.850 \\*
 & BBPE & 67.196 & 57.763 & 49.888 & 43.068 & 37.867 \\
\addlinespace[2pt]
Dinka & UBE & 63.980 & 55.801 & 52.088 & 50.110 & 48.201 \\*
 & BBPE & 64.145 & 55.947 & 52.095 & 50.118 & 48.219 \\
\addlinespace[2pt]
Dutch & UBE & 53.293 & 48.189 & 43.544 & 39.364 & 35.599 \\*
 & BBPE & 53.531 & 48.251 & 43.577 & 39.386 & 35.608 \\
\addlinespace[2pt]
Dyula & UBE & 61.540 & 57.182 & 54.221 & 50.734 & 48.086 \\*
 & BBPE & 61.633 & 57.230 & 54.233 & 50.757 & 48.093 \\
\addlinespace[2pt]
Dzongkha & UBE & 311.873 & 309.494 & 265.685 & 244.805 & 178.047 \\*
 & BBPE & 471.899 & 367.301 & 286.182 & 235.167 & 178.307 \\
\addlinespace[2pt]
Egyptian Arabic & UBE & 67.151 & 58.900 & 51.562 & 45.346 & 40.486 \\*
 & BBPE & 67.931 & 59.080 & 51.726 & 45.388 & 40.509 \\
\addlinespace[2pt]
English & UBE & 34.671 & 31.068 & 28.695 & 27.158 & 26.271 \\*
 & BBPE & 34.863 & 31.160 & 28.716 & 27.174 & 26.272 \\
\addlinespace[2pt]
Esperanto & UBE & 51.917 & 46.989 & 42.915 & 39.470 & 36.386 \\*
 & BBPE & 52.133 & 47.121 & 42.958 & 39.507 & 36.386 \\
\addlinespace[2pt]
Estonian & UBE & 55.151 & 49.897 & 45.625 & 41.158 & 36.723 \\*
 & BBPE & 55.427 & 50.057 & 45.656 & 41.181 & 36.735 \\
\addlinespace[2pt]
Ewe & UBE & 76.564 & 68.753 & 65.405 & 61.973 & 58.427 \\*
 & BBPE & 76.683 & 68.804 & 65.426 & 62.300 & 58.424 \\
\addlinespace[2pt]
Faroese & UBE & 60.310 & 55.507 & 50.721 & 47.020 & 43.441 \\*
 & BBPE & 60.559 & 55.625 & 50.798 & 47.048 & 43.440 \\
\addlinespace[2pt]
Fijian & UBE & 63.968 & 59.466 & 55.543 & 52.398 & 49.574 \\*
 & BBPE & 64.063 & 59.592 & 55.565 & 52.418 & 49.577 \\
\addlinespace[2pt]
Finnish & UBE & 60.395 & 54.506 & 49.171 & 44.018 & 39.285 \\*
 & BBPE & 60.664 & 54.867 & 49.185 & 44.038 & 39.304 \\
\addlinespace[2pt]
Fon & UBE & 103.407 & 87.433 & 84.325 & 78.079 & 73.668 \\*
 & BBPE & 103.517 & 87.460 & 84.337 & 78.090 & 73.673 \\
\addlinespace[2pt]
French & UBE & 57.925 & 52.676 & 48.367 & 44.186 & 40.611 \\*
 & BBPE & 58.072 & 52.753 & 48.390 & 44.231 & 40.626 \\
\addlinespace[2pt]
Friulian & UBE & 57.823 & 54.107 & 50.199 & 47.690 & 45.504 \\*
 & BBPE & 57.903 & 54.277 & 50.211 & 47.699 & 45.503 \\
\addlinespace[2pt]
Galician & UBE & 52.798 & 47.705 & 43.805 & 39.214 & 35.696 \\*
 & BBPE & 52.959 & 47.823 & 43.850 & 39.242 & 35.696 \\
\addlinespace[2pt]
Ganda & UBE & 58.761 & 54.324 & 51.193 & 48.448 & 45.609 \\*
 & BBPE & 58.819 & 54.426 & 51.207 & 48.500 & 45.613 \\
\addlinespace[2pt]
Georgian & UBE & 100.074 & 85.117 & 71.468 & 59.405 & 49.793 \\*
 & BBPE & 97.688 & 82.745 & 70.135 & 58.003 & 48.937 \\
\addlinespace[2pt]
German & UBE & 58.225 & 52.011 & 47.137 & 42.310 & 37.913 \\*
 & BBPE & 58.459 & 52.177 & 47.185 & 42.369 & 37.921 \\
\addlinespace[2pt]
Greek & UBE & 90.479 & 77.161 & 66.125 & 56.125 & 48.226 \\*
 & BBPE & 90.726 & 77.342 & 66.307 & 56.179 & 48.243 \\
\addlinespace[2pt]
Guarani & UBE & 65.314 & 60.675 & 56.804 & 53.292 & 50.566 \\*
 & BBPE & 65.500 & 60.755 & 56.818 & 53.304 & 50.563 \\
\addlinespace[2pt]
Gujarati & UBE & 105.362 & 99.345 & 93.423 & 89.804 & 86.962 \\*
 & BBPE & 105.835 & 98.813 & 93.399 & 89.846 & 86.943 \\
\addlinespace[2pt]
Haitian Creole & UBE & 49.862 & 46.222 & 42.663 & 39.065 & 35.862 \\*
 & BBPE & 49.914 & 46.280 & 42.720 & 39.073 & 35.858 \\
\addlinespace[2pt]
Hausa & UBE & 58.594 & 52.477 & 47.054 & 42.788 & 39.174 \\*
 & BBPE & 58.639 & 52.565 & 47.129 & 42.828 & 39.177 \\
\addlinespace[2pt]
Hebrew & UBE & 64.113 & 55.574 & 48.778 & 42.666 & 37.625 \\*
 & BBPE & 64.358 & 55.780 & 48.851 & 42.683 & 37.635 \\
\addlinespace[2pt]
Hindi & UBE & 89.002 & 86.442 & 85.105 & 84.375 & 83.987 \\*
 & BBPE & 89.078 & 86.494 & 85.114 & 84.376 & 83.991 \\
\addlinespace[2pt]
Hungarian & UBE & 65.203 & 59.007 & 52.771 & 47.145 & 41.420 \\*
 & BBPE & 65.443 & 59.067 & 52.841 & 47.198 & 41.448 \\
\addlinespace[2pt]
Icelandic & UBE & 59.486 & 53.666 & 48.718 & 44.026 & 39.751 \\*
 & BBPE & 59.710 & 53.771 & 48.808 & 44.042 & 39.766 \\
\addlinespace[2pt]
Igbo & UBE & 70.128 & 66.196 & 59.599 & 53.670 & 48.304 \\*
 & BBPE & 70.221 & 66.757 & 59.622 & 53.754 & 48.291 \\
\addlinespace[2pt]
Ilocano & UBE & 62.476 & 57.350 & 52.608 & 50.105 & 47.943 \\*
 & BBPE & 62.706 & 57.398 & 52.616 & 50.115 & 47.960 \\
\addlinespace[2pt]
Indonesian & UBE & 51.095 & 44.751 & 39.268 & 34.455 & 30.570 \\*
 & BBPE & 51.282 & 44.845 & 39.315 & 34.494 & 30.569 \\
\addlinespace[2pt]
Iraqi Arabic & UBE & 67.438 & 59.284 & 51.823 & 45.466 & 40.291 \\*
 & BBPE & 68.389 & 59.478 & 51.943 & 45.511 & 40.306 \\
\addlinespace[2pt]
Irish & UBE & 65.688 & 59.298 & 53.939 & 48.930 & 43.929 \\*
 & BBPE & 65.809 & 59.389 & 54.026 & 48.950 & 43.958 \\
\addlinespace[2pt]
Italian & UBE & 56.121 & 51.148 & 46.933 & 42.953 & 39.047 \\*
 & BBPE & 56.361 & 51.283 & 46.953 & 42.965 & 39.050 \\
\addlinespace[2pt]
Japanese & UBE & 52.419 & 46.273 & 40.783 & 36.098 & 32.338 \\*
 & BBPE & 52.668 & 46.194 & 40.630 & 35.965 & 32.145 \\
\addlinespace[2pt]
Javanese & UBE & 52.634 & 47.536 & 43.192 & 38.537 & 34.464 \\*
 & BBPE & 52.764 & 47.591 & 43.223 & 38.557 & 34.468 \\
\addlinespace[2pt]
Kabiyé & UBE & 129.830 & 112.313 & 109.938 & 107.228 & 87.811 \\*
 & BBPE & 129.873 & 112.336 & 109.958 & 107.237 & 87.816 \\
\addlinespace[2pt]
Kabuverdianu & UBE & 52.273 & 48.034 & 45.555 & 43.003 & 40.566 \\*
 & BBPE & 52.365 & 48.076 & 45.561 & 43.015 & 40.566 \\
\addlinespace[2pt]
Kabyle & UBE & 68.512 & 64.688 & 62.457 & 59.657 & 55.993 \\*
 & BBPE & 70.568 & 65.489 & 62.967 & 60.166 & 55.835 \\
\addlinespace[2pt]
Kachin & UBE & 70.703 & 64.437 & 61.013 & 56.744 & 54.767 \\*
 & BBPE & 70.808 & 64.487 & 61.032 & 56.751 & 54.767 \\
\addlinespace[2pt]
Kamba & UBE & 60.790 & 55.084 & 52.022 & 49.361 & 47.375 \\*
 & BBPE & 60.852 & 55.171 & 52.039 & 49.371 & 47.383 \\
\addlinespace[2pt]
Kannada & UBE & 121.886 & 114.405 & 108.854 & 104.483 & 101.329 \\*
 & BBPE & 121.878 & 115.094 & 109.013 & 104.505 & 101.329 \\
\addlinespace[2pt]
Kashmiri (Arabic) & UBE & 92.493 & 86.303 & 79.726 & 73.982 & 68.068 \\*
 & BBPE & 92.757 & 86.397 & 79.754 & 74.014 & 68.077 \\
\addlinespace[2pt]
Kashmiri (Devanagari) & UBE & 91.532 & 88.371 & 85.811 & 84.222 & 82.965 \\*
 & BBPE & 91.667 & 88.529 & 85.792 & 84.232 & 83.002 \\
\addlinespace[2pt]
Kazakh & UBE & 74.634 & 64.325 & 54.984 & 46.899 & 39.974 \\*
 & BBPE & 75.698 & 64.406 & 55.096 & 46.988 & 39.980 \\
\addlinespace[2pt]
Khmer & UBE & 145.965 & 135.804 & 127.642 & 121.825 & 118.121 \\*
 & BBPE & 145.946 & 136.047 & 127.747 & 121.860 & 118.116 \\
\addlinespace[2pt]
Kikongo & UBE & 59.435 & 55.104 & 51.684 & 47.888 & 46.230 \\*
 & BBPE & 59.529 & 55.150 & 51.697 & 47.888 & 46.231 \\
\addlinespace[2pt]
Kikuyu & UBE & 87.112 & 74.997 & 72.792 & 70.577 & 68.090 \\*
 & BBPE & 87.187 & 75.096 & 72.799 & 70.608 & 68.093 \\
\addlinespace[2pt]
Kimbundu & UBE & 62.657 & 57.136 & 53.250 & 50.544 & 48.355 \\*
 & BBPE & 62.847 & 57.305 & 53.274 & 50.551 & 48.368 \\
\addlinespace[2pt]
Kinyarwanda & UBE & 62.855 & 58.812 & 54.961 & 50.840 & 46.585 \\*
 & BBPE & 62.890 & 58.974 & 54.989 & 50.857 & 46.590 \\
\addlinespace[2pt]
Korean & UBE & 44.092 & 38.633 & 34.112 & 30.377 & 27.495 \\*
 & BBPE & 44.539 & 38.766 & 34.193 & 30.375 & 27.481 \\
\addlinespace[2pt]
Kyrgyz & UBE & 73.696 & 65.523 & 57.050 & 49.417 & 42.621 \\*
 & BBPE & 74.529 & 65.638 & 57.085 & 49.437 & 42.637 \\
\addlinespace[2pt]
Lao & UBE & 115.340 & 99.961 & 87.657 & 79.355 & 72.767 \\*
 & BBPE & 115.791 & 100.289 & 88.180 & 79.768 & 72.899 \\
\addlinespace[2pt]
Latgalian & UBE & 62.662 & 57.833 & 53.175 & 49.880 & 46.937 \\*
 & BBPE & 62.811 & 57.886 & 53.249 & 49.902 & 46.945 \\
\addlinespace[2pt]
Latvian & UBE & 63.681 & 56.927 & 50.618 & 44.915 & 39.578 \\*
 & BBPE & 63.818 & 57.020 & 50.708 & 44.976 & 39.583 \\
\addlinespace[2pt]
Ligurian & UBE & 61.958 & 58.355 & 55.042 & 52.183 & 49.449 \\*
 & BBPE & 62.092 & 58.480 & 55.069 & 52.199 & 49.449 \\
\addlinespace[2pt]
Limburgish & UBE & 55.518 & 51.811 & 48.811 & 44.882 & 42.739 \\*
 & BBPE & 55.757 & 51.863 & 48.829 & 44.896 & 42.743 \\
\addlinespace[2pt]
Lingala & UBE & 55.718 & 51.343 & 48.868 & 46.344 & 43.278 \\*
 & BBPE & 55.750 & 51.408 & 48.885 & 46.344 & 43.278 \\
\addlinespace[2pt]
Lithuanian & UBE & 57.804 & 52.376 & 47.147 & 42.651 & 37.911 \\*
 & BBPE & 58.010 & 52.514 & 47.202 & 42.671 & 37.916 \\
\addlinespace[2pt]
Lombard & UBE & 61.415 & 58.139 & 54.782 & 51.868 & 49.361 \\*
 & BBPE & 61.538 & 58.209 & 54.797 & 51.891 & 49.363 \\
\addlinespace[2pt]
Luba-Kasai & UBE & 57.894 & 53.066 & 49.802 & 46.561 & 44.423 \\*
 & BBPE & 57.994 & 53.200 & 49.825 & 46.584 & 44.422 \\
\addlinespace[2pt]
Luo & UBE & 55.680 & 50.832 & 48.359 & 46.138 & 44.131 \\*
 & BBPE & 55.740 & 51.055 & 48.375 & 46.149 & 44.134 \\
\addlinespace[2pt]
Luxembourgish & UBE & 62.170 & 57.709 & 52.636 & 48.048 & 44.665 \\*
 & BBPE & 62.391 & 57.796 & 52.678 & 48.112 & 44.670 \\
\addlinespace[2pt]
Macedonian & UBE & 66.926 & 58.824 & 51.662 & 45.542 & 40.673 \\*
 & BBPE & 67.254 & 58.984 & 51.717 & 45.597 & 40.676 \\
\addlinespace[2pt]
Magahi & UBE & 89.901 & 87.107 & 85.821 & 84.537 & 84.107 \\*
 & BBPE & 90.118 & 87.131 & 85.824 & 84.531 & 84.106 \\
\addlinespace[2pt]
Maithili & UBE & 92.055 & 89.345 & 87.868 & 85.727 & 84.039 \\*
 & BBPE & 92.080 & 89.391 & 87.870 & 85.723 & 84.004 \\
\addlinespace[2pt]
Malay & UBE & 52.465 & 45.849 & 40.387 & 35.232 & 31.275 \\*
 & BBPE & 52.641 & 45.912 & 40.434 & 35.286 & 31.274 \\
\addlinespace[2pt]
Malayalam & UBE & 132.650 & 126.980 & 120.760 & 115.976 & 113.153 \\*
 & BBPE & 132.444 & 126.831 & 120.714 & 115.951 & 113.167 \\
\addlinespace[2pt]
Maltese & UBE & 71.571 & 65.837 & 62.167 & 57.209 & 53.027 \\*
 & BBPE & 71.753 & 65.878 & 62.196 & 57.295 & 53.037 \\
\addlinespace[2pt]
Manipuri & UBE & 124.214 & 119.227 & 111.831 & 107.467 & 103.967 \\*
 & BBPE & 124.869 & 119.007 & 111.738 & 106.202 & 103.147 \\
\addlinespace[2pt]
Maori & UBE & 67.069 & 61.654 & 56.111 & 51.477 & 46.518 \\*
 & BBPE & 67.136 & 61.724 & 56.149 & 51.595 & 46.521 \\
\addlinespace[2pt]
Marathi & UBE & 104.165 & 100.329 & 98.342 & 96.261 & 95.370 \\*
 & BBPE & 104.264 & 100.625 & 98.368 & 96.236 & 95.363 \\
\addlinespace[2pt]
Minangkabau (Arabic) & UBE & 92.803 & 83.785 & 78.774 & 72.147 & 68.735 \\*
 & BBPE & 93.431 & 83.870 & 78.798 & 72.166 & 68.776 \\
\addlinespace[2pt]
Minangkabau (Latin) & UBE & 54.254 & 50.000 & 45.832 & 42.494 & 39.906 \\*
 & BBPE & 54.379 & 50.083 & 45.858 & 42.512 & 39.907 \\
\addlinespace[2pt]
Mizo & UBE & 60.398 & 56.024 & 51.791 & 49.498 & 46.544 \\*
 & BBPE & 60.598 & 56.080 & 51.802 & 49.518 & 46.534 \\
\addlinespace[2pt]
Mongolian & UBE & 80.031 & 69.169 & 59.959 & 50.714 & 43.704 \\*
 & BBPE & 80.340 & 69.547 & 60.095 & 50.801 & 43.726 \\
\addlinespace[2pt]
Moroccan Arabic & UBE & 67.443 & 59.507 & 52.555 & 46.515 & 42.130 \\*
 & BBPE & 68.267 & 59.719 & 52.646 & 46.535 & 42.162 \\
\addlinespace[2pt]
Mossi & UBE & 65.381 & 60.276 & 58.559 & 56.768 & 53.745 \\*
 & BBPE & 66.379 & 59.651 & 57.896 & 56.103 & 53.059 \\
\addlinespace[2pt]
Najdi Arabic & UBE & 69.042 & 60.559 & 52.814 & 46.434 & 41.090 \\*
 & BBPE & 69.959 & 60.746 & 52.919 & 46.484 & 41.105 \\
\addlinespace[2pt]
Nepali & UBE & 96.657 & 93.235 & 91.076 & 89.527 & 88.408 \\*
 & BBPE & 96.719 & 93.377 & 91.098 & 89.527 & 88.415 \\
\addlinespace[2pt]
Nigerian Fulfulde & UBE & 54.955 & 50.787 & 46.216 & 43.593 & 41.066 \\*
 & BBPE & 55.002 & 50.848 & 46.256 & 43.609 & 41.071 \\
\addlinespace[2pt]
North Levantine Arabic & UBE & 64.101 & 56.463 & 49.510 & 43.852 & 39.072 \\*
 & BBPE & 64.897 & 56.593 & 49.641 & 43.889 & 39.100 \\
\addlinespace[2pt]
Northern Kurdish & UBE & 61.977 & 56.102 & 50.753 & 46.007 & 41.816 \\*
 & BBPE & 62.231 & 56.165 & 50.822 & 46.066 & 41.831 \\
\addlinespace[2pt]
Northern Sotho & UBE & 63.217 & 59.307 & 56.489 & 53.723 & 50.016 \\*
 & BBPE & 63.522 & 59.399 & 56.526 & 53.741 & 50.020 \\
\addlinespace[2pt]
Northern Uzbek & UBE & 62.423 & 57.289 & 51.243 & 45.904 & 40.856 \\*
 & BBPE & 62.486 & 57.377 & 51.326 & 45.931 & 40.873 \\
\addlinespace[2pt]
Norwegian Bokmål & UBE & 49.458 & 45.037 & 41.152 & 37.581 & 34.627 \\*
 & BBPE & 49.628 & 45.210 & 41.177 & 37.622 & 34.633 \\
\addlinespace[2pt]
Norwegian Nynorsk & UBE & 51.084 & 46.710 & 42.946 & 39.627 & 36.563 \\*
 & BBPE & 51.270 & 46.811 & 43.007 & 39.681 & 36.568 \\
\addlinespace[2pt]
Nuer & UBE & 109.780 & 93.545 & 81.559 & 79.516 & 72.085 \\*
 & BBPE & 109.952 & 93.604 & 81.562 & 79.635 & 72.090 \\
\addlinespace[2pt]
Occitan & UBE & 57.495 & 53.097 & 49.753 & 46.640 & 44.212 \\*
 & BBPE & 57.662 & 53.154 & 49.778 & 46.650 & 44.223 \\
\addlinespace[2pt]
Odia & UBE & 188.713 & 161.316 & 135.328 & 118.017 & 109.633 \\*
 & BBPE & 182.258 & 164.863 & 133.265 & 118.017 & 108.809 \\
\addlinespace[2pt]
Pangasinan & UBE & 48.309 & 44.354 & 41.247 & 38.685 & 36.988 \\*
 & BBPE & 48.428 & 44.472 & 41.258 & 38.712 & 37.018 \\
\addlinespace[2pt]
Papiamento & UBE & 53.667 & 49.781 & 46.553 & 43.911 & 41.106 \\*
 & BBPE & 53.780 & 49.859 & 46.590 & 43.917 & 41.203 \\
\addlinespace[2pt]
Plateau Malagasy & UBE & 65.920 & 60.520 & 55.479 & 50.181 & 46.322 \\*
 & BBPE & 66.040 & 60.559 & 55.491 & 50.185 & 46.335 \\
\addlinespace[2pt]
Polish & UBE & 65.246 & 58.423 & 52.398 & 46.618 & 40.905 \\*
 & BBPE & 65.884 & 58.563 & 52.464 & 46.654 & 40.909 \\
\addlinespace[2pt]
Portuguese & UBE & 52.184 & 47.341 & 43.180 & 38.896 & 35.369 \\*
 & BBPE & 52.377 & 47.566 & 43.235 & 38.904 & 35.369 \\
\addlinespace[2pt]
Punjabi & UBE & 106.408 & 98.947 & 93.255 & 88.974 & 86.336 \\*
 & BBPE & 107.065 & 98.726 & 93.199 & 89.055 & 86.388 \\
\addlinespace[2pt]
Romanian & UBE & 58.321 & 54.096 & 49.314 & 44.658 & 40.676 \\*
 & BBPE & 58.457 & 54.194 & 49.376 & 44.672 & 40.681 \\
\addlinespace[2pt]
Rundi & UBE & 63.596 & 59.159 & 55.311 & 51.373 & 47.621 \\*
 & BBPE & 63.638 & 59.374 & 55.375 & 51.388 & 47.625 \\
\addlinespace[2pt]
Russian & UBE & 71.353 & 61.758 & 53.181 & 46.146 & 40.539 \\*
 & BBPE & 71.831 & 61.986 & 53.296 & 46.194 & 40.551 \\
\addlinespace[2pt]
Samoan & UBE & 65.732 & 62.070 & 59.037 & 55.253 & 51.648 \\*
 & BBPE & 65.797 & 62.167 & 59.042 & 55.294 & 51.650 \\
\addlinespace[2pt]
Sango & UBE & 60.273 & 56.579 & 53.367 & 50.404 & 47.674 \\*
 & BBPE & 60.302 & 56.607 & 53.372 & 50.420 & 47.680 \\
\addlinespace[2pt]
Sanskrit & UBE & 99.349 & 95.658 & 93.779 & 91.105 & 89.961 \\*
 & BBPE & 102.354 & 95.740 & 93.816 & 91.095 & 89.953 \\
\addlinespace[2pt]
Santali & UBE & 237.326 & 222.944 & 212.692 & 211.712 & 207.379 \\*
 & BBPE & 341.312 & 341.287 & 341.271 & 251.125 & 228.583 \\
\addlinespace[2pt]
Sardinian & UBE & 60.539 & 55.834 & 52.812 & 49.657 & 46.819 \\*
 & BBPE & 60.765 & 56.149 & 52.821 & 49.670 & 46.826 \\
\addlinespace[2pt]
Scottish Gaelic & UBE & 69.524 & 64.659 & 60.289 & 55.730 & 50.936 \\*
 & BBPE & 70.124 & 64.934 & 60.300 & 55.769 & 50.952 \\
\addlinespace[2pt]
Serbian & UBE & 67.635 & 59.429 & 51.874 & 46.119 & 41.318 \\*
 & BBPE & 67.838 & 59.612 & 51.923 & 46.154 & 41.318 \\
\addlinespace[2pt]
Shan & UBE & 236.768 & 221.375 & 218.032 & 167.262 & 151.673 \\*
 & BBPE & 233.891 & 218.301 & 204.210 & 170.984 & 151.355 \\
\addlinespace[2pt]
Shona & UBE & 63.249 & 57.294 & 53.068 & 49.035 & 45.174 \\*
 & BBPE & 63.323 & 57.509 & 53.125 & 49.199 & 45.180 \\
\addlinespace[2pt]
Sicilian & UBE & 59.735 & 56.201 & 52.945 & 49.747 & 46.600 \\*
 & BBPE & 60.358 & 56.648 & 53.291 & 50.107 & 46.944 \\
\addlinespace[2pt]
Silesian & UBE & 71.144 & 62.284 & 58.068 & 54.112 & 50.292 \\*
 & BBPE & 71.727 & 62.398 & 58.104 & 54.126 & 50.288 \\
\addlinespace[2pt]
Sindhi & UBE & 88.438 & 75.398 & 64.024 & 56.548 & 49.625 \\*
 & BBPE & 88.819 & 75.492 & 64.053 & 56.587 & 49.630 \\
\addlinespace[2pt]
Sinhala & UBE & 110.344 & 102.037 & 95.809 & 91.664 & 88.621 \\*
 & BBPE & 110.516 & 102.125 & 95.822 & 91.593 & 88.624 \\
\addlinespace[2pt]
Slovak & UBE & 59.940 & 54.424 & 49.093 & 43.599 & 38.802 \\*
 & BBPE & 60.335 & 54.685 & 49.178 & 43.622 & 38.815 \\
\addlinespace[2pt]
Slovenian & UBE & 55.101 & 50.241 & 46.221 & 41.769 & 37.653 \\*
 & BBPE & 55.371 & 50.386 & 46.295 & 41.807 & 37.662 \\
\addlinespace[2pt]
Somali & UBE & 64.171 & 59.198 & 53.844 & 49.272 & 45.173 \\*
 & BBPE & 64.639 & 59.266 & 53.864 & 49.355 & 45.187 \\
\addlinespace[2pt]
South Azerbaijani & UBE & 71.596 & 64.139 & 57.975 & 52.104 & 47.331 \\*
 & BBPE & 71.831 & 64.224 & 57.998 & 52.116 & 47.329 \\
\addlinespace[2pt]
South Levantine Arabic & UBE & 63.855 & 56.180 & 49.159 & 43.426 & 38.605 \\*
 & BBPE & 64.715 & 56.325 & 49.303 & 43.472 & 38.624 \\
\addlinespace[2pt]
Southern Pashto & UBE & 71.269 & 61.397 & 53.629 & 47.345 & 42.423 \\*
 & BBPE & 71.725 & 61.680 & 53.766 & 47.368 & 42.427 \\
\addlinespace[2pt]
Southern Sotho & UBE & 64.929 & 60.443 & 56.135 & 53.151 & 50.300 \\*
 & BBPE & 65.265 & 60.633 & 56.179 & 53.157 & 50.304 \\
\addlinespace[2pt]
Spanish & UBE & 54.925 & 50.021 & 45.200 & 40.793 & 37.298 \\*
 & BBPE & 55.205 & 50.161 & 45.262 & 40.805 & 37.300 \\
\addlinespace[2pt]
Sundanese & UBE & 54.776 & 50.185 & 44.769 & 39.723 & 35.630 \\*
 & BBPE & 55.007 & 50.275 & 44.827 & 39.742 & 35.632 \\
\addlinespace[2pt]
Swahili & UBE & 55.131 & 50.284 & 45.010 & 39.962 & 35.767 \\*
 & BBPE & 55.292 & 50.478 & 45.092 & 39.992 & 35.770 \\
\addlinespace[2pt]
Swati & UBE & 63.442 & 58.921 & 55.185 & 51.379 & 47.642 \\*
 & BBPE & 63.571 & 58.976 & 55.233 & 51.399 & 47.671 \\
\addlinespace[2pt]
Swedish & UBE & 50.691 & 45.972 & 42.010 & 38.131 & 34.788 \\*
 & BBPE & 50.818 & 46.118 & 42.037 & 38.165 & 34.792 \\
\addlinespace[2pt]
Ta'izzi Arabic & UBE & 68.238 & 59.942 & 52.392 & 45.931 & 40.679 \\*
 & BBPE & 69.195 & 60.128 & 52.545 & 45.988 & 40.692 \\
\addlinespace[2pt]
Tagalog & UBE & 60.836 & 55.433 & 50.975 & 46.782 & 43.141 \\*
 & BBPE & 60.985 & 55.543 & 51.022 & 46.794 & 43.153 \\
\addlinespace[2pt]
Tajik & UBE & 82.025 & 72.294 & 64.915 & 57.818 & 50.596 \\*
 & BBPE & 82.953 & 72.406 & 65.016 & 57.848 & 50.621 \\
\addlinespace[2pt]
Tamashek (Latin) & UBE & 64.325 & 60.727 & 57.668 & 55.734 & 51.245 \\*
 & BBPE & 64.648 & 60.888 & 57.751 & 55.803 & 51.292 \\
\addlinespace[2pt]
Tamashek (Tifinagh) & UBE & 194.050 & 188.255 & 185.609 & 184.284 & 184.142 \\*
 & BBPE & 271.147 & 270.655 & 270.571 & 270.528 & 204.566 \\
\addlinespace[2pt]
Tamil & UBE & 130.214 & 122.921 & 118.113 & 115.209 & 113.794 \\*
 & BBPE & 130.302 & 123.013 & 118.227 & 115.250 & 113.797 \\
\addlinespace[2pt]
Tatar & UBE & 77.965 & 66.208 & 57.925 & 50.683 & 43.549 \\*
 & BBPE & 78.247 & 66.621 & 57.987 & 50.709 & 43.575 \\
\addlinespace[2pt]
Telugu & UBE & 114.800 & 110.878 & 105.106 & 100.440 & 98.118 \\*
 & BBPE & 114.948 & 111.162 & 105.057 & 100.465 & 98.162 \\
\addlinespace[2pt]
Thai & UBE & 75.640 & 69.237 & 64.534 & 60.990 & 58.107 \\*
 & BBPE & 75.850 & 69.173 & 64.612 & 60.969 & 58.122 \\
\addlinespace[2pt]
Tibetan & UBE & 287.065 & 286.737 & 247.171 & 230.166 & 163.501 \\*
 & BBPE & 432.437 & 336.804 & 264.208 & 214.252 & 161.215 \\
\addlinespace[2pt]
Tigrinya & UBE & 91.849 & 79.221 & 69.298 & 61.681 & 54.403 \\*
 & BBPE & 92.910 & 78.665 & 69.399 & 61.069 & 53.665 \\
\addlinespace[2pt]
Tok Pisin & UBE & 62.981 & 57.913 & 54.843 & 52.291 & 50.361 \\*
 & BBPE & 63.056 & 57.952 & 54.853 & 52.299 & 50.365 \\
\addlinespace[2pt]
Tsonga & UBE & 65.786 & 61.653 & 58.231 & 55.727 & 52.549 \\*
 & BBPE & 65.937 & 61.708 & 58.242 & 55.729 & 52.577 \\
\addlinespace[2pt]
Tswana & UBE & 65.866 & 61.867 & 58.373 & 55.719 & 52.893 \\*
 & BBPE & 66.321 & 62.021 & 58.438 & 55.723 & 52.887 \\
\addlinespace[2pt]
Tumbuka & UBE & 74.959 & 68.298 & 64.038 & 58.219 & 54.144 \\*
 & BBPE & 75.073 & 68.649 & 64.180 & 58.237 & 54.148 \\
\addlinespace[2pt]
Tunisian Arabic & UBE & 66.258 & 58.247 & 51.588 & 45.781 & 41.058 \\*
 & BBPE & 67.125 & 58.427 & 51.682 & 45.815 & 41.079 \\
\addlinespace[2pt]
Turkish & UBE & 59.896 & 52.898 & 46.661 & 40.892 & 35.960 \\*
 & BBPE & 60.139 & 52.972 & 46.795 & 40.961 & 35.966 \\
\addlinespace[2pt]
Turkmen & UBE & 68.897 & 62.516 & 59.187 & 56.246 & 53.481 \\*
 & BBPE & 69.116 & 62.595 & 59.215 & 56.235 & 53.517 \\
\addlinespace[2pt]
Twi & UBE & 72.109 & 60.609 & 58.342 & 55.848 & 49.776 \\*
 & BBPE & 72.292 & 60.657 & 58.359 & 55.874 & 49.794 \\
\addlinespace[2pt]
Ukrainian & UBE & 72.472 & 63.801 & 56.567 & 49.195 & 43.529 \\*
 & BBPE & 72.595 & 64.037 & 56.634 & 49.225 & 43.552 \\
\addlinespace[2pt]
Umbundu & UBE & 59.733 & 55.394 & 52.952 & 50.009 & 47.377 \\*
 & BBPE & 59.838 & 55.542 & 52.970 & 50.014 & 47.378 \\
\addlinespace[2pt]
Urdu & UBE & 72.123 & 63.268 & 54.393 & 47.501 & 42.270 \\*
 & BBPE & 72.404 & 63.449 & 54.484 & 47.551 & 42.296 \\
\addlinespace[2pt]
Uyghur & UBE & 119.602 & 112.410 & 107.736 & 100.045 & 90.004 \\*
 & BBPE & 120.675 & 112.438 & 107.739 & 100.048 & 90.006 \\
\addlinespace[2pt]
Venetian & UBE & 53.993 & 50.199 & 46.913 & 43.812 & 41.147 \\*
 & BBPE & 54.225 & 50.297 & 46.931 & 43.834 & 41.148 \\
\addlinespace[2pt]
Vietnamese & UBE & 69.613 & 55.407 & 46.327 & 39.947 & 36.715 \\*
 & BBPE & 69.764 & 55.400 & 46.381 & 39.937 & 36.706 \\
\addlinespace[2pt]
Waray & UBE & 60.043 & 54.536 & 50.311 & 47.288 & 44.137 \\*
 & BBPE & 60.254 & 54.592 & 50.323 & 47.311 & 44.139 \\
\addlinespace[2pt]
Welsh & UBE & 66.409 & 61.948 & 58.707 & 54.242 & 50.725 \\*
 & BBPE & 66.582 & 61.978 & 58.736 & 54.253 & 50.723 \\
\addlinespace[2pt]
West Central Oromo & UBE & 69.103 & 64.926 & 60.678 & 57.090 & 53.208 \\*
 & BBPE & 69.242 & 64.967 & 60.695 & 57.106 & 53.213 \\
\addlinespace[2pt]
Western Persian & UBE & 69.852 & 60.530 & 52.618 & 45.601 & 40.308 \\*
 & BBPE & 70.073 & 60.675 & 52.650 & 45.659 & 40.328 \\
\addlinespace[2pt]
Wolof & UBE & 57.242 & 52.743 & 50.206 & 47.573 & 45.729 \\*
 & BBPE & 57.316 & 52.815 & 50.213 & 47.591 & 45.736 \\
\addlinespace[2pt]
Xhosa & UBE & 61.492 & 56.750 & 52.813 & 48.445 & 44.486 \\*
 & BBPE & 61.587 & 56.796 & 52.856 & 48.456 & 44.491 \\
\addlinespace[2pt]
Yiddish & UBE & 84.912 & 70.776 & 62.908 & 54.900 & 48.788 \\*
 & BBPE & 85.573 & 71.280 & 63.101 & 54.911 & 48.797 \\
\addlinespace[2pt]
Yoruba & UBE & 81.966 & 74.621 & 70.871 & 65.355 & 61.404 \\*
 & BBPE & 83.391 & 75.041 & 70.604 & 65.383 & 61.239 \\
\addlinespace[2pt]
Zulu & UBE & 65.370 & 60.117 & 55.212 & 50.784 & 45.533 \\*
 & BBPE & 65.469 & 60.134 & 55.272 & 50.792 & 45.556 \\
\addlinespace[2pt]
\end{longtable}
\endgroup

\subsection{Per-language downstream benchmarks and held-out BPB}
\label{app:perlang-benchmarks}
Tables~\ref{tab:perlang-xnli} and~\ref{tab:perlang-xcopa} report per-language zero-shot accuracy for XNLI and XCOPA, and Table~\ref{tab:perlang-bpb} reports held-out BPB, under the primary \texttt{101-lang/32K} configuration (182M). $\Delta$ denotes UBE $-$ BBPE\@. Section~\ref{sec:bpb-breakdown} summarizes held-out BPB differences by encoding-floor group; script-family aggregates can be recomputed directly from the per-language tables below.

\needspace{23\baselineskip}
\begin{table}[ht]
\centering
\small
\caption{Per-language XNLI zero-shot accuracy under the primary \texttt{101-lang/32K} 182M configuration (15 languages). $\Delta$ = UBE $-$ BBPE.}
\label{tab:perlang-xnli}
\begin{tabular}{lrrr}
\toprule
Language & UBE & BBPE & $\Delta$ \\
\midrule
ar & 33.6 & 31.8 & +1.8 \\
bg & 33.4 & 33.2 & +0.2 \\
de & 34.1 & 33.4 & +0.7 \\
el & 33.4 & 33.7 & -0.3 \\
en & 40.4 & 42.0 & -1.6 \\
es & 34.0 & 33.4 & +0.6 \\
fr & 35.2 & 34.7 & +0.5 \\
hi & 33.4 & 33.5 & -0.1 \\
ru & 35.9 & 34.7 & +1.1 \\
sw & 35.0 & 34.1 & +0.8 \\
th & 33.6 & 34.3 & -0.7 \\
tr & 33.3 & 33.1 & +0.2 \\
ur & 33.3 & 33.6 & -0.2 \\
vi & 33.1 & 34.5 & -1.4 \\
zh & 34.3 & 34.8 & -0.5 \\
\midrule
Mean & 34.4 & 34.3 & +0.1 \\
\bottomrule
\end{tabular}
\end{table}

\needspace{19\baselineskip}
\begin{table}[ht]
\centering
\small
\caption{Per-language XCOPA zero-shot accuracy under the primary \texttt{101-lang/32K} 182M configuration (11 languages). $\Delta$ = UBE $-$ BBPE.}
\label{tab:perlang-xcopa}
\begin{tabular}{lrrr}
\toprule
Language & UBE & BBPE & $\Delta$ \\
\midrule
et & 49.8 & 50.8 & -1.0 \\
ht & 52.4 & 52.0 & +0.4 \\
id & 52.2 & 53.4 & -1.2 \\
it & 49.6 & 48.4 & +1.2 \\
qu & 48.6 & 50.0 & -1.4 \\
sw & 53.4 & 54.6 & -1.2 \\
ta & 54.0 & 53.6 & +0.4 \\
th & 55.4 & 56.0 & -0.6 \\
tr & 53.2 & 52.8 & +0.4 \\
vi & 50.2 & 50.0 & +0.2 \\
zh & 51.4 & 50.2 & +1.2 \\
\midrule
Mean & 51.8 & 52.0 & -0.1 \\
\bottomrule
\end{tabular}
\end{table}

\paragraph{Belebele (122 languages).}
UBE wins 20/122 languages overall.

\needspace{6\baselineskip}
\paragraph{Held-out BPB (101 languages).}

{\small
\begin{longtable}{lrrrr}
\caption{Per-language held-out BPB for UBE and BBPE (\texttt{101-lang/32K}, 182M; 101 languages). UBE token-level perplexity is shown only as a scale reference; $\Delta$BPB = UBE $-$ BBPE.}\label{tab:perlang-bpb}\\
\toprule
Language & UBE BPB & BBPE BPB & $\Delta$BPB & UBE PPL \\
\midrule
\endfirsthead
\multicolumn{5}{c}{\tablename\ \thetable{} -- continued from previous page}\\
\toprule
Language & UBE BPB & BBPE BPB & $\Delta$BPB & UBE PPL \\
\midrule
\endhead
af & 1.9795 & 1.9670 & +0.0125 & 52.31 \\
am & 1.0483 & 1.0390 & +0.0093 & 9.50 \\
ar & 1.2544 & 1.2644 & -0.0100 & 17.86 \\
az & 1.9210 & 1.9380 & -0.0170 & 36.04 \\
be & 1.4363 & 1.4010 & +0.0353 & 30.66 \\
bg & 1.1601 & 1.1531 & +0.0070 & 21.12 \\
bn & 1.0159 & 0.9728 & +0.0431 & 9.55 \\
ca & 1.6012 & 1.5773 & +0.0239 & 25.30 \\
ceb & 1.6769 & 1.6802 & -0.0033 & 24.38 \\
co & 1.4023 & 1.3969 & +0.0054 & 12.53 \\
cs & 1.9726 & 1.9860 & -0.0134 & 32.29 \\
cy & 1.1069 & 1.1240 & -0.0171 & 10.97 \\
da & 1.4187 & 1.3983 & +0.0204 & 16.50 \\
de & 1.5568 & 1.5508 & +0.0060 & 21.65 \\
el & 0.9531 & 0.9781 & -0.0250 & 9.39 \\
en & 1.0282 & 1.0214 & +0.0068 & 14.10 \\
eo & 2.1549 & 2.1384 & +0.0165 & 59.73 \\
es & 1.3764 & 1.3663 & +0.0101 & 17.61 \\
et & 2.2668 & 2.2282 & +0.0386 & 57.55 \\
eu & 1.7250 & 1.6232 & +0.1018 & 22.09 \\
fa & 1.2241 & 1.2470 & -0.0229 & 19.65 \\
fi & 1.9675 & 1.9389 & +0.0286 & 32.96 \\
fil & 1.7348 & 1.7185 & +0.0163 & 37.87 \\
fr & 1.4166 & 1.4114 & +0.0052 & 17.89 \\
fy & 1.4440 & 1.4229 & +0.0211 & 13.91 \\
ga & 1.5320 & 1.4588 & +0.0732 & 17.52 \\
gd & 0.9800 & 0.9611 & +0.0189 & 4.59 \\
gl & 1.6125 & 1.6051 & +0.0074 & 25.43 \\
gu & 1.5372 & 1.5355 & +0.0017 & 31.87 \\
ha & 2.1483 & 2.1454 & +0.0029 & 57.05 \\
haw & 2.0418 & 2.0489 & -0.0071 & 29.10 \\
hi & 0.7297 & 0.7263 & +0.0034 & 6.25 \\
hmn & 1.8521 & 1.8462 & +0.0059 & 21.54 \\
ht & 2.1021 & 2.0850 & +0.0171 & 44.17 \\
hu & 2.1818 & 2.2069 & -0.0251 & 41.04 \\
hy & 1.0273 & 1.2238 & -0.1965 & 9.37 \\
id & 1.2899 & 1.2777 & +0.0122 & 14.39 \\
ig & 2.1886 & 2.2380 & -0.0494 & 45.18 \\
is & 2.2356 & 2.2252 & +0.0104 & 55.78 \\
it & 1.5506 & 1.5512 & -0.0006 & 23.09 \\
iw & 1.7980 & 1.7994 & -0.0014 & 50.13 \\
ja & 0.9237 & 0.9168 & +0.0069 & 8.65 \\
jv & 1.8007 & 1.8128 & -0.0121 & 28.74 \\
ka & 1.2649 & 1.0770 & +0.1879 & 42.92 \\
kk & 0.9030 & 0.7750 & +0.1280 & 9.59 \\
km & 1.3200 & 1.4587 & -0.1387 & 18.18 \\
kn & 0.6466 & 0.6689 & -0.0223 & 4.12 \\
ko & 0.8818 & 0.8818 & 0.0000 & 10.05 \\
ku & 1.3519 & 1.4384 & -0.0865 & 11.48 \\
ky & 0.8469 & 0.8660 & -0.0191 & 8.08 \\
la & 1.4653 & 1.4504 & +0.0149 & 18.31 \\
lb & 1.0129 & 1.0088 & +0.0041 & 6.25 \\
lo & 1.4836 & 1.9145 & -0.4309 & 35.47 \\
lt & 1.8119 & 1.7687 & +0.0432 & 24.75 \\
lv & 1.3041 & 1.3562 & -0.0521 & 9.76 \\
mg & 1.9600 & 1.9434 & +0.0166 & 36.28 \\
mi & 1.9730 & 1.9811 & -0.0081 & 31.00 \\
mk & 1.1655 & 1.1293 & +0.0362 & 24.85 \\
ml & 0.5524 & 0.5549 & -0.0025 & 3.47 \\
mn & 1.6991 & 1.7137 & -0.0146 & 43.94 \\
mr & 0.9251 & 0.9232 & +0.0019 & 10.53 \\
ms & 1.5119 & 1.4916 & +0.0203 & 21.15 \\
mt & 1.6258 & 1.6412 & -0.0154 & 10.10 \\
my & 1.4408 & 1.0904 & +0.3504 & 34.47 \\
ne & 0.9534 & 0.9511 & +0.0023 & 9.27 \\
nl & 1.5384 & 1.5239 & +0.0145 & 21.90 \\
no & 1.5625 & 1.5628 & -0.0003 & 22.33 \\
ny & 1.9924 & 2.0106 & -0.0182 & 40.73 \\
pa & 1.7135 & 1.5858 & +0.1277 & 52.47 \\
pl & 1.8785 & 1.8636 & +0.0149 & 23.20 \\
ps & 1.8104 & 1.7712 & +0.0392 & 73.08 \\
pt & 1.2234 & 1.2306 & -0.0072 & 11.92 \\
ro & 1.9984 & 1.9858 & +0.0126 & 49.56 \\
ru & 0.8143 & 0.8056 & +0.0087 & 8.97 \\
sd & 0.8137 & 0.8039 & +0.0098 & 5.42 \\
si & 1.5479 & 1.2656 & +0.2823 & 33.97 \\
sk & 2.1511 & 2.0940 & +0.0571 & 43.93 \\
sl & 1.4659 & 1.4815 & -0.0156 & 13.61 \\
sm & 1.8296 & 1.8198 & +0.0098 & 22.49 \\
sn & 1.9315 & 1.9302 & +0.0013 & 33.36 \\
so & 1.7567 & 1.7942 & -0.0375 & 26.43 \\
sq & 2.7314 & 2.7785 & -0.0471 & 180.35 \\
sr & 1.4024 & 1.3963 & +0.0061 & 35.68 \\
st & 2.2992 & 2.3017 & -0.0025 & 55.37 \\
su & 1.8481 & 1.8349 & +0.0132 & 35.17 \\
sv & 1.6863 & 1.7006 & -0.0143 & 28.82 \\
sw & 2.2022 & 2.2039 & -0.0017 & 58.59 \\
ta & 0.8556 & 0.8134 & +0.0422 & 7.15 \\
te & 0.9789 & 1.0941 & -0.1152 & 8.41 \\
tg & 1.4043 & 1.4377 & -0.0334 & 34.27 \\
th & 0.4792 & 0.4784 & +0.0008 & 4.85 \\
tr & 1.3062 & 1.3492 & -0.0430 & 10.93 \\
uk & 0.9754 & 0.9425 & +0.0329 & 11.30 \\
ur & 1.5161 & 1.5038 & +0.0123 & 40.99 \\
uz & 2.1720 & 2.1786 & -0.0066 & 48.10 \\
vi & 1.6018 & 1.5877 & +0.0141 & 25.09 \\
xh & 1.9580 & 1.9329 & +0.0251 & 33.52 \\
yi & 1.6290 & 1.6012 & +0.0278 & 39.59 \\
yo & 2.6140 & 2.5907 & +0.0233 & 90.06 \\
zh & 1.0238 & 1.0117 & +0.0121 & 9.87 \\
zu & 1.7078 & 1.6966 & +0.0112 & 25.30 \\
\midrule
Mean & 1.5237 & 1.5174 & +0.0063 & \\
\bottomrule
\end{longtable}
}

\clearpage

\end{document}